\documentclass[aos,preprint]{imsart}

\usepackage{graphicx}
\usepackage{booktabs} 
\usepackage{tikz}
\usepackage{hyperref,graphicx,amsmath,amsfonts,subcaption,amssymb,bm,url,breakurl,epsfig,epsf,color,MnSymbol,mathbbol,fmtcount,semtrans,caption,multirow,comment}

\usepackage{color}
\usepackage[utf8]{inputenc} 
\usepackage[T1]{fontenc}    
\usepackage{hyperref}       
\usepackage{url}            
\usepackage{booktabs}       
\usepackage{amsfonts}       
\usepackage{nicefrac}       
\usepackage{relsize}
\usepackage{enumitem}

\RequirePackage[OT1]{fontenc}
\RequirePackage{graphicx,float}
\RequirePackage[numbers]{natbib}
\RequirePackage[colorlinks,citecolor=blue,urlcolor=blue]{hyperref}
\RequirePackage{caption,subcaption}
\RequirePackage{multirow, slashbox}

\usepackage{hyperref,graphicx,amsmath,amsfonts,amssymb,bm,url,breakurl,epsfig,epsf,color,MnSymbol,mathbbol,fmtcount,semtrans,caption,subcaption,multirow,comment, boldline}
\usepackage[noend]{algpseudocode}
\usepackage{wrapfig}
\usepackage{amssymb}

\usepackage[utf8]{inputenc} 
\usepackage[T1]{fontenc}    
\usepackage{url}            
\usepackage{booktabs}       
\usepackage{amsfonts}       
\usepackage{nicefrac}       
\usepackage{microtype}      

\makeatletter
\providecommand*{\boxast}{%
  \mathbin{
    \mathpalette\@boxit{*}%
  }%
}
\newcommand*{\@boxit}[2]{%
  \sbox0{$\m@th#1\Box$}%
  \ifx#1\displaystyle \ht0=\dimexpr\ht0+.05ex\relax \fi
  \ifx#1\textstyle \ht0=\dimexpr\ht0+.05ex\relax \fi
  \ifx#1\scriptstyle \ht0=\dimexpr\ht0+.04ex\relax \fi
  \ifx#1\scriptscriptstyle \ht0=\dimexpr\ht0+.065ex\relax \fi
  \sbox2{$#1\vcenter{}$}
  \rlap{%
    \hbox to \wd0{%
      \hfill
      \raisebox{%
        \dimexpr.5\dimexpr\ht0+\dp0\relax-\ht2\relax
      }{$\m@th#1#2$}%
      \hfill
    }%
  }%
  \Box
}
\makeatother

  \makeatletter
\def\BState{\State\hskip-\ALG@thistlm}
\makeatother

  \usepackage{mathtools}

\usepackage{titlesec}

\usepackage{tikz}
\usepackage{pgfplots}
\usetikzlibrary{pgfplots.groupplots}

\titleformat{\paragraph}
{\normalfont\normalsize\bfseries}{\theparagraph}{1em}{}
\titlespacing*{\paragraph}
{0pt}{3.25ex plus 1ex minus .2ex}{1.5ex plus .2ex}

\usepackage{movie15}

\usepackage{caption}
\usepackage[bottom,hang,flushmargin]{footmisc} 

\newcommand{\tsn}[1]{{\left\vert\kern-0.25ex\left\vert\kern-0.25ex\left\vert #1 
    \right\vert\kern-0.25ex\right\vert\kern-0.25ex\right\vert}}

\definecolor{darkred}{RGB}{150,0,0}
\definecolor{darkgreen}{RGB}{0,150,0}
\definecolor{darkblue}{RGB}{0,0,200}
\hypersetup{colorlinks=true, linkcolor=black, citecolor=darkgreen, urlcolor=darkblue}

\newtheorem{theorem}{Theorem}[section]

\newtheorem{assumption}{Assumption}

\newtheorem{lemma}[theorem]{Lemma}
\newtheorem{corollary}[theorem]{Corollary}
\newtheorem{propo}[theorem]{Proposition}
\newtheorem{definition}[theorem]{Definition}

\newtheorem{remark}[subsection]{Remark}

\def\sF{{\sf F}}
\def\sG{{\sf G}}
\newcommand\tr{{{\operatorname{trace}}}}
\newcommand{\eps}{\varepsilon}

\newcommand{\cF}{\mathcal{F}}

\def\tgamma{\tilde{\gamma}}

\newcommand{\beq}{\begin{equation}}

\newcommand{\eeq}{\end{equation}}

\newcommand{\nn}{\nonumber}

\def\bJ{\mtx{J}}
\newcommand{\A}{{\mtx{A}}}

\newcommand{\Gb}{{\mtx{G}}}

\newcommand{\diag}[1]{{\rm{diag}}(#1)}

\newcommand{\Iden}{{\mtx{I}}}

\newcommand{\ub}{{\vct{u}}}

\newcommand{\z}{{\vct{z}}}

\newcommand{\bbeta}{{\boldsymbol{\beta}}}

\newcommand{\balpha}{{\boldsymbol{\alpha}}}

\newcommand{\Unif}{\rm{Unif}}

\newcommand{\w}{\vct{w}}

\newcommand{\ab}{\vct{a}}
\newcommand{\bb}{\vct{b}}

\newcommand{\event}{\mathcal{E}}

\newcommand{\opnorm}[1]{\left\|#1\right\|}
\newcommand{\fronorm}[1]{\left\|#1\right\|_{F}}
\newcommand{\onenorm}[1]{\left\|#1\right\|_{\ell_1}}
\newcommand{\twonorm}[1]{\left\|#1\right\|_{\ell_2}}

\newcommand{\infnorm}[1]{\left\|#1\right\|_{\ell_\infty}}

\newcommand{\abs}[1]{\left|#1\right|}

\newcommand{\x}{\vct{x}}
\newcommand{\rb}{\vct{r}}

\newcommand{\W}{\mtx{W}}

\definecolor{emmanuel}{RGB}{255,127,0}

\newcommand{\p}{{\vct{p}}}

\newcommand{\pb}{{\vct{p}}}
\newcommand{\qb}{{\vct{q}}}
\newcommand{\R}{\mathbb{R}}

\newcommand{\<}{\langle}
\renewcommand{\>}{\rangle}

\newcommand{\E}{\operatorname{\mathbb{E}}}

\newcommand{\eb}{\vct{e}}

\newcommand{\vct}[1]{\bm{#1}}
\newcommand{\mtx}[1]{\bm{#1}}

\numberwithin{equation}{section} 

\def \endprf{\hfill {\vrule height6pt width6pt depth0pt}\medskip}
\newenvironment{proof}{\noindent {\bf Proof} }{\endprf\par}

\newcommand{\qed}{{\unskip\nobreak\hfil\penalty50\hskip2em\vadjust{}
           \nobreak\hfil$\Box$\parfillskip=0pt\finalhyphendemerits=0\par}}

\newcommand{\hth}{{\widehat{\boldsymbol{\theta}}}}
\newcommand{\bth}{{\boldsymbol{\theta}}}

\newcommand\cL{\mathcal{L}}

\newcommand\reals{\mathbb{R}}
\newcommand\bSigma{\boldsymbol{\Sigma}}
\newcommand\normal{{\sf N}}
\newcommand\bdelta{\boldsymbol{\delta}}
\newcommand\sign{{\rm sign}}
\newcommand\sT{{\sf T}}

\newcommand\bA{\mtx{A}}
\newcommand\bB{\mtx{B}}
\newcommand\bv{\mtx{v}}
\newcommand\bh{\mtx{h}}
\newcommand\bH{\mtx{H}}

\newcommand\bz{\boldsymbol{z}}
\newcommand\bg{\boldsymbol{g}}

\def\cC{\mathcal{C}}
\def\de{{\rm d}}
\def\bM{\mtx{M}}

\def\bu{\boldsymbol{u}}
\def\by{\boldsymbol{y}}
\def\bw{\boldsymbol{w}}
\def\bX{\boldsymbol{X}}

\def\bx{\vct{x}}

\def\tbx{\widetilde{\boldsymbol{x}}}
\def\tbth{\widetilde{\boldsymbol{\theta}}}

\def\ones{\mathbf{1}}
\def\ind{\mathbb{I}}
\def\cS{\mathcal{S}}
\def\prob{\mathbb{P}}

\def\reals{\mathbb{R}}
\def\SR{{\sf SR}}
\def\AR{{\sf AR}}
\def\ST{{\sf ST}}
\def\bW{\boldsymbol{W}}

\newcommand{\rev}[1]{{{\color{black}{#1}}}}
\newcommand{\revv}[1]{{{\color{black}{#1}}}}

\def\cLo{\overset{\circ}{\cL}}
\def\cLoo{\overset{\circ\circ}{\cL}}
\def\ARo{\overset{\circ}{\AR}}
\def\ARoo{\overset{\circ\circ}{\AR}}
\def\bo{\mathbf{0}}
\def\mbf{\boldsymbol{f}}

\def\ARnl{\AR_{\rm nl}}

\def\elbar{\bar{\ell}}

\def\tqb{\widetilde{\qb}}
\def\br{\boldsymbol{r}}

\def\bK{\boldsymbol{K}}
\def\erf{{\rm erf}}
\def\erfc{{\rm erfc}}

\def\cR{\mathcal{R}}
\def\bOmega{\boldsymbol{\Omega}}
\def\hthoo{\overset{\circ\circ}{\bth}}
\def\delu{\bdelta^{\backslash u}}
\def\be{\boldsymbol{e}}

\def\etest{\eps_{{\rm test}}}

\input{latexdiff_prem}

\begin{document}

\begin{frontmatter}
\title{The curse of overparametrization in adversarial training:
Precise analysis of robust generalization for random features regression}
\runtitle{The curse of overparametrization in adversarial training}

\begin{aug}






\author[A]{\fnms{Hamed}~\snm{Hassani}\ead[label=e1]{hassani@seas.upenn.edu}},
\author[B]{\fnms{Adel}~\snm{Javanmard}\ead[label=e2]{ajavanma@usc.edu }\orcid{0000-0003-1934-8747}}
\address[A]{Department of Electrical and Systems Engineering, University
of Pennsylvania\printead[presep={,\ }]{e1}}

\address[B]{Data Sciences and Operations Department, University
of Southern California\printead[presep={,\ }]{e2}}
\end{aug}

\newcommand{\changelocaltocdepth}[1]{%
  \addtocontents{toc}{\protect\setcounter{tocdepth}{#1}}%
  \setcounter{tocdepth}{#1}%
}

\setcounter{tocdepth}{0}
\begin{abstract}
Successful deep learning models  often involve training neural network architectures that contain more parameters than the number of training samples. Such overparametrized models have recently been extensively studied, and the virtues of overparametrization have been established from both the statistical perspective, via the double-descent phenomenon, and the computational perspective via the structural properties of the optimization landscape. Despite this success, it is also well known that these models are  highly vulnerable to small adversarial perturbations in their inputs.  Even when  adversarially trained, their performance on perturbed inputs (robust generalization) is considerably worse than their best attainable performance on benign inputs (standard generalization). It is thus imperative to understand how overparametrization fundamentally affects robustness.  

In this paper, we will provide a precise characterization of the role of overparametrization on robustness by focusing on  random features regression models (two-layer neural networks with random first layer weights). We consider a regime where the sample size, the input dimension and the number of parameters grow proportionally, and derive an asymptotically exact formula for the robust generalization error when the model is adversarially trained. Our developed theory reveals the nontrivial effect of overparametrization on robustness and indicates that high overparametrization can hurt robust generalization.
\end{abstract}



\end{frontmatter}

\section{Introduction}
The success of deep learning models is often reliant on training highly complex neural networks whose number of parameters is much larger than the number of data points. Even though the large complexity of such models allows for perfect interpolation of the data, they often achieve low generalization error.  This behavior has resulted in a growing body of work aiming to analyze such so-called overparametrized models.

Recent work has demonstrated the virtues of overparametrization from statistical and optimization-based perspectives.  From the statistical viewpoint,  it is now well-documented that many overparametrized models exhibit a `double-descent' property \cite{belkin2018understand, belkin2019reconciling, mei2019generalization}: As the model complexity increases, the generalization error first follows  the traditional U-shaped curve until a specific point, after which the error decreases, and
attains a global minimum in the overparametrized regime. In fact, the minimum generalization error often appears to be at
infinite complexity -- the more overparametrized is the model, the smaller is the error. 
It is often argued that the good generalization behavior of highly overparametrized models is due to the inductive bias of gradient-based algorithms which helps with selecting models that generalize well --see e.g,~\cite{bartlett2021deep,hastie2019surprises,soudry2018implicit,NEURIPS2018_0e98aeeb}. From the optimization viewpoint,  training deep neural networks in general involves optimizing highly non-convex functions, but it has been conjectured that highly overparametrized models are easy to optimize despite non-convexity. Instances of this observation has been formally proved, e.g. in \cite{soltanolkotabi2018theoretical,javanmard2019analysis,NEURIPS2018_5a4be1fa,bartlett2021deep,montanari2021tractability}. The high-level intuition here is that in the highly overparametrized regimes, a model that perfectly interpolates the training data (and so is a global minimizer of the empirical risk) is found in the neighborhood of most initializations. 

Despite the remarkable success of deep neural networks, and the crucial role of overparametrization in both the generalization and the tractability aspects, these models are known to be highly vulnerable to perturbations in the input \cite{biggio2013evasion, szegedy2014intriguing}. With an unguarded training approach, these models show unsatisfactory \emph{robust generalization error} in the presence of ``small'' worst-case perturbations to their inputs, a.k.a \emph{asdversarial examples}.
This suggests that learning algorithms, even those with excellent performance on test data, may not be learning the true underlying concepts that determine the response; although they work well on naturally occurring data, adversarial examples have low probability in the data distribution and expose fundamental blind spots in the learning algorithms.

This observation stimulated significant effort to improve robustness using  a wide variety of \emph{adversarial training} methods which often involve augmenting the training loss so
as to become more robust to input perturbations (see e.g. \cite{DBLP:journals/corr/GoodfellowSS14, kurakin2016adversarial, jalal2017robust, DBLP:conf/iclr/MadryMSTV18, DBLP:conf/icml/WongK18, DBLP:conf/icml/ZhangYJXGJ19, cohen2019certified}).  However, there is still a large gap between the robust generalization error and the (standard) generalization error in adversarially-trained models. In summary, while modern overparametrized machine learning models perform very well on benign inputs, they still remain fragile to perturbations in the input. These findings raise a fundamental question: 

\vspace{0.5cm}
\begin{centering}
\hspace{1cm}\textit{How does overparametrization affect robustness  to  perturbations in the input?}
\end{centering}
\vspace{0.5cm}

A few recent papers have begun to answer the above question in specific settings with rather conflicting messages: 

\begin{itemize}
\item \cite{javanmard2020precise} and \cite{donhauser2021interpolation} have studied high-dimensional \emph{linear} models and showed that the robust generalization error of adversarially-trained models becomes \emph{worse} as the models become more overparametrized. It should be noted that for linear models, even in the case where there is no adversary, the best generalization error is attained when the model is  underparametrized~\cite{hastie2019surprises}. 

\item Another line of work provably shows that in order to \emph{interpolate} the training data smoothly, while being robust, overparametrization is  \emph{necessary} \cite{bubeck2021universal, bubeck2021law}. However, we note that, in order to train robust models, it may not be beneficial to interpolate the training data as robustness is measured with worst-case performance over all the points in a neighborhood around the input data. Indeed,~\cite{dohmatob2021fundamental} study the tradeoffs between memorization (of training data) and robustness of two-layer neural networks and established a lower-bound on the non-robustness of the model (via the Sobolev-seminorm of the model) as an increasing function of the amount of memorization.


We will provide a more detailed discussion of these points and other related works in Section~\ref{sec:related}. 
Despite such interesting recent progress, a comprehensive understanding on how overparametrization precisely affects robustness remains largely mysterious. 

\end{itemize}

\begin{figure}[!t]
     \centering
     \begin{subfigure}[b]{0.49\textwidth}
         \centering
         \includegraphics[width=\textwidth]{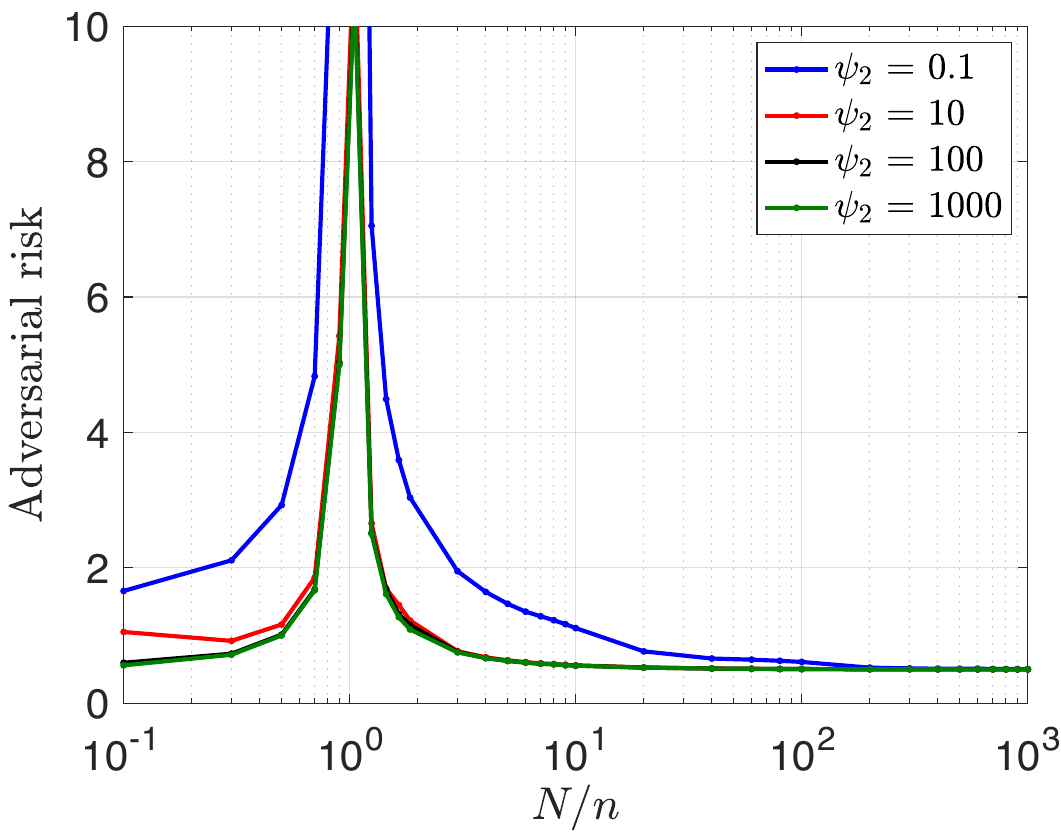}
         \caption{$\varepsilon = 10^{-7}$}
         \label{fig:y equals x}
     \end{subfigure}
     \hfill
     \begin{subfigure}[b]{0.475\textwidth}
         \centering
         \includegraphics[width=\textwidth]{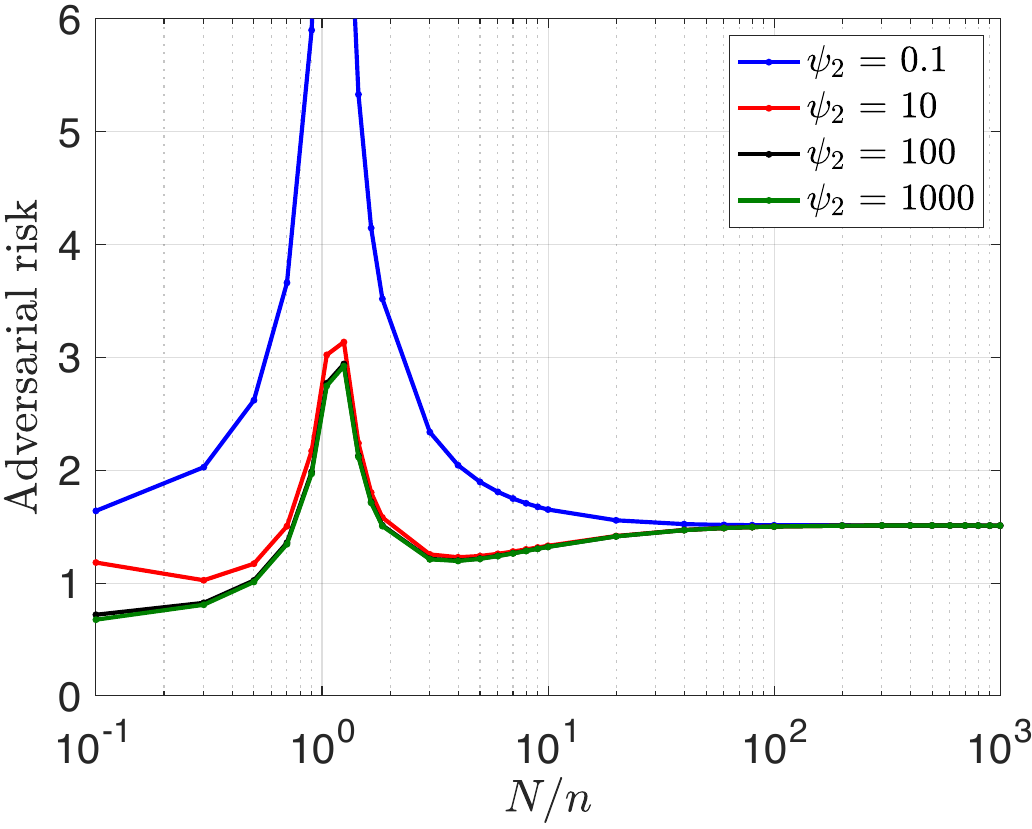}
         \caption{$\varepsilon = 0.1$}
         \label{fig:three sin x}
     \end{subfigure}
     \hfill
     \begin{subfigure}[b]{0.475\textwidth}
         \centering
         \includegraphics[width=\textwidth]{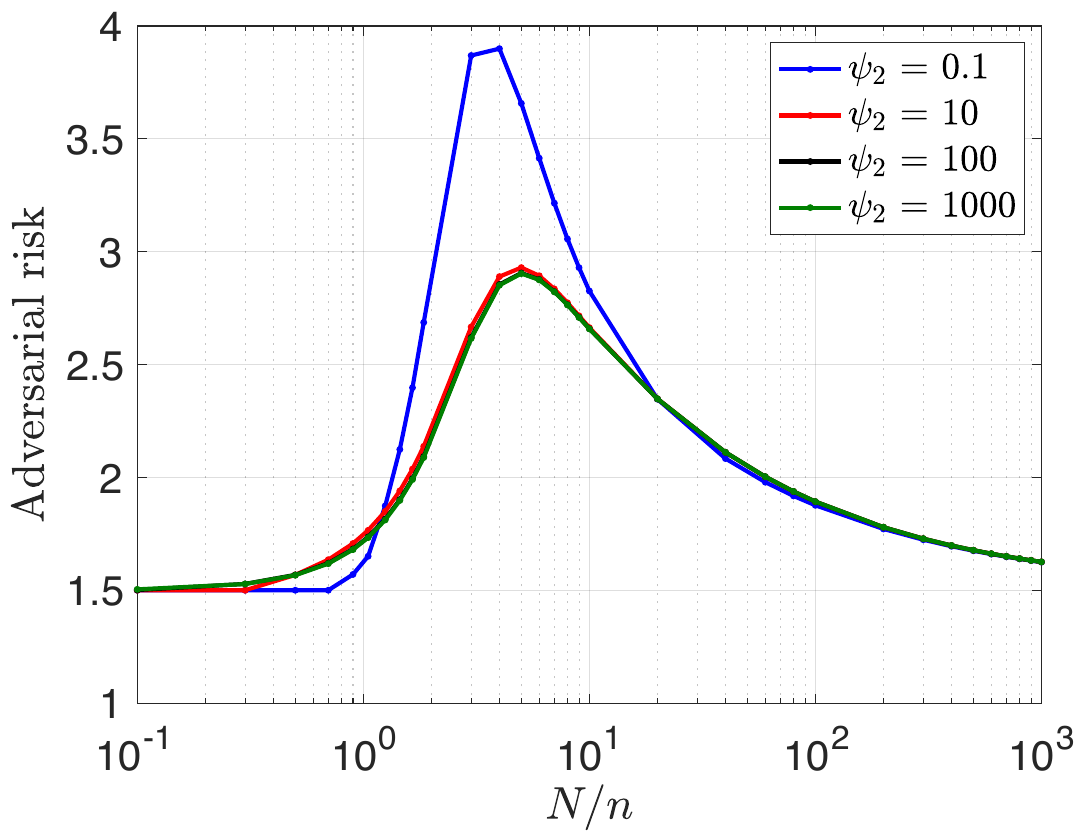}
         \caption{$\varepsilon = 1$}
         \label{fig:five over x}
     \end{subfigure}
        \caption{\small{Random features regression with (shifted) ReLU activation ($\sigma(x) = \max(x,0)-1/\sqrt{2\pi}$). Data $(\x_i, y_i)$ is generated with $d$-dimensional normal covariates $\x_i$ and $y_i = \bbeta^\sT \x_i + \xi_i$, where the noise variables $\xi_i \sim \mathcal{N}(0,0.5)$ and $\twonorm{\bbeta} = 1$. Perturbations are allowed within an Euclidean ball of radius $\varepsilon$, and the models are adversarially trained. We plot the robust generalization error (using Theorem \ref{thm:main}) versus the amount of overparametrization $N/n$, where $N$ is the number of parameters and $n$ is the number of training data points. The plots are obtained for different values of $\varepsilon$ and $\psi_2 = n/d$.} 
        \label{fig:intro_figures}}
\end{figure}

In this paper, we focus on random features regression models that are adversarially trained using robust empirical risk minimization and provide a ``precise characterization'' of the robust generalization.  Our analysis is carried out in a high-dimensional regime where the size of the training data $n$, the number of parameters $N$, and the dimension of the data $d$ grow proportional to each other, i.e. \rev{$N/d \to \psi_1$ and $n/d \to \psi_2$}. We further assume that  the perturbations are bounded in terms of $\ell_2$ norm by a value $\varepsilon >0$.  Our developed theory allows us to precisely characterize the effect of overparamterization on model robustness. One of the main consequences of our analysis is that, in general, higher overparametrization leads to a \emph{worse} robust generalization error  for the adversarially-trained models. Figure~\ref{fig:intro_figures} depicts how the robust generalization error varies with respect to the amount of overparametrization $N/n$. The left figure corresponds to the case where there is no adversary (i.e. $\varepsilon \approx 0$).  In this case, the robust generalization coincides with the (standard) generalization error, and is minimized at infinite overparametrization. However, as seen in the other two figures (for $\varepsilon > 0 $), overparametrization is in general hurting robustness. This is clearly seen in Figure~\ref{fig:intro_figures}(c) and (b) (for $\psi_2\ge 10$) where the minimum robust error is attained when the model is underparametrized. We refer to Figure~\ref{fig:tot} for an extended version of Figure~\ref{fig:intro_figures}  with more choices of $\eps$ and signal-to-noise ratios.

We proceed by providing an informal overview of our results and their implications in Section~\ref{sec:overview_of_results}.  Related works are discussed in Section~\ref{sec:related}. The main result of the paper, which characterizes the robust generalization error for the random features model is explained in Section~\ref{sec:main_results}. The  architecture of the proof of the main result is sketched in Section~\ref{sec:sketch}. 
Our analysis develops a set of techniques that are of independent interest: (i) We derive an asymptotic closed form for adversarial examples in trained random features models; (ii) While features are highly non-Gaussian in  random features models, we prove a Gaussian equivalence property which relates  robust generalization in these models to that of linear models with Gaussian features under the same correlation structure; (iii) Our analysis of the equivalent Gaussian model relies on the Convex Gaussian Min-max Theorem, which is a generalized and tight version of  Gordon's  Gaussian comparison inequalities.

\section{Results and discussion: An informal overview} \label{sec:overview_of_results}
\phantom{a}
\medskip

\noindent\textbf{Problem setting.} Consider a supervised learning scenario where 
we are given i.i.d data $\{(\bx_i,y_i)\}_{i\le n}$ generated according to the following distribution: 
\begin{align}\label{eq:linearModel}
y_i = \<\bx_i,\bbeta\>+\xi_i,\quad \text{with} \quad \bx_i\sim_{iid} \normal(0,\Iden_d),\quad \xi_i\sim \normal(0,\tau^2)\,.
\end{align}
The (linear) dependence between $(\bx_i,y_i)$ is unknown and the goal is to fit a model to this data which can be then used to predict labels for the unlabeled  examples at test time.

We consider modeling the relation between label $y$ and feature vector $\bx$ using the class of random features (RF) model, which can be described as 
\begin{align}\label{eq:RF}
\cF_{{\rm RF}}(\bW) = \Big\{f(\bx,\bth,\bW) = \sum_{\ell=1}^N \theta_\ell \sigma(\<\bw_\ell,\bx\>):\quad \bth= (\theta_1, \cdots, \theta_N) \in \reals^N \; \Big\}\,,
\end{align}
where $\bth$ is the parameter vector to be learned and $\bW\in \reals^{N\times d}$ is a fixed matrix whose rows $\bw_\ell$ are chosen randomly and independently of data. For simplicity we assume the normalization $\twonorm{\bw_\ell} = 1$. Namely, the vectors $\bw_\ell$ are chosen uniformly at random from the unit sphere, $\bw_\ell\sim\Unif(\mathbb{S}^{d-1})$, which implies that $\<\bw_\ell,\bx_j\>$ is of order one. In addition, $\sigma:\reals\mapsto \reals$ is a nonlinear activation function. 

Note that in random features model training is only done on $\bth$ and not on $\bW$. In other words, the random features model can be perceived as a two-layer neural network with the weights of the first layer chosen randomly and independently from data, while the weights in the second layer are learned during the training phase. The random features model was introduced by~\cite{rahimi2007random} for scaling kernel methods to large datasets, and there has been a large body of work drawing connections between random features models, kernel methods and fully trained neural networks~\cite{daniely2016toward,daniely2017sgd,jacot2018neural,li2018learning}. The random features models are arguably the simplest analytically tractable models  that capture
all the features of the double descent phenomenon without assuming ad hoc misspecification structures~\cite{montanari2019generalization}.
In particular, they allow to disentangle  the number
of parameters from the covariates dimension and hence isolate the effects of overparametrization from the effects of the ambient dimension.

To quantify robust generalization, we consider an adversarial framework where at the test time, the feature vector $\bx$ is corrupted by additive perturbation, chosen adversarially, from the Euclidean ball of radius $\varepsilon$. We measure the robust generalization via the \emph{adversarial risk} measure which is the expected test error of the model on perturbed test input.  We train a random features model, using a widely used adversarial training approach, which is based on the robust empirical risk minimizer (robust-ERM estimator)\rev{~\cite{madry2017towards,tsipras2018robustness}:
\begin{align}
\hth^\eps = \arg\min_{\bth\in\reals^N}  \max_{\twonorm{\bdelta_i}\le \eps}  \frac{1}{2n} \sum_{i=1}^n \left(y_i - \bth^\sT \sigma(\bW(\bx_i + \bdelta_i)\right)^2\,.
\end{align}
where $\bdelta_i$ is the norm-bounded adversarial perturbation on sample covariates $\bx_i$ and $\eps$ is the ``perceived'' adversary's power used in the training process.}

\vspace{.2cm}
\noindent\textbf{Results and discussion.}  We study the asymptotic setting, where $N,n,d\to\infty$ with $N/d \to\psi_1$ and $n/d\to \psi_2$ for some positive constants $\psi_1,\psi_2$. \emph{We derive the precise characterization of the adversarial risk of the robust-ERM estimator, as an explicit function of the dimension parameters $\psi_1,\psi_2$, the noise level $\tau^2$, and the adversarial power $\eps$. } We refer to Theorem~\ref{thm:main} for the specific formulae. 

\begin{figure}[!t]
     \centering
         \includegraphics[width=9cm]{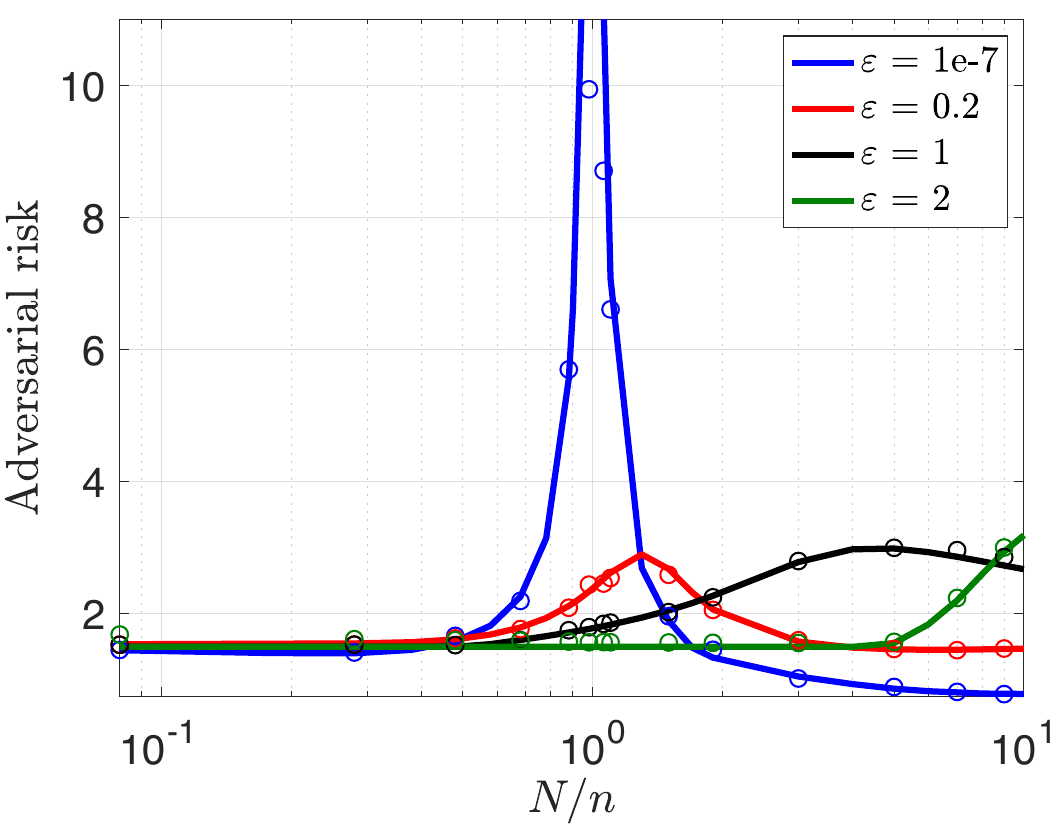}
        \caption{\small{Adversarial risk versus overparametrization $\psi_1/\psi_2 = N/n$ for different values of adversary's power $\eps_0$.  Solid curves
are theoretical predictions and dots are results obtained based on gradient descent on the robust ERM objective. Each dot represents the average of 100 trials. The data is generated according to model~\eqref{eq:linearModel}, with $d = 100$, $n = 300$, $\tau^2 = 0.5$, and $\bbeta\in\reals^d$ obtained by drawing a vector with i.i.d $\normal(0,1)$ entries and then normalizing it to have $\twonorm{\bbeta} = 1$.} 
        \label{fig:comp}}
\end{figure}

Let us now discuss the  behavior of the robust generalization curve under different settings.
We consider the data model~\eqref{eq:linearModel} and the random features regression with shifted ReLU activation: 
\[
\sigma(x) = \max(x,0)-\frac{1}{\sqrt{2\pi}}\,.
\]
The reason behind the intercept term is that since the  response variable is zero mean, we consider fitting a model using zero mean features. Note that $\<w_\ell, \x\>\sim \normal(0,1)$ and for $G\sim \normal(0,1)$, we have $\E[\sigma(G)] = \E[G\ind(G>0)- 1/\sqrt{2\pi}] = 0$.

 We start by Figure~\ref{fig:comp} which shows our theoretical curve versus the overparametrization ratio $\psi_1/\psi_2 = N/n$ along with the corresponding empirical results.  The solid lines depict
theoretical predictions with the dots representing the empirical performance of gradient descent in learning the robust ERM for data model~\eqref{eq:linearModel}, with $d = 100$, $n = 300$, $\tau^2 = 0.5$. In addition, $\bbeta\in\reals^d$ is generated by first drawing a $d$-dimensional vector with i.i.d standard normal entries and then normalizing it to have unit $\ell_2$ norm.
 Each dot represents the average of $20$ trials. As we see, even for moderate covariate dimensions ($d$), our theoretical curve is at excellent match with the empirical results. We note that when $\eps \to 0$ (we did not set $\eps= 0$ exactly for numerical stability), we are in non-adversarial regime and the robust generalization error reduces to the usual test error (blue curve). In this case, we observe the double-descent phenomena and recover the theoretical prediction of~\cite{mei2019generalization}. As $\eps$ grows the robust \rev{generalization curve starts} behaving differently.  For $\eps$ large enough ($\eps=1, 2$ in the figure), we  
see that overparametrization hurts robust generalization.

For a more complete picture, in Figure~\ref{fig:varyeps} we consider similar setting with more choices of $\eps$ and noise variance $\tau^2$, and also a larger range of overparametrization $\psi_2/\psi_1 = N/n$, as we fix $\psi_2 = 3$.  When $N/n\to 0$, we essentially have the risk of the zero estimator, which is $\twonorm{\bbeta}^2+\tau^2$. Several intriguing observations can be made from these plots:
\begin{itemize}
\item In the noiseless case (Figure~\ref{fig:varyeps1}) and for $\eps \le 0.5$, the global minimizer of the adversarial risk is at a finite overparametrization $(N/n>1)$, after which the risk becomes increasing as a function of $N/n$ (higher overparametrization hurst robust generalization).  Similar behavior is observed for $\tau^2 = 0.5$ and $\eps\le 0.05$.
\item In all three plots (corresponding to different SNR levels), when $\eps$ is large enough $(\eps \geq 1)$, the risk first goes up as overparametrization increases and after reaching its peak starts going down, but it remains above $1+\tau^2$ which is the risk at the highly underparametrized regime ($N/n \to 0$). Therefore, somewhat surprisingly, robust ERM estimator has larger adversarial risk compared to the trivial zero estimator, for all the range of overparametrization.
\item The peak of the adversarial risk occurs in the overparametrized regime; for the non-adversarial case $\eps= 0$, it occurs at the interpolation threshold $N/n = 1$ and for $\eps>0$ it occurs at $N/n>1$. The location  of the peak and the value of risk at the peak vary with $\eps$. As $\eps$ grows, the peak shifts to the right and occurs at a higher overparametrization ratio.    
\end{itemize}

\begin{figure}[!t]
     \centering
     \begin{subfigure}[b]{0.49\textwidth}
         \centering
         \includegraphics[width=\textwidth]{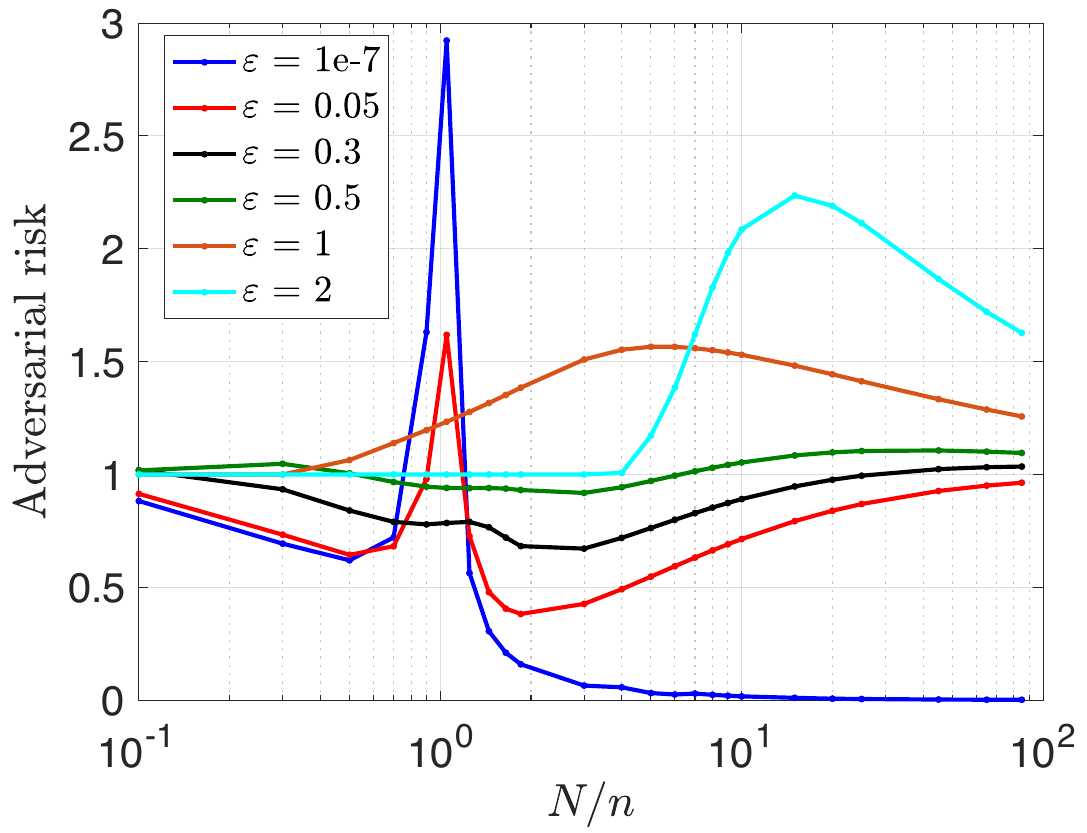}
         \caption{$\tau^2 = 0$}
         \label{fig:varyeps1}
     \end{subfigure}
     \hfill
     \begin{subfigure}[b]{0.475\textwidth}
         \centering
         \includegraphics[width=\textwidth]{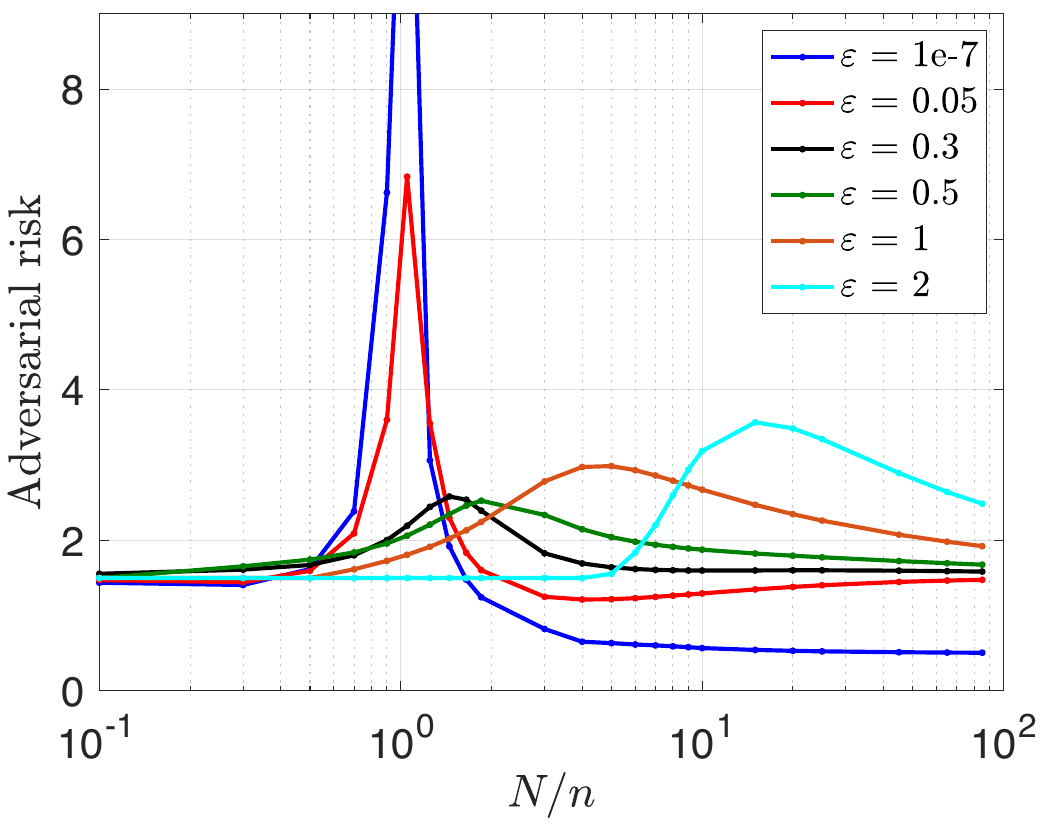}
         \caption{$\tau^2 = 0.5$}
         \label{fig:varyeps2}
     \end{subfigure}
     \hfill
     \begin{subfigure}[b]{0.495\textwidth}
         \centering
         \includegraphics[width=\textwidth]{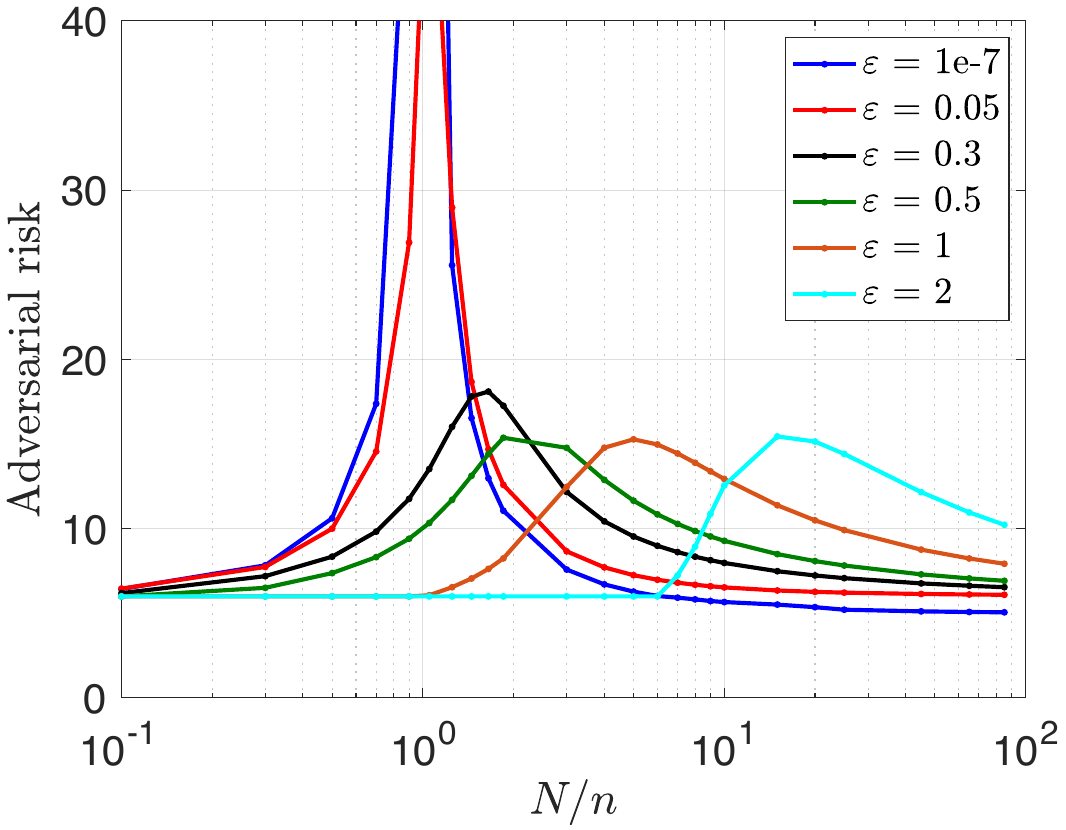}
         \caption{$\tau^2 = 5$}
         \label{fig:varyeps3}
     \end{subfigure}
        \caption{\small{Theoretical prediction curves for adversarial risk of robust ERM as a function of overparametrization $\psi_1/\psi_2 = N/n$ for different values of adversary's power $\eps$ and noise variance $\tau^2$, with the data model~\eqref{eq:linearModel}. Here we fix $\twonorm{\bbeta} = 1$ and $\psi_2 =3$.} 
        \label{fig:varyeps}}
\end{figure}

In Figure~\ref{fig:tot} we depict our theoretical prediction curves for the adversarial risk of the robust ERM estimator as a function of the overparametrization ratio $\psi_1/\psi_2 = N/n$ for different values of $\psi_2 = n/d$. The right panel  corresponds to $\eps = 1$ (strong adversary) and as we see for different values of $\tau^2$ and $\psi_2$, the adversarial risk is first an increasing function of overparametrization ratio, until it reaches its peak (in the overparametrized regime, $N/n>1$) and then becomes decreasing. But it never falls below its initial value at $N/n\approx 0$. The left panel corresponds to $\eps = 0.1$ (weak adversary) and as we see for large $\psi_2$,  overparametrization clearly has a negative effect on robust generalization. For example, in Figure~\ref{fig:tot1}, for $\psi_2 = 100, 1000$ the risk is an increasing function of $\psi_1/\psi_2$ over the entire range. Also in Figure~\ref{fig:tot3}, for $\psi_2\ge 10$ the global minimum of the adversarial risk is achieved in the underparametrized regime ($N/n<1$). 

\begin{figure}[]
     \centering
      \begin{subfigure}[b]{0.47\textwidth}
         \centering
         \includegraphics[width=\textwidth]{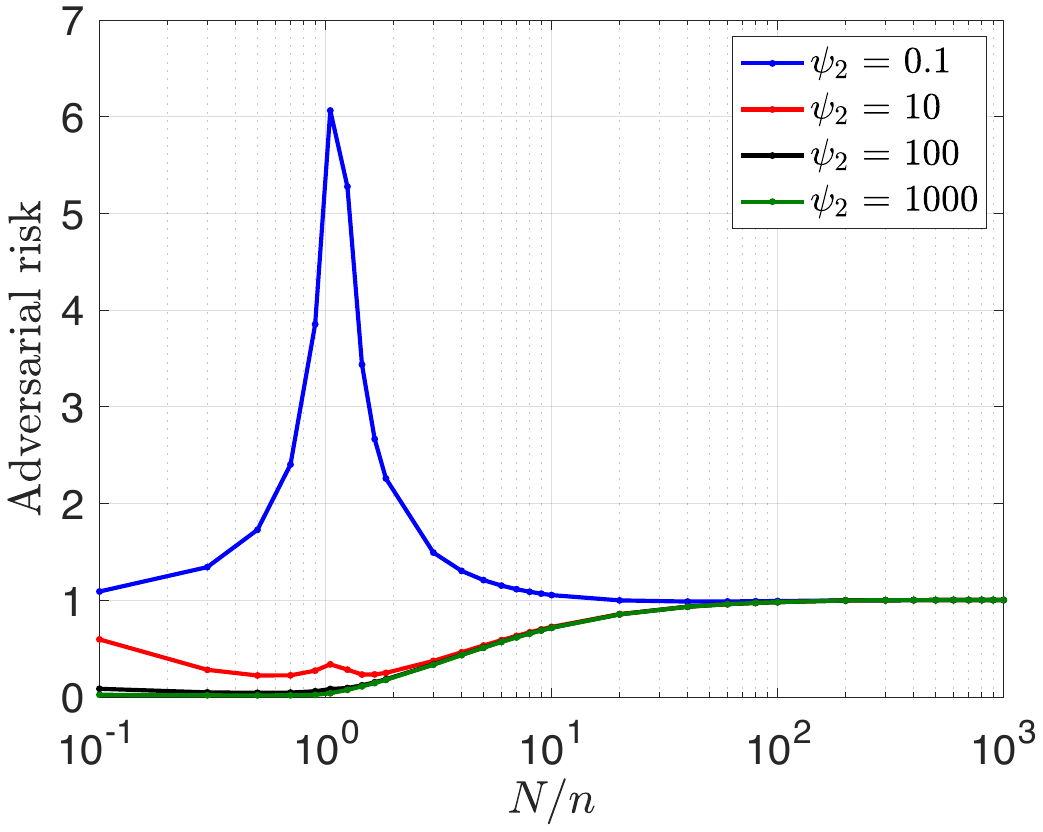}
         \caption{$\tau^2 = 0$, $\varepsilon = 0.1$}
         \label{fig:tot1}
     \end{subfigure}
     \hfill
      \begin{subfigure}[b]{0.485\textwidth}
         \centering
         \includegraphics[width=\textwidth]{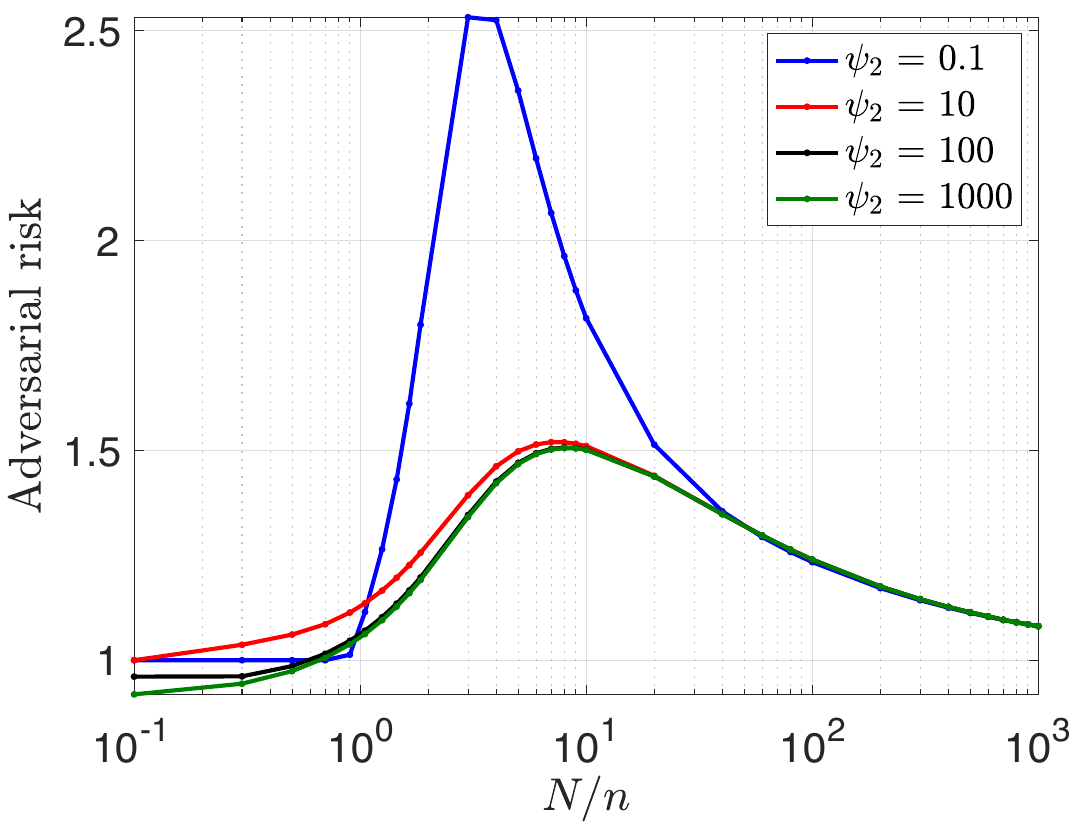}
         \caption{$\tau^2 = 0$, $\varepsilon = 1$}
         \label{fig:tot2}
     \end{subfigure}
     \hfill
     \begin{subfigure}[b]{0.48\textwidth}
         \centering
         \includegraphics[width=\textwidth]{FIG/noise05eps01}
         \caption{$\tau^2 = 0.5$, $\varepsilon = 0.1$}
         \label{fig:tot3}
     \end{subfigure}
     \hfill
     \begin{subfigure}[b]{0.5\textwidth}
         \centering
         \includegraphics[width=\textwidth]{FIG/noise05eps1}
         \caption{$\tau^2 = 0.5$, $\varepsilon = 1$}
         \label{fig:tot4}
     \end{subfigure}
     \hfill
     \begin{subfigure}[b]{0.47\textwidth}
         \centering
         \includegraphics[width=\textwidth]{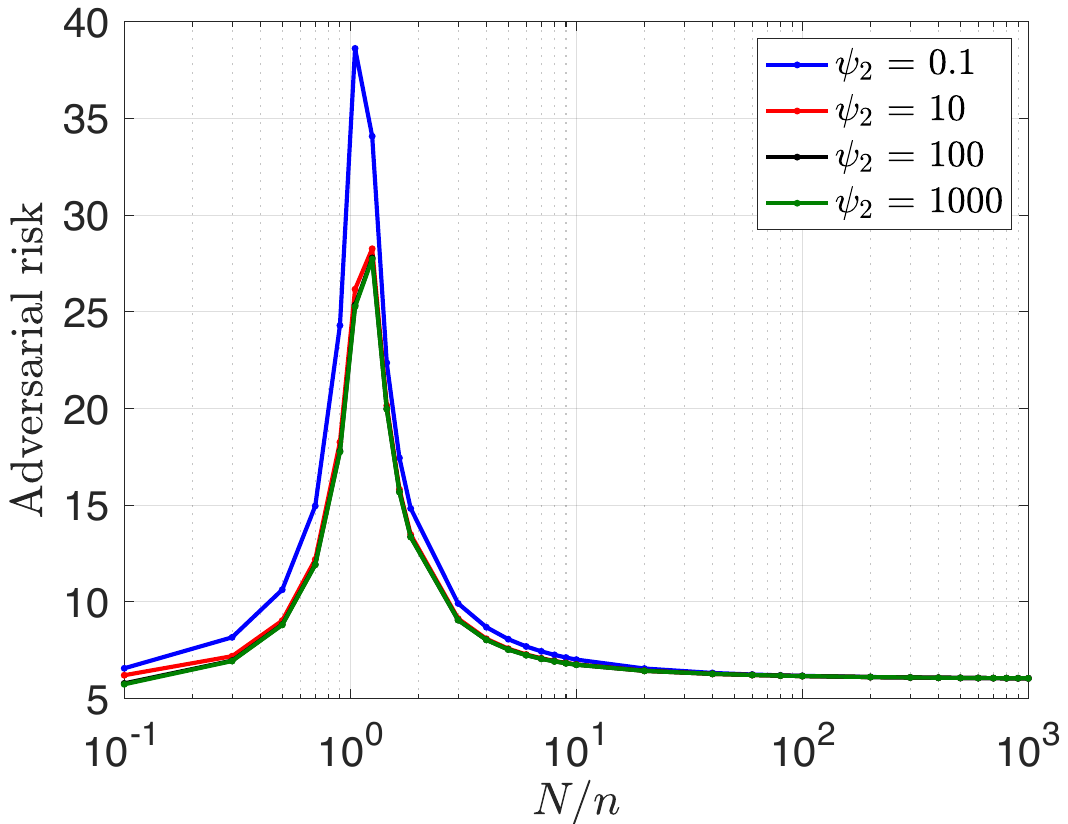}
         \caption{$\tau^2 = 5$, $\varepsilon = 0.1$}
         \label{fig:tot5}
     \end{subfigure}
     \hfill
      \begin{subfigure}[b]{0.465\textwidth}
         \centering
         \includegraphics[width=\textwidth]{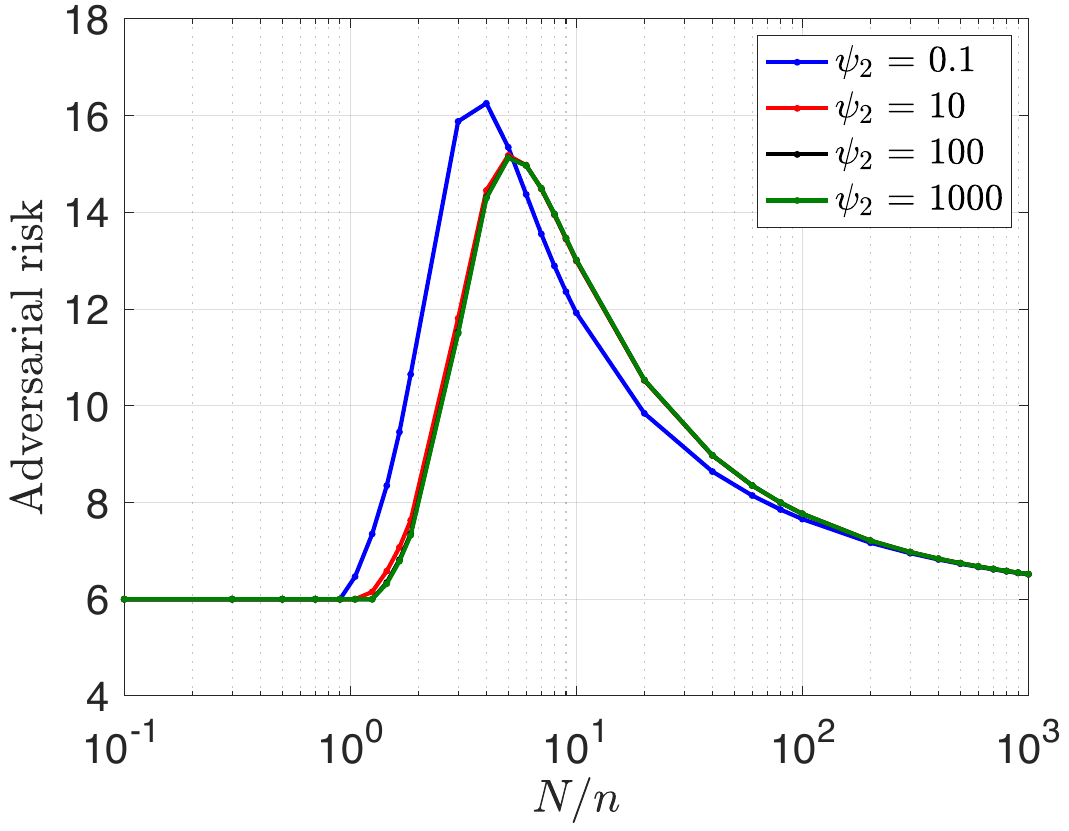}
         \caption{$\tau^2 = 5$, $\varepsilon = 1$}
         \label{fig:tot6}
     \end{subfigure}
        \caption{\small{Theoretical prediction curves for adversarial risk of robust ERM as a function of overparametrization $\psi_1/\psi_2 = N/n$ for different values of $\psi_2 = n/d$. Each plot corresponds to a specific value of adversary's power $\eps$ and noise variance $\tau^2$.} 
        \label{fig:tot}}
\end{figure}

\section{Related Work} \label{sec:related}
Several recent works have focused on the robustness of overparametrized models. On the one hand, \cite{bubeck2021law} 
shows that in order \emph{interpolate} the training data smoothly, the Lipschitz parameter of the resulting model should be at least of order $\sqrt{\frac{nd}{N}}$. \rev{This applies} to data distributions that satisfy a property called isometry--e.g. when the data covariates $\x_i$ are distributed on the unit sphere. For such data distributions, worst-case perturbations are meaningful only if their $\ell_2$ norm is upper-bounded by $\frac{\epsilon}{\sqrt{d}}$.  Otherwise, if the size of the perturbation can be allowed to be much larger than $O(\frac{1}{\sqrt{d}})$, it can be shown that the robust generalization error approaches one for any model--see \cite{DBLP:conf/iclr/ShafahiHSFG19, gilmer2018adversarial,mahloujifar2019curse, mahloujifar2019can}. Putting the above two results together, we can conclude that, in order to interpolate smoothly, while guaranteeing robustness to norm-bounded perturbations, it is necessary that the ratio $N/n$ is bounded away from zero. This is indeed the regime studied in our paper.
However, in this regime, it is not clear why interpolation to training data is beneficial for obtaining robust models. In fact, to obtain robust models, one may have to trade off  the performance on the original data points (i.e. interpolation to training data) with the performance on the points in a ball around each data point (i.e. extrapolation to adversarial examples). In other words, we may have underparametrized models that do not fit the training data perfectly, but have a small Lipschitz constant. Indeed, this can be implied from the main messages of our paper.  

On the other hand, the works in \cite{javanmard2020precise} and \cite{donhauser2021interpolation} have studied the performance of high-dimensional \emph{linear} models and showed that the robust generalization error of adversarially-trained models becomes worse as the models become more overparametrized. In particular, \cite{donhauser2021interpolation} provably shows that avoiding interpolation (and using underparametrized models) improves the robust generalization error in both linear regression and classification--which leads to the first theoretical result on robust overfitting. There are a few reasons on why we might prefer to study non-linear models (such as the random features model) compared to linear models \cite{mei2019generalization}:  First of all, for linear models,  we know that the best (standard) generalization error is attained when the model is highly underparametrized. Second, the number of parameters in a linear model is tied to the covariates dimension $d$ and hence the effects of overparametrization \rev{cannot be} isolated from the effects of the ambient dimensions. Third, a hypothesis put forward in~\cite{DBLP:journals/corr/GoodfellowSS14} is that the origins and ubiquity of adversarial examples is due to the (approximately) linear behavior of a model over large regions of the input space. Shallow linear models are not able to become constant near training points while also assigning different outputs to different
training points. However, the setting of random features is significantly different since this class can express any function to an arbitrary degree of accuracy so long as it has enough number of random features~\cite{rahimi2008uniform}.


In another recent work \cite{wu2021wider},  an extensive study on the robustness of wide neural networks with respect to norm-bounded perturbations is provided. By defining and analyzing a new metric, called perturbation stability, it is shown that while the (standard) generalization error is improved on wider models, the perturbation stability often
worsens, leading to a potential decrease in the overall model robustness. These empirical findings are aligned with the messages of our paper. 

A somehow different line of work~\cite{zhang2022many} studies the sample complexity of the robust interpolation problem where the goal is to interpolate (noisy) training data by a Lipschitz function, under generic covariate distribution (beyond isoperimetry distributions). This work  measures the (non)robustness of a model by its Lipschitz constant and similar to~\cite{dohmatob2021fundamental} establish a lower bound on Lipschitnzenss which is increasing with respect to the overfitting level. This result can be rephrased as an adverse effect of \rev{overparametrization} (thorough overfitting) on robustness. While these work study the effect of \rev{overparametrization} on robustness via memorization/interpolation, we will take a direct approach to study the effect of overparametrization on `adversarially trained' models. 

Several works have shown a non-trivial tradeoff between the robust generalization error and the standard generalization error for parametric models \cite{tsipras2018robustness,su2018robustness,raghunathan2019adversarial,DBLP:conf/icml/ZhangYJXGJ19, javanmard2020precise, edgar}.  It has also been shown that using more data can improve this tradeoff \cite{carmon2019unlabeled, min2021curious, raghunathan2019adversarial,sehwag2021improving,deng2021improving,zhai2019adversarially,najafi2019robustness,gowal2020uncovering, rebuffi2021data}. Again, these findings are aligned with the messages of our paper as more data can  mean less overparametrization.

This paper provides, for the first time, an analysis for the adversarially-trained random features model in the high-dimensional regime. For linear models, this analysis has been carried out in \cite{javanmard2020precise} for the  regression setting, and later on in \cite{javanmard2020precisec, taheri2020asymptotic} for the classification setting. A key ingredient of the analysis in these papers, as well as our paper, is a powerful extension of a classical Gaussian process inequality \cite{gordon1988milman}, known as the Convex Gaussian Minimax Theorem, developed in \cite{thrampoulidis2015regularized} and further extended in \cite{thrampoulidis2018precise, deng2019model}. Another key ingredient of our analysis is the Gaussian Equivalence Property for the random features model which was proposed and studied in \cite{montanari2019generalization} for  maximum-margin linear classifiers in the overparametrized
regime, as well as \cite{hastie2019surprises, abbasi2019universality, montanari2019generalization, mei2019generalization,  gerace2020generalisation,  dhifallah2020precise, hu2020universality, goldt2020modeling, goldt2020gaussian} for the linear Gaussian model under other settings. In particular, a part of our analysis, that establishes equivalence with the so-called noisy linear model in the adversarial setting, is heavily based on the machinery which was elegantly developed in \cite{hu2020universality} for the random features model. This machinery is itself based on the Lindeberg principle \cite{lindeberg1922neue} and the leave-one-out technique developed in \cite{el2018impact, abbasi2019universality}.  

We conclude this section by a broad comparison between adversarial setting and the literature of robust statistics.
\medskip

\noindent{\bf Comparison with robust statistics.}  This area traditionally considers a setting where perturbations are made to the \emph{training data}; a small fraction of data samples are grossly corrupted and the goal is to find estimators
that are robust against outliers (via measures like influence function, breakdown point, and change of variance, etc). In the adversarial training paradigm, one considers the so-called test-time adversarial setting, in which the training data is uncorrupted (say $(\bx_i,y_i)\sim \prob$ for some distribution $\prob$). However, the adversary can perturb \emph{each test data}. In other words, the test data $(\bx,y)$ is drawn from $\prob$ and then $\bx$ is perturbed by the adversary ($\bx\to \tbx$ with $\twonorm{\bx-\tbx}\le \eps$). The goal of adversarial training is to develop a model that can still predict the response $y$ from the perturbed feature $\tbx$. With this view, adversarially robust models are basically those that have good generalization and  are also smooth enough so that they do not change much on small neighborhoods (of radius at most $\eps$).

That said, another line of work (see e.g.~\cite{lai2020adversarial}) considers a different adversarial setup in which an attacker can
observe and modify all training data samples in an
adversarial manner so as to maximize the estimation error caused
by his attack. This work introduces the notion of adversarial influence function (AIF) to quantify the sensitivity of estimators to such adversarial attacks, and further derive the optimal estimator, among a certain class of estimator, that minimizes AIF.  \rev{Related to this setting, there is also a line of work based on the Median of Means approach, see e.g, \cite{huang2023deepmom, depersin2023robustness}), which concerns a data poisoning/ data contamination adversarial setting. In data poisoning, the adversary can pick a (small) fraction of the \emph{training data} and alter it in a way that it hurts the training process, and ultimately the generalization performance.  However, in this paper we consider a different type of adversarial act which has to do with adversarial perturbation (in a small ball) of the input data point at the \emph{test time}.}



\section{Main results} \label{sec:main_results}
Recall the  data distribution given in~\eqref{eq:linearModel}. Given $n$ i.i.d pairs $(\bx_i,y_i)$ drawn from this distribution, we fit a random features model, defined as the function class~\eqref{eq:RF}, with the shifted ReLU activation:
\begin{equation} \label{shifted_relu}
\sigma(x) = \max(x,0)-\frac{1}{\sqrt{2\pi}}\,.
\end{equation}

We consider sequences of parameters $(N,n,d)$ that diverge proportionally to each other and sometimes, we index such sequences by $d$, with $N = N(d)$ and $n= n(d)$ functions of $d$.
\begin{assumption}\label{assumption1}({Asymptotic setting.})
\begin{itemize}
\item[(a)] Defining $\psi_{1,d} = N/d$ and $\psi_{2,d} = n/d$, we assume that the following limits exist:
\[
\lim_{d\to\infty} \psi_{1,d} = \psi_1, \quad \lim_{d\to\infty} \psi_{2,d} = \psi_2\,,
\]
for some positive finite constants $\psi_1$ and $\psi_2$. 
\item[(b)] We assume that the $\ell_2$ norm of the signal $\bbeta$ converges, as $d\to\infty$. For the sake of
normalization and without loss of generality, we assume $\lim_{d\to\infty} \twonorm{\bbeta} = 1$.
\end{itemize}
\end{assumption} 

\rev{Recall that in the data model~\eqref{eq:linearModel}, $\bx_i\sim_{iid}\normal(0,\mtx{I}_d)$ and so its distribution is rotation-invariant. Likewise, in the random features
model \eqref{eq:RF}, the rows $\bw_\ell$ are chosen uniformly at random from the unit sphere, and so has a rotation-invariant distribution. In our adversarial setting, we also focus on
norm-bounded perturbations which is again a rotation-invariant constraint. Using these properties, it is easy to see that the adversarial risk will be invariant if we rotate the model $\bbeta$
in~\eqref{eq:linearModel} and hence only depends on $\twonorm{\bbeta}$. This justifies Assumption~\ref{assumption1}(b) made above.}

\rev{To study  robust generalization of the estimated models, we consider an adversarial framework with \rev{norm bounded} perturbations.
This can be formulated as a game between the learner and the adversary. Given access to a training dataset consisting of $n$ i.i.d. pairs $(\bx_i,y_i)$ generated from~\eqref{eq:linearModel}, the learner chooses a model $\bth$ from the class of random features model $\cF_{{\rm RF}}(\bW)$~\eqref{eq:RF}. Adversarial perturbations occur at the test time. After observing the learner's model, for every test point $\bx$, the adversary perturbs it to $\bx+\bdelta$ where $\bdelta$ is chosen from the Euclidean ball of radius $\eps$. Note that the choice of $\bdelta$ can in general depend on $\bx$ and the learner's model. The robust generalization of the learner's model is quantified via a measure called \emph{adversarial risk}, which \rev{is the expected} prediction loss of the model on an adversarially
corrupted test data point according to the described attack model.}
 
 \begin{definition}(\emph{Adversarial risk.})
 For a predictive model $f$ and a loss of choice $\ell:\reals\times\reals\to \reals_{\ge 0}$, the adversarial risk of model $f$ is defined as:
 \rev{
 \[
 \AR(f): = \E\Big[\max_{\twonorm{\bdelta}\le \etest} \ell(f(\bx+\bdelta), y)\Big]\,,
 \]
 }
  where the expectation is with respect to randomness of $(\bx,y)$. 
  
  In particular, for a random features model $\bth = (\theta_1, \cdots, \theta_N)^\sT$ from $\cF_{{\rm RF}}(\bW)$, defined in \eqref{eq:RF}, and with the choice of squared loss, the adversarial risk of $\bth$ becomes:
   \begin{align}\label{eq:AR}
\AR(\bth): = \E\Big[\max_{\twonorm{\bdelta}\le \etest} \left(y- \bth^\sT\sigma(\bW(\bx+\bdelta))\right)^2 \Big]\,.
\end{align}
 \end{definition}
Norm bounded adversarial attack models are widely used in the literature, motivated by a plethora of safety-critical applications in machine learning, computer vision, natural language processing, medical imaging, and robotics. A popular approach to adversarial training  is by considering the following robust empirical risk minimization (robust-ERM) problem~\cite{madry2017towards,tsipras2018robustness}:
\begin{align}\label{def:Rob-ERM}
\hth^\eps = \arg\min_{\bth\in\reals^N}  \max_{\twonorm{\bdelta_i}\le \eps}  \frac{1}{2n} \sum_{i=1}^n \left(y_i - \bth^\sT \sigma(\bW(\bx_i + \bdelta_i)\right)^2\,.
\end{align}
Here, $\etest$ is a measure of the adversary's power and $\eps$ is the ``perceived'' adversary's power used by the algorithm. Our theory allows for $\eps$ to be different from $\etest$; cf Theorem~\ref{thm:main}. In our numerical experiments in Section~\ref{sec:overview_of_results} we consider $\eps = \etest$ to focus on other relevant quantities, namely $\psi_1$, $\psi_2$ on adversarial risk.

Note that the above objective is the empirical surrogate of the adversarial risk~\eqref{eq:AR}, where the expectation is replaced by the empirical average over the training samples. This minimax approach can also be viewed as an implicit smoothing that
tries to fit the same response $y$ to all the covariate vectors in the $\eps$-neighborhood of $\bx$ simultaneously.

Our main result in this paper is a precise characterization of the adversarial risk of the robust ERM model~\eqref{def:Rob-ERM} under the asymptotic regime described in Assumption~\ref{assumption1}. Before stating our result, we introduce another piece of notation.

For $\psi_1\in (0,\infty)$, we define function $S(\cdot;\psi_1):\reals_{<0}\to \reals_{<0}$:
\begin{align}\label{def:S}
S(z;\psi_1) = \frac{1-\psi_1-z - \sqrt{(1-\psi_1-z)^2-4\psi_1 z}}{-2\psi_1 z}\,.
\end{align}
One may recognize that $S(z;\psi_1)$ is the Stieltjes transform of the Marchenko-Pastur distribution. We refer 
to \rev{Lemma}~\ref{propo:stiel} (Appendix~\ref{sec:useful-lemma}) for more details.

We are now ready to state our main result.
\begin{theorem}\label{thm:main}
Let $n$ i.i.d pairs $(\bx_i,y_i)$ be drawn from the data model~\eqref{eq:linearModel} and let $\hth^\eps$ be the robust ERM fit~\eqref{def:Rob-ERM}  to this data using the class of random features models $\cF_{{\rm RF}}(\bW)$, given by~\eqref{eq:RF} with the shifted ReLU activation. Consider the asymptotic regime, described in Assumption~\ref{assumption1}. 
With function $S(\cdot;\psi_1)$ given by~\eqref{def:S}, define
\[
\sigma^2 = \tau^2 +1 - \psi_1\left(1+\Big(1-\frac{2}{\pi}\Big) S\Big(\frac{2}{\pi}-1;\psi_1\Big)\right)\,.
\]
$(a)$\;  For $\eps>0$, the following convex-concave minimax scalar optimization has a unique solution \rev{$(\alpha_*,\tau_{g*}, \beta_*,\gamma_*,\tau_{q*})$}:
\begin{align}\label{eq:AO-final00}
 \max_{0\le\beta, \gamma,\tau_q}\;\; \min_{0\le \alpha, \tau_g}\;\;  & 
\cR(\alpha,\tau_g, \beta,\gamma,\tau_q) 
\,,
\end{align} 
where
\begin{align*}
\cR(\alpha,\tau_g, \beta,\gamma,\tau_q) : =  &\frac{\tau_q}{2\alpha} (\tau^2+1-\sigma^2)
 - \frac{\alpha \tau_q}{2}+\frac{\beta\tau_g}{2}\psi_2 + \frac{\beta}{2(\tau_g+\beta)} (\sigma^2+\alpha^2) \nn\\
 &+\mathbf{1}_{\Big\{\frac{\gamma(\tau_g+\beta)}{\eps\beta\sqrt{\alpha^2+\sigma^2}}>\sqrt{\frac{2}{\pi}} \Big\}} \frac{\beta^2(\alpha^2+\sigma^2)}{2\tau_g(\tau_g+\beta)}\left(\erf\left(\frac{\nu^*}{\sqrt{2}}\right)-\frac{\gamma(\tau_g+\beta)}{\eps\beta\sqrt{\alpha^2+\sigma^2}}\nu^*\right)\nn\\
 &-\frac{\alpha}{\tau_q} \sup_{0\le {\lambda}< 1}\;\;\left[
\frac{{\lambda}\psi_1}{2}\left\{\frac{\tau_q^2}{\alpha^2}+\beta^2 +\Big(\frac{\tau_q^2}{\alpha^2}\Big(1 - \frac{2}{\pi}{\lambda} \Big)+\frac{2}{\pi} (1- {\lambda})\beta^2\Big) S\Big(\frac{2}{\pi}{\lambda} - 1;\psi_1\Big)\right\}
-\frac{{\lambda}}{2(1- {\lambda})}\gamma^2\right]\,.
\end{align*}
Here, $\nu^*$ is the unique solution to 
\begin{align*}
\frac{\gamma(\tau_g+\beta)}{\eps\beta\sqrt{\alpha^2+\sigma^2}}-\frac{\beta}{\tau_g}\nu-\nu\cdot\erf\left(\frac{\nu}{\sqrt{2}}\right)-\sqrt{\frac{2}{\pi}} e^{-\frac{\nu^2}{2}}= 0\,.
\end{align*}
$(b)$\; The adversarial risk of the robust ERM $\hth^\eps$ converges in probability
\begin{align}
\AR(\hth^\eps) \stackrel{\mathcal{P}}{\to}  
\bigg[1+\Big(\frac{\etest \beta_*\nu_*}{\eps \tau_{g*}}\Big)^2 + 2\sqrt{\frac{2}{\pi}} \frac{\etest \beta_*\nu_*}{\eps\tau_{g*}} \bigg] (\alpha_*^2+\sigma^2)\,.\label{eq:AR-final-form}
\end{align}
Here, the probabilistic statement is with respect to the randomness in both the training data $\{(\bx_i,y_i)\}_{i=1}^n$ and the random features $\bW$.
\end{theorem}

We note that the robust ERM estimator is a \emph{random} and rather complicated high-dimensional function of the training data. 
However, in the asymptotic regime where $N,n,d\to\infty$ at the same order, the adversarial risk of the robust ERM concentrates and the above theorem provides an exact characterization of its limit as a \emph{deterministic} formula. The derived formula is based on a  five dimensional convex-concave mini-max optimization
problem and its  optimal solution can be
easily derived via a simple low-dimensional gradient descent rather quickly and accurately. Alternatively, one can form a system of equations by writing the KKT stationary conditions corresponding to \eqref{eq:AO-final00}. The adversarial risk prediction can then be written in terms of the fixed point of this system of deterministic equations.
   
\rev{Let us re-emphasize the contribution of theorem~\ref{thm:main}. Albeit its involved form, it describes the behavior of a \emph{high-dimensional random} problem in terms of a \emph{deterministic} optimization with a handful number of scalar variables. This theme of result is similar to state evolution equation for approximate message passing algorithms~\cite{AMP}, density evolution for LDPC codes~\cite[Chapter~4]{richardson2008modern}, and characterizing the trajectory of SGD for training neural networks in terms of partial differential equations~\cite{song2018mean,javanmard2019analysis}.}   
   
\rev{
\begin{remark} ({\bf Solving Optimization~\eqref{eq:AO-final00}})
We find the solution to this optimization by solving for the first-order optimality
conditions (stationary equations). We set the (sub)gradient with respect to the variables to zero since it is non-smooth, which results in a system of nonlinear equations with seven variables, namely $\alpha, \tau_g,\beta,\gamma,\tau_q, \lambda, \nu^*$. In the numerical experiments, we use \texttt{fsolve} command in \texttt{Matlab} to solve this system of equations, which is based on the trust-region algorithm. 
\end{remark}
}

\rev{\section{Discussion}
By virtue of Theorem~\ref{thm:main} we characterize the adversarial risk of the robust ERM in term of the solution of the deterministic optimization problem~\eqref{eq:AO-final00}. Given that it does not admit a closed-form solution in general, and is rather involved, in this section we discuss some of the applications of this theorem including optimal choice of $\eps$ during training, trend of adversarial risk with respect to different quantities, and implications for non-adversarial setting.

\subsection{Optimal $\eps$ for robust ERM estimator:}
An interesting application of our theory is to derive the optimal $\eps$ (perceived adversary's perturbation
level) in the robust ERM, while fixing the adversary’s
(actual) perturbation level on test inputs to $\etest$. The optimal $\eps$ here refers to the value which minimized the adversarial risk. An intriguing observation is that the optimal $\eps$ is different than $\etest$ in general, and depends on $\psi_1,\psi_2$
in a non-trivial way (There is no universal solution, which underscores the significance of possessing a precise theory that comprehends the impact of adversarial training, which constitutes the principal objective of the present work.)

In Figure~\ref{fig:vary_eps_a}, we fix $\psi_2 = 3$, $\tau = \sqrt{0.5}$, $\etest = 0.3$, and plot the adversarial risk of $\hth^\eps$ as we vary $\eps$ for different values of
$\psi_1$. As we see the optimal value of $\eps$ (resulting in minimum risk) changes with $\psi_1$, it is in general different from the test adversary's perturbation $\etest$.
In addition, the optimal $\eps$ increases with $\psi_1$. In Figure~\ref{fig:vary_eps_b}, we plot similar curves, fixing $\etest = 0$. In this case, the adversarial risk reduces to the notion of standard risk defined as
\begin{align}\label{eq:SR}
\SR(\bth): = \E\left[(y-\bth^\sT\sigma(\bW\bx))^2\right]\,.
\end{align}
As we see form the plots, even though $\etest = 0$, adversarial training can help as minimum risk is achieved at positive $\eps$. The reason is that adversarial training 
acts as a regularization (It becomes clearer after derivation~\eqref{L_2}, where adversarial training aims to find solution with small $\twonorm{\bJ\bth}$.) In particular, we observe that
\begin{itemize}
\item \revv{The optimal $\eps$ is always greater than or equal to $\etest$, the true test perturbation magnitude. This `additional' regularization helps with minimizing the adversarial risk.}
\item At higher overparametrization measured by $\psi_1/\psi_2 = N/n$ the benefit of this regularization becomes stronger. This is also evident from Figure~\ref{fig:vary_eps_c}, where with fixed $\psi_2$.  As we increase $\psi_1$, the optimal $\eps$ also increases. Likewise, in Figure~\ref{fig:vary_eps_d}, with fixed $\psi_1$, the optimal $\eps$ increases as $\psi_2$ decreases.
\end{itemize}

\begin{figure}[]
     \centering
      \begin{subfigure}[b]{0.48\textwidth}
         \centering
         \includegraphics[width=\textwidth]{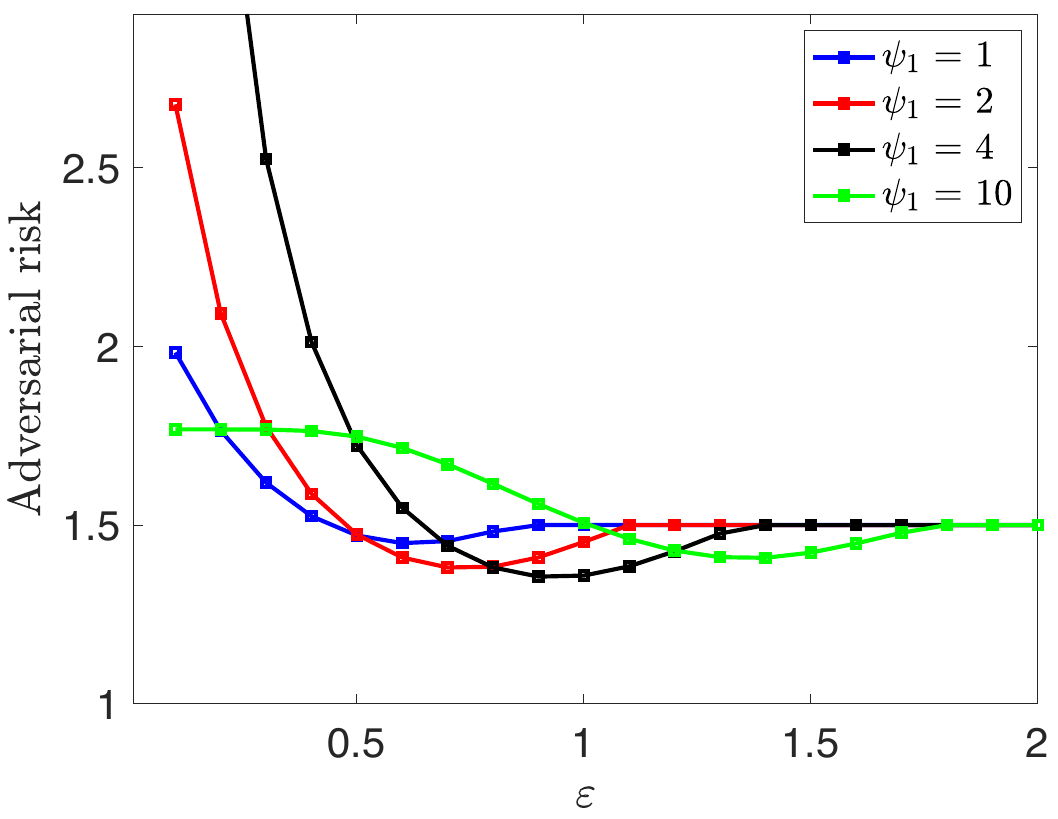}
         \caption{$\etest = 0.3$, $\psi_2 = 3$}
         \label{fig:vary_eps_a}
     \end{subfigure}
     \hfill
      \begin{subfigure}[b]{0.48\textwidth}
         \centering
         \includegraphics[width=\textwidth]{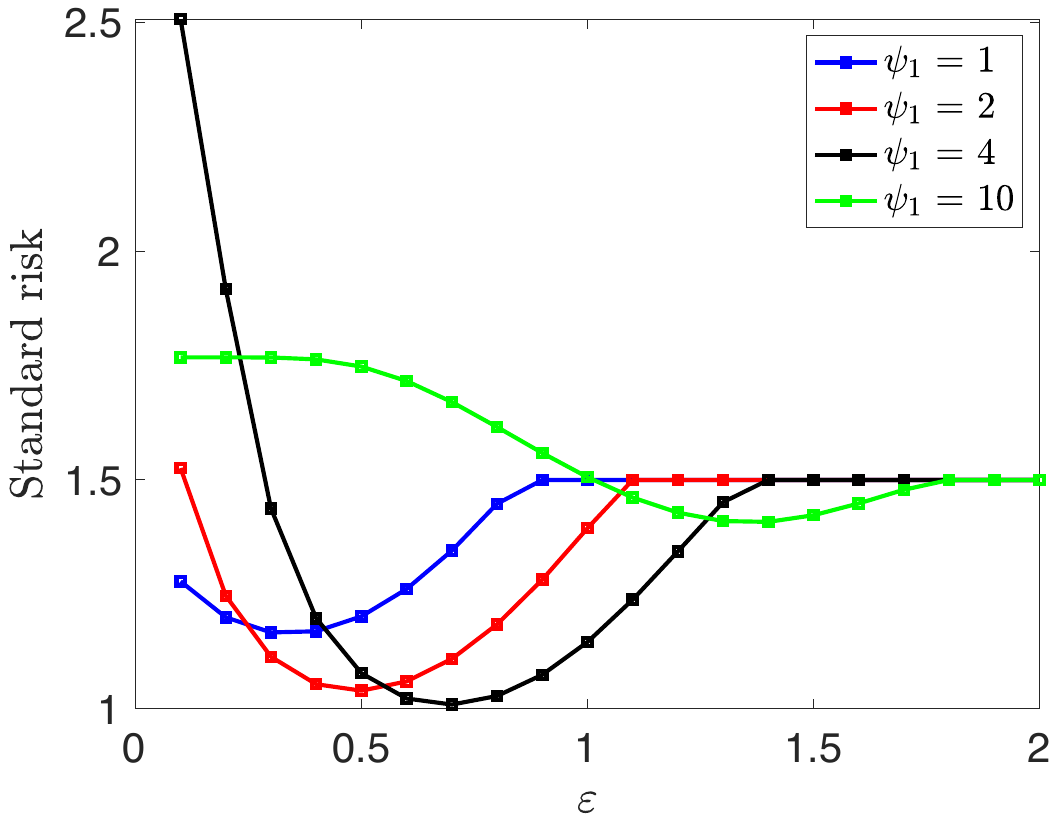}
         \caption{$\etest = 0$, $\psi_2 = 3$}
         \label{fig:vary_eps_b}
     \end{subfigure}
     \hfill
     \begin{subfigure}[b]{0.48\textwidth}
         \centering
         \includegraphics[width=\textwidth]{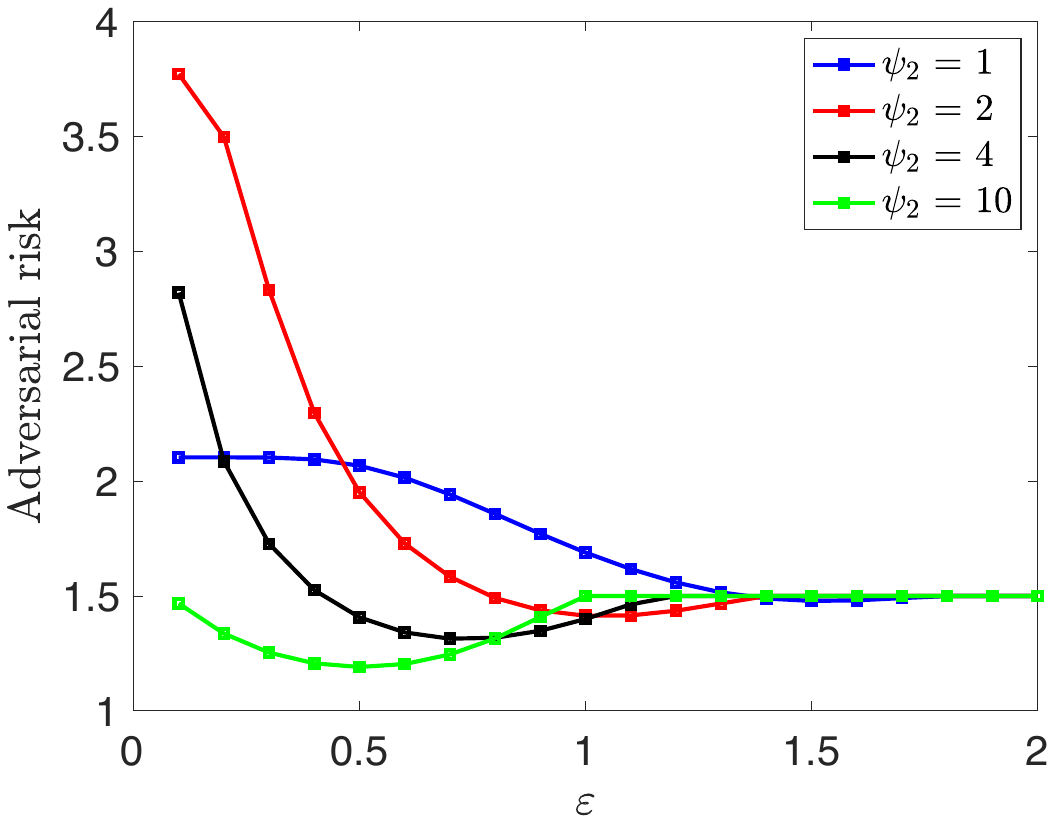}
         \caption{$\etest = 0.3$, $\psi_1 = 3$}
         \label{fig:vary_eps_c}
     \end{subfigure}
     \hfill
     \begin{subfigure}[b]{0.48\textwidth}
         \centering
         \includegraphics[width=\textwidth]{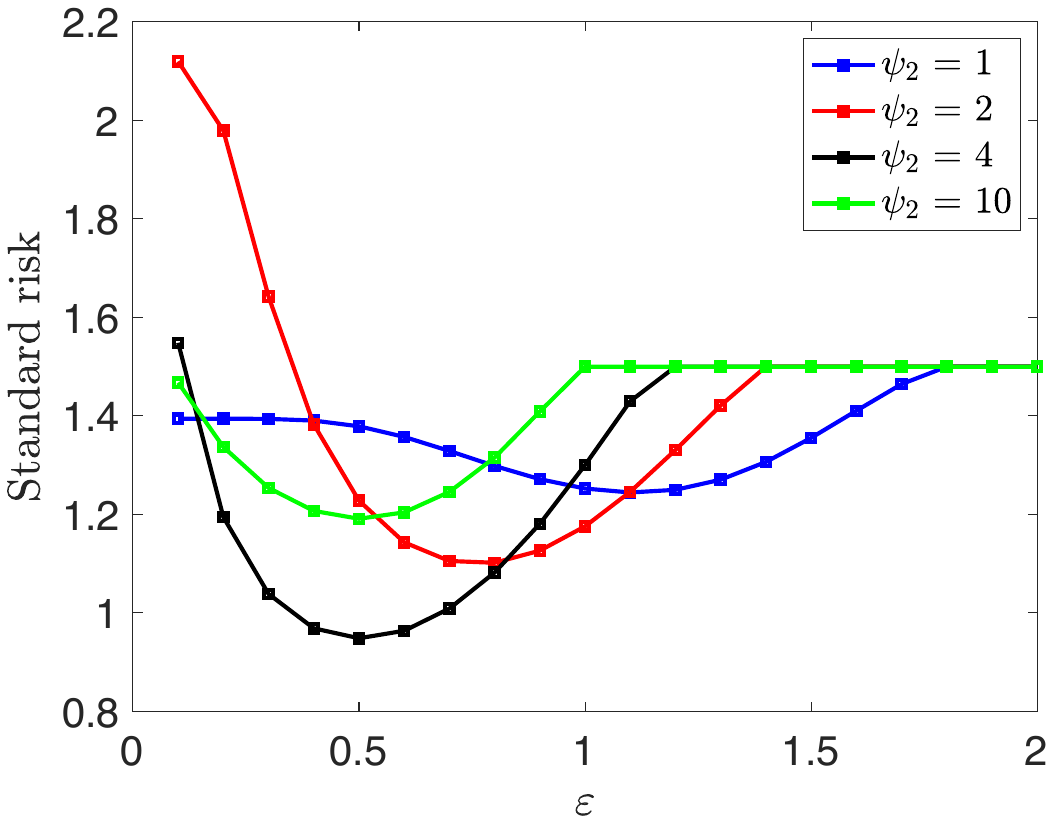}
         \caption{$\etest = 0$, $\psi_1 = 3$}
         \label{fig:vary_eps_d}
     \end{subfigure}
             \caption{\rev{\small{Behavior of adversarial/standard risk as we vary $\eps$, the ``perceived'' adversary's power used in the adversarial training.  In (a), (b), $\psi_2 = 3$ is fixed and in (c), (d), we fix $\psi_1 = 3$. Also, (b), (d) correspond to $\etest = 0$, and so there is no perturbation at the test time. In these cases, adversarial risk reduces to the standard risk. In these experiments, we set $\tau^2 = 0.5$, $\twonorm{\bbeta} = 1$.}} 
        \label{fig:vary_eps}}
\end{figure}

In summary, our theory allows to understand when adversarial training is beneficial and what is the optimal value of $\eps$ to use in training (depending on $\psi_1,\psi_2$, $\etest$ and $\tau$.)

\revv{\subsection{Dependence on $\eps$, $\psi_1$, $\psi_2$: } We discern the following trends in the analytical curves for adversarial risk which are derived based on Theorem~\ref{thm:main} (Equation~\eqref{eq:AR-final-form}).
\begin{itemize}
\item From Figure~\ref{fig:vary_eps}, fixing $\etest$, $\psi_1$ and $\psi_2$, the adversarial risk is first decreasing in $\eps$ until it gets to its minimum, after which it becomes increasing in $\eps$ indicating that the regularization effect from adversarial training is larger than it should be, and then eventually it levels at $\sigma^2+ \|\bth_0\|^2$ (the risk of model $\bth=\vct{0}$, which is the ERM when the regularization is very strong). 
\item For fixed $\psi_2$, if SNR is large enough and $\eps$ is small enough, we observe a double descent behavior in the adversarial risk (see Figure~\ref{fig:varyeps1}, $\eps = 1e-7$ and $\eps = 0.05$ and Figure~\ref{fig:tot3}).
\item Fixing $\etest$ and $\psi_1$, the adversarial risk becomes decreasing in $\psi_2$ after some point. In the next subsection, we characterize the adversarial risk in non-adversarial setting, by which we see this trend for $\psi_2>\psi_1$. 
\end{itemize}
}
\subsection{Non-adversarial training:} An important special case of the result is when $\eps = \etest = 0$. In words, there is no adversarial-training and also there is no adversary during the test time. Theorem~\ref{thm:main} allows us to characterize the standard risk (generalization error) of the ERM estimator.

We focus on the underparameterized regime ($n>N$ or equivalently $\psi_2> \psi_1$), since otherwise at $\eps=0$ the problem is underdetermined. In this case the objective $\mathcal{R}$ in~\eqref{eq:AO-final00} significantly simplifies and we can indeed obtain a closed-form expression for the risk.

\begin{propo}\label{propo:non-ad}
Let $n$ i.i.d pairs $(\bx_i,y_i)$ be drawn from the data model~\eqref{eq:linearModel} and let $\hth$ be the ERM~\eqref{def:Rob-ERM} fit to this data using the class of random features models $\cF_{{\rm RF}}(\bW)$, given by 
\[
\hth = \arg\min_{\bth\in\reals^N} \frac{1}{2n} \sum_{i=1}^n (y_i - \bth^\sT\sigma(\bW\bx_i))^2\,,
\]
with $\sigma(\cdot)$ the shifted ReLU activation. Consider the asymptotic regime, described in Assumption~\ref{assumption1}. 
With function $S(\cdot;\psi_1)$ given by~\eqref{def:S}, define
\[
\sigma^2 = \tau^2 +1 - \psi_1\left(1+\Big(1-\frac{2}{\pi}\Big) S\Big(\frac{2}{\pi}-1;\psi_1\Big)\right)\,.
\]
Then, the standard risk of the ERM $\hth$ converges in probability to
\begin{align}\label{eq:SR-RF}
\SR(\hth) \stackrel{\mathcal{P}}{\to} \sigma^2 \Big(\frac{\psi_2}{\psi_2-\psi_1}\Big)\,.
\end{align}
\end{propo} 
We refer to Section~\ref{proof:propo:non-ad} for the proof of this proposition.

We next use the result of Proposition~\ref{propo:non-ad} to discuss the role of $\psi_1$ and $\psi_2$ on the risk:
\begin{itemize}
\item Recall that $\sigma$ only depends on $\tau$ and $\psi_1$. Fixing $\psi_1$ the dependence of risk on $\psi_2$ is of form $\psi_2/(\psi_2-\psi_1)$. This is decreasing in $\psi_2 = n/d$. For example if $d, N$ are fixed, and we increase the sample size $n$, the risk goes down which is expected.
\item Dependence on $\psi_1$ is more involved as the term $\sigma^2$ also depends on $\psi_1$ through the Stieltjes transform $S$. In Figure~\ref{fig:nonadv_Psi} we plot the risk versus $\psi_1$ for 
different values of $\psi_2$. As we see up to some threshold, it is decreasing in $\psi_1$ but after that it becomes increasing. This is expected because for example fixing $n$, $d$ (and so $\psi_2$), as we increase $N$ (and so $\psi_1$), first the risk goes down because the model becomes richer to capture the data generative model, but after some point it has a reverse effect, because we need to estimate larger number of parameters $N$, from fixed sample size $n$, while this excess model complexity is not needed. As we see in the plots, this threshold  is increasing with $\psi_2$.   
\item It is also worth comparing the standard risk of random features model with that of linear models. For $\psi_2\ge 1$, using the result of~\cite[Proposition 2]{hastie2019surprises}, the risk of ridgeless least squares is given by $\tau^2 \psi_2/(\psi_2 -1)$. This is similar to our characterization~\eqref{eq:SR-RF}, where the noise variance $\tau^2$ is replaced with the effective noise variance $\sigma^2$, and $\psi_2$ is replaced by $\psi_2/\psi_1 = n/N$. (Note that the number of parameters to be learnt in the linear model is $d$, while in the random features model is $N$. So the sample size to parameter size ratio in the linear regression is $\psi_2$, while for the random features model it is $\psi_2/\psi_1$.)  
\end{itemize}

\begin{figure}[]
         \centering
         \includegraphics[width=0.48\textwidth]{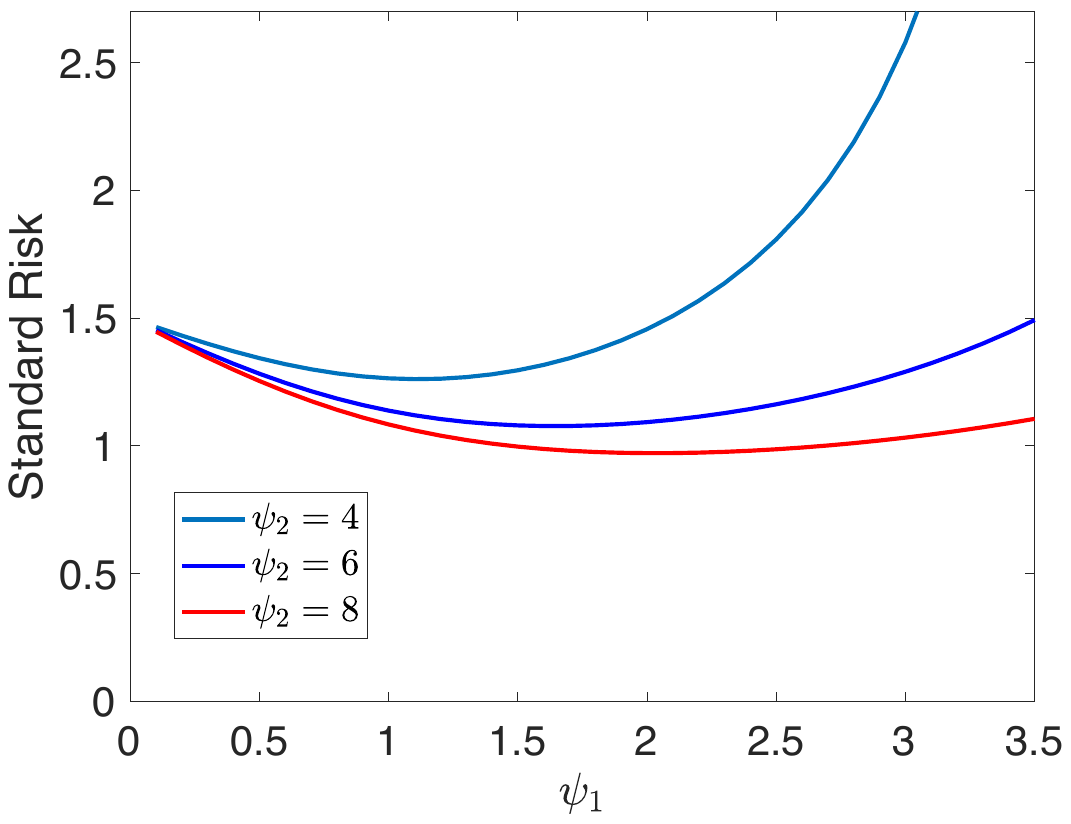}
             \caption{\rev{\small{Behavior of the standard risk of the ERM estimator $\hth$ versus $\psi_1$ for different values of $\psi_2$. As observed, the risk is first decreasing in $\psi_1$, up to some threshold depending on $\psi_2$, after which it becomes increasing. This threshold is increasing with $\psi_2$. }} 
        \label{fig:nonadv_Psi}}
\end{figure}

}

  
\section{Architecture of the proof} \label{sec:sketch}
This section introduces the key steps underlying the proof of our main result, Theorem~\ref{thm:main}. Our analysis is intricate and consists of a host of novel ideas which could be of separate interest. Here we discuss the major steps, along with an overview of the techniques and intermediate results.

Define  the loss $\cL(\bth)$ given by
\begin{equation} \label{R-ERM}
\cL(\bth):= \max_{\twonorm{\bdelta_i}\le \eps}  \frac{1}{2n} \sum_{i=1}^n \left(y_i - \bth^\sT \sigma(\bW(\bx_i + \bdelta_i)\right)^2 + \frac{\zeta}{2} \bth^\sT \bOmega \bth \,,
\end{equation}
where $\bOmega: = \Iden + \frac{\sqrt{\log (d)}}{d} \ones\ones^\sT$. 
Here, $\ones = (1,1,\dotsc,1)\in\reals^N$ and $\zeta>0$ is an arbitrary small but fixed constant. 
By definition, the robust ERM estimator~\eqref{def:Rob-ERM} is the minimizer of $\cL(\bth)$ for $\zeta = 0$. 

The regularization $\frac{\zeta}{2} (\twonorm{\bth}^2+ \frac{\sqrt{\log(d)}}{d} \,  \<\ones,\bth\>^2)$ in the loss $\cL(\bth)$ is added for technical reasons. In our analysis, we let $\zeta\to 0$, \emph{after} letting $d\to \infty$ to characterize the adversarial risk of the robust ERM $\hth^\eps$. \rev{We refer to Section~\ref{sec:interchange} in the supplementary for the justification of this step.}

Before we outline the main steps of the proof, we note that since the rows of the matrix $\bW\in\reals^{N\times d}$ are generated i.i.d. according to $ \bw_\ell \sim \Unif(\mathbb{S}^{d-1})$, then the matrix norm of $\W$ is bounded with high probability and the rows of $\W$ are almost orthogonal. More precisely, 
we define the event 
\begin{align}\label{eq:eventbW_main}
\event_{\bW}:= \left\{\opnorm{\bW}\le \sqrt{\psi_{1,d}} + C, \,\,\, \left| \w_\ell^\sT \w_k \right| \leq \log(d)/ \sqrt{d} \,\,\,\,\,\, \forall \ell \neq k  \right\}\,,
\end{align} 
for a large enough constant $C$ (note that $N/d := \psi_{1,d}$ -- see Assumption~\ref{assumption1}). Using well-known results on the norm of random matrices (see e.g. ~\cite[Theorem 5.39]{Vers}) as well as Hoeffding's inequality  we have
 $\prob(\event_{\bW}) \ge 1-c \exp(- \log^2(d)/c)$ for some constant $c>0$.  In the following, our statements are proven conditioned on the event $\event_{\bW}$ which holds with high probability. 
 
 \vspace{.2cm}
\noindent\textbf{Notation.} We  need to define a few pieces of notation which will be used in the following.   We use $O_d(\cdot)$, $o_d(\cdot)$ to denote the standard big-O and little-o notation, where we stress the asymptotic variable $d$. Likewise, we use $O_{d,\prob}$ and $o_{d,\prob}$ to indicate asymptotic behavior in probability. Specifically, $f(d) = O_{d,\prob}(g(d))$ if for any $\eps>0$, there exists $C_\eps>0$ and large enough $d_\eps$ such that $\prob(|f(d)/g(d)| > C_\eps)\le \eps$, for all $d\ge d_\eps$. Similarly, $f(d) = o_{d,\prob}(g(d))$ if $f(d)/g(d)$ converges to zero in probability. We write $f(d)\approx g(d)$ as $d\to\infty$, when $ f(d)-g(d) \to 0$, in probability.
Note that we consider the asymptotic regime where $n,d,N$ grow at the same scale, ($\lim N/d\to \psi_1$ and $\lim n/d\to \psi_2$ for some positive constants $\psi_1$ and $\psi_2$), the expression $d\to\infty$ implies that $n, N\to\infty$, as well. 

For a matrix $\bA$, we denote by $\opnorm{\bA}$ its operator norm, $\fronorm{\bA} = (\sum_{ij} A_{ij}^2)^{1/2}$ the Frobenius norm of $\bA$.  For an integer $n$, we use the shorthand $[n]= \{1,\dotsc, n\}$. 

Finally, the indicator function is denoted by $\ind(\cdot)$ -- i.e., \rev{$\ind(A)=1$ only if the event} $A$ holds true, and otherwise $\ind(A)=0$.

\subsection{An asymptotically-exact closed form for adversarial examples}
We start by simplifying the loss $\cL(\bth)$. If the activation function $\sigma$ was linear, then finding the worst-case perturbations $\bdelta_i$ (maximizers in the definition of loss $\cL(\bth)$) amounts to a trust-region subproblem that can be solved in closed form--see~\cite{javanmard2020precise}. A major challenge here is the nonlinearity of $\sigma$. In the first step we use the specific form of the activation to derive an asymptotically equivalent but simpler form of $\cL(\bth)$.

Note that $\sigma(z) = \max(z,0)-1/\sqrt{2\pi}$ is linear on the positive $z$ and constant on the negative $z$. Also by the constraint on perturbation we have $\|\<\bw,\bdelta\>\|\le \twonorm{\bw}\twonorm{\bdelta} \le \eps$. Therefore, if \rev{$\|\<\bw,\bx\>\| \ge\eps$,} then $\<\bw,\bx\>$ and the perturbed form $\<\bw,\bx+\bdelta\>$ share the same sign. 
In this case, the worst case $\bdelta$ can be solved exactly.  One can also use the randomness in $\bx$ to bound the number of rows of $\bW$ for which $|\<\bw,\bx\>| < \gamma$, where $\gamma $ is some small constant, and show that the contribution of these terms in the loss is asymptotically negligible. This argument is formalized in the next proposition. All the proofs of the statements in this section are relegated to  Appendix~\ref{proofs_of_step_1}.
\begin{propo}\label{propo:simple1}
Assume $(\bx,y)$ generated according to \eqref{eq:linearModel}. Further define 
\begin{align}\label{def:Cth}
\cC_{\bth}:=  \{\bth\in\reals^N:\; \infnorm{\bth} \le C_0 \sqrt{\log(d)/d},\; \twonorm{\bth}\le C_0\},
\end{align}
 for an arbitrary but fixed constant $C_0>0$. Then, we have 
\begin{align}\label{eq:adv1}
\max_{\twonorm{\bdelta}\le \eps}  \left| y - \bth^\sT \sigma(\bW(\bx + \bdelta))\right| =  |y - \bth^\sT \sigma(\bW\bx)| + \eps \twonorm{\bW^\sT  {{\rm diag}}\left({\ind(\bW \bx > 0)}\right)\; \bth}  + O_{d,\prob}\left(\frac{\log (d) }{d^{1/6}}\right)\,,
\end{align}
uniformly over $\bth\in\cC_{\bth}$. 
Here, the probability bound is with respect to the randomness in $\bx$. I.e. $\W$ is fixed and event $\event_{\W}$ in \eqref{eq:eventbW_main} is assumed to hold. 
\end{propo}
To be able to use the result of above proposition, we show that the minimizer of $\cL(\bth)$ falls in $\cC_{\bth}$ defined in \eqref{def:Cth}.
\begin{propo}\label{propo:Cth-L0}
Assume $(\bx,y)$ is generated according to \eqref{eq:linearModel}, and recall that the rows of $\bW\in\reals^{N\times d}$ are drawn i.i.d. from $\Unif(\mathbb{S}^{d-1})$. Let $\hth = \arg\min \cL(\bth)$. We have $\hth\in \cC_{\bth}$, with probability at least $1-4e^{-cn}$ for some absolute constant $c>0$.
\end{propo}
%

Motivated by Proposition~\ref{propo:simple1} we define loss $\cLo(\bth)$ as follows:
\begin{align} \label{eq: clo}
\cLo(\bth):=  \frac{1}{2n} \sum_{i=1}^n \left( |y_i - \bth^\sT\sigma(\bW\bx_i)|+ \eps \twonorm{\bW^\sT  {{\rm diag}}\left({\ind(\bW \bx_i > 0)} \right)\; \bth}\right)^2
+ \frac{\zeta}{2} \bth^\sT \bOmega \bth\,.
\end{align}
By using Proposition~\ref{propo:simple1}, we can prove the following. 
\begin{propo} \label{propo:L-L1}
Under the setting of Proposition~\ref{propo:simple1} we have
\begin{align}\label{eq:L-L1}
\sup_{\bth\in\cC_{\bth}} \left|\cL(\bth)-\cLo(\bth) \right| = o_{d,\prob}(1)\,.
\end{align}
\end{propo}

\subsection{Concentration of the adversarial effects}
As can be observed from the loss $\cLo(\bth)$, the effect of adversarial perturbation in reflected via the terms $\eta_i:= \twonorm{\bW^\sT  \diag{\ind(\bW \bx_i > 0)}\; \bth}$. In the next proposition, we show that apart from a negligible fraction of data points $i\in[n]$, the perturbation terms $\eta_i(\bth)^2$ concentrate around their expectation. All the proofs of the statements in this section are relegated to  Appendix~\ref{proofs_in_step_2}.

\begin{propo}\label{propo:concentration}
Let $\eta_i(\bth):=\twonorm{\bW^\sT  \diag{\ind(\bW \bx_i > 0)}\; \bth}$ and
$\nu_i(\bth,\gamma): =  \ind\left(|\eta_i(\bth)^2 -\E[\eta_i(\bth)^2]|>\gamma \right)$. 
Under the setting of Proposition~\ref{propo:simple1}  and for any sequence $\gamma_d$ such that $1/\gamma_d= e^{o(\sqrt{\log (d)})}$, we have  
\begin{align}\label{eq:propo-eq}
\sup_{\bth\in \cC_\bth}\,\, \frac{1}{n}\; \sum_{i=1}^n \nu_i(\bth;\gamma_d) = O_{d,\prob}(\log(d\gamma_d^2)^{-0.5})\,.
\end{align} 
\end{propo}

\begin{corollary}\label{coro:nui}
By choosing sequence $\gamma_d = 1/\log(d)$ we obtain
\begin{align}\label{eq:propo-eq-cor}
\sup_{\bth\in \cC_\bth}\,\, \frac{1}{n}\; \sum_{i=1}^n \nu_i(\bth;\tfrac{1}{\log (d)}) = o_{d,\prob}(1)\,.
\end{align} 
\end{corollary}
By Corollary~\ref{coro:nui}, other than at most an $o_d(1)$ fraction of data points $i\in[n]$, the terms $\eta_i(\bth)^2$ concentrate, in the sense $|\eta_i(\bth)^2-\E[\eta_i(\bth)^2]| \le 1/\log(d)$, uniformly over $\bth\in\cC_{\bth}$.  This suggests that in the loss function we can replace the terms $\eta_i(\bth)$ by $\sqrt{\E[\eta_i(\bth)^2]}$. This observation will be formally stated in the next lemma. Before proceeding to it, let us compute the expectation of terms $\eta_i(\bth)^2$. We write
\begin{align}\label{eq:tokhmi}
\E[\eta_i(\bth)^2] = \E\Big[\twonorm{\bW^\sT\diag{\ind(\bW\bx_i>0)}\bth}^2\Big] =  \bth^\sT \E\Big[(\bW\bW^\sT) \odot \left(\ind(\bW\bx_i>0)\ind(\bW\bx_i>0)^\sT \right) \Big]\bth\,,
\end{align}
where the expectation is with respect to $\bx_i$ and $\W$ is fixed. Hence, this can be written as
$\E[\eta_i(\bth)^2] = \twonorm{\bJ\bth}^2$ with 
\[
\bJ := \left((\bW\bW^\sT) \odot \E[\ind(\bW\bx_i>0)\ind(\bW\bx_i>0)^\sT]\right)^{1/2}\,.
\]
Note that the $\bJ$ is well-defined since the matrix under the square root is positive semidefinite. (This follows from the observation that the expression~\eqref{eq:tokhmi} is positive for all $\bth$.)

Since $\twonorm{\bw_\ell} = 1$ and $\bx_i\sim \normal(0,\Iden_d)$, we have that $\<\bw_\ell,\bx_i\>$ and $\<\bw_k,\bx_i\>$ are jointly Gaussian with 
\[\E(\<\bw_\ell,\bx_i\>^2) = \E(\<\bw_k,\bx_i\>^2) = 1,\quad  \E(\<\bw_\ell,\bx_i\>\<\bw_k,\bx_i\>) = \<\bw_k,\bw_\ell\>\,,\]
and by using \cite[Table 1]{daniely2016toward} we get
\[
\E[\ind(\bW\bx_i>0)\ind(\bW\bx_i>0)^\sT] = \frac{\pi-\cos^{-1}(\bW\bW^\sT)}{2\pi}\,.
\]
Therefore, we  obtain the following explicit formulation for $\bJ_{\bW}$:
\begin{align}\label{eq:J}
\bJ = \left( (\bW\bW^\sT) \odot \Big(\frac{\pi-\cos^{-1}(\bW\bW^\sT)}{2\pi}\Big) \right)^{1/2}\,.
\end{align} 
Motivated by Corollary~\ref{coro:nui} and the interpretation after it we define the loss function
\begin{align} \label{L_2}
\cLoo(\bth): =  \frac{1}{2n} \sum_{i=1}^n \left( |y_i - \bth^\sT\sigma(\bW\bx_i)|+  \eps \twonorm{\bJ \bth}\right)^2
+ \frac{\zeta}{2} \bth^\sT \bOmega \bth\,.
\end{align}
By our next lemma, the minimizer of $\cLoo(\bth)$ converges to the minimizer of the original loss $\cL(\bth)$ and therefore we can work with $\cLoo(\bth)$ for our asymptotic analysis.

\begin{lemma}\label{lem:Loss-oo}
We have
\[
\sup_{\bth\in\cC_{\bth}} \frac{|\cLoo(\bth)-\cL(\bth)|}{1+ \min(\cL(\bth),\cLoo(\bth))} = o_{d,\prob}(1)\,.
\]
Also, by denoting by \rev{$\hth^*$} and $\hth$ the minimizers of $\cLoo(\bth)$ and $\cL(\bth)$, we have $\twonorm{\hth^* - \hth}\to 0$, in probability.
\end{lemma}

\rev{
Motivated by the result of Lemma~\ref{lem:Loss-oo}, we define a notion of adversarial risk based on the modified loss $\cLoo(\bth)$. Specifically, we define
\begin{align}
\ARoo(\bth)&:= \E_{\bx,y}\left[\left(|y-\bth^\sT\sigma(\bW\bx)|+\etest \twonorm{\bJ \bth} \right)^2\right]\,.\label{eq:AR*}
\end{align}
In the next lemma, we show that $\ARoo(\bth)$ converges to $\AR(\bth)$ uniformly over $\cC_{\bth}$.
\begin{lemma}\label{lem:AR2-AR}
Recall the adversarial risk of a model $\bth$, denoted by $\AR(\bth)$ and given by~\eqref{eq:AR}. Let $\ARoo(\bth)$ be defined as \eqref{eq:AR*}. We then have,
\[
\sup_{\bth\in\cC_{\bth}} \frac{|\ARoo(\bth) -\AR(\bth)|}{\sqrt{\AR(\bth)}} = o_{d,\prob}(1)\,.
\]
\end{lemma}
}

\subsection{The Gaussian equivalence property and the noisy linear model}

  
 In this section, we will show that in order to characterize the robust generalization error of the random features model, \rev{we can} equivalently consider the so-called Gaussian features model (a.k.a. the noisy linear model).  This equivalency is often termed as the \emph{Gaussian Equivalence Property (GEP)}, and has recently been proven in several contexts \cite{montanari2019generalization, hu2020universality, gerace2020generalisation,  dhifallah2020precise}. We prove this equivalency for adversarially-trained random features models in this section.
 
We begin with decomposing the nonlinear activation function $\sigma(z)$ as
\begin{equation} \label{shifted_relu_decompose}
\sigma(z) = \mu_0 + \mu_1 z+ \mu_2 \sigma_\perp(z)\,,
\end{equation}
where for $G\sim\normal(0,1)$,  
\[\mu_0 := \E[\sigma(G)], \quad\mu_1 = \E[G\sigma(G)],\quad  \mu_2 := \sqrt{\E[\sigma^2(G)] - \mu_0^2-\mu_1^2}\,.\]
For the case of shifted ReLU activation, defined in \eqref{shifted_relu},  we have $\mu_0 = 0 $, 
$\mu_1 = \tfrac{1}{2}$ and $\mu_2 = \sqrt{\tfrac{1}{4}- \tfrac{1}{2\pi}}$.
Also, $\sigma_\perp(z)$ is the nonlinear component of the activation function which is orthogonal to the constant and linear components in the following sense: $\E[\sigma_\perp(G)] = 0$ and $\E[G\sigma_\perp(G)] = 0$. We can then write the random features  $\sigma(\bW\bx)$ as follows 
\begin{align}\label{eq:relu-decompose}
\sigma(\bW\x) =   \mu_0 \mathbf{1} + \mu_1 \bW\x + \mu_2  \sigma_\perp(\bW\x),
\end{align}
 Note that the random variables $\sigma(\bw_i^\sT \x)$ have zero mean and unit variance, by construction. Further, $\E_{\x}\{(\bw_i^\sT \x) \sigma_{\perp}(\bw_i^\sT \x)\} = 0$ since by construction $\E[\sigma_{\perp}(G)G] = 0$. This suggests to replace the variables $\sigma_{\perp}(\bw_i^\sT \x)$ by a set of i.i.d standard normal variables and consider the following \emph{noisy linear model}
\begin{equation} \label{noisy_model}
\mbf := \mu_0 \mathbf{1} + \mu_1 \bW\x + \mu_2 \ub,
\end{equation} 
with $\mbf, \ub\in\reals^N$ and $\ub\sim\normal(\bo,\Iden_N)$ is generated independently from $\x$.

Consequently, we define the loss of the noisy linear model as
\begin{align} 
\cL_{\rm nl}(\bth): =  \frac{1}{2n} \sum_{i=1}^n \left( |y_i - \bth^\sT \mbf_i|+  \eps \twonorm{\bJ \bth}\right)^2 + \frac{\zeta}{2} \bth^\sT \bOmega \bth\,,
\end{align}
where $f_i$ are generated i.i.d. according to \eqref{noisy_model}. We note that compared to the loss $\cLoo(\bth)$ defined in \eqref{L_2}, we have only replaced the feature vectors $\sigma(\W \x_i)$ with the noisy linear features $\mbf_i$.
%

Let $\hth^*$ and $\hth^*_{\rm nl}$ respectively denote the minimizers of $\cLoo(\bth)$ and $\cL_{\rm nl}(\bth)$. Roughly speaking, the Gaussian equivalence property (GEP) states that under certain conditions on $\bW$ and the activation function $\sigma$,  we have
\rev{
\begin{align}\label{eq:GEP}
\ARoo(\hth^*) \approx \ARoo_{\rm nl}(\hth^*_{\rm nl}) \quad \text{ as }d\to \infty\,, 
\end{align}
where \rev{$\ARoo(\cdot)$ is defined by~\eqref{eq:AR*} and $\ARoo_{\rm nl}(\cdot)$} is its counterpart defined based on the noisy linear model, as follows: 
}
\begin{align}
\ARoo_{\rm nl}(\bth)&:= \E_{\mbf,y}\left[\left(|y-\bth^\sT\mbf|+\etest \twonorm{\bJ \bth} \right)^2\right]\,.\label{eq:ARnl}
\end{align}
\rev{Therefore, by virtue of Lemma~\ref{lem:AR2-AR} and \eqref{eq:GEP}, we can henceforth focus on characterizing $\ARoo_{\rm nl}(\hth^*_{\rm nl})$.}

In order to prove \eqref{eq:GEP}, we first show  the asymptotic equality of   $\ARoo_{\rm nl}(\bth)$ and $\ARoo(\bth)$. All the proofs of the statements in this section are provided in Appendix~\ref{proofs_of_step_3}.


\begin{propo}\label{pro:risk-equi} 
Consider model \eqref{eq:linearModel} under the asymptotic setting in Assumption~\ref{assumption1} and define the set 
$$\cC'_{\bth} : =  \left\{ \bth: \infnorm{\bth} \leq C_0 \sqrt{(\log (d))/d}\,\,\,\, {\rm and }\,\,\,\,  \twonorm{\bth} \leq C_0 \,\,\,\, {\rm and }\,\,\,\,      \left| \ones^\sT \bth \right| \leq C_0 \sqrt{d/(\log (d))} \right\},$$
 for an arbitrary but fixed constant $C_0>0$. Let $\ARoo(\bth)$ and $\ARoo_{\rm nl}(\bth)$ be defined by~\eqref{eq:AR*}-\eqref{eq:ARnl}. Then, for any $\bth\in \cC'_{\bth}$ we have
\begin{align}\label{eq:risk-equi}
 \ARoo(\bth) =\ARoo_{\rm nl}(\bth) +o_d(1)\,,
\end{align}
In addition, we have the following characterizations for  $\ARoo_{\rm nl}(\bth)$: 
\begin{align}
\ARoo_{\rm nl}(\bth) &=  M(\bth)^2+ \etest^2 \rev{\twonorm{\bJ \bth}^2} + 2 \sqrt{\frac{2}{\pi}} \etest M(\bth)  \twonorm{\bJ \bth} \,,\label{ARnl-char}
\end{align}
with \rev{$M(\bth)$} given by
\begin{align}\label{eq:SR-a}
M(\bth)^2 = \tau^2 + \twonorm{\frac{1}{2}\bW^\sT\bth - \bbeta}^2 + \Big(\frac{1}{4}-\frac{1}{2\pi}\Big) \twonorm{\bth}^2\,.
\end{align}
\end{propo}
Proof of Proposition~\ref{pro:risk-equi}, i.e. equation \eqref{eq:risk-equi}, follows from a Central limit theorem (CLT) for weakly correlated variables proved in~\cite{goldt2020gaussian}. Specifically,~\cite{goldt2020gaussian} shows that $(\bth^\sT\sigma(\bW\x), \bbeta^\sT \x)$ converges in distribution to $(\bth^\sT \mbf, \bbeta^\sT\x)$, where $\bbeta$ is a fixed vector with bounded norm. In~\cite{hu2020universality}, the authors provide an alternative proof of this CLT using Stein's method and the Lindeberg approach. Their analysis assumes that the activation function $\sigma(z)$ is an odd function with bounded first derivatives. In addition, their analysis gives the convergence rate in terms of $\infnorm{\bth}$ (a Berry-Esseen type result).    
By the characterizations~\eqref{ARnl-char}, we know that, provided that $\bth \in \cC'_{\bth}$, the quantity  $\ARoo_{\rm nl}(\bth)$ depends on $\bth$ through \rev{the quantities $M(\bth)$ and $\twonorm{\bJ\bth}$. As we show in Lemma~\ref{pro:spectral}, $\|\bJ^2-\bK\|\to 0$, in probability with $\bK = (\bW\bW^\sT+\mtx{I})/4$. Since by definition, for $\bth\in \cC'_{\bth}$ we have $\twonorm{\bth}\le C_0$, therefore $\twonorm{\bJ\bth}^2\to (\twonorm{\bW^\sT \bth}^2+ \twonorm{\bth}^2)/4$.} 
So, in order to show the GEP relation of the form \eqref{eq:GEP},  it suffices to show that the quantities \rev{$\twonorm{\bth}$, $\twonorm{\tfrac{1}{2}\bW^\sT\bth - \bbeta}$ and $\twonorm{\bW^\sT\bth}$}, evaluated at $\hth^*$  converge to the corresponding quantities evaluated at $\hth^*_{\rm nl}$, and also $\hth^*, \hth^*_{\rm nl} \in  \cC'_{\bth}$. 

 \begin{theorem} \label{main_theorem:GEP}
Consider the quantities $\Phi_A$ and $\Phi_B$ defined as
\begin{align*}
\Phi_A &:= \min_{\bth} \frac{1}{n}\sum_{i=1}^n  \left( |y_i - \bth^\sT\sigma(\bW\bx_i)|+  \eps \twonorm{\bJ \bth} \right)^2   + \lambda\twonorm{\bth}^2 + \lambda_w \twonorm{\frac{1}{2}\bW^\sT\bth - \bbeta}^2 +  \lambda_s\frac{\log (d)}{d} (\ones^\sT \bth)^2, \\
\Phi_B &:= \min_{\bth} \frac{1}{n}\sum_{i=1}^n  \left(y_i \, - \, \bth^\sT\mbf_i + \eps \,\twonorm{\bJ \bth} \right )^2   +  \lambda\twonorm{\bth}^2 + \lambda_w \twonorm{\frac{1}{2}\bW^\sT\bth - \bbeta}^2 +   \lambda_s\frac{\log (d)}{d} (\ones^\sT \bth)^2,
\end{align*}
where $\lambda, \lambda_s, \lambda_w > 0$, and  $(\bx_i,y_i)$ is generated i.i.d. according to \eqref{eq:linearModel}. We further assume that the event $\event_{\W}$ holds. Then, we have
\begin{equation} \label{main_equivalence} 
\Phi_A \stackrel{\mathcal{P}}{\longrightarrow} c \,\,\,\,\text{ if and only if} \,\,\,\,  \Phi_B \stackrel{\mathcal{P}}{\longrightarrow} c,
\end{equation}
where $\stackrel{\mathcal{P}}{\longrightarrow}$ denotes convergence in probability. 
\end{theorem}
Using this theorem, we can then prove the following proposition. 
\begin{propo}\label{pro:opt-equiz}
Recall $\hth^*$ and $\hth^*_{\rm nl}$ given by
\begin{align}
\hth^* &= \arg\min_{\bth\in\reals^N} \cLoo(\bth) = \arg\min_{\bth\in\reals^N} \frac{1}{2n} \sum_{i=1}^n \left( |y_i - \bth^\sT\sigma(\bW\bx_i)|+  \eps \twonorm{\bJ \bth} \right)^2  + \frac{\zeta}{2} \bth^\sT \bOmega \bth\,,\nn\\
\hth^*_{\rm nl} &= \arg\min_{\bth\in\reals^N} \cL_{\rm nl}(\bth) = \arg\min_{\bth\in\reals^N} \frac{1}{2n} \sum_{i=1}^n \left( |y_i - \bth^\sT\mbf_i|+  \eps \twonorm{\bJ \bth} \right)^2+ \frac{\zeta}{2} \bth^\sT \bOmega \bth\,.\label{eq:nl-opt}
\end{align}
Then, under the asymptotic regime of Assumption~\ref{assumption1} we have
$$\hth^* ,  \hth^*_{\rm nl}  \in   \cC'_{\bth}, \\  $$
with probability $1-o_d(1)$, and
\begin{eqnarray}
M(\hth^*) - M(\hth^*_{\rm nl})
\stackrel{\mathcal{P}}{\longrightarrow} 0\,,
\rev{\quad \twonorm{\bJ\hth^*} - \twonorm{\bJ \hth^*_{\rm nl}}
\stackrel{\mathcal{P}}{\longrightarrow} 0\,,}
\end{eqnarray}
where $\stackrel{\mathcal{P}}{\longrightarrow}$ denotes convergence in probability. 

\end{propo}
Therefore, the GEP~\eqref{eq:GEP} follows from combining Propositions~\ref{pro:risk-equi} and \ref{pro:opt-equiz}. 

Finally, the result of Theorem~\ref{thm:main} follows by computing $\ARoo_{\rm nl}(\hth^*_{\rm nl})$ when we send $\zeta \to 0$ \emph{after} $d \to \infty$. This characterization will be carried out in Step 4 of the proof which will be described in the next section.  

\vspace{.2cm}
In non-adversarial contexts and for the standard risk  (a.k.a. the generalization error) GEP \rev{has been observed} by several previous works (see, e.g. \cite{hastie2019surprises, montanari2019generalization, abbasi2019universality, mei2019generalization, hu2020universality, goldt2020modeling, gerace2020generalisation, goldt2020gaussian} and also \cite{louart2018random, cheng2013spectrum, pennington2019nonlinear} in the context of random kernel matrices). In~\cite{mei2019generalization} the authors provide a precise characterization of the standard risk for the random features model (in non-adversarial setting) and observed that it corresponds to that of its noisy linear counterpart model. A similar GEP phenomena was conjectured for maximum-margin linear classifiers in binary classification~\cite{montanari2019generalization}. Subsequently, GEP has been proved for more general settings by~\cite{goldt2020modeling} and \cite{hu2020universality} for a teacher-student framework. In \cite{goldt2020modeling} the authors show GEP for learning
with one-pass stochastic gradient descent (SGD). The work~\cite{hu2020universality} considers the empirical risk minimization (with all data), which results in complicated correlations between the estimator and the samples, and proves the GEP for these settings. However, we cannot directly apply the result of~\cite{hu2020universality} since it assumes that the activation function is an odd function, thrice continuously differentiable with bounded first three derivatives, which are violated for the ReLU activation. Also, our adversarial loss function has additional terms that are beyond the setting considered in \cite{hu2020universality}. Nevertheless, our proof of Theorem~\ref{main_theorem:GEP} is based on the machinery developed in \cite{hu2020universality}.  Here we use a   
central limit theorem for weakly correlated random variables proved by~\cite{goldt2020gaussian} to show GEP in the context of adversarial training.

\subsection{Analysis of the Gaussian noisy linear model via convex Gaussian
minimax framework}
In our final step, we provide a sharp characterization of the adversarial risk $\ARoo_{\rm nl}(\hth^*_{\rm nl})$ using the Convex Gaussian Minimax Theorem (CGMT), which is a powerful and tight extension of Gordon’s Gaussian process
inequality \cite{gordon1988milman} with the presence of convexity. The underlying idea of the CGMT framework dates back to~\cite{StojLAS,stojnic2013meshes,stojnic2013upper} where the constrained LASSO was analyzed in the high signal-to-noise ratio regimes. The seminal work~\cite{thrampoulidis2015regularized,thrampoulidis2018precise} significantly extended these ideas and developed the CGMT framework to precisely characterize the mean-squared errors of regularized M- estimators in high-dimensional linear models. 

At a more technical level, the CGMT provides a principled machinery to characterize
the asymptotic behavior of certain mini-max optimization problems that are affine in a Gaussian
matrix $\bX$, namely problems of the form
\begin{align}\label{eq:sketch-PO}
\min_{\bth} \max_{\bu} \quad \bu^\sT \bX\bth + \phi(\bth,\bu)
\end{align}
where $\phi(\bth,\bu)$is convex in $\bth$ and concave in $\bu$. The CGMT decouples the above objective into a much simpler Gaussian process with essentially the same limit, yet much easier to analyze:
\begin{align}\label{eq:sketch-AO}
\min_{\bth} \max_{\bu} \quad \twonorm{\bth}\bg^\sT\bu + \twonorm{\bu}\bh^\sT\bth + \phi(\bth,\bu)\,,
\end{align}
where $\bg$ and $\bh$ are independent Gaussian vectors with i.i.d $\normal(0,1)$ entries. We refer to \cite[Theorem 3]{thrampoulidis2015regularized} for a precise statement on the relation between the optimization~\eqref{eq:sketch-PO}, often referred to as \emph{Primary Optimization (PO)} and~\eqref{eq:sketch-AO}, called \emph{Auxiliary Optimization (AO)}. The next step is to derive the point-wise limit of the AO objective in the large dimension limit and showing that it concentrates around a deterministic function with a small number of scalar variables (called \emph{scalarization} step). By showing that this convergence is uniform over a neighborhood of solution and using convexity-concavity of the function, we obtain a precise characterization of the adversarial risk in terms of the solutions of the corresponding convex-concave (deterministic) optimization~\eqref{eq:AO-final00}. 

Note that although the CGMT is a general machinery, the derivation and the study of the AO problem is entirely problem-specific and is usually rather challenging, often requiring the development of non-trivial probabilistic analysis. In relation to \emph{Approximate message passing (AMP)}, which is another powerful tool for deriving asymptotically exact characterization of high-dimensional estimators (see e.g.~\cite{AMPmain}), it is worth noting that both of these techniques provide a deterministic equation (called state evolution in the AMP parlance) which describes the large limit behavior of a random system. 

The CGMT has been recently used in several contexts, e.g., to characterize the performance of high-dimensional regularized logistic regression~\cite{NEURIPS2019_ab49ef78}, SLOPE estimator in sparse linear regression~\cite{hu2019asymptotics}, boosting and min $\ell_1$ norm classifier~\cite{liang2020precise}, multi-class classification~\cite{thrampoulidis2020theoretical}, and phase retrieval~\cite{dhifallah2018phase}.  More closely to our work, the CGMT has been used to study the effect of adversarial training in the context of linear regression~\cite{javanmard2020precise} and linear classifiers~\cite{javanmard2020precisec, taheri2020asymptotic}. On a technical side, the CGMT analysis for our current problem is more involved and intricate than the analysis carried out in~\cite{javanmard2020precise} for linear regression due to: (i) features $\mbf_i$ in \eqref{eq:nl-opt} being correlated; (ii) the presence of the matrix $\bJ$ in the loss which introduces more interactions among the model parameters.

\begin{acks}[Acknowledgments]
The authors thank Alexander Robey for interesting discussions and feedback on an early draft.
\end{acks}

\begin{funding}
The research of H. Hassani is supported by the NSF CAREER award, AFOSR YIP, the Intel Rising Star award, as well as the  AI Institute for Learning-Enabled Optimization at Scale (TILOS).  A. Javanmard is partially supported by the Sloan Research Fellowship in mathematics, an Adobe Data Science
Faculty Research Award, the NSF CAREER Award DMS-1844481 and NSF Award DMS-2311024.
\end{funding}


\begin{supplement}[id=suppA]
\sname{}
  \sname{Supplement to}
  \stitle{``Precise Statistical Analysis of Classification Accuracies for Adversarial Training''}
 \sdescription{Due to space constraints, proofs of theorems and some of the technical details are provided in the Supplementary Material~\cite{supp:displ}.}
\end{supplement}

\bibliographystyle{imsart-number} 
\bibliography{Bibfiles.bib}       

\clearpage
\setcounter{page}{1}
\appendix

\begin{frontmatter}
\title{Supplementary Material to ``Precise Statistical Analysis of Classification Accuracies for Adversarial Training"}
\runtitle{Precise Statistical Analysis of Adversarial Training}

\begin{aug}
\author[A]{\fnms{Hamed}~\snm{Hassani}},
\author[B]{\fnms{Adel}~\snm{Javanmard}}
\address[A]{Department of Electrical and Systems Engineering, University
of Pennsylvania\printead[presep={,\ }]{e1}}

\address[B]{Data Sciences and Operations Department, University
of Southern California\printead[presep={,\ }]{e2}}
\end{aug}

\end{frontmatter}

\appendix
\newcommand{\changelocaltocdepth}[1]{%
  \addtocontents{toc}{\protect\setcounter{tocdepth}{#1}}%
  \setcounter{tocdepth}{#1}%
}

The supplementary materials contain the proofs of theorems and technical lemmas.
It is structured around the main four steps outlined in Section~\ref{sec:sketch}.
\medskip

For the sake of completeness, we reintroduce the notation used throughout the proofs. 

\noindent\textbf{Notations.} Throughout the paper, we use $O_d(\cdot)$, $o_d(\cdot)$ to denote the standard big-O and little-o notation, where we stress the asymptotic variable $d$. Likewise, we denote by $O_{d,\prob}$ and $o_{d,\prob}$ to indicate asymptotic behavior in probability. Specifically, $f(d) = O_{d,\prob}(g(d))$ if for any $\eps>0$, there exists $C_\eps>0$ and large enough $d_\eps$ such that $\prob(|f(d)/g(d)| > C_\eps)\le \eps$, for all $d\ge d_\eps$. Similarly, $f(d) = o_{d,\prob}(g(d))$ if $f(d)/g(d)$ converges to zero in probability. We write $f(d)\approx g(d)$ as $d\to\infty$, when $ f(d)-g(d) \to 0$, in probability.
Note that we consider the asymptotic regime where $n,d,N$ grow at the same scale, ($\lim N/d\to \psi_1$ and $\lim n/d\to \psi_2$ for some positive constants $\psi_1$ and $\psi_2$), the expression $d\to\infty$ implies that $n, N\to\infty$, as well. 

For a matrix $\bA$, we denote by $\opnorm{\bA}$ its operator norm, $\fronorm{\bA} = (\sum_{ij} A_{ij}^2)^{1/2}$ the Frobenius norm of $\bA$.
For two matrices $\bA$ and $\bB$ of same size, we let $\bA\odot \bB$ be the element-wise product of $\bA$ and $\bB$. In addition, $[\bA; \bB]$ concatenates the two matrices row-wise and $[\bA, \bB]$ denotes the column-wise concatenation.  For an integer $n$, we use the shorthand $[n]= \{1,\dotsc, n\}$.

\tableofcontents

\changelocaltocdepth{2}


\rev{\section{Interchanging the limits of $d\to\infty$ and $\zeta\to 0$}\label{sec:interchange}

Consider the loss function (5.1) given by
\[
\cL(\bth,\zeta,d) = \max_{\twonorm{\bdelta_i}\le \eps} \frac{1}{2n} \sum_{i=1}^n \left(y_i - \bth^\sT\sigma(\bW(\bx_i+\bdelta_i))\right)^2+\frac{\zeta}{2}\bth^\sT \bOmega \bth\,,
\]
where with a slight abuse of notation, we made the dependence on $\zeta$ and $d$ explicit.

In the next lemma we show that the order of the two limits $d\to\infty$ and $\zeta\to 0$ can be interchanged.

\begin{lemma}\label{lem:interchange}
Under the assumptions of Theorem~\ref{thm:main}, we have
\[
\lim_{d\to\infty} \min_{\bth} \cL(\bth,0,d) =  \lim_{\zeta\to 0} \lim_{d\to\infty} \min_{\bth} \cL(\bth,\zeta,d)\,.
\]
\end{lemma} 

\begin{proof}(Proof of Lemma~\ref{lem:interchange})
First note that
\begin{align}\label{eq:tmp1}
\min_{\bth} \cL(\bth, 0, d) = \min_{\bth,\zeta\ge0} \cL(\bth,\zeta, d) = \min_{\zeta\ge0} \min_{\bth} \cL(\bth,\zeta, d) = \lim_{\zeta\to 0}  \min_{\bth} \cL(\bth,\zeta, d)\,.
\end{align}
The last step holds since $\cL(\bth,\zeta,d)$ is increasing in $\zeta$ for all $\bth$:
\[
\cL(\bth,\zeta_1,d) \le \cL(\bth,\zeta_2,d), \quad \text{ if }\zeta_1\le \zeta_2\,.
\]
Minimizing both sides over $\bth$, we get that $\min_{\bth}\cL(\bth,\zeta,d)$ is increasing in $\zeta$.

We next show that
\begin{align}\label{eq:interchange}
\lim_{d\to\infty} \lim_{\zeta\to 0} \min_{\bth} \cL(\bth,\zeta,d) =  \lim_{\zeta\to 0} \lim_{d\to\infty} \min_{\bth} \cL(\bth,\zeta,d)\,,
\end{align}
where the limits are in probability. Without loss of generality we restrict the domain of $\zeta$ to $[0,\zeta_*]$, for an arbitrary but fixed $\zeta_*$. The reason is that
in our proofs provided in the paper we allow $\zeta$ to be an arbitrarily small fixed value (i.e. we need $\zeta$ to be arbitrarily small, but fixed). We next use the Moore-Osgood theorem on exchanging limits, by which we need to verify that
\begin{align}
\lim_{d\to\infty} \min_{\bth} \cL(\bth,\zeta,d) &= f(\zeta), \quad \text{uniformly on }\zeta\in(0,\zeta_*],\label{eq:first-limit}\\
\lim_{\zeta\to 0} \min_{\bth} \cL(\bth,\zeta,d) & = A_d, \quad \text{pointwise over } d\in\mathbb{N}\,.
\end{align} 
The second identity follows from~\eqref{eq:tmp1}. To prove the first identity, note that $\cL(\bth,\zeta,d)$  is convex in $(\bth, \zeta)$. Now, since 
partial minimization preserves convexity~\cite[Section 3.2.5]{Boyd}, $\min_{\bth} \cL(\bth,\zeta,d)$ is convex in $\zeta$. The point-wise limit of \eqref{eq:first-limit} is already established in the paper,
and we obtain uniform convergence using the convexity lemma~\cite[Lemma 7.75]{liese2008statistical}. In words, the lemma states that pointwise convergence of convex
functions implies uniform convergence in compact subsets.

Combining~\eqref{eq:tmp1} and~\eqref{eq:interchange} we obtain
\begin{align}\label{eq:com}
\lim_{d\to\infty} \min_{\bth} \cL(\bth,0,d) =  \lim_{\zeta\to 0} \lim_{d\to\infty} \min_{\bth} \cL(\bth,\zeta,d)\,.
\end{align} 
\end{proof}

Recall that our main goal in the paper is to characterize the in-probability limit of $\AR(\hth^\eps)$, with 
\begin{align}\label{eq:hth_eps}
\hth^\eps = \arg\min_{\bth}\cL(\bth,0,d).
\end{align}

Define 
\begin{align}\label{eq:hth_eps_zeta}
\hth^\eps_\zeta = \arg\min_{\bth}\cL(\bth,\zeta, d).
\end{align}

Our next lemma relates $\AR(\hth^\eps)$ to $\AR(\hth^\eps_\zeta)$.

\begin{propo}\label{pro:AR_inter}
Let $\hth^\eps$ and $\hth^\eps_\zeta$ be respectively given by~\eqref{eq:hth_eps} and \eqref{eq:hth_eps_zeta}. Under assumptions of Theorem~\ref{thm:main}, we have
\[
\lim_{\zeta\to 0}\lim_{d\to\infty}\AR(\hth^\eps_\zeta) = \lim_{d\to\infty}\AR(\hth^\eps)\,.
\]
\end{propo}
\begin{proof}(Proof of Proposition~\ref{pro:AR_inter})
To proof the claim, we use a standard trick to translate the question on the optimal
solution of the minimization problem (i.e. $\hth^\eps$, $\hth^\eps_\zeta$) to one regarding the optimal costs.  

Let $B =  \lim_{\zeta\to 0}\lim_{d\to\infty}\AR(\hth^\eps_\zeta)$ (which exists and is calculated in Theorem \ref{thm:main}). We need to show that $\hth^\eps$ belongs to the following set
$S_\delta: = \{\bth: \; |\AR(\bth)- B|\le \delta\}$, with probability converging to one (as $d\to \infty$) for all $\delta>0$. Let $S_\delta^c$ denotes the complement set. If we show that
\begin{align}\label{eq:compare}
\min_{\bth\in S_\delta^c} \cL(\bth,0,d) > \cL(\hth^\eps, 0,d), 
\end{align}
then $\hth^\eps$ must lie in $S_{\delta}$. We formalize it in the next lemma.

\begin{lemma}\label{lem:inter-step}
Suppose that there exist constants $\ell$, $\tilde{\ell}$ and $\eta>0$ such that
\begin{itemize}
\item $\tilde{\ell} \ge \ell+ 2\eta$,
\item $\cL(\hth^\eps, 0,d)<\ell+\eta$ with probability at least $1-p$,
\item $\min_{\bth\in S_\delta^c} \cL(\bth,0,d) > \tilde{\ell}-\eta$ with probability at least $1-p$.
\end{itemize}
Then, $\prob(\hth^\eps\in S_\delta)\ge 1-2p$.
\end{lemma}
We then have the following corollary.
\begin{corollary}
Suppose that there exist constants $\ell < \tilde{\ell}$ such that $ \cL(\hth^\eps, 0,d)\stackrel{p}{\to} \ell$
and $\min_{\bth\in S_\delta^c} \cL(\bth,0,d) \stackrel{p}{\to} \tilde{\ell}$. Then, $\lim_{d\to\infty}\prob(\hth^\eps\in S_\delta) = 1$, for every $\delta>0$.
\end{corollary} 

In light of the above corollary, we compare the converging limits: Let $\ell:= \lim_{d\to\infty} \cL(\hth^\eps, 0,d)$ and $\tilde{\ell}:= \lim_{d\to\infty} \min_{\bth\in S_\delta^c} \cL(\bth,0,d)$.  
We need to show $\ell< \tilde{\ell}$. A similar trick has been used in~\cite[Theorem 6.1 (iii)]{thrampoulidis2015precise}. We next use~\eqref{eq:com}, by which
$$\ell =  \lim_{\zeta\to 0} \lim_{d\to\infty} \min_{\bth} \cL(\bth,\zeta,d).$$ 
By a similar argument, 
$$\tilde{\ell}= \lim_{\zeta\to 0}\lim_{d\to \infty}\min_{\bth\in S_\delta^c} \cL(\bth,\zeta,d),$$ 
where the difference is the domain over which we optimize. Note that in Theorem~\ref{thm:main} we calculate $\ell$ as the optimal value of a deterministic convex-concave optimization problem and show that it has a unique solution. Likewise, one can obtain a similar optimization for $\tilde{\ell}$ with the difference that its variables come from a restricted domain, which excludes the optimal solution of the former. By uniqueness of the solution, we conclude that $\ell<\tilde{\ell}$, which completes the proof of proposition.
\end{proof}
\subsection{Proof of Lemma~\ref{lem:inter-step}}
Define the event 
\[
\mathcal{E}: = \{\min_{\bth\in S_\delta^c} \cL(\bth,0,d) > \tilde{\ell}-\eta, \cL(\hth^\eps, 0,d)<\ell+\eta\}\,.
\]
On this event, using the first condition, we see that~\eqref{eq:compare} holds and so $\hth^\eps\in S_\delta$.
So we need to show that $\prob(\mathcal{E})\ge 1-2p$, which follows easily from union bounding and using the second and third conditions.
}
\revv{\begin{remark} In~\cite{mei2019generalization} the authors derive a precise characterization of the generalization of random features regression in a non-adversarial setting. This work makes a conjecture (see Remark 1 therein) that the generalization error of ridgeless estimator is the same as the min-norm least square estimator. This conjecture amounts to showing that the limits $\lambda\to 0$ ($\lambda$ the ridge penalty parameter) and $d\to\infty$ can be exchanged. We believe that this conjecture can be proved by following a similar argument of the proof of Lemma~\ref{lem:interchange}.
\end{remark}}
\section{Proofs of step 1: Asymptotically-exact closed form of adversarial examples} \label{proofs_of_step_1}


Recall that $\bx\sim\normal(0,\Iden_d)$. In the following we will show that conditioned on the event $\event_{\bW}$, defined in  \eqref{eq:eventbW_main}, with probability at least $1 - c/(\log (d))^2 - N^2 d^{-C}$ over the choice of $\bx$, we have 
\begin{equation} \label{adv_pert_linear}
\sup_{\bth\in\cC_{\bth}, \twonorm{\bdelta}\le \eps}\;\; \biggl | \bth^\sT \sigma(\bW(\bx + \bdelta))  -   \bth^\sT \sigma(\bW \bx)    -   \bth^\sT   \diag{\ind(\bW \bx > 0)} \bW\bdelta   \biggr| = C \frac{\log (d)}{ d^{\frac{1}{6}}}.
\end{equation}
As a result, we can write
\begin{align} \label{adv_pert_linear2}
\max_{\twonorm{\bdelta}\le \eps}  \left| y - \bth^\sT \sigma(\bW(\bx + \bdelta) \right| 
& = \max_{\twonorm{\bdelta}\le \eps}  \left| y - \bth^\sT\sigma(\bW\bx) - \< \bW^\sT  {{\rm diag}}({\ind(\bW \bx > 0)}) \bth,  \bdelta \> \right| + C  \frac{\log (d)}{ d^{\frac{1}{6}}},
\end{align}
for an absolute constant $C>0$, uniformly over $\bth\in\cC_{\bth}$. 
The maximization problem in the right-hand side of the above relation has a closed-form solution: 
\begin{equation} \label{adv_pert_linear3}
\bdelta = \eps \; \sign( y - \bth^\sT \sigma(\bW\bx) ) \; \frac{\bW^\sT  {{\rm diag}}{\ind(\bW \bx > 0)} \bth}{\twonorm{\bW^\sT  {{\rm diag}}{\ind(\bW \bx > 0)} \bth} },    
\end{equation}
which gives us the desired result~\eqref{eq:adv1}. It thus remains to prove \eqref{adv_pert_linear}. 

Denote the rows of matrix $\bW$ by $\{\bw_1, \dotsc, \bw_N\}$ with $\bw_\ell\in\reals^d$. Given $\x$,  we define the three sets
 $$ A (\x)= \biggl\{ \ell: \<\w_\ell, \x \> >  d^{-\frac13} \biggr\}, $$
$$ B(\x) = \biggl\{ \ell: \<\w_\ell, \x \>  <  -d^{-\frac13} \biggr\}, $$
$$ C(\x) = \biggl\{ \ell: \<\w_\ell, \x \>  \in [-d^{-\frac13}, d^{-\frac13}] \biggr\}. $$
We first need to bound the cardinality of the set $C(\x)$. 
\begin{lemma} \label{size_of_C}
With probability $1 - c/(\log (d))^2 - N^2 d^{-C}$ we have
$$ |C(\x)| \leq C d^\frac{2}{3} \log(d). $$
\end{lemma}
The proof of this lemma is given in Section~\ref{proof:size_of_C}.

Now, for a vector $\bdelta$ we define: 
 $$ \Delta A (\x,\bdelta)= \biggl\{ \ell: \ell \in A(\x)  \text{ and } \<\w_\ell, \x + \bdelta \> <   0 \biggr\}, $$
$$\Delta B(\x,\bdelta) = \biggl\{ \ell: \ell \in B(\x) \text{ and }  \<\w_\ell, \x + \bdelta \>  >  0  \biggr\}. $$
In other words, the set  $\Delta A (\x,\bdelta)$ (respectively $\Delta B (\x,\bdelta)$) contains all the indices $\ell$ in $A(\x)$ (respectively $B(\x)$) in which the sign of $\<w_\ell, \x\>$ and $\<w_\ell, \x+\bdelta\>$ are  different. We now prove that when $\twonorm{\bdelta} \leq \eps$  we have 
\begin{equation} 
|\Delta A (\x,\bdelta)| ,\,  |\Delta B (\x,\bdelta)| \leq C d^{\frac 23},
\end{equation}
for an absolute constant $C> 0$. By definition, on the event $\event_{\bW}$, $\bW$ has bounded operator norm, say at most $C$ for some constant $C>0$. In addition, $\twonorm{\bdelta}\le\eps$. Therefore, 
$\twonorm{\bW \bdelta}\le \eps C$. As a result, the number of entries of the vector $\W \bdelta$ whose absolute value is larger than $d^{-\frac 13}$ is bounded by $\eps^2 C^2 d^{\frac 23}$. But from the definitions of the sets $A(\x)$ and $\Delta A (\x,\bdelta)$ it is immediate that for $\ell \in \Delta A (\x,\bdelta)$ we have $| \< \w_\ell, \bdelta\> | > d^{-\frac 13}$. And this results in the fact that $|\Delta A (\x,\bdelta)| \leq C' d^{\frac 23}$ with $C' = \eps^2 C^2$. The same argument holds for $|\Delta B (\x,\bdelta)|$.

Let us now consider the two vectors $\sigma(\bW \x)$ and $\sigma(\W (\x + \bdelta))$. We would like to find out how these vectors are different on the indices that belong to the set $A(\x)$ or $B(\x)$. Let us start with the indices in $B(\x)$. Note that for any $\ell \in B(\x)$ we have $\<\w_\ell, \x \> < 0$. For this entry, it is easy to see that the two vectors $\sigma(\bW \x)$ and $\sigma(\W (\x + \bdelta))$ take different values only if we also have $\ell \in \Delta B(\x,\bdelta)$. As a result, we can conclude that the two vectors $\sigma(\bW \x)$ and $\sigma(\W (\x + \bdelta))$ are the same on all the indices belonging to the set $B(\x)$ except at most $C d^{\frac 23}$ indices. 
In other words, the difference $ \sigma(\W (\x + \bdelta)) - \sigma(\W \x)$ takes zero value on all the indices belonging to the set $B(\x)$ except at most $C d^{\frac 23}$ indices.

For the indices in the set $A(\x)$ the situation is different as we are operating in the non-constant part of the ReLU function (note that for any $\ell \in A(\x)$ we have $\<\w_\ell, \x \> > 0 $). We first claim the following:  The two vectors $\sigma(\bW (\x +  \bdelta)) $ and  $\bW (\x +  \bdelta) - 1/\sqrt{2\pi} $ are the same on all the entries in the set $A(\x)$ except the indices in the set $\Delta A(\x, \bdelta)$. The justification is as follows: Consider an index $\ell \in A(\x)$ such that $\sigma(\<\w_\ell, \x + \bdelta \>) \neq \<\w_\ell, \x + \bdelta \> - 1/\sqrt{2\pi}$. Since $\ell \in A(\x)$, we have $\sigma(\<\w_\ell, \x \>) = \<\w_\ell, \x \> - 1/\sqrt{2\pi}$. Now, since  $\sigma(\<\w_\ell, \x + \bdelta \>) \neq \<\w_\ell, \x + \bdelta \> - 1/\sqrt{2\pi}$ only if $\<\w_\ell, \x + \bdelta \> < 0$, we obtain that $\sigma(\<\w_\ell, \x + \bdelta \>) \neq \<\w_\ell, \x + \bdelta \> - 1/\sqrt{2\pi}$ only if  $\ell \in \Delta A(\x, \bdelta)$.

In summary, we have shown that (i) On indices belonging to the set $A(\x) \backslash \Delta A(\x, \bdelta)$: the two vectors $\sigma(\bW (\x +  \bdelta)) $ and  $\bW (\x +  \bdelta) - 1/\sqrt{2\pi} $  are the same, and (ii) on the indices belonging to the set 
$B(\x)  \backslash \Delta B(\x, \bdelta) $ the vector $\sigma(\W(\x + \bdelta)) - \sigma(\W\x)$ takes value $0$. Also, (iii) both sets  $\Delta A(\x, \bdelta)$  and $\Delta B(\x, \bdelta)$ have cardinality at most $C d^{\frac 23}$. We can thus write: 
\begin{align}  \label{sum-decompose}
& \bth^\sT \sigma(\W (\x + \bdelta))  -  \bth^\sT \sigma(\W \x)  \nonumber\\ 
&=  \sum_{\ell \in A(\x)} \theta_\ell (\sigma(\<\w_\ell, \x + \bdelta \> )  -  \sigma(\<\w_\ell ,\x \>)) +  \sum_{\ell \in B(\x)} \theta_\ell (\sigma(\<\w_\ell, \x + \bdelta \> )  -  \sigma(\<\w_\ell ,\x \>))  \\
&\quad\quad+  \sum_{\ell \in C(\x)}  \theta_\ell (\sigma(\<\w_\ell, \x + \bdelta \> )  -  \sigma(\<\w_\ell ,\x \>)). \nonumber
\end{align}
We first bound the second and third terms. Using the fact that $|\sigma(a+ b ) - \sigma(b) | \leq |a|$ we obtain for the third term that: 
\begin{align} \nonumber
\Big|  \sum_{\ell \in C(\x)} \theta_\ell (\sigma(\<\w_\ell, \x + \bdelta \> )  -  \sigma(\<\w_\ell ,\x \>)) \Big| & \leq  \sum_{\rev{\ell} \in C(\x)} | \theta_\ell|   | \<\w_\ell,  \bdelta \> |  \\ \nonumber
& \leq ||\bth||_\infty \sqrt{| C(\x)| } \twonorm{ \bW \bdelta} \\ \nonumber
& \leq \frac{C}{d^{\frac 12}} (C d^{\frac 23} \log (d))^\frac{1}{2} C \eps \\
& \leq C' d^{-\frac{1}{6}} \log (d) , \label{C-bound}
\end{align}
where we have use the result of Lemma~\ref{size_of_C} to bound the size of the set $C(\x)$ (and hence the above holds with the probability given in that lemma). For the second term we have
 \begin{equation}  \label{B-bound}
\Big| \sum_{\ell \in B(\x)} \theta_\ell (\sigma(\<\w_\ell, \x + \bdelta \> )  -  \sigma(\<\w_\ell ,\x \>)) \Big|  \leq   \sum_{\ell \in \Delta B(\x, \bdelta)} |\theta_\ell |\; |\<\w_\ell, \bdelta\>| \leq  ||\bth||_\infty \sqrt{| \Delta B(\x,\bdelta)| } \twonorm{ \bW \bdelta} \leq C' d^{-\frac 16}.
 \end{equation}
 Finally, in a similar manner as above we can bound
  \begin{equation}  \label{A-bound}
\Big|   \sum_{\ell \in \Delta A(\x, \bdelta)} \theta_\ell(  \sigma(\<\w_\ell, \x + \bdelta \> ) - \rev{ \sigma(\<\w_\ell ,\x \>) } \Big| \leq  ||\bth||_\infty \sqrt{| \Delta A(\x,\bdelta)| } \twonorm{ \bW \bdelta} \leq C' d^{-\frac 16}.
 \end{equation}
 
 By plugging\eqref{C-bound}, \eqref{B-bound}, and \eqref{A-bound} into \eqref{sum-decompose} we have shown that 
 \begin{align}
 \bth^\sT \sigma(\W (\x + \bdelta))  -   \bth^\sT \sigma(\W \x)  \nonumber
& =  \sum_{\ell \in A(\x) \backslash \Delta A(\x,\bdelta)} \theta_\ell (\sigma(\<\w_\ell, \x + \bdelta \> )  -  \sigma(\<\w_\ell ,\x \>))  + C' d^{-\frac{1}{6}} \log (d) \\ \nonumber
& =  \sum_{\ell \in A(\x) \backslash \Delta A(\x,\bdelta)}  \theta_\ell ( \<\w_\ell, \x + \bdelta \>   -  \< w_\ell ,\x \>) +  C' d^{-\frac{1}{6}} \log (d) \\
&  =  \sum_{\rev{\ell} \in A(\x)  \backslash \Delta A(\x,\bdelta)} \theta_\ell \<\w_\ell, \bdelta \> +  C' d^{-\frac{1}{6}} \log (d),  \label{dota-be-akhar}
 \end{align}
 where the second equality follows form the definition of the set $ \Delta A(\x,\bdelta)$. 
 
 As a final step, we define the set $A^+(\x) = \{ \ell: \< \w_\ell, \x \> >0 \}$. Note that $A(\x) \subseteq A^+(\x)$ and $A^+(\x) \backslash A(\x) \subseteq C(\x)$. As a result $|A^+(\x) \backslash A(\x)| \leq C d^{\frac{2}{3}} \log (d)$. We thus obtain
 \begin{align} \nonumber
 \sum_{\ell \in A(\x) \backslash\Delta A(\x,\bdelta)} \theta_\ell \rev{\<\w_\ell, \bdelta \>}  & =  \sum_{\ell \in A^+(\x) } \theta_\ell \<\w_\ell, \bdelta \> -  \sum_{\ell \in (A^+(\x) \backslash A(\x)) \cup \Delta A(\x,\bdelta)}   \theta_\ell \<\w_\ell, \bdelta \>  \\
 & =  \sum_{\ell \in A^+(\x) } \theta_\ell \<\w_\ell, \bdelta \>  +  C' d^{-\frac{1}{6}} \log (d), \label{yeki-be-akhar}
   \end{align}
   where the last relations follows from the fact that $| (A^+(\x) \backslash A(\x)) \cup \Delta A(\x,\bdelta) | \leq | C(\x)| + | \Delta A(\x,\bdelta)| = O(d^{\frac{2}{3}} \log (d)) $. By plugging \eqref{yeki-be-akhar} into \eqref{dota-be-akhar} we have
\begin{align*}   \bth^\sT \sigma(\W (\x + \bdelta)) -  \bth^\sT \sigma(\W \x)  &=  \sum_{\ell \in A^+(\x) } \theta_\ell \<\w_\ell, \bdelta \>   + C d^{-\frac{1}{6}} \log (d) \\
& = \sum_{\ell: \< \w_\ell, \x\> >0}  \theta_\ell \<\w_\ell, \bdelta \>   + C' d^{-\frac{1}{6}} \log (d) \\
 &=   \bth^\sT  {{\rm diag}}{\ind(\bW \bx > 0)} \W \bdelta +  C' d^{-\frac{1}{6}} \log (d), 
\end{align*}
which is the result of \eqref{adv_pert_linear}.  

\subsection{Proof of Lemma~\ref{size_of_C}}\label{proof:size_of_C}
Define the random variables $\mu_\ell:= \w_\ell^\sT \x $ for $\ell=1, \dotsc, N$. Note that since $\x$ is gaussian and $\twonorm{\w_\ell} = 1$, then $\mu_\ell \sim \normal (0, 1)$. Also, note that $\mu_\ell$'s are correlated with each other and each pair $(\mu_\ell, \mu_k)$ is a jointly-gaussian random variable with  correlation $\rho_{\ell,k}:=\mathbb{E}[\mu_\ell, \mu_k] = \w_\ell^\sT \w_k$. Define $z_\ell = \mathbf{1} \{\mu_\ell \in [-d^{-\frac13} , d^{-\frac13}] \}$. Note that $\mathbb{E}[z_\ell]   \leq c d^{-\frac13} $. 

Define the event $\event:= \{|\bw_\ell^\sT\bw_k|\le d^{-1/2}\sqrt{C \log (d)},\; \forall \ell,k\in[N]\}$. Since $\bw_\ell\sim_{i.i.d} \Unif(\mathbb{S}^{d-1})$, it is easy to see that $\prob(\event) \ge 1 - N^2 d^{-C}$. We also have 
\begin{align*}
\prob(\mu_\ell = 1, \mu_k = 1) = \int_{\mu \in  [-d^{-\frac13} , d^{-\frac13}] } f(\mu_\ell = \mu) \prob (\mu_k \in  [-d^{-\frac13} , d^{-\frac13}] \mid \mu_\ell = \mu) \text{d}\mu,
\end{align*}
where $f$ denotes pdf of $\mu_\ell$. Now, given $\mu_\ell = \mu$, the distribution of $\mu_k$ is $\normal(\mu \rho_{\ell,k}, (1-\rho_{\ell,k}^2) )$. It is easy to see that on the event $\event$, for $|\mu| \leq d^{-\frac 13} $ we have 
$$ \prob(\mu_k \in  [-d^{-\frac13} , d^{-\frac13}] \mid \mu_\ell = \mu)  \leq c d^{-\frac 13}, $$
and thus
$$\mathbb{E}[z_\ell z_k] = \prob(\mu_\ell = 1, \mu_k = 1) \leq c d^{-\frac 13} \int_{\mu \in  [-d^{-\frac13} , d^{-\frac13}] } f(\mu_\ell = \mu)  \text{d}\mu \leq c d^{-\frac23}.$$

 Let us now consider the average $\bar{z} = \frac{1}{N} \sum_{\ell=1}^N z_\ell$. We have
\begin{equation*}
\mathbb{E}[ \left | \bar{z} - \mathbb{E}[\bar{z}] \right|^2]= \E[\bar{z}^2] - \E[\bar{z}]^2\le \E[\bar{z}^2] \leq \frac{1}{N^2} \sum_{\ell,k =1}^N \mathbb{E}[z_\ell z_k] \leq c d^{-\frac 23},
\end{equation*}
and thus we obtain via the Chebyshev's inequality that
\begin{align*}
\prob \left \{ | \bar{z} - \mathbb{E}[\bar{z}] | \geq d^{-\frac{1}{3}} \log (d);\; \event  \right \} \leq \frac{c}{(\log (d))^2}\,.
\end{align*}
Now, by noticing that $|C(\x)|/N = \bar{z}$, and $\mathbb{E}[\bar{z}] \leq c d^{-\frac 13}$, along with the assumption that $N,d$ grows proportionally we obtain
$$ \prob\left\{ |C(\x)| \geq C d\times d^{-\frac{1}{3}} \log (d);\,\event  \right \} \leq \frac{c}{(\log (d))^2}. $$
 Finally, we have
 $$ \prob\left\{ |C(\x)| \geq C  d^{\frac{2}{3}} \log (d)  \right \} \leq \frac{c}{\log(d)^2} + \prob(\event^c) \le \frac{c}{(\log (d))^2} + N^2 d^{-C}\,. $$

\subsection{Proof of Proposition~\ref{propo:Cth-L0}}
Recall the loss $\cL(\bth)$ given by
\[
\cL(\bth):=\max_{\twonorm{\bdelta_i}\le \eps}  \frac{1}{2n} \sum_{i=1}^n \left(y_i - \bth^\sT \sigma(\bW(\bx_i + \bdelta_i)\right)^2 + \frac{\zeta}{2} \bth^\sT \bOmega \bth\,,
\]
and $\hth=\arg\min \cL(\bth)$. 
Bounding $\twonorm{\hth}$ is straightforward. By optimality of $\hth$ and comparing the loss at $\hth$ and $\boldsymbol{0}$
we get
\[
\frac{\zeta}{2} \hth^\sT\bOmega \hth \le \cL(\boldsymbol{0}) = \frac{1}{2n} \sum_{i=1}^n y_i^2< C,
\]
with probability at least $1-e^{-cn}$, for absolute constants $c,C>0$. Since $\bOmega\succeq \Iden$, this implies that
$\twonorm{\bth} \le C_0$ for sufficiently large $C_0$. 

To bound $\infnorm{\hth}$ we just need to bound any given entry of $\hth$, e.g. its last entry, with high probability.  By symmetry, all the entries
have the same marginal distribution. Consequently, each entry of $\hth$ can be analyzed in the same way
and $\infnorm{\hth}$ can then be controlled by using the union bound.

With a slight abuse of notation, we consider a $(N+1)$ dimensional version of the above optimization over $[\bth; u]$ and denote the last coordinate of the optimal solution by $\hat{u}$. Let $\lambda: = \frac{\sqrt{\log(d)}}{d}$ and
\begin{align*}
\bOmega = \begin{pmatrix}
\widetilde{\bOmega}& \lambda \ones\\
\lambda \ones^\sT & \lambda+1\,
 \end{pmatrix},
\end{align*}
 where $\widetilde{\bOmega}$ is of size $N$ and $\ones = (1, 1, \dotsc, 1)\in \reals^{N}$.
The last coordinate $\hat{u}$ can be expressed as
\begin{align}\label{eq:fu}
\hat{u} = \arg\min_u \min_{\bth} &\bigg[\frac{1}{2n} \sum_{i=1}^n \max_{\twonorm{\bdelta_i}\le \eps} \left(y_i - \bth^\sT \sigma(\bW(\bx_i + \bdelta^*_i) - {u} \sigma(\bw_{N+1}^\sT(\bx_i+\bdelta^*_i)) \right)^2\nonumber\\
& + \frac{\zeta}{2} (\bth^\sT \widetilde{\bOmega}\bth +2 \lambda (\ones^\sT \bth) u +(\lambda+1) u^2)\bigg]\,.
\end{align}
We next define $f(u)$ as the objective function of $u$ in \eqref{eq:fu}. We proceed by deriving a lower bound for $f(u)$. 

Let $\bth_*$ be the optimal $\bth$ if we set $u = 0$ and denote by $\delu_i$ the maximizing $\delta_i$, when $u=0$. Note that $\delu_i$ is in general a function of $\bth$. In addition, define
\begin{align*}
\ell([\bth, u]) &=  \frac{1}{2n} \sum_{i=1}^n \left(y_i - \bth^\sT \sigma(\bW(\bx_i + \delu_i) - {u} \sigma(\bw_{N+1}^\sT(\bx_i+\delu_i)) \right)^2\,,\\
Q([\bth, u]) &=  \frac{\zeta}{2} (\bth^\sT \widetilde{\bOmega}\bth +2 \lambda (\ones^\sT \bth) u + (\lambda+1) u^2)\,.
\end{align*}
Since the pointwise maximum of convex functions is convex, the function $\ell(\cdot)$ is convex and hence we have
\begin{align}\label{eq:ell1}
\ell([\bth;u]) \ge \ell([\bth_*;0])+ \<\nabla_{\bth}\ell([\bth,u]) |_{[\bth_*;0]}, \bth-\bth_*\> +  \nabla_{u}\ell([\bth,u])|_{[\bth_*;0]} u\,.
\end{align}

For quadratic function $Q([\bth;u])$ we have
\begin{align}\label{eq:Q1}
Q([\bth;u]) &=  \frac{\zeta}{2}\bth_*^\sT\widetilde{\bOmega} \bth_*  + \zeta \bth_*^\sT\widetilde{\bOmega}(\bth-\bth_*) + \frac{\zeta}{2} (\bth-\bth_*)^\sT \widetilde{\bOmega} (\bth-\bth_*) + \frac{\zeta}{2}(2\lambda u \ones^\sT \bth + (\lambda+1)u^2 )\nonumber\\
&=Q([\bth_*;0]) + \zeta \left( \bth_*^\sT \widetilde{\bOmega} (\bth - \bth_*) + \lambda (\ones^\sT \bth_*) u\right)+ \frac{\zeta}{2}
\bigl\{(\bth-\bth_*)^\sT \widetilde{\bOmega} (\bth-\bth_*) \\
& \quad  + (\lambda+1)u^2+ 2\lambda u \ones^\sT(\bth-\bth_*)\bigr \}\,.
\end{align}
Combining \eqref{eq:ell1} and \eqref{eq:Q1} we get
\begin{align}
\cL([\bth;u]) &\ge \ell([\bth;u])+ Q([\bth;u])\nonumber\\
&\ge \cL([\bth_*;0]) + \<\nabla_{\bth}\ell([\bth,u]) |_{[\bth_*;0]} + \zeta \bth_*^\sT\widetilde{\bOmega}, \bth-\bth_* \>
+ ( \nabla_{u}\ell([\bth,u])|_{[\bth_*;0]} + \zeta \lambda \ones^\sT \bth_*) u\nonumber\\
&+ \frac{\zeta}{2}
\left\{(\bth-\bth_*)^\sT \widetilde{\bOmega} (\bth-\bth_*) + (\lambda+1) u^2+ 2\lambda u \ones^\sT(\bth-\bth_*)\right\}\label{eq:L1}.
\end{align}
Here, the first inequality holds since $\cL(\cdot)$ involves maximization over $\bdelta_i$, while in definition of $\ell(\cdot)$ we consider $\delu_i$. Though, note that $\cL([\bth_*,0]) = \ell([\bth_*;0])+ Q([\bth_*;0])$ because when $u=0$, $\delu_i$ are the maximizing perturbations by definition. We used this observation in the second inequality above.
 
We argue that the second term in the right-hand side is zero. To see this, first write the partial derivative $\nabla_\bth \ell$ as
\begin{align*}
&\nabla_{\bth}\ell([\bth,u]) \\
&= -\frac{1}{n} \sum_{i=1}^n \left(y_i - \bth^\sT \sigma(\bW(\bx_i + \delu_i) - {u} \sigma(\bw_{N+1}^\sT(\bx_i+\delu_i)) \right)\times \\
&\quad \quad\quad\quad\quad\quad\quad\quad\quad\quad\quad\quad\quad\quad
 \left(\frac{\partial}{\partial \bth}[\bth^\sT \sigma(\bW(\bx_i + \delu_i)] + u \frac{\partial }{\partial \bth}\sigma(\bw_{N+1}^\sT(\bx_i+\delu_i)) \right)\,.
\end{align*}
(Note that $\delu_i$ is a function of $\bth$.) Therefore,
\begin{align}\label{eq:dev-ell-th}
\nabla_{\bth}\ell([\bth,u])|_{[\bth^*;0]} = -\frac{1}{n} \sum_{i=1}^n \left(y_i - \bth_*^\sT \sigma(\bW(\bx_i + \delu_i) \right) \left(\frac{\partial}{\partial \bth}[\bth^\sT \sigma(\bW(\bx_i + \delu_i)] \Big|_{[\bth_*;0]} \right)\,.
\end{align}
By using the first-order optimality condition for $\bth_*$ we have
\begin{align*}
&\nabla_{\bth}\ell([\bth,u]) |_{[\bth_*;0]} + \zeta \bth_*^\sT\widetilde{\bOmega}\\
&= -\frac{1}{n} \sum_{i=1}^n \left(y_i - \bth_*^\sT \sigma(\bW(\bx_i + \delu_i) \right) \left(\frac{\partial}{\partial \bth}[\bth^\sT \sigma(\bW(\bx_i + \delu_i)] \Big|_{[\bth_*;0]} \right) + \zeta \bth_*^\sT\widetilde{\bOmega} = 0\,.
\end{align*}
Using the above relation in~\eqref{eq:L1} we arrive at
\begin{align*}
\cL([\bth;u]) 
&\ge \cL([\bth_*;0])  
+ ( \nabla_{u}\ell([\bth,u])|_{[\bth_*;0]} + \zeta \lambda \ones^\sT \bth_*) u \\
&\quad\quad\quad \quad\quad\quad + \frac{\zeta}{2}
\left\{(\bth-\bth_*)^\sT \widetilde{\bOmega} (\bth-\bth_*) + (\lambda+1) u^2+ 2\lambda u \ones^\sT(\bth-\bth_*)\right\}\,.
\end{align*}
Therefore, by minimizing the both sides over $\bth$ we obtain
\begin{align}
f(u) &= \min_{\bth} \cL([\bth;u])\nonumber\\
& \ge f(0) + ( \nabla_{u}\ell([\bth,u])|_{[\bth_*;0]} + \zeta \lambda \ones^\sT \bth_*) u\nonumber\\
&+ \min_{\bth} \frac{\zeta}{2}
\left\{(\bth-\bth_*)^\sT \widetilde{\bOmega} (\bth-\bth_*) + (\lambda+1) u^2+ 2\lambda u \ones^\sT(\bth-\bth_*)\right\}\nonumber\\
&=  f(0) + ( \nabla_{u}\ell([\bth,u])|_{[\bth_*;0]} + \lambda \ones^\sT \bth_*) u + \frac{\zeta}{2}u^2(1+\lambda - \lambda^2\ones^\sT \widetilde{\bOmega}^{-1}\ones)\label{eq:fu_0}.
\end{align}
By definition of $\widetilde{\bOmega}$, it has $\ones$ as an eigenvector with eigenvalue $1+\lambda d$. So, 
\begin{align}\label{eq:lambda_LB}
1 +\lambda- \lambda^2\ones^\sT \widetilde{\bOmega}^{-1}\ones \ge 1 +\lambda - \frac{\lambda^2d}{1+\lambda d} > 1\,.
\end{align}
By optimality of $\hat{u}$, we have $f(\hat{u})\le f(0)$, which together with~\eqref{eq:fu_0} and \eqref{eq:lambda_LB} imply that  
\begin{align}\label{eq:uB0}
|\hat{u}| \le \frac{2}{\zeta} \left| \nabla_{u}\ell([\bth,u])|_{[\bth_*;0]} + \zeta \lambda \ones^\sT \bth_* \right|.
\end{align}
We next bound the terms on the right-hand side separately. We have
\[
\lambda \ones^\sT \bth_* \le \lambda \twonorm{\ones}\twonorm{\bth_*} \le \sqrt{\frac{\log (d)}{d}} \twonorm{\bth_*}\,.
\]
By optimality of $\bth_*$ (when we set $u=0$) and comparing it with $\boldsymbol{0}$ we get 
\[
\frac{\zeta}{2} \bth_*^\sT \widetilde{\bOmega}\bth_* \le \frac{1}{2n}\sum_{i=1}^n y_i^2 < C\,,
\]
with probability at least $1 - e^{-cn}$.

Sine $\widetilde{\bOmega}\succeq \Iden$, this implies that $\lambda \ones^\sT \bth_*  = O_{\prob}(\sqrt{{\log (d)}/{d}})$.

To bound the other term, recall that by definition $\frac{\partial}{\partial u}\delu_i = 0$ and so  $\nabla_{u}\ell([\bth,u])|_{[\bth_*;0]}$ is given by
\begin{align}
\nabla_{u}\ell([\bth,u])|_{[\bth_*;0]} =  \frac{1}{n} \sum_{i=1}^n \left(y_i - \bth_*^\sT \sigma(\bW(\bx_i + \delu_i) \right) \sigma(\bw_{N+1}^\sT(\bx_i+\delu_i))\,.
\end{align}
To simplify the notation define $m_i: = \frac{1}{\sqrt{n}} (y_i - \bth_*^\sT \sigma(\bW(\bx_i + \delu_i) )$ and $\bX = [\bx_1|\dotsc|\bx_n]^\sT$. Consider the following event:
\[
\event: = \left\{\twonorm{\boldsymbol{m}} \le C, \; \frac{1}{\sqrt{d}} \opnorm{\bX}\le C \right\}\,,
\]
where $\boldsymbol{m}= (m_1,\dotsc, m_n)^\sT$ and $C>0$ is a sufficiently large constant. We show that $\event$ is a high probability event. To see this, first observe that
\begin{align}\label{eq:m-B}
\twonorm{\boldsymbol{m}}^2 = \frac{1}{n}  \sum_{i=1}^n \left(y_i - \bth_*^\sT \sigma(\bW(\bx_i + \delu_i) \right)^2\le \frac{1}{n}
\sum_{i=1}^n y_i^2 \,,
\end{align}
where the inequality follows from optimality of $\bth_*$ and comparing the loss value $\cL([\bth_*; 0])$ with $\cL([\boldsymbol{0};0])$. Therefore,
\[
\prob(\twonorm{\boldsymbol{m}} > C) \le \prob(\frac{1}{\sqrt{n}} \twonorm{\by}> C)\le e^{-C' n}\,,
\]
for absolute constants $C, C'$ (depending on the noise variance $\tau^2$). Also, given that $\bX$ has i.i.d standard normal entries we have
\[
\prob(\frac{1}{\sqrt{d}} \opnorm{\bX} > C) \le 2e^{-c n}\,.
\]
 Putting the last two bounds together we obtain $\prob(\event^c)\le 3 e^{-cn}$. 
 
Let $\cF$ be the $\sigma$-algebra generated by an arbitrary $\bX,\bW,\by$ in $\event$. 
 Clearly, $m_i$ are measurable with respect to $\cF$. Also, $\bw_{N+1}$ is drawn independently from $\cF$ and hence conditioned on that $\bw_{N+1}^\sT \bx_i \sim\normal(0,1)$. Since $\E[\sigma(G)] = 0$ for $G\sim \normal(0,1)$, we have $\E[\sigma(\bw_{N+1}^\sT \bx_i )|\cF] = 0$. To bound $\nabla_{u}\ell([\bth,u])|_{[\bth_*;0]}$ we view that as a function of $\bw_{N+1}$ and condition on $\cF$. We then have
 \[
 \E\left[\nabla_{u}\ell([\bth,u])|_{[\bth_*;0]} | \cF \right] = \frac{1}{\sqrt{n}}\sum_{i=1}^n m_i \E[\sigma(\bw_{N+1}^\sT \bx_i )|\cF] = 0\,.
 \]
 Also this is a Lipschitz continuous function of $\bw_{N+1}$ with a Lipschitz factor at most $\frac{C}{\sqrt{d}} \twonorm{\bX\boldsymbol{m}} \le \frac{1}{\sqrt{\psi_2}} \frac{1}{\sqrt{d}}\opnorm{\bX} \twonorm{\boldsymbol{m}} \le \frac{C^2}{\sqrt{\psi_2}}: = C_0$.
Since $\bw_{N+1}$ is chosen uniformly at random from the unit sphere, we can apply the concentration bound for Lipschitz function (see e.g.~\cite[Theorem 5.1.4]{vershynin2018high}), which implies that
 \begin{align}
 \prob\left(\nabla_{u}\ell([\bth,u])|_{[\bth_*;0]} > t \right)\le 2e^{-c' dt^2}\,.
 \end{align}
 Choosing $t= C\sqrt{\frac{\log(d)}{d}}$ and using this bound in~\eqref{eq:uB0} we get 
 \[
 |\hat{u}|\le C'\sqrt{\frac{\log(d)}{d}}\,, 
 \]
 with probability at least $1 - 4e^{-cn} - 2d^{-c'C^2}$. The result follows by choosing $C>0$ large enough so that $c'C^2>1$ and union bounding over the $N$ coordinates of $\hth$, along with the assumption that $N, n, d$ grow at the same order.

\subsection{Proof of Proposition~\ref{propo:L-L1}}
We know from the result of Proposition~\ref{propo:simple1}, or more precisely the equations \eqref{adv_pert_linear}-\eqref{adv_pert_linear2} in its proof, that with probability $1 - o_d(1)$ we have

$$ \sup_{\bth \in \cC_{\bth} } \left|  \max_{\twonorm{\bdelta}\le \eps}  \left| y - \bth^\sT \sigma(\bW(\bx + \bdelta))\right| - \left(  |y - \bth^\sT \sigma(\bW\bx)| + \eps \twonorm{\bW^\sT  {{\rm diag}}\left({\ind(\bW \bx > 0)}\right)\; \bth} \right)    \right| = \alpha_d, $$
where $\alpha_d =   O\left(\frac{\log (d) }{d^{1/6}}\right) $.

One can thus write from \eqref{R-ERM} and \eqref{eq: clo} that for any $\bth \in \mathcal{C}_{\bth}$
\begin{align*}
\left|\cL(\bth)-\cLo(\bth) \right| & \leq \alpha_d^2 + \alpha_d \frac{1}{n} \sum_{i=1}^n \max_{\twonorm{\bdelta_i} \leq \varepsilon} \left |y_i - \bth^\sT\sigma(\W (\x_i + \bdelta_i)) \right|  \\
&\leq \alpha_d^2 + \alpha_d \frac{1}{n} \sum_{i=1}^n   \left |y_i - \bth^\sT\sigma(\W \x_i ) \right|  + \alpha_d \twonorm{\bth} || \W|| \, \varepsilon \\
&\leq \alpha_d^2 + \alpha_d \frac{1}{n} \sum_{i=1}^n   \left |y_i - \bth^\sT\sigma(\W \x_i) \right|  + c_1 \alpha_d,
\end{align*}
where $c_1 > 0$ is an absolute constant.  The second inequality follows from the fact that the ReLU function is 1-Lipschitz, and the third inequality follows from $\bth \in \mathcal{C}_{\bth}$ as well as the fact that $||\W||$ is bounded. 

We will now show that 
\begin{equation} \label{sum_abs_diff_Ls}
\sup_{\bth \in \mathcal{C}_{\bth}} \frac{1}{n} \sum_{i=1}^n   \left |y_i - \bth^\sT\sigma(\W \x_i) \right| = O_{d,\prob}(d^{{\frac{1}{12}}}).
\end{equation}
It is easy to see that proving the above relation will finish the proof as $\alpha_d =   O\left(\frac{\log (d) }{d^{1/6}}\right) $.

Fix a $\bth$ such that $\twonorm{\bth} \leq C$. Recall that $(\x_i,y_i)$ are generated i.i.d. according to the the distribution~\eqref{eq:linearModel}. Since the random variables $|y_i - \bth^\sT\sigma(\W \x_i)|$ are sub-gaussian (see e.g. \eqref{sub-gauss-a} in Lemma~\ref{concentration_feature}), we can write 
\begin{equation} \label{eps-net-diff-L-1}
 \prob\left( \frac{1}{n} \sum_{i=1}^n   \left |y_i - \bth^\sT\sigma(\W \x_i) \right| \geq d^{\frac{1}{12}}\right) \leq c_2 e^{-c_2 \, d^{\frac{7}{6}}},
 \end{equation}
for an absolute constant $c_2>0$. 

Now, to prove \eqref{sum_abs_diff_Ls}, we use an $\epsilon$-net argument. Consider a $1$-net of the set $\{\bth: \twonorm{\bth} \leq C\}$. We know that such a $1$-net $\mathcal{S}$ exists with size at most $|\mathcal{S}| \leq 2^{c_3 d}$ where $c_3 >0$ is an absolute constant. Let $\bth_1 \in \mathcal{S}$ be a vector in this net, and consider another vector $\bth_2$ in the $1$-neighborhood of $\bth_1$ -- i.e. $\twonorm{\bth_1 - \bth_2} \leq 1$. We can write
\begin{align} \nonumber
\left| \frac{1}{n} \sum_{i=1}^n   \left |y_i - \bth_1^\sT\sigma(\W \x_i) \right|  -  \frac{1}{n} \sum_{i=1}^n   \left |y_i - \bth_2^\sT\sigma(\W \x_i) \right| \right| \nonumber
&\leq \frac{1}{n} \sum_{i=1}^n   \left| (\bth_1 - \bth_2)^\sT\sigma(\W \x_i)  \right| \\\nonumber
&= \frac{1}{n} \onenorm{(\bth_1 - \bth_2)^\sT \bM} \\\nonumber
 &\leq \frac{1}{\sqrt{n}} ||M|| \twonorm{\bth_1 - \bth_2}\\ \label{eps-net-diff-L-2}
 & \leq \frac{1}{\sqrt{n}} ||M||,
\end{align}
where the matrix $\bM$ is defined as $\bM = \left[\sigma(\W \x_1) \,|\, \sigma(\W \x_2) \,| \, \cdots \, | \, \sigma(\W \x_n) \right]$. Now, since the random vectors $\sigma(\W \x_i)$ are independently generated and  sub-gaussian (see \eqref{sub-gauss-a}), we can conclude that
\begin{equation} \label{eps-net-diff-L-3}
\prob \left( ||M|| \geq c_4 \sqrt{n} \right) \leq c_5 e^{- c_5 d},
\end{equation}
 for absolute constants $c_4, c_5 > 0$ (recall that  $d$, $n$, and $N$ grow proportionally as per Assumption~\ref{assumption1}).  As a result, from \eqref{eps-net-diff-L-2} and \eqref{eps-net-diff-L-3}, we have
  \begin{align}  \label{eps-net-diff-L-4}
\prob \left( \sup_{\bth_1, \bth_2: \twonorm{\bth_1 - \bth_2} \leq 1} \left| \frac{1}{n} \sum_{i=1}^n   \left |y_i - \bth_1^\sT\sigma(\W \x_i) \right|  -  \frac{1}{n} \sum_{i=1}^n   \left |y_i - \bth_2^\sT\sigma(\W \x_i) \right| \right| \geq c_4 \right) \leq c_5 e^{- c_5 d}. 
\end{align}

 Now, by using \eqref{eps-net-diff-L-1} and \eqref{eps-net-diff-L-4}, and a union bound argument over $\mathcal{S}$, we obtain:
 \begin{align*} 
 \prob\left(\sup_{\bth: \twonorm{\bth} \leq C} \frac{1}{n} \sum_{i=1}^n   \left |y_i - \bth^\sT\sigma(\W \x_i) \right| \geq d^{\frac{1}{12}} + c_4 \right)  &\leq  \prob\left(\sup_{\bth \in \mathcal{S} }  \frac{1}{n} \sum_{i=1}^n   \left |y_i - \bth^\sT\sigma(\W \x_i) \right| \geq d^{\frac{1}{12}}\right)  +  c_5 e^{- c_5 d}\\
 &\leq c_2 e^{c_3 d -c_2 d^{\frac{7}{6}}}+  c_5 e^{- c_5 d} = O(e^{- c_5 d}). 
 \end{align*}
 The claim \eqref{sum_abs_diff_Ls} now follows because $\mathcal{C}_{\bth} \subseteq \{\bth: \twonorm{\bth} \leq C\}$.

\section{Proofs of step 2: Concentration of the adversarial effects} \label{proofs_in_step_2}
\subsection{Proof of Proposition~\ref{propo:concentration}}
Recall the high probability event \rev{$\event_{\bW}$} given by~\eqref{eq:eventbW_main}.
We also define the event $\event_{\bx}:= \{\twonorm{\bx_i}\le \sqrt{5d}, \; \forall i\in[n]\}$. Since $\bx_i\sim\normal(0,\Iden_d)$, $\twonorm{\bx_i}^2\sim\chi^2_d$ is a chi-squared distribution with $d$ degrees of freedom. Using chi-squared distribution tail bound (see e.g.~\cite[lemma 1]{laurent2000adaptive}) along with a union bound over $i\in[n]$, we obtain
$\prob(\event_{\bx}) \ge 1- ne^{-d}$. Since $d$ and $n$ grow proportionally as per Assumption~\ref{assumption1}, both of the events $\event_{\bW}$ and $\event_{\bx}$ are high probability events, and so it suffices to prove the claim~\ref{eq:propo-eq} on the event $\event_{\bW}\cap \event_{\bx}$. 

To prove the proposition, we first state the following lemma which establishes a deviation bound for a fixed $\bth\in \cC_\bth$ and fixed $i\in[n]$.
\begin{lemma}\label{lem:HW}
For any fixed $\bth\in\cC_\bth$ and fixed $i\in[n]$, the following holds :
\[
\prob_{\bx_i}\left\{|\eta_i(\bth)^2-\E[\eta_i(\bth)^2] |\ge \gamma;\;\event_{\bW}\cap\event_{\bx} \right\}\le \frac{c \log^6(d)}{d\gamma^2}\,, 
\]
for some absolute constant $c>0$.
\end{lemma}

Proof of Lemma~\ref{lem:HW} is given in Section~\ref{proof:lem:HW}.

Fix $\bth\in \cC_\bth$ and recall our notation $\nu_i(\bth;\gamma):= \ind(|\eta_i(\bth)^2 -\E[\eta_i(\bth)^2]|>\gamma)$. Given that $\bx_i$ are i.i.d,  the random variables $\nu_i\in\{0,1\}$ are also i.i.d. Bernoulli random variables. Therefore, 
\begin{align}
\prob\left(\frac{1}{n}\sum_{i=1}^n \nu_i(\bth;\gamma) \ge \frac{1}{\sqrt{\log(d\gamma^2)}} \right) 
&=\prob\left(\sum_{i=1}^n \nu_i(\bth;\gamma) \ge \frac{n}{\sqrt{\log(d\gamma^2)}} \right) \nonumber\\ 
&\le\sum_{\ell = \frac{n}{\sqrt{\log(d\gamma^2)}}}^n {n \choose \ell} \E[\nu_1(\bth;\gamma)]^\ell (1-\E[\nu_1(\bth;\gamma)])^{n-\ell} \nonumber\\ 
&\le \E[\nu_1(\bth;\gamma)]^{\frac{n}{\sqrt{\log(d\gamma^2)}}} \sum_{\ell = \frac{n}{\sqrt{\log(d\gamma^2)}}}^n  {n \choose \ell}\nonumber\\
&\le 2^n \left(\frac{c \log^6(d)}{d\gamma^2}\right)^{\frac{n}{\sqrt{\log(d\gamma^2)}}} \le \left(2 c \log^6(d) \right)^n e^{-n\sqrt{\log(d\gamma^2)}}\,, \label{eq:nu_i}
\end{align}
where the last step follows from Lemma~\ref{lem:HW} by which  $\E[\nu_1(\bth;\gamma)] \le c\log^6(d)/(d\gamma^2)$ on the event $\event_{\bW}\cap \event_{\bx}$. 

Note that the above bound was for a fixed $\bth\in\cC_{\bth}$. In order to prove claim~\ref{eq:propo-eq} we use an $\eps$-net argument. 
We write
\begin{align}
\sup_{\bth\in \cC_\bth}\,\, \frac{1}{n}\sum_{i=1}^n \nu_i(\bth;\gamma)
&\le \sup_{\twonorm{\bth}\le C_0}\,\, \frac{1}{n}\sum_{i=1}^n \nu_i(\bth;\gamma) \nonumber\\
&= \sup_{\twonorm{\bth}= C_0}\,\, \frac{1}{n}\;\sum_{i=1}^n \nu_i(\bth;\gamma)  \nonumber\\
&= \sup_{\bth\in \mathbb{S}^{d-1}}\,\, \frac{1}{n}\;\sum_{i=1}^n \nu_i(\bth;\tfrac{\gamma}{C_0^2}) \label{eq:Sphere1}\,,
\end{align}
where the first step follows from definition of $\cC_\bth$; the second step follows from a simple scaling argument, and the third step follows from definition of $\eta_i^2(\bth)$ and $\nu_i(\bth;\gamma)$. We recall that $\mathbb{S}^{d-1}$ denotes the unit $(d-1)$-dimensional sphere. 

Next consider a $\eps$-net $\mathcal{N}$ of $\mathbb{S}^{d-1}$ for $\eps = c_0 \gamma$. By~\cite[Lemma 5.2]{Vers} we can choose the net $\mathcal{N}$ so that $|\mathcal{N}|\le (1+\frac{2}{c_0\gamma})^d$.  We use the lemma below to relate the quantity $\nu_i(\bth;\gamma)$ for $\bth\in \mathbb{S}^{d-1}$ to a $\bth\in \mathcal{N}$.
\begin{lemma}\label{lem:net}
For $\bth\in  \mathbb{S}^{d-1}$ choose $\tbth\in\mathcal{N}$ which approximates $\bth$ as $\twonorm{\bth-\tbth}\le c_0\gamma$.
On the event $\event_{\bW}$ we have the following for all $i\in[n]$:
\begin{align}
\nu_i(\bth;\gamma) = 1 \implies \nu_i(\tbth;\gamma(1 -2c_0\sqrt{\psi_{1,d}}-2c_0C)) = 1\,.
\end{align}
\end{lemma}
We refer to Section~\ref{proof:lem:net} for the proof of Lemma~\ref{lem:net}.

Continuing from~\eqref{eq:Sphere1} and using Lemma~\ref{lem:net} we get
\begin{align}
\sup_{\bth\in \cC_\bth}\,\, \frac{1}{n}\; \sum_{i=1}^n \nu_i(\bth;\gamma)  
&\le  \sup_{\bth\in \mathbb{S}^{d-1}}\,\, \frac{1}{n}\; \sum_{i=1}^n \nu_i(\bth;\tfrac{\gamma}{C_0^2}) \nonumber\\
&\le  \sup_{\tbth\in \mathcal{N}}\,\, \frac{1}{n}\; \sum_{i=1}^n \nu_i(\tbth;\tfrac{\gamma}{C_0^2}(1 -2c_0\sqrt{\psi_{1,d}}-2c_0C))\,.\label{eq:net}
\end{align}
Let $\tgamma: = \tfrac{\gamma}{C_0^2}(1 -2c_0\sqrt{\psi_{1,d}}-2c_0C)$. By choosing the constant $c_0$ small enough we have $\tgamma \ge 0$. Using~\eqref{eq:nu_i} along with union-bounding over the net $\mathcal{N}$ we get
\[
\prob\left(\sup_{\tbth\in \mathcal{N}}\,\, \frac{1}{n}\; \sum_{i=1}^n \nu_i(\tbth;\tgamma)\ge \frac{1}{\sqrt{\log(d\tgamma^2)}}  \right)\le\left(1+\frac{2}{c_0\gamma}\right)^d  \left(2c \log^6(d) \right)^n e^{-n\sqrt{\log(d\tgamma^2)}}\,.
\]
Since $n$ and $d$ grow proportionally and also $\gamma,\tgamma$ are of same order,  the above event is a high probability event if $\log(1/\gamma) = o(\sqrt{\log (d)})$ or equivalently if $\tfrac{1}{\gamma}= e^{o(\sqrt{\log (d)})}$. 
The result follows by combining the above bound with~\eqref{eq:net}.
%
\subsection{Proof of Lemma~\ref{lem:HW}}\label{proof:lem:HW}
We decompose the step function as
\[
\ind(z>0) = \mu_0 + \mu_1 z+ \mu_* \varphi(z)\,,
\]
where for $G\sim\normal(0,1)$,  
\[\mu_0 := \E[\ind(G>0)] = \frac{1}{2}, \quad\mu_1 = \E[G\ind(G>0)] = \frac{1}{\sqrt{2\pi}},\quad  \mu_*^2 := \E[\ind(G>0)] - \mu_0^2-\mu_1^2 = \frac{1}{4}- \frac{1}{2\pi}\,.\]
Here, $\varphi(z)$ is the nonlinear component of the step function which is orthogonal to the constant and linear components in the following sense: $\E[\varphi(G)] = 0$ and $\E[G\varphi(G)] = 0$.  We write
\begin{align}\label{eq:sign-decompose}
\ind(\<\bw_\ell,\bx_i\> >0 ) = \frac{1}{2} + \frac{1}{\sqrt{2\pi}} \<\bw_\ell,\bx_i\> + \mu_* u_{\ell i}, \quad \text{where: }  u_{\ell i} :=  \varphi(\<\bw_\ell,\bx_i\>)\,,
\end{align}
noting that $\<\bw_\ell,\bx_i\>$ and $\<\bw_k,\bx_i\>$ are jointly Gaussian with 
\[\E(\<\bw_\ell,\bx_i\>^2) = \E(\<\bw_k,\bx_i\>^2) = 1,\quad  \E(\<\bw_\ell,\bx_i\>\<\bw_k,\bx_i\>) = \<\bw_k,\bw_\ell\>\,.\] 
Therefore, we have (see e.g.,~\cite[Table 1]{daniely2016toward})
\begin{align}
\E[\ind(\<\bw_\ell,\bx_i\>>0)\ind(\<\bw_k,\bx_i\>>0)] &=  \frac{\pi-\cos^{-1}(\<\bw_k,\bw_\ell\>)}{2\pi} \nonumber\\
&= \frac{1}{4}+\frac{1}{2\pi}\<\bw_\ell,\bw_k\> + O(\<\bw_\ell,\bw_k\>^3)\nonumber\\
& =\frac{1}{4}+\frac{1}{2\pi}\<\bw_\ell,\bw_k\> +O\left( d^{-3/2} \log^3(d) \right) \,. \label{eq:correlation1}
\end{align}
To bound the correlation of variables $u_{\ell i}, u_{k i}$, we write
\begin{align}
\mu_*^2 \E[u_{\ell i} u_{k i}] &= \E\Big[\Big\{\ind(\<\bw_\ell,\bx_i\>>0) - \frac{1}{2} - \frac{1}{\sqrt{2\pi}} \<\bw_\ell,\bx_i\> \Big\} \Big\{\ind(\<\bw_k,\bx_i\>>0) - \frac{1}{2} - \frac{1}{\sqrt{2\pi}} \<\bw_k,\bx_i\> \Big\}\Big] \nonumber\\
&= \E[\ind(\<\bw_\ell,\bx_i\>>0)\ind(\<\bw_k,\bx_i\>>0)] + \E\Big[\Big(\frac{1}{2} + \frac{1}{\sqrt{2\pi}} \<\bw_\ell,\bx_i\>\Big)\Big(\frac{1}{2} + \frac{1}{\sqrt{2\pi}} \<\bw_k,\bx_i\>\Big)\Big] \nonumber\\
&\quad- \E\Big[\ind(\<\bw_\ell,\bx_i\> > 0) \Big(\frac{1}{2} + \frac{1}{\sqrt{2\pi}} \<\bw_k,\bx_i\>\Big)\Big]
- \E\Big[\ind(\<\bw_k,\bx_i\> > 0)\Big(\frac{1}{2} + \frac{1}{\sqrt{2\pi}} \<\bw_\ell,\bx_i\>\Big)\Big].\label{eq:ulk}
%
\end{align}
The first term above is calculated in~\eqref{eq:correlation1}. For the second term, we have
\begin{align}\label{eq:correlation2}
\E\Big[\Big(\frac{1}{2} + \frac{1}{\sqrt{2\pi}} \<\bw_\ell,\bx_i\>\Big)\Big(\frac{1}{2} + \frac{1}{\sqrt{2\pi}} \<\bw_k,\bx_i\>\Big)\Big] = \frac{1}{4} + \frac{\<\bw_\ell,\bw_k\>}{2\pi}\,.
\end{align}
For the third term we write
\begin{align}
\E\Big[\ind(\<\bw_\ell,\bx_i\> > 0) \Big(\frac{1}{2} + \frac{1}{\sqrt{2\pi}} \<\bw_k,\bx_i\>\Big)\Big] &= 
\frac{1}{4} + \frac{1}{\sqrt{2\pi}} \E\Big[\ind(\<\bw_\ell,\bx_i\> > 0)\<\bw_k,\bx_i\>\Big]\nonumber\\
&=\frac{1}{4} + \frac{1}{\sqrt{2\pi}} \E\Big[\<\bw_k,\bx_i\>\Big| \<\bw_\ell,\bx_i\> > 0\Big]\, \prob(\<\bw_\ell,\bx_i\> > 0)\nonumber\\
&\stackrel{(a)}{=} \frac{1}{4} + \frac{1}{2\sqrt{2\pi}} \<\bw_\ell,\bw_k\> \frac{\phi(0)}{1-\Phi(0)} \nonumber\\
&= \frac{1}{4} + \frac{1}{2\pi} \<\bw_\ell,\bw_k\>\,. \label{eq:correlation3}
\end{align} 
Here $(a)$ follows from lemma below.
\begin{lemma}\label{lem:aux1}
For $Z_1$, $Z_2\sim\normal(0,1)$ with $\E[Z_1Z_2]=\rho$ we have
\[
\E[Z_1|\;Z_2 > z] = \rho \frac{\phi(z)}{(1-\Phi(z))}\,,
\]
where $\phi(z) = \tfrac{1}{\sqrt{2\pi}} e^{-z^2/2}$is the density of standard normal and $\Phi(z) = \int_{-\infty}^z \phi(t) \de t$ is its CDF.
\end{lemma}
Using Equations~\eqref{eq:correlation1}, \eqref{eq:correlation2} and \eqref{eq:correlation3} in~\eqref{eq:ulk} we obtain 
\begin{align}\label{eq:u-correlation}
\E[u_{\ell i} u_{k i}] = O\left(d^{-3/2} \log^3(d) \right)\,.
 \end{align}
 Substituting for the sign function $\ind(\<\bw_\ell,\bx_i\> >0 )$ from~\eqref{eq:sign-decompose} we get
 \begin{align}
\bW^\sT  \diag{\ind(\bW \bx_i > 0)}\; \bth &= \bW^\sT  \diag{\bth} {\ind(\bW \bx_i > 0)} \nonumber\\
&=\bW^\sT  \diag{\bth} \left(\frac{1}{2} \ones+  \mu_* \bu_{i}+ \frac{1}{\sqrt{2\pi}} \bW \bx_i \right)\,,
 \end{align}
 with $\bu_i = (u_{\ell i})_{\ell=1}^N $. \rev{To lighten} the notation, we use the shorthand $\bh_i := \bW^\sT  \diag{\bth} (\frac{1}{2} \ones+  \mu_* \bu_{i})$. 
 We next decompose $\eta_i(\bth)^2$ into three terms as follows:
 \begin{align}\label{eq:eta-dec}
 \eta_i(\bth)^2 = \frac{1}{2\pi}\twonorm{\bW^\sT\diag{\bth}\bW\bx_i}^2+\twonorm{\bh_i}^2+\sqrt{\frac{2}{\pi}}\<\bh_i, \bW^\sT  \diag{\bth}\bW \bx_i\>.
 \end{align}
 We next provide deviation bounds for each of these terms by putting which together we obtain the desired claim.
 

We start by the first term in~\eqref{eq:eta-dec}. Note that since $\bth\in \cC_{\bth}$, we have the following bounds conditioned on the event $\event_{\bW}$:
 \begin{align*}
 &\opnorm{\bW^\sT\diag{\bth} \bW \bW^\sT\diag{\bth}\bW} \le \opnorm{\bW}^4 \infnorm{\bth}^2 = O\left(d^{-1} \log(d) \right)\,,\\
 &\quad  \fronorm{\bW^\sT\diag{\bth} \bW \bW^\sT\diag{\bth}\bW}\le \sqrt{\min(d,N)}  \opnorm{\bW^\sT\diag{\bth} \bW \bW^\sT\diag{\bth}\bW} =O(d^{-0.5})\,.
\end{align*}
Therefore,  by applying Hanson-Wright's inequality~\cite{rudelson2013hanson} we get
\begin{align}\label{eq:term1}
\prob \left\{\left|\twonorm{\bW^\sT\diag{\bth}\bW\bx_i}^2 - \E\left[\twonorm{\bW^\sT\diag{\bth}\bW\bx_i}^2\right] \right| > {\gamma \log(d)};\,\event_{\bW }\right\} 
\le 2e^{-c \gamma^2 d}\,,
\end{align}
for an absolute constant $c>0$.

For the second term we bound variation in the vector $\bh_i$ itself from which we obtain a deviation bound on its norm $\twonorm{\bh_i}$.
We write
\begin{align}\label{eq:expec-B}
\E\Big[\twonorm{\bh_i-\E[\bh_i]}^2\Big] &= \mu_*^2 \E[\twonorm{\bW^\sT \diag{\bth} \bu_i}^2]\nonumber\\
&= \sum_{\ell,k} \<\bw_\ell,\bw_k\> \theta_\ell\theta_k \E[u_{\ell i} u_{ki}] \le \sum_{i,j} C \frac{1}{\sqrt{d}} \infnorm{\bth}^2 d^{-1.5} \log^3(d) = O\left(d^{-1} \log^4(d) \right)\,,
\end{align}
where we used the assumption $\bth\in\cC_\bth$ along with \eqref{eq:u-correlation}. 
We write $\bh_i = \E[\bh_i] + \bdelta$ and define the event $\event_{\bdelta}:= \{\twonorm{\bdelta} \leq \gamma\}$.
Therefore by using Markov's inequality along with~\eqref{eq:expec-B} we obtain $\prob(\event_{\bdelta})\ge 1-  c\frac{\log^4(d)}{d\gamma^2}$.
Furthermore,
\begin{align}
\twonorm{\bh_i}^2 = \twonorm{\E[\bh_i]}^2+ \twonorm{\bdelta}^2+ 2\<\bdelta,\E[\bh_i] \>.
\end{align}
On the event $\event_{\bW}$, we have $\twonorm{\E[\bh_i]} = \frac{1}{2}\twonorm{\bW^\sT \diag{\bth} \ones} \le \frac{1}{2} \opnorm{\bW} \|\bth\|_\infty \sqrt{N} \le C \sqrt{ \log(d)}$, and so  
$|\<\bdelta,\E[\bh_i] \>|\le C \twonorm{\bdelta}$. Hence, on the event $\event_{\bW}\cap \event_{\bdelta}$,
\[
\Big|\twonorm{\bh_i}^2- \twonorm{\E[\bh_i]}^2\Big| \le \gamma^2 + 2C  \sqrt{ \log(d)} \gamma = O\left(\gamma   \sqrt{ \log(d)} \right)\,.
\]
This implies that $\twonorm{\E[\bh_i]}^2 = \E[\twonorm{\bh_i}^2] +O(\gamma \log(d))$, and therefore
\begin{align}\label{eq:term2}
\rev{\Big|\twonorm{\bh_i}^2 - \E[\twonorm{\bh_i}^2] \Big|} = O\left(\gamma  \sqrt{ \log(d)} \right)\,.
\end{align}
We next proceed to the third term in~\eqref{eq:eta-dec}. 
\begin{align}\label{eq:term3}
\<\bh_i,\bW^\sT\diag{\bth} \bW\bx_i\> &= \<\E[\bh_i],\bW^\sT\diag{\bth} \bW\bx_i\> + \<\bdelta,\bW^\sT\diag{\bth} \bW\bx_i\>\nonumber\\
&=  \frac{1}{2}\<\bW^\sT\diag{\bth}\ones,\bW^\sT\diag{\bth} \bW\bx_i\> + \<\bdelta,\bW^\sT\diag{\bth} \bW\bx_i\>\,.
\end{align}
Note that on the event $\event_{\bW}$ the first term above is a Lipschitz continuous function of the Gaussian vector $\bx_i$ with Lipschitz constant 
\[
L= \twonorm{\ones^\sT\diag{\bth}\bW\bW^\sT \diag{\bth}\bW}\le \sqrt{N} \infnorm{\bth}^2 \opnorm{\bW}^3 = O\left(\log(d)/\sqrt{d} \right)\,.
\] 
By Gaussian isoperimetry~\cite{ledoux}, we have 
\begin{align}\label{eq:term3-1}
\prob\left(\bigg|\<\bW^\sT\diag{\bth}\ones,\bW^\sT\diag{\bth} \bW\bx_i\> - \E[\<\bW^\sT\diag{\bth}\ones,\bW^\sT\diag{\bth} \bW\bx_i\>]\bigg| \ge \gamma \log(d);\,\event_{\bW} \right)\le 2e^{-c\gamma^2 d}\,,
\end{align}
for some constant $c>0$. For the second term of~\eqref{eq:term3}, note that on the event $\event_{\bdelta}\cap \event_{\bW}\cap \event_{\bx}$,
\begin{align}\label{eq:term3-2}
|\<\bdelta,\bW^\sT\diag{\bth} \bW\bx_i\>|\le \twonorm{\bdelta} \twonorm{\bW^\sT\diag{\bth} \bW\bx_i} \le \twonorm{\bdelta} \opnorm{\bW}^2 \infnorm{\bth} \twonorm{\bx_i}
= O(\gamma \log (d))\,.
\end{align}
Combining~\eqref{eq:term3-1} and \eqref{eq:term3-2} with the decomposition~\eqref{eq:term3} we get that on the event $\event_{\bdelta}\cap \event_{\bW}\cap \event_{\bx}$,
\begin{align}\label{eq:term3*}
\bigg|\<\bh_i,\bW^\sT\diag{\bth} \bW\bx_i\> - \E\Big[\<\bh_i,\bW^\sT\diag{\bth} \bW\bx_i\>\Big] \bigg| =O(\gamma \log(d))\,,
\end{align}
with probability at least $1-2e^{-c\gamma^2d}$.
Putting together the deviation bounds for the three terms, given by~\eqref{eq:term1}, \eqref{eq:term2} and \eqref{eq:term3*}, we arrive at
\[
\prob\left(|\eta_i(\bth)^2 - \E[\eta_i(\bth)^2]| > C\gamma \log(d);\;  \event_{\bW}\cap \event_{\bx} \right) \le 4e^{-c\gamma^2d} + \frac{c \log^4(d)}{d\gamma^2} = O\left(\frac{\log^4(d)}{d\gamma^2}\right)  \,.
\]
Note that the above relation holds for any $\gamma > 0$. The result of the lemma now follows by letting $\gamma \leftarrow C\gamma \log(d)$.

%
\subsection{Proof of Lemma~\ref{lem:net}}\label{proof:lem:net}

Define the matrix
$$\bA :=  \diag{\ind(\bW \bx_i > 0)} \bW\bW^\sT \diag{\ind(\bW \bx_i > 0)} \,.$$ 
By triangle inequality we have 
\begin{align*}
|\eta_i(\bth)^2-\eta_i(\tbth)^2| &= |\<\bA\bth,\bth\> - \<\bA\tbth,\tbth\>|\\
&=|\<\bA\bth,\bth-\tbth\> + \<\bA(\bth-\tbth),\tbth\>|\\
&\le \opnorm{\bA} \twonorm{\bth}\twonorm{\bth-\tbth} + \opnorm{\bA} \twonorm{\bth-\tbth}\twonorm{\tbth} \le 2c_0\gamma \opnorm{\bA}\,, 
\end{align*}
where in the last step we used the fact that $\bth,\tbth\in\mathcal{S}^{d-1}$ and $\twonorm{\bth-\tbth}\le c_0\gamma$. So it suffices to bound $\opnorm{\bA}$. By definition, the matrix $\bA$ is obtained by selecting a subset of rows and columns of $\bW\bW^\sT$ and replacing them with zeros. Therefore, $\opnorm{\bA}\le \opnorm{\bW\bW^\sT}\le \sqrt{\psi_{1,d}}+C$, on the event $\event_{\bW}$.

Since the above bound in Lemma~\ref{lem:net} holds for any vector $\bx_i$, a similar bound also holds if the terms are replaced by their expectation with respect to $\bx_i$, whence we obtain  $|\E[\eta_i(\bth)^2]-\E[\eta_i(\tbth)^2]| \le2c_0\gamma(\sqrt{\psi_{1,d}}+C)$.

By definition of $\nu_i(\bth;\gamma)$ we have
\begin{align}
\nu_i(\bth;\gamma) = 1 &\implies |\eta_i(\bth)^2- \E[\eta_i(\bth)^2]| \ge \gamma\nonumber\\
&\implies |\eta_i(\tbth)^2- \E[\eta_i(\tbth)^2]| + \Big|(\eta_i(\bth)^2- \E[\eta_i(\bth)^2]) - (\eta_i(\tbth)^2- \E[\eta_i(\tbth)^2]) \Big|  \ge \gamma\nonumber\\
&\implies |\eta_i(\tbth)^2- \E[\eta_i(\tbth)^2]| \ge \gamma - 2c_0\gamma(\sqrt{\psi_{1,d}}+C)\nonumber\\
&\implies \nu_i(\tbth; \gamma(1 -2c_0\sqrt{\psi_{1,d}}-2c_0C)) = 1\,.
\end{align}

\subsubsection{Proof of Lemma~\ref{lem:aux1}}
The conditional distribution of $Z_1$ given $Z_2$ is 
\[
Z_1|Z_2 = z_2 \sim\normal(\rho z_2,(1-\rho^2))\,.
\]
Therefore,  $\E[Z_1|\;Z_2 = z_2] = \rho z_2$ and so 
\[
\E[Z_1|\; Z_2 >z] = \E[Z_1|\; Z_2 = z_2] \prob(Z_2=z_2 | Z_2>z) \de z_2 = \rho \E[Z_2| Z_2>z]. \]
Using the properties of the expectation of a truncated normal distribution, we have
\[
 \E[Z_2| Z_2>z] =  \frac{\phi(z)}{(1-\Phi(z))}\,,
\]
which completes the proof.

\subsection{Proof of Lemma~\ref{lem:Loss-oo}}

By~\eqref{eq:L-L1} it suffices to show that 
\[
\sup_{\bth\in\cC_{\bth}} \frac{|\cLoo(\bth)-\cLo(\bth)|}{1+\min(\cLo(\bth),\cLoo(\bth))} = o_{d,\prob}(1)\,.
\]
To lighten the notation we define the shorthand $\alpha_i: =  |y_i - \bth^\sT\sigma(\bW\bx_i)|+ \eps \twonorm{\bW^\sT  {{\rm diag}}{\ind(\bW \bx_i > 0)}\; \bth}$ and so
$\cLo(\bth) = 1/(2n) \sum_{i=1}^n \alpha_i^2 +\frac{\zeta}{2}\bth^\sT \bOmega \bth$. We write $\cLoo(\bth) = 1/(2n)\sum_{i=1}^n (\alpha_i+\beta_i)^2+\frac{\zeta}{2}\bth^\sT \bOmega \bth$ with 
\[
\beta_i:= \eps\rev{\twonorm{\bJ \bth}} - \eps \twonorm{\bW^\sT  {{\rm diag}}{\ind(\bW \bx_i > 0)}\; \bth}\,.
\]
Since for any two positive values $a,b$ we have $|a-b|\le\sqrt{|a^2-b^2|}$, we can write
\begin{align}\label{eq:beta_i-B}
|\beta_i|\le \eps \Big(\twonorm{\bW^\sT  {{\rm diag}}{\ind(\bW \bx_i > 0)}\; \bth}^2 - \rev{\twonorm{\bJ \bth}^2}\Big)^{1/2}
=[\eta_i(\bth)^2 - \E[\eta_i(\bth)^2]]^{1/2}\,.
\end{align}

Note that on the event $\event_{\W}$, $\opnorm{\bW}$ is bounded and so $\opnorm{\bW^\sT  {{\rm diag}}{\ind(\bW \bx_i > 0)}}$ is also bounded. Since $\bth\in\cC_{\bth}$, we have $\twonorm{\bth} = O(1)$, which along with Lemma~\ref{pro:spectral} imply that $\max_{i\in [n]} |\eta_i(\bth)|$ and $\max_{i\in [n]} |\E[\eta_i(\bth)]|$ are both $O_{d,\prob}(1)$. Therefore: $(i)$ defining $\bbeta = (\beta_1,\dotsc,\beta_n)$,  we have $\twonorm{\bbeta} = O_{d,\prob}(1)$; $(ii)$ Using~\eqref{eq:beta_i-B} along with Corollary~\ref{coro:nui}, we get $\frac 1n |\{i:\, |\beta_i|> \tfrac{1}{\sqrt{\log (d)}}\}| = o_{d,\prob}(1)$.

From the above we can deduce that $\frac1n \twonorm{\bbeta}^2 = o_{d,\prob}(1)$. We also define $\balpha = (\alpha_1,\dotsc, \alpha_n)$. For any $\bth\in\cC_{\bth}$ we have
\begin{align*}
| \cLoo(\bth) - \cLo(\bth) | &= \bigg| \frac{1}{2n} \sum_{i=1}^n  (\alpha_i + \beta_i)^2 - \sum_{i=1}^n \alpha_i^2\bigg|\\
& =  \frac{\twonorm{\bbeta}^2}{2n} +   \frac{1}{n} | \sum_{i=1}^n \alpha_i \beta_i | \\
& \leq  \frac{\twonorm{\bbeta}^2}{2n}  + \frac{1}{n} \twonorm{\bbeta} \twonorm{\balpha} \\
&  \leq  \frac{\twonorm{\bbeta}^2}{2n}  +  \frac{\twonorm{\bbeta}}{\sqrt{n}}  \frac{\twonorm{\balpha}}{\sqrt{n}} \\
& \leq  \frac{\twonorm{\bbeta}^2}{2n}  +   \frac{\twonorm{\mathbf{\bbeta}}}{2\sqrt{n}} \Big( 1 +    \frac{\twonorm{\balpha}^2}{n}  \Big) \\
& \leq   \frac{\twonorm{\bbeta}^2}{2n}  +   \frac{\twonorm{\bbeta}}{2\sqrt{n}} +   \frac{\twonorm{\bbeta}}{\sqrt{n}} \cLo(\bth)\\
& = o_{d,\prob}(1) (1+ \cLo(\bth)),
\end{align*}
where the last step holds because $\frac{\twonorm{\bbeta}}{\sqrt{n}} = o_{d,\prob}(1)$.

By a similar argument, we also get
\[
| \cLoo(\bth) - \cLo(\bth) | \le o_{d,\prob}(1) (1+\cLoo(\bth))\,.
\]
Combining these two bounds we get $| \cLoo(\bth) - \cLo(\bth) | \le o_{d,\prob}(1) \left(1+\min(\cLo(\bth),\cLoo(\bth)) \right)$ .

We next proceed to the second part. By optimality of $\hth$ and $\hthoo$ we have
\begin{align}
\cLoo(\hth) < \cLoo(\boldsymbol{0}) = \frac{1}{n}\sum_{i=1}^n y_i^2 = O_{d,\prob}(1),
\quad 
 \cL(\hth) < \cL(\boldsymbol{0}) = \frac{1}{n}\sum_{i=1}^n y_i^2= O_{d,\prob}(1).
\end{align}
As shown in Proposition~\ref{propo:Cth-L0}, $\hth\in \cC_{\bth}$ with high probability. Likewise we have \rev{$\hth^*\in \cC_{\bth}$}, with high probability (this follows from Lemma~\ref{th_inf_lemma} for the special case of $k = n$ in that lemma.)  

Therefore using the first part of the current lemma, 
\begin{align*}
|\cLoo(\hth)-\cL(\hth)| = o_{d,\prob}(1), \quad \rev{|\cLoo(\hth^*)-\cL(\hth^*)| = o_{d,\prob}(1)\,.}
\end{align*}
We therefore obtain
\[
0\le \cL(\hth^*) - \cL(\hth) < (\cL(\hth^*) - \cLoo(\hth^*)) + \underbrace{(\cLoo(\hth^*) - \cLoo(\hth))}_{\le 0} + (\cLoo(\hth) - \cL(\hth)) \le o_{d,\prob}(1)\,.
\]
Since $\cL(\bth)$ is $\frac{\zeta}{2}$-strongly convex we have
\[
\twonorm{\hth^*-\hth} \le o_{d,\prob}(1)/\zeta \to 0, \text{ as } d\to \infty\,.
\]

\rev{\subsection{Proof of Lemma~\ref{lem:AR2-AR}}
We define
\[
\ARo(\bth):= \E\left[ \left(|y - \bth^\sT \sigma(\bW\bx)| + \etest \twonorm{\bW^\sT  {{\rm diag}}{\ind(\bW \bx > 0)}\; \bth}\right)^2\right]\,.
\]
As an immediate result of Proposition~\ref{propo:simple1}, we have $\sup_{\bth\in \cC_{\bth}} |\ARo(\bth) - \AR(\bth)| = o_d(1)$. Therefore, it suffices to show that
\begin{align}\label{eq:inter-ARs}
\sup_{\bth\in \cC_{\bth}} \frac{|\ARoo(\bth)-\ARo(\bth)|}{\sqrt{\ARo(\bth)}} = o_{d,\prob}(1)\,.
\end{align}
By expanding the terms in $\ARo(\bth)$ and invoking our notation $\eta_i(\bth) =\twonorm{\bW^\sT  {{\rm diag}}{\ind(\bW \bx > 0)}\; \bth}$, we have
\[
\ARo(\bth) = \E[(y - \bth^\sT \sigma(\bW\bx))^2] + \etest^2 \E[ \eta_i(\bth)^2] + 
2\etest \E\left[ |y - \bth^\sT \sigma(\bW\bx)|\;  \eta_i(\bth) \right]\,.
\]
Likewise we have
\[
\ARoo(\bth) = \E[(y - \bth^\sT \sigma(\bW\bx))^2] + \etest^2  \twonorm{\bJ\bth}^2 + 
2\etest\E[ |y - \bth^\sT \sigma(\bW\bx)|] \;  \twonorm{ \bJ \bth }\,.
\]
Recall that by definition of $\bJ$ we have $\E[ \eta_i(\bth)^2] = \twonorm{\bJ\bth}^2$. Hence,
\begin{align}\label{eq:AR1-AR2}
\Big|\ARo(\bth)-\ARoo(\bth)\Big| &= 2\etest\; \Big|\E\left[ |y - \bth^\sT \sigma(\bW\bx)|\; ( \eta_i(\bth) - \sqrt{\E[\eta_i(\bth)^2]}) \right]\Big|\nonumber\\
&\le 2\etest \E\left[(y - \bth^\sT \sigma(\bW\bx))^2\right]^{1/2}  \E\left[(\eta_i(\bth)- \sqrt{\E[\eta_i(\bth)^2]})^2\right]^{1/2}\nonumber\\
&\le 2\etest \sqrt{\ARo(\bth)}  \E\left[(\eta_i(\bth)- \sqrt{\E[\eta_i(\bth)^2]})^2\right]^{1/2}\,.
\end{align}
To bound the right-hand side, note that for any two positive values $a,b$ we have $(a-b)^2\le |a^2-b^2|$. Therefore, 
\begin{align}\label{eq:eta-chain1}
\E\left[(\eta_i(\bth)- \sqrt{\E[\eta_i(\bth)^2]})^2\right] \le \E\left[\Big|\eta_i(\bth)^2- {\E[\eta_i(\bth)^2]}\Big|\right] \,.
\end{align} 
Also recall that for any non-negative random variable $Z$, we have $\E[Z] = \int_0^\infty \prob(Z\ge z)\de z$. Therefore, 
\begin{align}\label{eq:eta-chain2}
\E\left[\Big|\eta_i(\bth)^2- {\E[\eta_i(\bth)^2]}\Big|\right]  &= \E\left[\Big|\eta_i(\bth)^2- {\E[\eta_i(\bth)^2]}\Big|; \event_{\bW}\cap \event_{\bx}\right] +  \prob((\event_{\bW}\cap \event_{\bx})^c)\nonumber\\  
&\int_0^\infty \prob\left(\Big|\eta_i(\bth)^2- {\E[\eta_i(\bth)^2]}\Big| \ge \gamma \right)\de \gamma+ \prob((\event_{\bW}\cap \event_{\bx})^c)\nonumber\\
&\le  \int_0^\infty \min\left(\frac{c}{d\gamma^2},1\right) \de \gamma + c \exp(- \log^2(d)/c) + n e^{-d}
\end{align}
where the inequality follows from Lemma~\ref{lem:HW}. Next, we have
\begin{align}\label{eq:eta-chain3}
 \int_0^\infty \min\left(\frac{c\log^6(d)}{d\gamma^2},1\right) \de \gamma &=
\int_0^{\sqrt{c\log^6(d)/d}} \de \gamma + \int_{\sqrt{c\log^6(d)/d}}^\infty \frac{c\log^6(d)}{d\gamma^2} \de \gamma\nonumber\\
&= \sqrt{\frac{c\log^6(d)}{d}} + \frac{c\log^6(d)}{d} \sqrt{\frac{d}{c\log^6(d)}}  = 2\sqrt{\frac{c\log^6(d)}{d}}\,.
\end{align} 
Combining Eqs.~\eqref{eq:eta-chain1}, \eqref{eq:eta-chain2} and \eqref{eq:eta-chain3} we arrive at
\[
\E\left[(\eta_i(\bth)- \sqrt{\E[\eta_i(\bth)^2]})^2\right] \le 2\sqrt{\frac{c\log^6(d)}{d}} +  c \exp(- \log^2(d)/c) + n e^{-d} = o_d(\log^3(d) d^{-1/2})\,.
\]
Using the above bound in~\eqref{eq:AR1-AR2} we get that uniformly over $\bth\in \cC_{\bth}$,
\[
|\ARo(\bth)-\ARoo(\bth)| \le \sqrt{\ARo(\bth)} \;o_{d,\prob}(1)\,.
\]
This completes the proof of claim~\eqref{eq:inter-ARs}.
}
\section{Proofs of step 3: The Gaussian equivalence property} \label{proofs_of_step_3}
\subsection{Proof of Proposition~\ref{pro:risk-equi}}
As proved in~\cite[Theorem 2]{goldt2020gaussian},  under the assumptions of Proposition~\ref{pro:risk-equi},
$(\bth^\sT\sigma(\bW\x), \bbeta^\sT \x_0)$ converges in distribution to $(\bth^\sT \mbf, \bbeta^\sT\x)$. We first show that \rev{$\ARoo(\bth)= \ARoo_{\rm nl}(\bth)+o_d(1)$}. Recalling the definition of $\ARoo_{\rm nl}(\bth)$ given by~\eqref{eq:ARnl}, and plugging for $y = \beta^\sT\x+\xi$, we write
\begin{align}
\ARoo_{\rm nl}(\bth) &= \E\left[\left(|\bbeta^\sT\x-\bth^\sT\mbf+\xi| + \eps \twonorm{\bJ \bth}\right)^2\right]\nn\\
&= \E\left[\left(\bbeta^\sT\x-\bth^\sT\mbf+\xi\right)^2\right]
+ \eps\twonorm{\bJ \bth} \E\left[\left|\bbeta^\sT\x-\bth^\sT\mbf+\xi\right|\right] + \eps^2 \twonorm{\bJ \bth}^2
\end{align}
Therefore, $\ARoo_{\rm nl}(\bth)$ can be written in terms of the first and second moment of random variable $\left|\bbeta^\sT\x-\bth^\sT\mbf+\xi\right|$ which converges in distribution to  $\left|\bbeta^\sT\x-\bth^\sT\sigma(\bW\x)+\xi\right|$. To show that $\ARoo_{\rm nl}(\bth)- \ARoo(\bth)\to 0$ as $d\to\infty$, we need to show that the first and second moments of $\left|\bbeta^\sT\x-\bth^\sT\mbf+\xi\right|$ converge respectively to the first and second moments of $\left|\bbeta^\sT\x-\bth^\sT\sigma(\bW\x)+\xi\right|$. As an application of \cite[Corollary of Theorem 25.12]{billingsley1995probability}, it suffices to show that $\left|\bbeta^\sT\x-\bth^\sT\mbf+\xi\right|$ has bounded third moment. To show this, note that by Holder's inequality, $|a+b+c|^3\le 3(|a|^3+|b|^3+|c|^3)$. Furthermore, $\E[|\xi|^3] = 2$, $\E[|\bbeta^\sT\x|^3] =2\twonorm{\bbeta}^3 = 2$. Hence,
\begin{align}\label{eq:3rd-mom}
\E\left[\left|\bbeta^\sT\x-\bth^\sT\mbf+\xi\right|^3\right]\le 3(16+\E[|\bth^\sT \mbf|^3])\,.
\end{align}
By using the Holder's inequality again we have
\begin{align}
E[|\bth^\sT \mbf|^3] &= \E\left[\left|\frac{1}{\sqrt{2\pi}}\mathbf{1}^\sT \bth+\frac{1}{2} \bth^\sT\bW\x + \sqrt{\frac{1}{4}-\frac{1}{2\pi}}\; \bth^\sT\bu\right |^3\right]\nn\\
&\le 3\left(\frac{1}{\sqrt{2\pi}^3} (\mathbf{1}^\sT\bth)^3 + \frac{1}{4} \twonorm{\bW^\sT\bth}^3 + 2(\frac{1}{4}-\frac{1}{2\pi})^{3/2} \twonorm{\bth}^3 \right)\,.
\end{align}
Now note that by our assumption $\ARoo_{\rm nl}(\bth)$ is bounded, which in conjunction with characterization~\eqref{ARnl-char} implies that $\twonorm{\bth}$, $\mathbf{1}^\sT\bth$ and $\twonorm{\bW^\sT\bth - 2\bbeta}^2$ are bounded as $d\to \infty$. This also implies that $\twonorm{\bW^\sT\bth}^2\le (\twonorm{\bW^\sT\bth - \bbeta} + \twonorm{\bbeta})^2\le (\twonorm{\bW^\sT\bth - \bbeta} + 1)^2$ is bounded. Putting these together, we obtain that 
$\E\left[\left|\bbeta^\sT\x-\bth^\sT\mbf+\xi\right|^3\right]$ is bounded, which completes the argument for showing that $\ARoo_{\rm nl}(\bth) = \ARoo(\bth)+o_d(1)$.

We next prove the characterization~\eqref{ARnl-char}. 
Note that 
\begin{align}
y - \bth^\sT\mbf &= \xi + \bbeta^\sT\x - \frac{1}{2} \bth^\sT\bW\x - \sqrt{\frac{1}{4}-\frac{1}{2\pi}}\; \bth^\sT\bu \nn\\
&= \xi + \<\bbeta - \frac{1}{2} \bW^\sT\bth, \x\> - \sqrt{\frac{1}{4}-\frac{1}{2\pi}}\; \bth^\sT\bu\nn \\
& \sim  \normal(0,  M(\bth)^2) \,,
\end{align}
with
\[
 M(\bth)^2 = \tau^2 + \twonorm{\frac{1}{2}\bW^\sT\bth - \bbeta}^2 + \Big(\frac{1}{4}-\frac{1}{2\pi}\Big) \twonorm{\bth}^2\,.
\]
We then write
\begin{align}
\ARoo_{\rm nl}(\bth) &= \E\left[\left(|y-\bth^\sT\mbf| + \rev{\etest} \twonorm{\bJ \bth} \right)^2\right]\nn\\
&= \E\left[\left(y-\bth^\sT\mbf\right)^2\right]
+ 2\rev{\etest} \twonorm{\bJ \bth} \E\left[\left|y-\bth^\sT\mbf\right|\right] + \rev{\etest^2} \twonorm{\bJ \bth}^2\nn\\
&=  \rev{M(\bth)^2+ \etest^2 \twonorm{\bJ \bth}^2 + 2\sqrt{\frac{2}{\pi}}\etest M(\bth) \twonorm{\bJ \bth}} \,, 
\end{align}
using the first and second moment of folded normal distribution.
\subsection{Proof of Theorem~\ref{main_theorem:GEP}}
Before stating the proof, we remark that our proof is an adaptation of the  powerful machinery developed \cite{hu2020universality};  however, since our individual adversarial losses have an additional term $ \varepsilon \twonorm{\bJ \bth}$, there are some additional details in the proofs which we provide in the following. Also, since our activation function is not odd, we will use the CLT-type result of \cite{goldt2020gaussian} instead of the one provided in  \cite{hu2020universality}. 

Recall that we are seeking to analyze the asymptotic values of the following quantities:
\begin{align*}
\Phi_A &:= \min_{\bth} \frac{1}{n}\sum_{i=1}^n  \left( |y_i - \bth^\sT\sigma(\bW\bx_i)|+  \eps \twonorm{\bJ \bth} \right)^2   + \lambda\twonorm{\bth}^2 + \lambda_w \twonorm{\frac{1}{2}\bW^\sT\bth - \bbeta}^2 +  \lambda_s\frac{\log (d)}{d} (\ones^\sT \bth)^2, \\
\Phi_B &:= \min_{\bth} \frac{1}{n}\sum_{i=1}^n  \left(\,\, |\,y_i \, - \, \bth^\sT\mbf_i \,|\,+\,  \eps \,\twonorm{\bJ \bth} \,\, \right )^2   \,\, \,\,\,\,\,+  \,\,\lambda\twonorm{\bth}^2 + \lambda_w \twonorm{\frac{1}{2}\bW^\sT\bth - \bbeta}^2 +   \lambda_s\frac{\log (d)}{d} (\ones^\sT \bth)^2,
\end{align*}
 To simplify our notation, and without loss of generality, we absorb the value $\epsilon$ into $\bJ$ and, with a slight abuse of notation, consider $\bJ \longleftarrow \epsilon \bJ$ (and hence the eigenvalues of the matrix $\bJ$ depend on $\epsilon$). Hence, above the quantities of interest become
\begin{align*}
\Phi_A &:= \min_{\bth} \frac{1}{n}\sum_{i=1}^n  \left( |y_i - \bth^\sT\sigma(\bW\bx_i)|+   \twonorm{\bJ \bth} \right)^2   + \lambda\twonorm{\bth}^2 + \lambda_w \twonorm{\frac{1}{2}\bW^\sT\bth - \bbeta}^2 +  \lambda_s\frac{\log (d)}{d} (\ones^\sT \bth)^2, \\
\Phi_B &:= \min_{\bth} \frac{1}{n}\sum_{i=1}^n  \left(\,\, |\,y_i \, - \, \bth^\sT\mbf_i \,|\,+\,  \,\twonorm{\bJ \bth} \,\, \right )^2   \,\, \,\,\,\,\,+  \,\,\lambda\twonorm{\bth}^2 + \lambda_w \twonorm{\frac{1}{2}\bW^\sT\bth - \bbeta}^2 +   \lambda_s\frac{\log (d)}{d} (\ones^\sT \bth)^2,
\end{align*}

For technical reasons, we first need to make the objectives smooth. We thus define $g(x) = \sqrt{x + \gamma}$, and define the smoothed loss 
\begin{equation} \label{ell_def}
\ell(\bth; \rb,  y) = (\bth^\sT \rb - y)^2 + \twonorm{\bJ \bth}^2 + 2g\left((\bth^\sT \rb - y)^2 \twonorm{\bJ \bth}^2 \right),
\end{equation}
Note that when $\gamma =0$,  we have $\ell(\bth; \sigma(\W \x_i), y_i) =  \left( |y_i - \bth^\sT\sigma(\bW\bx_i)|+   \twonorm{\bJ \bth} \right)^2$ and $\ell(\bth; \mbf_i, y_i) =  \left( |y_i - \bth^\sT\mbf_i +   \twonorm{\bJ \bth} \right)^2$.  

In the following, we consider an arbitrary but \emph{fixed} value $\gamma >0$, and with some abuse of notation, we let
\begin{align*}
\Phi_A &:= \min_{\bth} \frac{1}{n}\sum_{i=1}^n  \ell \left(\bth; \sigma(\W \x_i), y_i \right)   + \lambda\twonorm{\bth}^2 + \lambda_w \twonorm{\frac{1}{2}\bW^\sT\bth - \bbeta}^2 +  \lambda_s\frac{\log (d)}{d} (\ones^\sT \bth)^2, \\
\Phi_B &:= \min_{\bth} \frac{1}{n}\sum_{i=1}^n \ell\left(\bth; \mbf_i, y_i \right)   \,\, \,\,+  \,\,\lambda\twonorm{\bth}^2 + \lambda_w \twonorm{\frac{1}{2}\bW^\sT\bth - \bbeta}^2 +   \lambda_s\frac{\log (d)}{d} (\ones^\sT \bth)^2,
\end{align*}
For these quantities $\Phi_A, \Phi_B$, we show then show in the following that the statement of the Theorem is true.  Then, the result of the Theorem for the original losses -- i.e. when $\gamma = 0$ -- follows simply by taking  the limit $\gamma \to 0$ (note that this limit is taken {\emph{after}} the limit $d \to \infty$; also, note that $\sup_{x\geq 0}\{g(x) - x\} = \sqrt{\gamma}$).

We will use the Lindeberg's leave-one-out technique. In a nutshell,  we start with the quantity $\Phi_B$, and through $n$ consecutive steps, we reach to the quantity $\Phi_A$. In the $k$-th step, we will replace the feature vector $\mbf_k$ with $\sigma(\W \x_k)$. We will then show that each of these replacements has a negligible effect (i.e. $o_n(1)/n$) on our quantities of interest, leading to the proof of the theorem.

Let us now proceed with the details. The proof has multiple steps which will be put together in Section~\ref{putting_things_together_GEP} to obtain the proof of the theorem. 

We begin by defining  
 \begin{equation} \label{Phi_k}
\Phi_k = \min_{\bth} \frac{1}{n} \sum_{i=1}^k \ell(\bth; \sigma(\W \x_i)_i, y_i)  +  \frac{1}{n} \sum_{i=k+1}^n \ell(\bth; \mbf_i, y_i) + \lambda\twonorm{\bth}^2 + \lambda_w \twonorm{\frac{1}{2}\bW^\sT\bth - \bbeta}^2 +   \lambda_s\frac{\log (d)}{d} (\ones^\sT \bth)^2,
\end{equation}
Roughly speaking, our goal is to show that for all $k \in [n]$, we have $\Phi_k \approx \Phi_{k-1} + o_n(1)/n$. To make this entirely rigorous, we need to define several new quantities and understand their relations. For $k \in [n]$ let

\begin{equation} \label{R_-k}
R_{-k}(\bth) = \frac{1}{n}\sum_{i=1}^{k-1} \ell \left( \bth; \sigma(\W \x_i), y_i \right) +  \frac{1}{n} \sum_{i=k+1}^n \ell \left(\bth; \mbf_i, y_i \right)  
+\lambda\twonorm{\bth}^2 + \lambda_w \twonorm{\frac{1}{2}\bW^\sT\bth - \bbeta}^2 +   \lambda_s\frac{\log (d)}{d} (\ones^\sT \bth)^2,
\end{equation}
and 
\begin{equation} \label{relation_Rs}
R_{k}(\bth, \rb) =  \frac{1}{n} \ell(\bth; \rb, y_k)  + R_{-k}(\bth).
\end{equation}
Let us denote the minimizers of the above two objectives by
\begin{equation} \label{minimizers_LOO}
\bth^*_{-k} = \arg \min_\bth R_{-k}(\bth),\text{ and } \bth^*_{k}(\rb) = \arg \min_\bth R_{k}(\bth, \rb)
\end{equation}
and 
\begin{equation}
\Phi_{-k} = \min_{\bth} R_{-k}(\bth),
\end{equation}
and

\begin{equation} \label{Phi_def}
\Phi_{k}(\rb) = \min_{\bth} R_{k}(\bth, \rb).
\end{equation}

It will also be convenient to work with approximate versions of the term $R_k(\bth, \rb)$ in \eqref{relation_Rs}. Hence, we define below we define $\R_k(\bth, \rb)$ which is essentially obtained by Taylor-expanding the term $R_{-k}(\bth)$ in \eqref{relation_Rs}.

\begin{equation} \label{S_k}
S_k(\bth, \rb) = \Phi_{-k} + \frac12(\bth - \bth^*_{-k})^\sT \bH_{-k} (\bth - \bth^*_{-k}) + \frac{1}{n}\ell(\bth; \rb, y_k),
\end{equation}
where $\bH_{-k}$ is the Hessian of $R_{-k}(\bth)$ at $\bth^*_{-k}$, i.e.
\begin{equation} \label{H_def}
\bH_{-k} = \nabla^2 R_{-k}(\bth) \mid_{\bth = \bth^*_{-k}}.
\end{equation}
Finally, we denote the minimizer of $S(\bth, \rb)$ by
\begin{equation}
 \tilde{\bth}_{k}(\rb) = \arg \min_\bth S_{k}(\bth, \rb),
\end{equation}
and
\begin{equation} \label{Psi_def}
 \Psi_{k}(\rb) = \min_{\bth} S_{k}(\bth, \rb).
\end{equation}

\noindent\textbf{Simplification of Notation.}
We note from \eqref{Phi_k}  that in our analysis the feature vectors are either $\rb_i = \ab_i$ or $\rb_i = \bb_i$; i.e. we can write $\Phi_k$ as
\begin{equation} \label{Phi_k_r}
\Phi_k = \min_{\bth} \frac{1}{n} \sum_{i=1}^n \ell(\bth; \rb_i, y_i) +\lambda\twonorm{\bth}^2 + \lambda_w \twonorm{\frac{1}{2}\bW^\sT\bth - \bbeta}^2 +   \lambda_s\frac{\log (d)}{d} (\ones^\sT \bth)^2,
\end{equation}
where the feature vectors $\rb_i$ they are generated according to one of the following distributions
\begin{equation} \label{rb_distrbutions}
\rb_i = \sigma(\W \x _i) \,\,\,\,\,\,\,\,\text{ or } \,\,\,\,\,\,\,\, \rb_i = \mbf_i := \mu_1 \W \x_i + \mu_2 \ub_i,
\end{equation}
It will be sometimes easier in our analysis to use \eqref{Phi_k_r}, i.e. use $\rb_i$ for both $\sigma(\W \x_i)$ and $\mbf_i$, but we will keep in mind that for $i \leq k$ we have $\rb_i = \mbf_i$ and for $i > k$ we have $\rb_i = \sigma(\W \x_i)$.

\noindent\textbf{Details of the Gradient and Hessian of $\ell$.} In the following, we will need to work out the first and second derivatives of the loss function $\ell$, given in \eqref{ell_def}, at multiple points. In order to present the derivations more compactly, let us denote 
\begin{equation} \label{h_def}
h(\bth; \rb, y) := (\bth^\sT \rb - y )^2 \twonorm{\bJ \bth}^2,
\end{equation}
and provide the details for the derivatives of the loss function $\ell$ here. Given how the function $h$ is defined, we can write 
$$  \ell( \bth; \rb, y) = (\bth^\sT \rb - y)^2 + \twonorm{\bJ \bth}^2 + 2g\left(h(\bth; \rb, y)   \right)  $$
Using the notation $\nabla$ for gradient w.r.t. $\bth$, and $\nabla^2$ for hessian w.r.t. $\bth$, we can  write
\begin{equation} \label{grad_ell}
\nabla  \ell(\bth; \rb, y) =  2(\bth^\sT \rb- y) \rb   + 2 \bJ^\sT \bJ \bth  + 2\nabla h(\bth; \rb, y)   g'\left(h(\bth; \rb, y)   \right) 
\end{equation}
and
\begin{equation} \label{hessian_ell}
\nabla^2  \ell( \bth; \rb, y) = 2\left( \rb \rb^\sT + \bJ^\sT \bJ + \nabla^2 h(\bth; \rb, y)   g'\left(h(\bth; \rb, y)\right)  + \nabla h(\bth; \rb, y) \left(\nabla h(\bth; \rb, y) \right)^\sT   g''\left(h(\bth; \rb, y \right) \right),
\end{equation}
where
\begin{equation} \label{gradient_h}
\nabla  h(\bth; \rb, y) = 2\left( (\bth^\sT \rb - y) \twonorm{\bJ \bth}^2 \rb + (\bth^\sT \rb - y)^2  \bJ^\sT \bJ \bth \right),
\end{equation}
and 
\begin{equation}
\nabla^2  h(\bth; \rb, y) = 2 \left(  \twonorm{\bJ \bth}^2 \rb \rb^\sT + 2(\bth^\sT \rb - y) \rb \bth^\sT  \bJ^\sT \bJ +    2(\bth^\sT \rb - y) \bJ^\sT \bJ \bth \rb^\sT + (\bth^\sT \rb - y)^2 \bJ^\sT \bJ   \right). 
\end{equation}

%

\subsubsection{Some properties of the minimizers in \eqref{minimizers_LOO}}
In this section, we will analyze some of the properties of the vector $\bth_{-k}^*$ and its relation with $\tilde{\bth}_{-k}(\rb)$, for $k \in [n]$. We first show some basic properties of the vectors $\bth_{-k}^*$ and $\bth_{k}^*(\rb)$.

\begin{lemma} \label{properties_of_bth_star}
Fix $k \in [n]$. The following hold with  absolute constants $c,C > 0$: 
\begin{itemize}
\item[(a)] The vector $\bth_{-k}^*$ is bounded in the $\ell_2$ norm: 
\begin{equation}
\mathbb{P} \left( \twonorm{\bth_{-k}^*} \geq v + C \right) \leq c \exp\left(- nv^2/c\right).
\end{equation} 
\item[(b)] The vector $\bth_{k}^*(\rb)$ is bounded in the $\ell_2$ norm: 
\begin{equation} \label{theta_star_norm_bound}
\mathbb{P}\left( \twonorm{\bth_{k}^*(\rb)} \geq v + C  \right) \leq c \exp\left(- nv^2/c\right).
\end{equation} 
\item [(c)] For an independently generated vector $\rb$ we have 
\begin{equation}
\mathbb{P}\left(| \rb^\sT \bth_{-k}^*(\rb) | \geq v \right) \leq c \exp(-v/c).
\end{equation}
\item [(d)] We also have
\begin{equation}
\left| \ones^\sT \bth^*_{-k} \right| \leq C\sqrt{\frac{d}{\log (d)}},
\end{equation}
with probability at least $1-ce^{-c n}$.
\end{itemize}
\end{lemma}
The proof of this lemma is provided in Section~\ref{leave-one-out-auxilliary-lemmas}. 
We now show that the distance between the two minimizers $\bth_{-k}^*$ and $\tilde{\bth}_k(\rb)$ is of order $O(1/\sqrt{d})$.  

\begin{lemma} \label{ell_2_diff_first}
Fix $k \in [n]$. Assuming that $\rb$ is generated independently from $\bth_{-k}^*$ and according to one of the distributions in \eqref{rb_distrbutions}. Then, there exist absolute constants $c, c' > 0$ such that
\begin{equation}
\mathbb{P}\left(\twonorm{\tilde{\bth}_{-k}(\rb) - \bth_{-k}^*} \geq \frac{v}{\sqrt{ d}} \right) \leq c \exp\left( -  v^{c'}/c \right).
\end{equation}
\end{lemma}
\begin{proof}
We start by noting that since  $\tilde{\bth}_k(\rb)$ is the minimizer of \eqref{S_k}:
\begin{equation}
\tilde{\bth}_k(\rb) = \arg\min_\bth\left\{S(\bth): = \frac12(\bth - \bth^*_{-k})^\sT \bH_{-k} (\bth - \bth^*_{-k}) + \frac{1}{n}\ell(\bth; \rb, y_k) \right\}
\end{equation}
Observe that  (i) the function $S$ is $\lambda$-strongly convex due to the fact that $R_{-k}(\bth)$ is strongly-convex, and thus its Hessian $\bH_{-k}$ is a  PSD matrix with smallest eigenvalue lower-bounded by $\lambda$; (ii) $S(\bth) \geq 0$ for any $\bth$, as $\bH_{-k} $ is a PSD matrix and $\ell$ is  always positive-valued.  As a result, we can write
\begin{equation}
\twonorm{\tilde{\bth}_k(\rb) - \bth_{-k}^*}^2 \leq \frac{1}{\lambda} S(\bth_{-k}^*) 
\end{equation}
We can then write from \eqref{ell_def} that
\begin{align*}
S(\bth_{-k}^*)  &= \frac{1}{n} \ell(\bth_{-k}^*; \rb, y_k) \\
& \leq \frac{1}{n} C\max\{ 1, ({\bth_{-k}^*}^\sT \rb - y)^2, \twonorm{\bJ \bth_{-k}^*}^2 \},
\end{align*}
where $C > 0$ is an absolute constant. The proof now follows from the result of Lemma~\ref{properties_of_bth_star} and the fact that $||\bJ||$ is bounded, as well as the fact that $d$ and $n$ grow in proportion to each other. 
\end{proof}
Given the above lemma, we can analyze the behavior of $\Psi_k(\rb)$, defined in \eqref{Psi_def}, in more detail. 
 \begin{lemma} \label{simpler_min_theta_taus_lemma}
Fix $k \in [n]$.  We have 
\begin{equation}  \label{simpler_min_theta_taus}
 \Psi_k(\rb) = \Phi_{-k} +  \frac{1}{n} \min_{\tau_1} 
  \left\{ \frac{1}{2n}\left(  \frac{\partial \tilde{\ell}(\tau_1, 0) }{\partial \tau_1} \rb + \frac{\partial \tilde{\ell}(\tau_1,0) }{\partial \tau_2} \pb\right)^\sT  \bH_{-k}^{-1} 
  \left(  \frac{\partial \tilde{\ell}(\tau_1, 0) }{\partial \tau_1} \rb + \frac{\partial \tilde{\ell}(\tau_1,0) }{\partial \tau_2} \pb\right) 
+ \tilde{\ell}(\tau_1, 0) \right\} + e, 
 \end{equation}
 where (i) we have $\pb^\sT =  2{\bth_{-k}^*}^\sT \bJ^\sT \bJ$; (ii) the function $\tilde{\ell}(\tau_1, \tau_2)$ is defined as
 \begin{equation} \label{ell_tilde_GEP}
\tilde{\ell}(\tau_1, \tau_2) :=  \rho_1 + \rho_2 + \rho_3 \tau_1 + \tau_1^2 + \tau_2+
g\left( (\rho_1 + \rho_3 \tau_1 + \tau_1^2) (\rho_2 + \tau_2) \right),
 \end{equation}
 with  $ \rho_1 = ({\bth_{-k}^*}^\sT \rb - y_k )^2$, $\rho_2 = \twonorm{\bJ \bth_{-k}^*}^2$, and $ \rho_3 = 2 ({\bth_{-k}^*}^\sT \rb - y_k)$; and (iii) 
 the value $e$ satisfies
 $$ \mathbb{P}( |e| \geq \frac{v}{d^{\frac 32}}) \leq c \exp\left( - v^{c'}/c \right), $$
 for absolute constants $c, c' > 0$.
 
 
Furthermore, assuming that $\tau_1^*$ is the minimizer of the optimization problem in \eqref{simpler_min_theta_taus}, we have

\begin{equation} \label{refined_bth_diff_lemma}
\tilde{\bth} - \bth_{-k}^* =  \frac{1}{n}\left(  \beta_1 \bH_{-k}^{-1}\rb + \beta_2 \bH_{-k}^{-1}\pb\right) + \eb,
\end{equation}
where $\beta_1, \beta_2$ depend only on $\tau_1^*$, as well as $\rb^\sT \bth_{-k}^*$, and $\twonorm{\bJ \bth_{-k}^*}$; and
\begin{equation} \label{bounded_diff_constants}
\mathbb{P} \left( \max\{d^{\frac32} \twonorm{\eb}, |\beta_1|, |\beta_2|, \twonorm{\pb}\} \geq  v \right) \leq  c \exp( -v^{c'}/c).
\end{equation}

\end{lemma}
\begin{proof}
In the following, to simplify notation, we use $\tilde{\bth}$ instead of $\tilde{\bth}_k(\rb)$.
We have from \eqref{S_k} and \eqref{Psi_def} that
\begin{equation} \label{Psi_min}
 \Psi_k(\rb) = \Phi_{-k} +  \min_{\bth} \left\{\frac12(\bth - \bth^*_{-k})^\sT \bH_{-k} (\bth - \bth^*_{-k}) + \frac{1}{n}\ell(\bth; \rb, y_k) \right\}. 
 \end{equation}
We now decompose the term $\frac{1}{n}\ell(\bth; \rb, y_k)$ (see \eqref{ell_def}) according to the following set of simple relations:
\begin{align*}
 \bth^\sT \rb - y &=  {\bth_{-k}^*}^\sT \rb - y + ( \bth - \bth_{-k}^*)^\sT \rb, \\
 \twonorm{\bJ \bth}^2 &= \twonorm{\bJ \bth_{-k}^*}^2 +  2{\bth_{-k}^*}^\sT \bJ^\sT \bJ (  \bth -\bth_{-k}^*) + (  \bth-\bth_{-k}^*)^\sT \bJ^\sT \bJ (  \bth - \bth_{-k}^*) ,\\
 (\bth^\sT \rb - y)^2 &=  ({\bth_{-k}^*}^\sT \rb - y)^2 + 2 ({\bth_{-k}^*}^\sT \rb - y) \rb^\sT ( \bth- \bth_{-k}^*) + (\bth -\bth_{-k}^*)^\sT \rb \rb^\sT (\bth- \bth_{-k}^*),
\end{align*} 
As a result, it is easy to obtain the following: 
\begin{align*}
\frac{1}{n}\ell(\bth; \rb, y_k) =&  \frac{1}{n} \left\{ \rho_1 + \rho_2 + \rho_3 (\rb^\sT(\bth - \bth_{-k}^*)) + (\rb^\sT(\bth - \bth_{-k}^*))^2 + \pb^\sT (\bth- \bth_{-k}^*) \right\} \\
&+ \frac{1}{n} \twonorm{\bJ (\bth - \bth_{-k}^*)}^2 \\
&+ \frac{1}{n}g\left( ( \rho_1 + \rho_3 (\rb^\sT(\bth - \bth_{-k}^*)) +  (\rb^\sT(\bth - \bth_{-k}^*))^2) (\rho_2 + \pb^\sT (\bth - \bth_{-k}^*) +\twonorm{\bJ (\bth - \bth_{-k}^*)}^2 )  \right),
\end{align*}
where the parameters $\rho_j$, $j=1,2,3$, and the vector $\pb$ are defined in the following:
\rev{
\begin{align} \label{alphas_and_p}
\rho_1 &= ({\bth_{-k}^*}^\sT \rb - y_k )^2, \\ \nonumber
\rho_2 &= \twonorm{\bJ \bth_{-k}^*}^2, \\ \nonumber
\rho_3 &= 2 ({\bth_{-k}^*}^\sT \rb - y_k), \\ \nonumber
\pb^\sT &=  2{\bth_{-k}^*}^\sT \bJ^\sT \bJ.
\end{align}
}
We note that none of these quantities depend on the optimization variable $\bth$ and hence can be considered as constants w.r.t. the minimization procedure in \eqref{Psi_min}.

It will be  convenient to consider the following variables: 
\begin{equation} \label{tau_def}
\tau_1 = \rb^\sT(\bth - \bth_{-k}^*) , \,\,\,\,\,\ \text{ and } \,\,\,\,\,\,  \tau_2 = \pb^\sT (\bth - \bth_{-k}^*),  \,\,\,\,\,\ \text{ and } \,\,\,\,\,\,  \tau_3 = \twonorm{\bJ (\bth - \bth_{-k}^*)}^2.
\end{equation}
As a result,  we can write
\begin{align} \nonumber
\frac{1}{n}\ell(\bth; \rb, y_k) &=  \frac{1}{n} \left\{ \rho_1 + \rho_2 + \rho_3 \tau_1 + \tau_1^2 + \tau_2 + \tau_3+
g\left( (\rho_1 + \rho_3 \tau_1 + \tau_1^2) (\rho_2 + \tau_2+ \tau_3) \right) \right\} \\
& := \frac{1}{n}\tilde{\ell}(\tau_1, \tau_2 , \tau_3). \label{l_tilde_proof_diff}
\end{align}
Now, from \eqref{Psi_min}, we can write the following equation for $\tilde{\bth}$ (as it is the minimizer):
\begin{equation}
\bH_{-k}(\tilde{\bth} - \bth_{-k}^*) = - \frac{1}{n}\left(  \frac{\partial \tilde{\ell} }{\partial \tau_1} \rb + \frac{\partial \tilde{\ell} }{\partial \tau_2} \pb + \frac{\partial \tilde{\ell}}{\partial \tau_3} \bJ^\sT \bJ (\tilde{\bth} - \bth_{-k}^*)  \right),
\end{equation}
where the partial derivatives are evaluated at $\tau_1, \tau_2,\tau_3$  when  $\bth = \tilde{\bth}$. Consequently, we have
\begin{equation} \label{tilde_bth_eq}
\tilde{\bth} - \bth_{-k}^* = - \frac{1}{n}\left(  \frac{\partial \tilde{\ell} }{\partial \tau_1} \bH_{-k}^{-1}\rb + \frac{\partial \tilde{\ell} }{\partial \tau_2} \bH_{-k}^{-1}\pb + \frac{\partial \tilde{\ell}}{\partial \tau_3} \bH_{-k}^{-1}  \bJ^\sT \bJ (\tilde{\bth} - \bth_{-k}^*)  \right),
\end{equation}
Using the above relation, we can derive a few useful properties. First, given how $\rho_j$'s are defined in \eqref{alphas_and_p}, and by using Lemma~\ref{properties_of_bth_star}, and since $g(x) = \sqrt{x + \gamma}$ has uniformly bounded first and second derivative, and by using the fact that the operator norm of $\bJ$ is bounded, it is easy to show that for absolute constants $c,c'$ we have  
\begin{align}
\mathbb{P}\left( \max \left\{ \left|  \frac{\partial \tilde{\ell} }{\partial \tau_1} \right|,  \left| \frac{\partial \tilde{\ell} }{\partial \tau_2}\right|, \left| \frac{\partial \tilde{\ell} }{\partial \tau_3}\right| \right\} \geq v \right) \leq c \exp(-v^{c'}/c),
\end{align}
where in the above relation the partial derivatives are evaluated at $\tau_1, \tau_2,\tau_3$, when $\bth = \tilde{\bth}$.

Second, by using Lemma~\ref{ell_2_diff_first}, and the fact that the norm of the matrix $\bJ$ is bounded, as well as the fact that the operator norm of $\bH_{-k}^{-1}$ is upper-bounded by $1/\lambda$, we have  for absolute constants $c,c'$ that
\begin{equation}
\mathbb{P}\left( \twonorm{   \frac{1}{n}\bH_{-k}^{-1} \bJ^\sT  \bJ (\tilde{\bth} - \bth_{-k}^*)    }  \geq \frac{v}{d^{\frac{3}{2}}}\right) \leq   c \exp( -v^{c'}/c). 
\end{equation}

Third, from the definition of $\pb$ in \eqref{alphas_and_p}, and by using Lemma~\ref{properties_of_bth_star} as well as \eqref{tilde_bth_eq} we obtain  for absolute constants $c,c'$ that

\begin{equation}
\mathbb{P}\left( \pb^\sT(\tilde{\bth} - \bth_{-k}^*)      \geq \frac{v}{d^{\frac{1}{2}}}\right) \leq   c \exp( -v^{c'}/c). 
\end{equation}
The above relation shows that the value of $\tau_2$, evaluated at $\bth = \tilde{\bth}$, is of order $O(d^{-\frac 12})$. 

Fourth, we can write using Lemma~\ref{ell_2_diff_first} that
\begin{equation}
\mathbb{P}\left( \twonorm{\bJ (\tilde{\bth} - \bth_{-k}^*)  }     \geq \frac{v}{d^{\frac 12}}\right) \leq   c \exp( -v^{c'}/c),
\end{equation}
which essentially results in $\tau_3$, evaluated at $\bth = \tilde{\bth}$, to be  of the order $O(d^{-1})$. 

Finally, we note that for each of the partial derivatives we can write
\begin{equation} 
\mathbb{P}\left( \left|  \frac{\partial \tilde{\ell} (\tau_1, \tau_2, \tau_3) }{\partial \tau_j} -  \frac{\partial \tilde{\ell} (\tau_1, 0, 0) }{\partial \tau_j}    \right| \geq \frac{v}{d^{\frac 12}} \right) \leq c \exp( -v^{c'}/c),
\end{equation}
for $j=1,2,3$ and for $\tau_j$'s that are evaluated at $\bth = \tilde{\bth}$.

Using the above five properties, we can conclude that
\begin{equation} \label{refined_bth_diff}
\tilde{\bth} - \bth_{-k}^* = - \frac{1}{n}\left(  \frac{\partial \tilde{\ell}(\tau_1, 0, 0) }{\partial \tau_1} \bH_{-k}^{-1}\rb + \frac{\partial \tilde{\ell}(\tau_1,0,0) }{\partial \tau_2} \bH_{-k}^{-1}\pb\right) + \eb,
\end{equation}
and 
\begin{equation}
\frac{1}{n}\tilde{\ell} (\tau_1, \tau_2, \tau_3) = \frac{1}{n} \tilde{\ell} (\tau_1, 0, 0) + e', 
\end{equation}
where $\tau_1, \tau_2, \tau_3$ are computed from \eqref{tau_def} at $\bth = \tilde{\bth}$, and  
$$\mathbb{P} \left( \max\{\twonorm{\eb}, |e'|\} \geq \frac{v}{d^{\frac32}} \right) \leq  c \exp( -v^{c'}/c).$$

As a result, by defining 
$$\tilde{\ell}(\tau_1, \tau_2) := \tilde{\ell}(\tau_1, \tau_2, 0), $$
where  $\tilde{\ell}$ is given in  \eqref{l_tilde_proof_diff}, and by plugging the solution \eqref{refined_bth_diff} into the optimization in \eqref{Psi_min}, we obtain
\begin{equation}
 \Psi_k(\rb) = \Phi_{-k} +  \frac{1}{n} \min_{\tau_1} 
  \left\{ \frac{1}{2n}\left(  \frac{\partial \tilde{\ell}(\tau_1, 0) }{\partial \tau_1} \rb + \frac{\partial \tilde{\ell}(\tau_1,0) }{\partial \tau_2} \pb\right)^\sT  \bH_{-k}^{-1} 
  \left(  \frac{\partial \tilde{\ell}(\tau_1, 0) }{\partial \tau_1} \rb + \frac{\partial \tilde{\ell}(\tau_1,0) }{\partial \tau_2} \pb\right) 
+ \tilde{\ell}(\tau_1, 0) \right\} + e, 
 \end{equation}
 where 
 $$\mathbb{P} \left( |e| \geq \frac{v}{d^{\frac32}} \right) \leq  c \exp( -v^{c'}/c).$$

\end{proof}

\subsubsection{Bounding \texorpdfstring{$\twonorm{ \tilde{\bth}(\rb) - \bth^*(\rb)}$}{TEXT}} 
\begin{lemma} \label{ell_2_diff}
Fix $k \in [n]$. There exist absolute constants $b, c, c'>0$ such that for any $v \geq 0$
\begin{equation}
\mathbb{P} \left( \twonorm{ \tilde{\bth}_k(\rb) - \bth_k^*(\rb) } \geq v\frac{\left(\log (d) \right)^{b}}{d} \right) 
\leq c  \exp\left( -\frac{ v^{c'}}{c} \right) + c \exp\left(- (\log (d))^2 / c \right)
\end{equation}
\end{lemma}

\begin{proof}
To simplify notation, in this lemma we use $\bth^*$ instead of $\bth_k^*(\rb)$ and $\tilde{\bth}$ instead of $\tilde{\bth}_k(\rb)$. Also, we define
$$q(\bth) = \lambda\twonorm{\bth}^2 + \lambda_w \twonorm{\frac{1}{2}\bW^\sT\bth - \bbeta}^2 +  \lambda_s\frac{\log (d)}{d} (\ones^\sT \bth)^2$$

Due to the the fact that $R(\bth, \rb)$ is $\lambda$-strongly convex, we can write
\begin{equation}
\twonorm{\bth^* - \tilde{\bth}} \leq \frac{1}{\lambda} \twonorm{\nabla_\bth R(\tilde{\bth}, \rb) }
\end{equation}

Consequently, we will bound the right-hand-side in the above relation. By using the fact that $\bth^*_{-k}$ is the minimizer of the function $R_{-k}$, we can write
\begin{equation} \label{gradient_bth_tilde}
 \nabla_\bth R(\tilde{\bth}, \rb)  =  \nabla_\bth R(\tilde{\bth}, \rb)  -  \nabla_\bth R_{-k}(\bth_{-k}^*).
 \end{equation}

 From \eqref{relation_Rs} and \eqref{gradient_bth_tilde} we have 
 \begin{align*}
  \nabla_\bth R(\tilde{\bth}, \rb)  &= \nabla R_{-k}(\tilde{\bth}) + \frac{1}{n} \nabla \ell( \tilde{\bth}; \rb, y_k) - \nabla R_{-k}(\bth_{-k}^*)  \\
  & = \frac{1}{n} \sum_{t \neq k} ( \nabla \ell( \tilde{\bth}; \rb_t, y_t) -  \nabla \ell( \bth_{-k}^*; \rb_t, y_t) ) +  \nabla q(\tilde{\bth}) - \nabla q(\bth_{-k}^*)  + \frac{1}{n}\nabla \ell(\tilde{\bth};\rb, y_k), 
 \end{align*}

By using the fact that $\tilde{\bth}$ is the minimizer of \eqref{S_k}, we can write
  \begin{align*}
  \nabla_\bth R(\tilde{\bth}, \rb) 
  & = \frac{1}{n} \sum_{t \neq k} ( \nabla \ell( \tilde{\bth}; \rb_t, y_t) - \nabla \ell(\bth_{-k}^*;\rb_t, y_t) )  - \bH_{-k} (\tilde{\bth} - \bth_{-k}^*) + \nabla q(\tilde{\bth}) - \nabla q(\bth_{-k}^*) 
 \end{align*}
 Now, from \eqref{H_def}, we obtain 
  \begin{align}    \label{bound_gradient_f_part}
  \nabla_\bth R(\tilde{\bth}, \rb) =
 \frac{1}{n}  \sum_{t \neq k}  \nabla \ell(\tilde{\bth}; \rb_t, y_t) - \nabla \ell( \bth_{-k}^*; \rb_t, y_t)  - \nabla^2 \ell( & \bth_{-k}^*;\rb_t, y_t)(\tilde{\bth} - \bth_{-k}^*),    
 \end{align}
where in the above we have used the fact that, since $q$ is a quadratic function, we have $\nabla q(\tilde{\bth}) - \nabla q(\bth_{-k}^*) - \nabla^2 q(\bth_{-k}^*) (\tilde{\bth} - \bth_{-k}^*) = 0$. We will now analyze each of the terms above.  We first bound the first term (i.e. the sum involving the derivatives of $\ell$). We will use the following simple relations for any choice of $\rb, y$:

\begin{align*}
 \tilde{\bth}^\sT \rb - y &=  {\bth_{-k}^*}^\sT \rb - y + ( \tilde{\bth} - \bth_{-k}^*)^\sT \rb, \\
 \twonorm{J \tilde{\bth}}^2 &= \twonorm{\bJ \bth_{-k}^*}^2 +  2{\bth_{-k}^*}^\sT \bJ^\sT \bJ ( \tilde{\bth} -\bth_{-k}^*) + ( \tilde{\bth} -\bth_{-k}^*)^\sT \bJ^\sT \bJ ( \tilde{\bth} - \bth_{-k}^*) ,\\
 (\tilde{\bth}^\sT \rb - y)^2 &=  ({\bth_{-k}^*}^\sT \rb - y)^2 + 2 ({\bth_{-k}^*}^\sT \rb - y) \rb^\sT ( \tilde{\bth} - \bth_{-k}^*) + ( \tilde{\bth} -\bth_{-k}^*)^\sT \rb \rb^\sT ( \tilde{\bth} - \bth_{-k}^*),
\end{align*}
and 
\begin{align*}
 h(\tilde{\bth}; \rb, y)  - h(\bth_{-k}^*; \rb, y ) & = (\tilde{\bth}^\sT \rb - y)^2  \twonorm{J \tilde{\bth}}^2  -   ({\bth_{-k}^*}^\sT \rb - y)^2 \twonorm{\bJ \bth_{-k}^*}^2 \\
 & = \nabla h(\bth_{-k}^*; \rb, y )^\sT ( \tilde{\bth} - \bth_{-k}^*) + e(\rb, y)\\
 & =  2( \twonorm{\bJ \bth_{-k}^*}^2 \rb^\sT + ({\bth_{-k}^*}^\sT \rb - y)^2 \bJ^\sT \bJ {\bth_{-k}^*}^\sT ) ( \tilde{\bth} - \bth_{-k}^*) + e(\rb, y),
\end{align*}
where the error term $e(\rb, y)$ can be written as 
$$e(\rb, y) = ( \tilde{\bth} - \bth_{-k}^*)^\sT \nabla_\bth^2 h(\bth; \rb, y)  \mid_{\bth = \bth(\rb,y)} ( \tilde{\bth} - \bth_{-k}^*),$$
and $\bth(\rb,y) =  \zeta \bth_{-k}^*  + (1-\zeta) \tilde{\bth}  $ for some $\zeta \in [0,1]$ which depends on $\rb$ and $y$. 

We will also use  the Taylor expansion:
\begin{align*}
g'\left(  h(\tilde{\bth}; \rb, y_t ) \right) = 
g'\left( h(\bth_{-k}^*; \rb_t, y_t )  \right) &
+ g''\left( h(\bth_{-k}^*; \rb_t, y_t ) \right) (h(\tilde{\bth}; \rb_t, y_t)  - h(\bth_{-k}^*; \rb_t, y_t )) \\
&+ \frac 12 g'''(v_t) \left( h(\tilde{\bth}; \rb_t, y_t)  - h(\bth_{-k}^*; \rb_t, y_t ) \right)^2, 
\end{align*}
where $v_t$ is a number between $h(\tilde{\bth}; \rb_t, y_t) $ and $h(\bth_{-k}^*; \rb_t, y_t )$.

From \eqref{grad_ell} and \eqref{hessian_ell} we will decompose:
\begin{align*}
 \nabla \ell(\tilde{\bth}; \rb_t, y_t) - \nabla \ell( \bth_{-k}^*; \rb_t, y_t)  - \nabla^2 \ell( & \bth_{-k}^*;\rb_t, y_t)(\tilde{\bth} - \bth_{-k}^*)  = \text{Term}_1 +  \text{Term}_2 -   \text{Term}_3,
\end{align*}
where the terms are given in \eqref{diff_theta_term_1}, \eqref{diff_theta_term_2}, and \eqref{diff_theta_term_3}. We will bound each of these terms in the following. We have
\begin{align}\label{diff_theta_term_1}
& \text{Term}_1 = 2  \left( \nabla h(\tilde{\bth}; \rb_t, y_t) -  \nabla h(\bth_{-k}^*; \rb_t, y_t)  \right) g'( h(\bth_{-k}^*; \rb_t, y_t ) )  
\end{align}
where 
\begin{align*}
&\nabla h(\tilde{\bth}; \rb_t, y_t) -  \nabla h(\bth_{-k}^*; \rb_t, y_t)    \\
&=  2\left( (\tilde{\bth}^\sT \rb_t - y_t) \twonorm{\bJ \tilde{\bth}}^2 \rb_t + (\tilde{\bth}^\sT \rb_t - y_t)^2  \bJ^\sT \bJ \tilde{\bth} - ({\bth_{-k}^*}^\sT \rb_t - y_t) \twonorm{\bJ \bth_{-k}^*}^2 \rb_t + ({\bth_{-k}^*}^\sT \rb_t - y_t)^2  \bJ^\sT \bJ \bth_{-k}^*  \right)
\end{align*}
And thus,
\begin{align}
&\nabla h(\tilde{\bth}; \rb_t, y_t) -  \nabla h(\bth_{-k}^*; \rb_t, y_t)  \label{diff_gradients}   \\
&=2 \left( (\tilde{\bth} - \bth_{-k}^*  )^\sT \rb_t \twonorm{\bJ \bth_{-k}^* }^2 + (\tilde{\bth}^\sT \rb_t - y_t) \left(2 {\bth_{-k}^*}^\sT  \bJ^\sT \bJ   (\tilde{\bth} - \bth_{-k}^*  )    + (\tilde{\bth} - \bth_{-k}^*  )^\sT \bJ^\sT \bJ   (\tilde{\bth} - \bth_{-k}^*  )   \right)   \right)  \rb_t  \nonumber \\
&\,\,\,+  2\left((\tilde{\bth}^\sT \rb_t - y)^2 \bJ^\sT \bJ (\tilde{\bth} - \bth_{-k}^*  ) + \left(2 (\tilde{\bth}^\sT \rb_t- y_t) \rb^\sT  (\tilde{\bth} - \bth_{-k}^*  ) + (\tilde{\bth} - \bth_{-k}^*  )^\sT \rb_t\rb_t^\sT  (\tilde{\bth} - \bth_{-k}^*  ) \nonumber\right) \bJ^\sT  \bJ \tilde{\bth}  \right)
\end{align}
We also have
\begin{align} \label{diff_theta_term_2}
 \text{Term}_2 & =   2\nabla h(\tilde{\bth}; \rb_t, y_t)   g''\left( h(\bth_{-k}^*; \rb_t, y_t )\right) \left(h(\tilde{\bth}; \rb_t, y_t)  - h(\bth_{-k}^*; \rb_t, y_t ) \right)  \\
 & = 2\nabla h(\tilde{\bth}; \rb_t, y_t) g''\left( h(\bth_{-k}^*; \rb_t, y_t ) \right) \left( \nabla  h(\bth_{-k}^*; \rb_t, y_t )^\sT ( \tilde{\bth} - \bth_{-k}^*) + e(\rb_t, y_t) \right), \nonumber
\end{align}
and
\begin{align} \nonumber 
 \text{Term}_3  = & 2\left(\nabla^2 h(\bth_{-k}^*; \rb_t, y_t)   g'\left(h(\bth_{-k}^*; \rb_t, y_t)\right)  
 + \nabla h(\bth_{-k}^*; \rb_t, y_t) (\nabla h(\bth_{-k}^*; \rb_t, y_t))^\sT   g''\left(h(\bth_{-k}^*; \rb_t, y_t ) \right) \right)(\tilde{\bth} - \bth_{-k}^*)\\ 
 &\!-\frac 12 g'''(v_t) \left( h(\tilde{\bth}; \rb_t, y_t)  - h(\bth_{-k}^*; \rb_t, y_t ) \right)^2, \label{diff_theta_term_3}
\end{align}
After some straight-forward steps, we can write
\begin{align} \label{three_terms_together}
  &\text{Term}_1 +  \text{Term}_2 -   \text{Term}_3 \\
 &=  4 g'\left( h(\bth_{-k}^*; \rb_t, y_t ) \right) \biggl(2 \rb_t^\sT(\tilde{\bth} - \bth_{-k}^*  ) {\bth_{-k}^*}^\sT \bJ^\sT \bJ (\tilde{\bth} - \bth_{-k}^*  ) \rb_t    \nonumber\\
 &  \quad\quad\quad \quad\quad\quad\quad\quad\quad +  2\rb_t^\sT(\tilde{\bth} - \bth_{-k}^*  ) \left( ({\bth_{-k}^*}^\sT \rb_t - y_t) \bJ^\sT \bJ (\tilde{\bth} - \bth_{-k}^*  ) + \rb_t^\sT  (\tilde{\bth} - \bth_{-k}^*  ) \bJ^\sT \bJ  \tilde{\bth}   \right)   \nonumber\\
&   \quad\quad\quad \quad\quad\quad\quad\quad\quad+ \left(2 ({\bth_{-k}^*}^\sT \rb_t - y_t) \rb_t^\sT ( \tilde{\bth} - \bth_{-k}^*) + ( \tilde{\bth} -\bth_{-k}^*)^\sT \rb_t \rb_t^\sT ( \tilde{\bth} - \bth_{-k}^*) ) \bJ^\sT \bJ (\tilde{\bth} - \bth_{-k}^*)  \right)  \nonumber\\
& \quad\quad\quad \quad\quad\quad\quad\quad\quad + (\rb_t^\sT(\tilde{\bth} - \bth_{-k}^*))^2 \bJ^\sT \bJ \tilde{\bth} + \twonorm{\bJ(\tilde{\bth} - \bth_{-k}^*)}^2 \rb_t    \biggr)  \nonumber\\
  &+    2g''\left(h(\bth_{-k}^*; \rb_t, y_t) \right)   \biggl( \left(\nabla  h(\tilde{\bth}; \rb_t, y_t )  - \nabla  h(\bth_{-k}^*; \rb_t, y_t )\right) (\tilde{\bth} - \bth_{-k}^*)^\sT \nabla  h(\bth_{-k}^*; \rb_t, y_t )  +   \nabla h(\tilde{\bth}; \rb_t, y_t) e(\rb_t, y_t)  \biggr) \nonumber\\
  & \,\, -\frac 12 g'''(v_t)  \bigl( h(\tilde{\bth}; \rb_t, y_t) \, - h(\bth_{-k}^*; \rb_t, y_t ) \bigr)^2.  \nonumber
\end{align}
The relation \eqref{three_terms_together} has itself three different terms. We will now simplify and bound each of the terms above. However, we remark that all the three terms will be bounded in a similar way. Let's consider the first term in the right-hand-side of \eqref{three_terms_together}.  The first part of this term is: 
$$ 4 g'\left( h(\bth_{-k}^*; \rb_t, y_t ) \right) \times 2 \rb_t^\sT(\tilde{\bth} - \bth_{-k}^*  ) {\bth_{-k}^*}^\sT \bJ^\sT \bJ (\tilde{\bth} - \bth_{-k}^*  ) \rb_t,  $$
which can be rewritten as 
$$ 8 g'\left( h(\bth_{-k}^*; \rb_t, y_t \right)  {\bth_{-k}^*}^\sT \bJ^\sT \bJ (\tilde{\bth} - \bth_{-k}^*  ) \rb_t \rb_t^\sT (\tilde{\bth} - \bth_{-k}^*  ). $$
Now, by using the fact that the first derivatives of the function $g$ is uniformly bounded, and by some straight-forward usages of the Cauchy-Schwartz inequality, we can easily rewrite the above part as 
$$ \alpha_{1,t} \pb^\sT \,   (\tilde{\bth} - \bth_{-k}^*  ) \rb_t \rb_t^\sT (\tilde{\bth} - \bth_{-k}^*  ),$$
where the vector $\pb = {\bth_{-k}^*}^\sT \bJ^\sT \bJ$ \emph{does not depend} on $t$, and $\alpha_{1,t}$ is a constant. We can further write:  
$$ | \alpha_{1,t}| \leq C \text{ and } \twonorm{\pb} \leq \twonorm{J}^2 \twonorm{ \bth_{-k}^*} \leq  \twonorm{\bJ}^4 + \twonorm{ \bth_{-k}^*}^2, $$
where $C$ is an absolute constant.  

Let us now consider the second part of the first term in \eqref{three_terms_together}, which is: 
$$ 4 g'\left( h(\bth_{-k}^*; \rb_t, y_t ) \right)  \times 2\rb_t^\sT(\tilde{\bth} - \bth_{-k}^*  )  ({\bth_{-k}^*}^\sT \rb_t - y_t) \bJ^\sT \bJ (\tilde{\bth} - \bth_{-k}^*  ). $$
We can rewrite this part as 
$$ 8 g'\left( h(\bth_{-k}^*; \rb_t, y_t ) \right)  \times  ({\bth_{-k}^*}^\sT \rb_t - y_t) \bJ^\sT \bJ  (\tilde{\bth} - \bth_{-k}^*  )^\sT    \rb_t (\tilde{\bth} - \bth_{-k}^*  ) =  \alpha_{2,t} \A \, (\tilde{\bth} - \bth_{-k}^*  )^\sT    \rb_t (\tilde{\bth} - \bth_{-k}^*  ),   $$
where, the matrix $\A$ is the same for all  $t$, and 
$$ |\alpha_{2,t}| \leq C |{\bth_{-k}^*}^\sT \rb_t - y_t | \text{ and } \twonorm{\A} \leq \twonorm{\bJ}^2.$$
In a similar way, one can inspect all the parts of the first term in \eqref{three_terms_together} and show that they take the form of one of the following: 
\begin{align} \label{forms1} \nonumber
& \bullet \alpha_{1,t}\, \pb_1^\sT \,   (\tilde{\bth} - \bth_{-k}^*  ) \rb_t \rb_t^\sT (\tilde{\bth} - \bth_{-k}^*  ), \\
\nonumber
& \bullet \alpha_{2,t}\, \A_1 \, (\tilde{\bth} - \bth_{-k}^*  )^\sT    \rb_t (\tilde{\bth} - \bth_{-k}^*  ), \\
\nonumber
& \bullet \alpha_{3,t}\, \pb_2 (\tilde{\bth} - \bth_{-k}^*  )^\sT \rb_t \rb_t^\sT (\tilde{\bth} - \bth_{-k}^*  ), \\
\nonumber
& \bullet \alpha_{4,t}\, (\tilde{\bth} - \bth_{-k}^*  )^\sT \A_2 (\tilde{\bth} - \bth_{-k}^*  )   \rb_t ,\\
& \bullet \alpha_{5,t}\, ((\tilde{\bth} - \bth_{-k}^*  )^\sT \rb_t)^2 \pb_3,
\end{align}
where \rev{$\pb_1, \pb_2, \pb_3, \A_1,\A_2$} do not depend on $t$, and we have
\begin{equation}
|\alpha_{j,t}| \leq C(1 + ({\bth_{-k}^*}^\sT \rb_t - y)^2) \text{ and } \max\{|\gamma|, \twonorm{\pb_1}, \twonorm{\pb_2},\twonorm{\pb_3}, \twonorm{\A_1}, \twonorm{\A_2} \} \leq 
C(1 + \twonorm{\bJ}^2). 
\end{equation}
We will now consider the second term in the right-hand-side of \eqref{three_terms_together} which can be expanded using  \eqref{gradient_h} and \eqref{diff_gradients}. Again, one can inspect all the parts and show that they take  one of the forms in the following (in addition to the forms presented in \eqref{forms1}): 
\begin{align} \label{forms2} 
& \bullet \alpha_{3,t} \,   (\tilde{\bth} - \bth_{-k}^*  ) \rb_t \rb_t^\sT  (\tilde{\bth} - \bth_{-k}^*  )  (\tilde{\bth} - \bth_{-k}^*  )^\sT \, \pb_3, \\
\nonumber
& \bullet \alpha_{4,t} \,  (\rb_t^\sT (\tilde{\bth} - \bth_{-k}^*  ))^2 \rb_t, \\
\nonumber
& \bullet  \alpha_{5,t}\, \pb_4^\sT \,   (\tilde{\bth} - \bth_{-k}^*  ) \rb_t \rb_t^\sT (\tilde{\bth} - \bth_{-k}^*  ),
\end{align}
where, for some positive constant $C$ and even integer $D$ we have
\begin{equation} \label{bounds_on_unimportants}
|\alpha_{j,t}| \leq C(1 + ({\bth_{-k}^*}^\sT \rb - y)^D + (\tilde{{\bth}}^\sT \rb - y)^D) \text{ and } \max \{ \twonorm{\pb_3}, \twonorm{\pb_4} \} \leq 
C(1 + \twonorm{\bJ}^D + \twonorm{\bth_{-k}^*}^D). 
\end{equation}

A similar bounding can be done for the third term in \eqref{three_terms_together}.  

We now claim that the sum of each of the terms in \eqref{forms1} and \eqref{forms2} over $t$ is at most  of order $O(\text{polylog }(n) / n )$. Let's consider the first term in \eqref{forms1}. We can write
\begin{align*} 
\frac{1}{n}\twonorm{ \sum_t  \alpha_{1,t}\, \pb_1^\sT \,   (\tilde{\bth} - \bth_{-k}^*  ) \rb_t \rb_t^\sT (\tilde{\bth} - \bth_{-k}^*  ) } & \leq \frac{1}{n} \sup_t\{|\alpha_{1,t}| \} \left| \pb_1^\sT (\tilde{\bth} - \bth_{-k}^*) \right|   \twonorm{ \sum_{t \neq k} \rb_t \rb_t^\sT} \twonorm{\tilde{\bth} - \bth_{-k}^*} \\
& = \sup_t\{|\alpha_{1,t}| \} \left|  \pb_1^\sT (\tilde{\bth} - \bth_{-k}^*) \right|   \twonorm{ \frac{1}{n}\sum_{t \neq k} \rb_t \rb_t^\sT} \twonorm{\tilde{\bth} - \bth_{-k}^*} 
\end{align*}
Now, from Lemma~\ref{all_the_bounds} it should be clear why the above quantity is small. Informally, and neglecting the polylog factors,  the lemma asserts that with high probability the terms $\sup_t\{|\alpha_{1,t}|\}$, and $\twonorm{ \frac{1}{n}\sum_{t \neq k} \rb_t \rb_t^\sT}$ are all $O(1)$; but the term $\left|  \pb_1^\sT (\tilde{\bth} - \bth_{-k}^*) \right|$, is $O(1/n)$ as $\p_1$ is a fixed vector (independent of $t$), and $\twonorm{\tilde{\bth} - \bth_{-k}^*} $ is $O(n^{-\frac12})$. As a result, the whole expression is $O(n^{-\frac32})$. Formally, it is easy to conclude from Lemma~\ref{all_the_bounds} that for a given $k \in [d]$:
\begin{align} \label{probabilistic_diff_thetas}
\mathbb{P}\left( \frac{1}{n}\twonorm{ \sum_t  \alpha_{1,t}\, \pb_1^\sT \,   (\tilde{\bth} - \bth_{-k}^*  ) \rb \rb^\sT (\tilde{\bth} - \bth_{-k}^*  ) } \geq v \frac{(\log (d))^b}{d^{\frac 32}}\right) \leq c \exp(- v^{c'}/c) + c \exp\left(- (\log(d))^2/c\right),
\end{align}
for some absolute constants $c,c', c'' > 0$.  

Let's now consider the second term in \eqref{forms1}. We can write
\begin{align*} 
 \twonorm{ \frac{1}{n} \sum_t  \alpha_{2,t}\, \A_1 \, (\tilde{\bth} - \bth_{-k}^*  )^\sT    \rb_t (\tilde{\bth} - \bth_{-k}^*  )} 
& =\twonorm{\A_1} \twonorm{\tilde{\bth} - \bth_{-k}^*}  \twonorm{ (\tilde{\bth} - \bth_{-k}^*)^\sT \frac{1}{n}  \sum_{t \neq k}  \alpha_{2,t} \rb_t} \\
& =\twonorm{\A_1} \twonorm{\tilde{\bth} - \bth_{-k}^*}  \twonorm{ (\tilde{\bth} - \bth_{-k}^*)^\sT  \times \frac{1}{n}\sum_{t \neq k}  \alpha_{2,t} \rb_t}. 
\end{align*}
Now, note from the first part of Lemma~\ref{all_the_bounds} that the norm of the vector $\frac{1}{n}\sum_{t \neq k}  \alpha_{2,t} \rb_t^\sT$  is w.h.p. $O(1)$. Also, this vector is independent from $\rb$, and hence from the second part of Lemma~\ref{all_the_bounds} we obtain that w.h.p. $\twonorm{ (\tilde{\bth} - \bth_{-k}^*)^\sT  \times \frac{1}{n}\sum_{t \neq k}  \alpha_{2,t} \rb_t} $ is $O(n^{-\frac12})$. Consequently, by noting that $\twonorm{\tilde{\bth} - \bth_{-k}^*}$ is w.h.p $O(n^{-\frac12})$ we obtain that the whole expression is w.h.p. $O(n^{-\frac32})$. The formal expression would be like the probabilistic expression given in \eqref{probabilistic_diff_thetas}. 

Similarly, for the term $ \alpha_{4,t} \,  (\rb_t^\sT (\tilde{\bth} - \bth_{-k}^*  ))^2 \rb_t$ we can write
\begin{align} \label{important_term_diff_thetas}
\twonorm{\frac{1}{n} \sum_{t \neq k}  \alpha_{4,t} \,  (\rb_t^\sT (\tilde{\bth} - \bth_{-k}^*  ))^2 \rb_t} \leq 
\sup_{t} \left\{ (\rb_t^\sT (\tilde{\bth} - \bth_{-k}^*  ))^2 \right\} \twonorm{\frac{1}{n} \sum_{t \neq k}  \alpha_{4,t} \, \rb_t}.  
\end{align}
Now, by part (e) of Lemma~\ref{all_the_bounds} we can easily conclude that 
$$\mathbb{P}\left(  \sup_{t} \left\{ (\rb_t^\sT (\tilde{\bth} - \bth_{-k}^*  ))^2 \right\}  \geq (\log (d) )^b \right) \leq c \exp\left(- (\log(d))^2/c\right), $$
for absolute constants $b,c > 0$ that are suitably chosen. A similar conclusion can be made for $\sup_t \{|\alpha_{4,t}|\}$ from \eqref{bounds_on_unimportants} and part (g) of Lemma~\ref{all_the_bounds}--i.e.  
$$\mathbb{P}\left(  \sup_{t} \left\{  | \alpha_{4,t}|  \right\}  \geq (\log (d) )^b \right) \leq c \exp\left(- (\log(d))^2/c\right). $$
As a result, by using the above bounds, as well as \eqref{important_term_diff_thetas}, and part (a) of Lemma~\ref{all_the_bounds} we obtain
 \begin{align*}
\mathbb{P}\left( \twonorm{\frac{1}{n} \sum_{t \neq k}  \alpha_{4,t} \,  (\rb_t^\sT (\tilde{\bth} - \bth_{-k}^*  ))^2 \rb_t}  \geq v \frac{(\log(d))^b}{d}\right) \leq c \exp(- v^{c'}/c) + c \exp\left(- (\log(d))^2/c\right),
\end{align*}
In a similar way as the above, we can show that the sum of all the terms in \eqref{forms1} and \eqref{forms2} over $t$ have similar bounds. As a result, going back to \eqref{bound_gradient_f_part}, we have shown that the sum can be bounded to give the desired result as in the lemma.  

\end{proof}
\begin{lemma} \label{all_the_bounds}
For some absolute constants $b, c,c' > 0$ we have:
\begin{itemize}
\item[(a)] $$\mathbb{P} \left(\twonorm{\frac{1}{n} \sum_{t \neq k} \rb_t \rb_t^\sT} \geq v \right) \leq c\exp(-v^{c'}/c).  $$ 
 
\item[(b)] Given any sequence of numbers $\{\alpha_t\}_{t=1}^n$ such that $|\alpha_t| \leq 1$, we have
$$\mathbb{P}\left(\twonorm{\frac{1}{n} \sum_{t \neq k} \alpha_t \rb_t} \geq v \right) \leq ce^{-v^{c'}/c}.  $$
 
 \item[(c)] Given a vector $\ub$, with $\twonorm{\ub} = 1$, we can write
 $$ \mathbb{P} \left( | \ub^\sT \rb| \geq v  \right) \leq c e^{-v/c}. $$
 
\item[(d)] Given a vector $\ub$, with $\twonorm{\ub} = 1$, which is independent from 
 $\rb$, we have $$\mathbb{P}\left( \left| \ub^\sT (\tilde{\bth} - \bth_{-k}^*) \right| \geq \frac{v}{n}  \right) \leq ce^{- v/c}.$$ 
 
 \item[(e)] For $\rb$ generated according to either of the distributions in \eqref{rb_distrbutions}, we have
 $$\mathbb{P} \left( \twonorm{\rb} \geq v \sqrt{n} \right) \leq c \exp(- v^2/c).  $$
 
 \item[(f)] We further have
  $$\mathbb{P}\left( \left| \rb_t^\sT (\tilde{\bth} - \bth_{-k}^*) \right| \geq \frac{v}{\sqrt{n}}  \right) \leq ce^{- v/c}.$$ 
 
\item[(g)] For a given even integer $D>0$ we have: 
$$\mathbb{P}\left(\max\left\{{(\bth_{-k}^*}^\sT \rb - y)^D , (\tilde{{\bth}}^\sT \rb - y)^D , \twonorm{\bth_{-k}^*}^D \right\} \geq v \right) \leq ce^{-v^{c'}/c}.$$ 
As a corollary, we have
$$\mathbb{P}\left(\max\left\{{(\bth_{-k}^*}^\sT \rb - y)^D , (\tilde{{\bth}}^\sT \rb - y)^D , \twonorm{\bth_{-k}^*}^D \right\} \geq (\log (d))^b \right) \leq c\exp \left\{-(\log (d))^2/c \right\}.$$

 \end{itemize}
\end{lemma}
Proof of this lemma is provided in Section~\ref{leave-one-out-auxilliary-lemmas}.

\subsubsection{Bounding \texorpdfstring{$|\rb^\sT \tilde{\bth}(\rb)|$}{TEXT}}
\begin{lemma} \label{bound_product_r_bth}
There exist absolute constants $c, c'> 0$ such that for every $k \in [n]$ and $v \geq 0$ we have
\begin{equation}
\mathbb{P}\left(| \rb^\sT \tilde{\bth}(\rb) | \geq v  \right) \leq c \exp(-v^{c'}/c),
\end{equation}
and 
\begin{equation}
\mathbb{P}\left(| \rb^\sT \bth^*_{k}(\rb) | \geq v \right) \leq cd \exp(-v^{c'}/c) + c \exp(-(\log(d))^2/c).
\end{equation}
\end{lemma}
\begin{proof}
This lemma can also be proven similarly to \cite{hu2020universality} (see Section D.3 in \cite{hu2020universality}). There are, however, some small differences that we will mention the details here. 

For the first part, the proof proceeds in two steps. In this first step, we use Lemma~\ref{properties_of_bth_star} to obtain
\begin{equation} \label{r_bth_star}
\mathbb{P}\left(| \rb^\sT \bth_{-k}^* | \geq v \right) \leq c \exp(-v/c).
\end{equation}

We will now bound the term $| \rb^\sT (\tilde{\bth}(\rb) - \bth_{-k}^*(\rb) )|$. We can write:
\begin{align*}
\mathbb{P}\left(| \rb^\sT (\tilde{\bth}(\rb) - \bth_{-k}^* )| \geq v  \right) & = \mathbb{P}\left(| \frac{1}{\sqrt{d}}\rb^\sT \times \sqrt{d}(\tilde{\bth}(\rb) - \bth_{-k}^*(\rb) )| \geq v  \right)\\
& \leq \mathbb{P}\left( \twonorm{\tilde{\bth}(\rb) - \bth^*_{-k} } \geq \sqrt{\frac{v}{d }}  \right) +  \mathbb{P}\left( \twonorm{\rb } \geq \sqrt{v d}  \right)  
\end{align*}
Now, the first term above can be bounded using Lemma~\ref{ell_2_diff_first}, and the second can be bounded  from Lemma~\ref{all_the_bounds}, and thus the proof of the lemma follows. 

The proof of the second part is similar to the first part (and we use Lemma~\ref{ell_2_diff}).
\end{proof}

\subsubsection{Bounding \texorpdfstring{$\infnorm{\bth_{-k}^*}$}{TEXT}} 
\begin{lemma} \label{th_inf_lemma}
Let $\bth_{-k}^*$ be the minimizer of $R_{-k}$ defined in \eqref{R_-k}. 
There exist absolute constants $c_0,c_1,c_\infty > 0$ such that for any $k\in[n]$:
 \begin{equation} 
 \mathbb{P}\left\{ \infnorm{\bth_{-k}^*} \geq c_\infty \sqrt{\frac{\log(d)}{d}} \right\} \leq  5d^{-c_0}+ 3e^{-c_1n}
 \end{equation}
\end{lemma}
\begin{proof}
For convenience we remind the definition of $\bth_{-k}^*$ given by \rev{$\bth_{-k}^* = \arg\min_{\bth} R_{-k}(\bth)$} where
\begin{equation} \label{eq:Rk-0}
R_{-k}(\bth) = \frac{1}{n}\sum_{i=1}^{k-1} \ell(\bth; \ab_i, y_i) +  \frac{1}{n} \sum_{i=k+1}^n \ell(\bth; \bb_i, y_i)  +\lambda \twonorm{\bth}^2 + \lambda_w \twonorm{\frac12 \W^\sT \bth - \bbeta}^2 + \lambda_s \frac{\sqrt{\log (d)}}{d} (\ones^\sT \bth)^2,
\end{equation}
and
\[\ell(\bth; \rb,  y) = (\bth^\sT \rb - y)^2 + \twonorm{\bJ \bth}^2 + 2g\left((\bth^\sT \rb - y)^2 \twonorm{\bJ \bth}^2 \right)\,,\]
with $g(x) = \sqrt{x+\gamma}$.

We use a similar strategy as in the proof of Proposition \ref{propo:Cth-L0}. Specifically we first bound the last coordinate of $\bth^*_{-k}$. Next, by symmetry we conclude that the same bound holds for all of its coordinates and control its $\ell_\infty$ norm by union bounding.  

With a slight abuse of notation, we consider a $(N+1)$ dimensional version of the above optimization over $[\bth; u]$ and denote the last coordinate of the optimal solution by $\hat{u}$.  We define 
\[
\br_i = \begin{cases}
\sigma(\W \x_i)   & \text{ if }i\le k-1,\\
\boldsymbol{0} &\text{ if } i = k,\\
\mbf_i &\text{ if }k\le i\le n\,.
\end{cases}
\]
Also define $\be = [a_1, \dotsc, a_{k-1}, b_{k+1}, \dotsc, b_n]$ with $a_i = \sigma(\bw_{N+1}^\sT \bx_i)$ and $b_i = \mu_1 \bw_{N+1}^\sT \bx_i + \mu_2 z_i$ (with $z_i\sim\normal(0,1)$ independent of $\bx_i$).
From \eqref{eq:Rk-0}, $\hat{u}$ can be expressed as
\begin{align*} 
\hat{u} = \arg\min_u\; \min_{\bth} R_{-k}([\bth;u])
\end{align*}
where 
\begin{align}
R_{-k}([\bth;u]) = &\frac{1}{n}\sum_{i=1}^{n}  (\bth^\sT \rb_i + u e_i - y_i)^2 + \twonorm{\bJ \bth + u \bh}^2 + 2g\left((\bth^\sT \rb_i + u e_i - y_i)^2 \twonorm{\bJ \bth + u\bh}^2 \right) \nonumber\\
&+\lambda \twonorm{\bth}^2 +\lambda u^2 + \lambda_w \twonorm{\frac12 \W^\sT \bth + \bw_{N+1} u - \bbeta}^2 + \lambda_s \frac{\sqrt{\log (d)}}{d} (\ones^\sT \bth + u)^2,
\end{align}
Here $\bh\in \reals^{N+1}$ is the last column of the $(N+1)$-dimensional matrix $\bJ$. Namely,
\begin{align}\label{eq:def-h}
\bh:= \begin{bmatrix}(\bW\bw_{N+1}) \odot \Big(\frac{\pi-\cos^{-1}(\bW\bw_{N+1})}{2\pi}\Big)\\\frac{1}{2}\end{bmatrix}\,.
\end{align}

Let $f(u)$ denote the objective function of $u$ above, i.e., $f(u) = \min_{\bth} R_{-k}([\bth;u])$. We also let $\bth_*$ be the minimizing $\bth$ in this objective if we set $u=0$, i.e., $\bth_*= \min_{\bth} R_{-k}([\bth;0])$.

Following a similar argument in the proof of Proposition~\ref{propo:Cth-L0}, we can obtain a lower bound on $f(u)$ by considering a second-order Taylor expansion of $f(u)$ around $[\bth_*, 0]$ and using the strong-convexity of the loss function to arrive at the following upper bound on $\hat{u}$:
\begin{align}\label{eq:uhat-00}
|\hat{u}|\le \frac{1}{\lambda+\lambda_w} \bigg|\nabla_u R_{-k}([\bth;u])|_{[\bth_*;0]}\bigg|\,.
\end{align}
Calculating $\nabla_u R_{-k}([\bth;u])|_{[\bth_*;0]}$ we have
\begin{align}
\nabla_u R_{-k}([\bth;u])|_{[\bth_*;0]} = & \frac{1}{n}\sum_{i=1}^{n} \bigg[ 2(\bth_*^\sT \rb_i - y_i) e_i + 2 \bh^\sT \bJ \bth_*\nonumber\\
& \quad\quad + \frac{2}{\sqrt{(\bth_*^\sT \rb_i  - y_i)^2 \twonorm{\bJ \bth_* }^2 +\gamma}} \left((\bth_*^\sT \rb_i - y_i) \twonorm{\bJ \bth_*}^2 e_i +(\bth_*^\sT \rb_i - y_i)^2 \bh^\sT\bJ\bth_* \right)\bigg]\nonumber\\
&+ 2\lambda_w \bw_{N+1}^\sT \Big(\frac{1}{2}\bW^\sT\bth_*-\bbeta\Big) + 2\lambda_s \frac{\sqrt{\log (d)}}{d} \ones^\sT \bth_*\,.\label{eq:Rk-01}
\end{align}
We treat each of these terms separately.

Define the event
\[
\event: = \left\{\frac{1}{n}\sum_{i=1}^n y_i^2<C , \; \frac{1}{\sqrt{d}} \opnorm{\bX}\le C,\; \opnorm{\bW}\le C \right\}\,.
\]
Using the concentration bounds \rev{for the} operator norm of Gaussian matrices and also the tail bound for chi-square random variables, we have that $\prob(\event)\ge 1- 3e^{-cn}$ for some constant $c>0$.

Note that by optimality of $\bth_*$, we have $R_{-k}([\bth_*;0])\le R_{-k}([\boldsymbol{0};0]) = \frac{1}{n}\sum_{i=1}^n y_i^2$ from which we get $\twonorm{\bth_*} \le C$ and $\twonorm{\tfrac{1}{2}\bW^\sT\bth_*-\bbeta}\le C$, on the event $\event$.
Therefore,
\begin{align}\label{eq:term-inf-3}
2\lambda_s \frac{\sqrt{\log (d)}}{d}  |\ones^\sT \bth_*| \le 2\lambda_s \frac{\sqrt{\log (d)}}{d} \twonorm{\bth_*}\twonorm{\ones}
\le C \sqrt{\frac{\log (d)}{d}}\,.
\end{align} 
Also note that $\bw_{N+1}$ is drawn independently from $\bW$, $\bth_*$ and $\bbeta$. 
Since $\bw_{N+1}\sim \Unif(\mathbb{S}^{d-1})$, given $\tfrac{1}{2}\bW^\sT\bth_*-\bbeta$ the conditional distribution of $\bw_{N+1}^\sT(\tfrac{1}{2}\bW^\sT\bth_*-\bbeta)$ converges to $\normal(0, \frac{1}{d}\twonorm{\tfrac{1}{2}\bW^\sT\bth_*-\bbeta}^2)$ form which we obtain
\begin{align}\label{eq:term-inf-4}
\Big|\bw_{N+1}^\sT \Big(\frac{1}{2}\bW^\sT\bth_*-\bbeta\Big) \Big| \le \sqrt{2c' \frac{\log (d)}{d}} \twonorm{\tfrac{1}{2}\bW^\sT\bth_*-\bbeta} < C\sqrt{2c'\frac{\log (d)}{d}}\,,
\end{align}
with probability at least $1 - d^{-c'}$.

We next focus on the terms in the right-hand side of~\eqref{eq:Rk-01}, which involve $e_i$. This part can be written as $\frac{1}{\sqrt{n}} \sum_{i=1}^n m_i e_i$ with
\[
m_i = \frac{1}{\sqrt{n}} \left[  2(\bth_*^\sT \rb_i - y_i) + \frac{2}{\sqrt{(\bth_*^\sT \rb_i  - y_i)^2 \twonorm{\bJ \bth_* }^2 +\gamma}} (\bth_*^\sT \rb_i - y_i) \twonorm{\bJ \bth_*}^2\right]\,.
\]
Note that 
\begin{align*}
|m_i|\le \frac{2}{\sqrt{n}} \left(|\bth_*^\sT\br_i - y_i|+ \twonorm{\bJ\bth_*} \right)\,.
\end{align*}
By optimality of $\bth_*$, on the event $\event$ we have 
\[
\twonorm{\boldsymbol{m}}^2 \le \R_{-k}([\bth_*;0])\le \R_{-k}([\boldsymbol{0};0]) = \frac{1}{n}\sum_{i=1}^n y_i^2<C\,.
\]

Observe that $\bw_{N+1}$ is independent from $\{m_i\}_{i\in[n]}$ (recall that $\bth_*$ does not depend on $\bw_{N+1}$ by its definition.) Following the same strategy in the proof of Proposition~\ref{propo:Cth-L0}, we only consider the randomness in $\bw_{N+1}$ and condition on everything else.  Write $\frac{1}{\sqrt{n}} \sum_{i=1}^n m_i e_i$ as a function of $\bw_{N+1}$ as follows:
\begin{align}
V(\bw_{N+1}): = \frac{1}{\sqrt{n}} \sum_{i=1}^{k-1} m_i \sigma(\bw_{N+1}^\sT \bx_i) + \frac{1}{\sqrt{n}} \sum_{i=k+1}^{n} m_i 
(\mu_1 \bw_{N+1}^\sT \bx_i + \mu_2 z_i)\,.
\end{align}
Observe that the (conditional) expectation $\E[V(\bw_{N+1})| \bW,\bX] = 0$. In addition, $V(\cdot)$ is a Lipschitz function with Lipschitz factor at most $\frac{C}{\sqrt{d}} \twonorm{\bX\boldsymbol{m}}$. Therefore, using the concentration bound for Lipschitz function on unit sphere (see e.g.~\cite[Theorem 5.1.4]{vershynin2018high}), we obtain
\begin{align*}
\prob\left(|V(\bw_{N+1})| \ge t \right) &\le  \prob\left(|V(\bw_{N+1})| \ge t ;\event\right) + \prob(\event^c)\nonumber\\
&\le 2e^{-c'd t^2} + 3e^{-cn}\,.
\end{align*}
Choosing $t = C\sqrt{\frac{\log (d)}{d}}$, we get
\begin{align}\label{eq:term-inf-5}
\prob\left(|V(\bw_{N+1})| \ge C\sqrt{\frac{\log (d)}{d}} \right) \le 2d^{-c'C^2} + 3e^{-cn}\,.
\end{align}
The remaining terms in~\eqref{eq:Rk-01} can be rearranged and written as $A \bh^\sT \bJ\bth_*$, with
\begin{align*}
A: = 2\left(1+ \frac{1}{n}\sum_{i=1}^n \frac{(\bth_*\br_i-y_i)^2}{\sqrt{(\bth_*^\sT \rb_i  - y_i)^2 \twonorm{\bJ \bth_* }^2 +\gamma}}\right)\,.
\end{align*}
We next bound $A>0$:
\begin{align*}
|A| &< 2\left(1+ \frac{1}{n}\sum_{i=1}^n \frac{|\bth_*\br_i-y_i|}{\twonorm{\bJ \bth_* }}\right) \\
&2\left(1+ \frac{1}{\twonorm{\bJ\bth_*}}\left(\frac{1}{n} \sum_{i=1}^n (\bth_*^\sT \br_i - y_i)^2 \right)^{1/2}\right)\\
&< 2\left(1+ \frac{1}{\twonorm{\bJ\bth_*}}\left(R_{-k}([\bth_*;0]) \right)^{1/2}\right)\\
&\le 2\left(1+ \frac{1}{\twonorm{\bJ\bth_*}}\left(R_{-k}([\boldsymbol{0};0]) \right)^{1/2}\right)\\
& = 2\left(1+ \frac{1}{\twonorm{\bJ\bth_*}}\left(\frac{1}{n}\sum_{i=1}^n y_i^2\right)^{1/2}\right)\,.
\end{align*} 
Using Lemma~\ref{pro:spectral}, on the event $\event$, the right-hand side of the above equation is of order one ($A<C$, for some constant $C>0$). 

We next bound $\bh^\sT \bJ\bth_*$. Using the relation $\frac{1}{2\pi}(\pi- \cos^{-1}(\rho)) = \frac{1}{4} + \frac{\rho}{2\pi}+O(\rho^3)$, we define $\tilde{\bh}$ as follows
\[
\tilde{\bh} : = \begin{bmatrix}\frac{1}{4} \bW\bw_{n+1}^\sT\\ \frac{1}{2} \end{bmatrix}\,.
\]
Recalling $\bh$ given by~\eqref{eq:def-h}, we have $\twonorm{\bh-\tilde{\bh}} = O(1/\sqrt{d})$. On the event $\event$, 
we have $\twonorm{\bth_*} = O(1)$. Also by invoking Lemma~\ref{pro:spectral}, on event $\event$, we have $\opnorm{\bJ} = O(1)$. Hence,
\begin{align}\label{eq:diff-h-k}
|\bh^\sT \bJ\bth_* - \tilde{\bh}^\sT \bJ\bth_*|\le O(1/\sqrt{d})\,.
\end{align}
We henceforth focus on bounding $\tilde{\bh}^\sT \bJ\bth_*$.
Recall that $\bw_{N+1}$ is independent of $\bW$ and $\bth_*$. Viewing $\tilde{\bh}^\sT\bJ\bth_*$ as a function of $\bw_{N+1}$, it has zero expectation (w.r.t $\bw_{N+1}$ conditioned on $\bth_*$ and $\bW$). In addition, it is a Lipschitz function with Lipschitz factor at most $\frac{1}{4} \twonorm{\bW^\sT\bJ\bth_*}$, which is $O(1)$ on the event $\event$. Next, by employing the concentration bound for Lipschitz functions on unit sphere (see e.g.~\cite[Theorem 5.1.4]{vershynin2018high}), we obtain
\begin{align*}
\prob\left(|\tilde{\bh}^\sT\bJ\bth_*| \ge t \right) &\le  \prob\left(|\tilde{\bh}^\sT\bJ\bth_*| \ge t ;\event\right) + \prob(\event^c)\nonumber\\
&\le 2e^{-c'd t^2} + 3e^{-cn}\,.
\end{align*}
Choosing $t = C\sqrt{\frac{\log (d)}{d}}$, and invoking~\eqref{eq:diff-h-k} we get
\begin{align}\label{eq:term-inf-6}
\prob\left(A\; |\bh^\sT \bJ\bth_*| \ge C\sqrt{\frac{\log (d)}{d}} \right) \le 2d^{-c'C^2} + 3e^{-cn}\,.
\end{align}
Combining the bounds~\eqref{eq:term-inf-3} to \eqref{eq:term-inf-6} into \eqref{eq:uhat-00} 
we get
\[
 |\hat{u}|\le C'\frac{\log(d)}{\sqrt{d}}\,, 
 \]
 with probability at least $4d^{-c'C^2}+ 3e^{-cn}+ d^{-c'}$.

The result follows by union bounding over the $N$ coordinates of $\hth$, along with the assumption that $N, n, d$ grow at the same order. (Note that the event $\event$ is common across all these bounds and so we count its complement probability once.)
\end{proof}

\subsubsection{Bounding the Difference of \texorpdfstring{$\Phi_k, \Phi_{k-1} \text{ with } \Psi_k(\rb)$}{TEXT}}
\begin{lemma} \label{gholoom_main_lemma_1}
We have
\begin{equation} \label{Phi-Phi-diff1}
\max\left\{ \E[(\Psi_k(\sigma(\W \x_k)) - \Phi_{-k})^2], (\Psi_k(\mbf_k) - \Phi_{-k})^2 \right\} \leq \frac{{\rm{polylog }}\, (d)}{d^2},
\end{equation}
and
\begin{equation} \label{Phi-Phi-diff2}
\max\left\{ \E[(\Psi_k(\sigma(\W \x_k)) - \Phi_{k-1})^2], (\Psi_k(\mbf_k) - \Phi_{k})^2 \right\} \leq \frac{{\rm{polylog }} \, (d)}{d^3}
\end{equation}
\end{lemma}
\begin{proof}
To prove \eqref{Phi-Phi-diff1}, \rev{we note} from \eqref{S_k} that
\begin{align*}
\Psi_k(\rb) - \Phi_{-k} &= \min_\bth S_k(\bth, \rb ) \leq S_k(\bth_{-k}^*, \rb )\\ &= \frac{1}{n} \ell(\bth_{-k}^*; \rb, y_k)   \\
&\leq  \frac{C}{n}\left( 1 + |y_k| + \twonorm{\bJ \bth_{-k}^*} + \rb^\sT\bth_{-k}^*  \right)^2,
\end{align*}
where $C$ is an absolute constant. Consequently, by using Lemma~\ref{properties_of_bth_star} and the fact that $\|\bJ \|$ is \rev{bounded}, we obtain \eqref{Phi-Phi-diff1}.

To prove \eqref{Phi-Phi-diff2}, we adapt the proof of Lemma 1  in \cite{hu2020universality} to our setting. In the following we use 
$$q(\bth) = \lambda\twonorm{\bth}^2 + \lambda_w \twonorm{\frac{1}{2}\bW^\sT\bth - \bbeta}^2 +  \lambda_s\frac{\log (d)}{d} (\ones^\sT \bth)^2$$

We start with writing the Taylor expansion of $R_k(\bth, \rb)$, defined in \eqref{relation_Rs}, around the point $\bth^*_{-k}$. Note that $\bth^*_{-k}$ is the minimizer of $R_{-k}(\bth)$, and hence 
\begin{align*}
R_k(\bth, \rb) = R_{-k} (\bth^*_{-k}) + \frac{1}{n} \ell(\bth; \rb, y_k)+  \frac{1}{2n} \sum_{t \neq k} &(\bth - \bth^*_{-k})^\sT \nabla^2 \ell(\bth'; \rb_t, y_t) (\bth - \bth^*_{-k})  \\
& + \frac 12 (\bth - \bth^*_{-k})^\sT \nabla^2  q(\bth')  (\bth - \bth^*_{-k}) 
\end{align*}
where $\bth'$ can be written as 
\begin{equation} \label{taylor_r}
\bth' = \omega \bth_{-k}^* + (1-\omega) \bth,
\end{equation}
for some $\omega \in [0,1]$. 
As a result, we can write using the definition \eqref{S_k}
\begin{align} \nonumber
R_k(\bth, \rb) - &S_k(\bth, \rb)\\
&= \frac 12 (\bth - \bth^*_{-k})^\sT \left[ \frac{1}{n} \sum_{t \neq k} \nabla^2 \ell(\bth'; \rb_t, y_t) -  \nabla^2 \ell(\bth_{-k}^*; \rb_t, y_t)    \right]  (\bth - \bth^*_{-k}), \label{diff_R_S}
\end{align}
where we have noted that, since $q$ is a quadratic function, we have $\nabla^2 q(\bth_{-k}^*) = \nabla^2 q(\bth')$. 
 
 Let us now consider the sum involving the terms of the form
\begin{align} \label{second_diff_diff}
(\bth - \bth^*_{-k})^\sT \left( \nabla^2 \ell(\bth'; \rb_t, y_t) -  \nabla^2 \ell(\bth_{-k}^*; \rb_t, y_t) \right) (\bth - \bth^*_{-k}). 
\end{align}
We can now use the expansion in \eqref{hessian_ell} to bound the above term. A straight-forward calculation, similar to what was done in the proof of Lemma~\ref{ell_2_diff}, shows that the above term involves several terms, among which the dominant term has the following form:
\begin{align} \label{second_diff_domin}
\alpha_t (\bth - \bth^*_{-k})^\sT \left(\rb_t \rb_t^\sT (\rb_t^\sT(\bth' - \bth^*_{-k})) \right) (\bth - \bth^*_{-k}) = \alpha_t \left( \rb_t^\sT(\bth - \bth^*_{-k})\right)^2 \left( \rb_t^\sT(\bth' - \bth^*_{-k})\right) = \alpha_t \omega \left( \rb_t^\sT(\bth - \bth^*_{-k})\right)^3,
\end{align}
where  $\alpha_t$ satisfies (noting that the derivatives of $g$ are uniformly bounded):
\begin{equation} \label{alpha_diff_R_S_bound}
\mathbb{P}\left( |\alpha_t| \geq v \right) \leq c \exp( -v^{c'}/c), 
\end{equation}
for absolute constants $c,c'>0$.  A straight-forward calculation (similar to what is done in the proof of Lemma~\ref{ell_2_diff}) shows that all the other terms in the expansion of \eqref{second_diff_diff} are in absolute value less than the term given in \eqref{second_diff_domin}.  As a result, one can write
\begin{align*}
 \left| \frac 12 (\bth - \bth^*_{-k})^\sT \left[ \frac{1}{n} \sum_{t \neq k} \nabla^2 \ell(\bth'; \rb_t, y_t) -  \nabla^2 \ell(\bth_{-k}^*; \rb_t, y_t)  \right]  (\bth - \bth^*_{-k}) \right| 
 \leq \frac{1}{n} \sum_{t\neq k} \left|  \alpha_t \right| \left| \rb_t^\sT(\bth - \bth^*_{-k}) \right|^3,  
\end{align*}
Using the above bound, we can now bound \eqref{diff_R_S} as
\begin{align} \label{diff_R_S_1}
\left| R_k(\bth, \rb) - S_k(\bth, \rb) \right| \leq   \frac{1}{n} \sum_{t\neq k} \left|  \alpha_t \right| \left| \rb_t^\sT(\bth - \bth^*_{-k}) \right|^3.  
\end{align}
The rest of the proof follows almost line-by-line according to the proof of Lemma 1 in \cite{hu2020universality}. 
Let $\mathcal{B} = \{\bth_k^*(\rb)\} \cup \{\tilde{\bth}_k(\rb)\} $. By using the definitions \eqref{Phi_def} and \eqref{Psi_def}, we have
\begin{align*}
\left| \Phi_k(\rb) - \Psi_k(\rb) \right| &= \left| \min_{\bth \in \mathcal{B}} R_k(\bth, \rb) - \min_{\bth \in \mathcal{B}}  S_k(\bth, \rb)  \right| \\
&\leq \max_{\bth \in \mathcal{B}} \left| R_k(\bth, \rb) -   S_k(\bth, \rb) \right|
\end{align*}
We thus obtain using \eqref{diff_R_S_1} that
\begin{align*}
\left| \Phi_k(\rb) - \Psi_k(\rb) \right| \leq &
 C \frac{1}{n}\sum_{t \neq k}  |\alpha_t| \left( \left| \rb_t^\sT(\bth^*_k(\rb) - \tilde{\bth}_{k}(\rb)) \right|^3  +  \left| \rb_t^\sT(\tilde{\bth}_k(\rb) - \bth^*_{-k}) \right|^3 \right)
\end{align*}
Let us now bound each of the terms above. 
We have
$$\frac{1}{n}  \sum_{t \neq k}  |\alpha_t|  \left| \rb_t^\sT(\bth^*_k(\rb) - \tilde{\bth}_{k}(\rb)) \right|^3 \leq \twonorm{\bth^*_k(\rb) - \tilde{\bth}_{k}(\rb)}^3 \frac{1}{n} \sum_{t \neq k} |\alpha_t| \twonorm{\rb_t}^3.   $$
Now, from Lemma~\ref{simpler_min_theta_taus_lemma}, up to negligible $O(\frac{\text{polylog}(n)}{n})  $ terms,  we have
$$\rb_t^\sT (\tilde{\bth}_k(\rb) - \bth^*_{-k}) = \frac{1}{n} \beta_1 \rb_t^\sT \bH_{-k}^{-1} \rb + \frac{1}{n} \beta_2 \rb_t^\sT H_{-k}^{-1} \pb. $$
Now, using the above relations, and the  inequality $(\sum_{i=1}^n |a_i|)^2 \leq n \sum_{i=1}^n a_i^2 $, as well as the Holder's inequality, we can write
\begin{align*}
\left| \Phi_k(\rb) - \Psi_k(\rb) \right|^2 \leq & C''  \twonorm{\bth^*_{k}(\rb) - \tilde{\bth}_{k}(\rb)}^6 ( \frac{1}{n} \sum_{t\neq k} |\alpha_t|^2 \twonorm{\rb_t}^6) \\
&+ \frac{C''}{n} \sum_{t\neq k} |\alpha_t|^2 \left( \left| \beta_1 \frac{\rb_t^\sT \bH_{-k}^{-1} \rb }{n}\right|^6 + \left| \beta_2 \frac{\rb^\sT \bH_{-k}^{-1} \pb}{n} \right|^6 \right) 
  \\
&+ O(\frac{\text{polylog}(n)}{n^4}),
\end{align*}
where $C'' > 0$ is an absolute constant. Now, from  Lemma~\ref{ell_2_diff}, parts (c) and (e) of Lemma~\ref{all_the_bounds}, and \eqref{bounded_diff_constants}, and  \eqref{alpha_diff_R_S_bound}, we obtain for any integer $D > 0$ that
\begin{align*}
& \E\left[\twonorm{\bth^*_{k}(\rb) - \tilde{\bth}_{k}(\rb)}^{2D} \right] \leq C_D \frac{\rev{\text{polylog}}(n)}{n^{2D}},\\
& \mathbb{P} \left( \twonorm{\rb_t}^{2D}  \geq n^D (\log(n))^b \right) \leq c \exp\left( -(\log(d))^2/c \right), \\
&  \mathbb{P} \left( |\alpha_t|^{2D} \geq   (\log(n))^b \right) \leq  c \exp\left( -(\log(d))^2/c \right),\\
& \mathbb{P} \left( \max\{|\beta_1|^{2D}, |\beta_2|^{2D}\} \geq   (\log(n))^b \right) \leq  c \exp\left( -(\log(d))^2/c \right),
\end{align*}
for suitably chosen absolute constants $b, c, C_D>0$. Finally, since the matrix $\bH_{-k}^{-1}$ has bounded norm, and $\rb_t$ and $\rb$ are independent sub-gaussian random vectors, we obtain
\begin{align*} \E\left[ \left| \frac{\rb_t \bH_{-k}^{-1} \rb}{n}\right|^{2D} \right] 
& \leq C_D \frac{\rev{\text{polylog}}(n)}{n^{D}} \\
& \,\,\, \text{ and } \,\,\,  \E\left[ \left| \frac{\rb_t \bH_{-k}^{-1} \pb}{n}\right|^{2D} \right] \leq \E\left[ \left| \frac{\twonorm{\rb_t} \twonorm{\bH_{-k}^{-1}} \twonorm{\pb}}{n}\right|^{2D} \right] \leq  C_D \frac{\rev{\text{polylog}}(n)}{n^{D}} .
\end{align*}

By using the above relations, the following result now follows in a straight-forward manner using the Holder's inequality: 
$$\E\left[ \left| \Phi_k(\rb) - \Psi_k(\rb) \right|^2 \right] \leq \frac{\rev{\text{polylog}}(n)}{n^3}. $$  
And the result follows since $d$ and $n$ grow in proportion to each other.
\end{proof}

\subsubsection{Putting things together}
\label{putting_things_together_GEP}
To prove  Theorem~\ref{main_theorem:GEP}, we consider any test function $\phi: \mathbb{R} \to \mathbb{R}$ which is uniformly bounded in terms of its value as well as its first and second derivatives. We will show that 
\begin{equation} \label{Phi_A_diff_Phi_B}
 \left| \E[\varphi(\Phi_A)] - \E[\varphi(\Phi_B)]  \right| = \frac{\rev{\text{polylog}} (d)}{d^{\frac 12}} + o_d(1) . 
 \end{equation}
Using this result, one immediately obtains the theorem (see Sections 2.3 and 2.4 of \cite{hu2020universality}). As a result, in the rest of this section we focus on proving the above relation for any test function $\varphi$.  In order to prove this result, we use the so-called Lindeberg's method: We consider the quantities $\Phi_k$ defined in \eqref{Phi_k}, and show for any $k \in [n]$ that
\begin{equation}\label{Phi_k_diff_Phi_k-1}
 \left| \E[\varphi(\Phi_k)] - \E[\varphi(\Phi_{k-1})]  \right| = \frac{\rev{\text{polylog}}(d)}{d^{\frac 32}} + \frac{o_d(1)}{d} .
\end{equation}
The above bound immediately results in \eqref{Phi_A_diff_Phi_B} via a telescopic sum over $k$. It thus remains to prove \eqref{Phi_k_diff_Phi_k-1}.  

Using the Taylor expansion, we can write
$$ \varphi(\Phi_k) = \varphi(\Phi_{-k}) + \varphi'(\Phi_{-k}) (\Phi_{k} - \Phi_{-k} ) + \frac 12 \varphi''(\alpha)(\Phi_{k} - \Phi_{-k} )^2,  $$
where $\alpha$ is a number between $\Phi_{-k}$ and $\Phi_{k}  $.  Using the above expansion, and a similar expansion for $\Phi_{k-1}$, we obtain
\begin{equation}
 \left| \E[\varphi(\Phi_k)] - \E[\varphi(\Phi_{k-1})] \right| \leq ||\varphi'||_\infty \left| \E[\Phi_k - \Phi_{k-1}] \right| + \frac 12 ||\varphi''||_\infty \left( (\Phi_{k} - \Phi_{-k})^2 + (\Phi_{k-1} - \Phi_{-k})^2  \right),
\end{equation}
where $||\varphi'||_\infty$ and $||\varphi''||_\infty$ are the maximum (absolute) values of the first and second derivative of $\varphi$. 

By using Lemma~\ref{gholoom_main_lemma_1} we obtain that
\begin{align*}
| \E[\Phi_k - \Phi_{k-1}] | &\leq  | \E[\Psi_k(\mbf_k) - \Psi_k(\sigma(\W\x_k)] |  +  \E[|\Phi_k - \Psi_k(\mbf_k)|]   +   \E[ |\Psi_k(\sigma(\W \x_k)) - \Phi_{k-1}|]  \\
&\leq  | \E[\Psi_k(\mbf_k) - \Psi_k(\sigma(\W \x_k))] |  + \frac{{\rm polylog} \,(d)}{d^{\frac 32}},
\end{align*}
where the last step follows simply from \eqref{Phi-Phi-diff2}.  Also, from \eqref{Phi-Phi-diff1} and \eqref{Phi-Phi-diff2}, we can conclude that
$$\E[ (\Phi_k - \Phi_{-k})^2] \leq 2\E[ (\Phi_k - \Psi_k(\mbf_k))^2] +  2\E[ (\Phi_{-k} - \Psi_k(\mbf_k))^2] \leq   \frac{\rev{\text{polylog}}(d)}{d^{2}},$$ 
and similarly
$$\E[ (\Phi_{k-1} - \Phi_{-k})^2] \leq   \frac{\rev{\text{polylog}}(d)}{d^{2}}.$$ 

Finally, the only term that is left to be analyzed is $| \E[\Psi_k(\mbf_k) - \Psi_k(\sigma(\W \x_k))] |  $, for which we use Lemma~\ref{simpler_min_theta_taus_lemma}, \ref{th_inf_lemma}, as well as a CLT-type result from \cite{goldt2020gaussian}. We note the following three facts: 

\begin{itemize}
\item [(i)] The quantity $\rb^\sT \bth_{-k}^*$ converges in distribution to a gaussian with the same mean and variance when $\rb$ is generated according to the distributions in \eqref{rb_distrbutions}.  This is due to the CLT-type theorem given in \cite[Theorem 2]{goldt2020gaussian}. More precisely, we have shown in Lemma~\ref{th_inf_lemma} that with probability $1- c d^{-c}$ we have: $\infnorm{\bth_{-k}^*}$ is at most $C \sqrt{(\log (d))/d}$, where $c, C$ are absolute constants. Also, according to part (e) of Lemma~\ref{properties_of_bth_star}, we have $|\ones^\sT \bth_{-k}^*| \leq C' \sqrt{d/(\log (d))}$, where $C'$ is an absolute constant.    

Now, let us define $\bth' = \bth^*_{-k}/\sqrt{\log (d)}$. Note that, with high probability (as specified above), we have $\infnorm{\bth'} \leq C/\sqrt{d}$ and $|\ones^\sT \bth'| \leq C' \sqrt{d}/(\log (d))$. According to \cite[Theorem 2]{goldt2020gaussian}, for $\ab$ and $\bb$ generated according \eqref{rb_distrbutions}, and fixing $\bth'$, the random variables $\ab^\sT \bth'$ and $\bb^\sT \bth'$ have the same mean and variance, and we have  
$$ d_{\rm MS}\left(\ab^\sT \bth', \bb^\sT \bth' \right) \leq C''\infnorm{\bth'} \left( \frac{\left| \ones^\sT \bth' \right|}{\sqrt{d}} + \frac{1}{\sqrt{d}} \right),$$ 
where $C'' >0$ is an absolute constant, and $d_{\rm MS}$ is the so-called maximum-sliced distance, and $d_{\rm MS}\left(\ab^\sT \bth', \bb^\sT \bth' \right)$ defines the distance between the distributions of $\ab^\sT \bth'$ and $\bb^\sT \bth'$. As a result, since $\bth^*_{-k} = \bth' \times \sqrt{\log (d)}$, we obtain that
$$ d_{\rm MS}\left(\ab^\sT \bth_{-k}^*, \bb^\sT \bth_{-k}^* \right) \leq C'' \sqrt{\log (d)} \infnorm{\bth'} \left( \frac{\left| \ones^\sT \bth' \right|}{\sqrt{d}} + \frac{1}{\sqrt{d}} \right) =   O\left(\frac{1}{\sqrt{\log (d)}} \right).$$


\item[(ii)] Consider the result of Lemma~\ref{simpler_min_theta_taus_lemma}. From Lemma~\ref{properties_of_bth_star}, the norm of the vector $\bth^*_{-k}$ is bounded by an absolute constant with probability at least $1 - \exp(-c n)$. Hence, since the matrix $\bJ$ is also of bounded operator norm, then the norm of the vector $\pb$ given in Lemma~\ref{simpler_min_theta_taus_lemma} is bounded by an absolute constant. Also, the quantity $\twonorm{\bJ \bth^*_{-k}}$ is bounded by an absolute constant. Given fixed matrix $\bH^{-1}$ with bounded norm, and a fixed vector $\pb$ with bounded norm, the quantity $\frac{1}{n} \rb^\sT \bH^{-1} \pb$ is,  with probability at least $1- c \exp(-(\log (d))^2/c)$, of order $O(\rev{\text{polylog}}(d)/d)$ according to Lemma~\ref{concentration_feature}.  Hence, in the formula \eqref{simpler_min_theta_taus}, the overall contribution of the terms which include  
$\frac{1}{n} \rb^\sT \bH_{-k}^{-1} \pb$ or $\frac{1}{n} \pb^\sT \bH_{-k}^{-1} \pb$  is of order $O(\rev{\text{polylog}}(d)/d^2)$. Therefore, neglecting these terms adds an additional error of at most $O(\text{polylog}(d)/d^2)$ in  computing $\E[\Psi(\rb)] - \Phi_{-k}$. Consequently, from the result of Lemma~\ref{simpler_min_theta_taus_lemma} we can write

\begin{equation}  \label{E_diff_phi_psi}
\E[\Psi_k(\rb)] = \Phi_{-k} +  \frac{1}{n} \E\left[ \min_{\tau_1} 
  \left\{ \frac{ \rb^\sT \bH_{-k}^{-1}\rb}{2n} \left(  \frac{\partial \tilde{\ell}(\tau_1, 0) }{\partial \tau_1} \right)^2
+ \tilde{\ell}(\tau_1, 0) \right\}  \right]
+ O\left(\frac{\rev{\text{polylog}(d)}}{d^{\frac{3}{2}}} \right), 
 \end{equation}
where 
$\tilde{\ell}(\tau_1, \tau_2) $ is given in \eqref{ell_tilde_GEP}.

\item [(iii)] Given a matrix $\bH$, the value 
$\frac{1}{n} \rb^\sT \bH_{-k}^{-1} \rb$ concentrates on the same quantity if $\rb$ is generated from either of the distributions in \eqref{rb_distrbutions}. More precisely,  from  \cite[Lemma 13]{hu2020universality} (or \cite[Lemma 1]{louart2018random}) 
we obtain 
$$ \mathbb{P}\left( \left| \frac{1}{n} \rb^\sT \bH_{-k}^{-1} \rb - \E[\frac{1}{n} \rb^\sT \bH_{-k}^{-1} \rb] \right| \geq  c\frac{\log(d)}{\sqrt{d}}  \right)  \leq 1 - c \exp(- (\log(d)^2)),$$
for $\rb$ being generated according to either of the distributions in \eqref{rb_distrbutions}, and $c > 0$ being an absolute constant. Also, we have that 
$$ \left| \E[\frac{1}{n} \sigma(\W \x)^\sT \bH_{-k}^{-1} \sigma(\W \x) ] -  \E[\frac{1}{n} \mbf^\sT \bH_{-k}^{-1} \mbf ] \right| = 
\frac{1}{n} \text{Trace}\left(  \bH_{-k}^{-1} (\Sigma_s - \Sigma_f) \right), $$
where $\Sigma_2 = \mathbb{E}[ \sigma(\W \x)  \sigma(\W \x)^\sT] $ and $\mathbb{E}[ \mbf  \mbf^\sT]$, and the last inequality follows from Lemma~\ref{covariances_are_the_same}.  As a result, we obtain
\begin{align*}
 \mathbb{P}\left( \left|\frac{1}{n}  \sigma(\W \x_k)^\sT \bH_{-k}^{-1} \sigma(\W \x_k) - \frac{1}{n} \E\left[\mbf^\sT \bH_{-k}^{-1} \mbf \right]   \right| \geq  c\frac{(\log(d))^{3/2}}{\sqrt{d}}  \right)  &\leq 1 - c \exp(- (\log(d)^2)), \\
 \mathbb{P}\left( \left| \frac{1}{n}  \mbf^\sT \bH_{-k}^{-1}\mbf - \frac{1}{n} \E\left[\mbf^\sT \bH_{-k}^{-1} \mbf \right]   \right| \geq  c\frac{\log(d)}{\sqrt{d}}  \right) & \leq 1 - c \exp(- (\log(d)^2)), 
\end{align*}
\end{itemize}

Let us now put all the above facts together to bound $| \E[\Psi_k(\mbf_k) - \Psi_k(\sigma(\W \x_k))] |  $. 
Consider the function  $\tilde{\ell}(\tau_1, \tau_2) $ given in \eqref{ell_tilde_GEP}. This function depends on $\rb$ only through 
$ \bth_{-k}^* \rb$. Using fact (i) above, we know that $ \bth_{-k}^* \rb$ will asymptotically have the same (gaussian) distribution for both $\rb \sim \mbf_k$ and $\rb \sim \sigma(\W \x_k)$.  Also, it is easy to conclude using part (c) of Lemma~\ref{properties_of_bth_star} \rev{that all of the moments} of random variable $ \bth_{-k}^* \rb$ are bounded (i.e. $\mathbb{E}[|\bth_{-k}^* \rb|^D ] \leq C_D$ for an absolute constant $C_D > 0$). Further, in fact (ii) we have argued that $\twonorm{J \bth_{-k}^*}$ is bounded with probability $1- e \exp(-cn)$. Also, from fact (iii) above, we know that the term $\frac{1}{n} \rb^\sT \bH_{-k}^{-1} \rb$ concentrates sharply on the same value for $\rb$ being either $\mbf_k$ or $\sigma(\W \x_k)$. Putting all these together, and using \cite[Corollary of Theorem 25.12]{billingsley1995probability}, we obtain that

\begin{align*}
&  \E_{\rb = \mbf_k} \left[ \min_{\tau_1} 
  \left\{ \frac{ \rb^\sT \bH_{-k}^{-1}\rb}{2n} \left(  \frac{\partial \tilde{\ell}(\tau_1, 0) }{\partial \tau_1} \right)^2
+ \tilde{\ell}(\tau_1, 0) \right\}  \right]  \\
&\quad \quad \quad  \,\,- \,\,  
 \E_{\rb = \sigma(\W \x_k)} \left[ \min_{\tau_1} 
  \left\{ \frac{ \rb^\sT \bH_{-k}^{-1}\rb}{2n} \left(  \frac{\partial \tilde{\ell}(\tau_1, 0) }{\partial \tau_1} \right)^2
+ \tilde{\ell}(\tau_1, 0) \right\}  \right] \\
& = o_d(1), 
\end{align*}
and therefore from \eqref{E_diff_phi_psi} we obtain 
$$\left| \E[\Psi_k(\mbf_k) - \Psi_k(\sigma(\W \x_k))] \right|  \leq   \frac{o_d(1)}{d},$$
for an absolute constant $C > 0$, and hence we obtain \eqref{Phi_k_diff_Phi_k-1}.

\subsubsection{Proofs of the Auxiliary Lemmas} \label{leave-one-out-auxilliary-lemmas}
Here we provide the proofs of some of the auxiliary lemmas used in our analysis.
\medskip



%
%
%
%

\noindent \textbf{Proof of Lemma~\ref{properties_of_bth_star}.} 
We will prove part (b) here, but part (a) will have the exact same proof. 
Let $\bth^*$ be the minimizer of $R_k(\bth, \rb)$. We can write
$$ R_k(\bth^*, \rb) \leq R_k(\mathbf{0}, \rb),$$
as $\bth^*$ is the minimizer. 

On the one  hand we have
$$R_k(\mathbf{0},\rb) = \frac{1}{n} \sum_{i=1}^n y_i^2 + \twonorm{\bbeta}^2,$$
and thus for any $v \geq 0$:
\begin{equation} \label{R_0_bound}
\prob\left( R_k(\mathbf{0},\rb)  \geq v +  \twonorm{\bbeta}^2 + 2\E[y_1^2]  \right) \leq c_1 \exp\left( - nv^2/c_1 \right), 
\end{equation}
for an absolute constant $c_1 > 0$.  

On the other hand, since $R(\bth, \rb)$ is $\lambda$-strongly convex, and $R(\bth, \rb) \geq 0$, we can write
$$\twonorm{\bth^*}^2 \leq \frac{1}{\lambda} R(\mathbf{0}, \rb),$$
which together with \eqref{R_0_bound} gives use the result.

To prove part (c), we note that $\rb$ is generated independently from $\bth_{-k}^*$. We can thus write:
\begin{align*}
\mathbb{P}\left( | \rb^\sT \bth_{-k}^* | \geq v  \right)  &\leq  \mathbb{P}\left( \twonorm{\bth_{-k}^*} \geq \sqrt{v} \right) +   \mathbb{P}\left( | \rb^\sT \bth_{-k}^* | \geq v \bigg\mid \twonorm{\bth_{-k}^*} < \sqrt{v}  \right) \\
& \leq c_2 e^{-v/c_2} + \mathbb{P}\left( | \rb^\sT \bth_{-k}^* | \geq v \bigg\mid \twonorm{\bth_{-k}^*} < \sqrt{v}  \right) \\
&\leq c'e^{-v/c'}. 
\end{align*}
where the second step follows from  part (a) of the lemma (with $c_2$ chosen to be sufficiently large); and the last step follows from the independence of $\rb$ and $\bth_{-k}^*$ as well as Lemma~\ref{concentration_feature}.  

Also, the proof of part (d) follows simply because
$$ \lambda_s\frac{d}{\log(d)}( \ones^\sT \bth_{-k}^*)^2 \leq R_{-k}(\bth_{-k}^*) \leq  R_{-k}(\mathbf{0}) = \frac{1}{n} \sum_{i \neq k} y_i^2 + \twonorm{\bbeta}^2, $$ 
where $\mathbf{0}$ is the all-zero vector. By using a bound similar to \eqref{R_0_bound} for $R_{-k}(\mathbf{0})$ we obtain the result.

\vspace{.2cm}
\noindent\textbf{Proof of Lemma~\ref{all_the_bounds}.} Part (a) is exactly Lemma 12 in \cite{hu2020universality}. For part (b), consider the matrix $\vct{R}$ whose columns are $\rb_t$'s, i.e. $\vct{R}= [\rb_1 | \rb_2| \cdots| \rb_n]$. Since $\rb_t$'s are zero-mean sub-gaussian vectors (see Lemma~\ref{concentration_feature}), we know that its operator norm satisfies:  
$$\mathbb{P} \left( ||\vct{R}|| \geq c_1\sqrt{d} + v \right) \leq c \exp(-v^2/c).$$
Also, define the vector $\alpha = [\alpha_t]^\sT_{t \neq k}$. Note that $\twonorm{\alpha} \leq \sqrt{n}$. We have
$$\mathbb{P} \left( \twonorm{\frac{1}{n} \sum_{t \neq k} \alpha_t \rb_t} \geq v + c_1 \right) 
= \mathbb{P} \left( \twonorm{\frac{1}{n} \vct{R} \alpha} \geq v + c_1\right)
= \mathbb{P} \left( || \vct{R}|| \geq  v \sqrt{n} + c_1 \sqrt{n} \right).$$
Given the above relation, and the fact that $d$ and $n$ grow proportionally, the result of the second part of the lemma follows easily. 

Part (c) follows from Lemma~\ref{concentration_feature}. Also, part(d) follows from part (c) as well as  Lemma~\ref{ell_2_diff_first} (specifically \eqref{refined_bth_diff_lemma}) and the fact that the operator norm of the matrix $\bH_{-k}^{-1}$ is upper-bounded by $2/\lambda$. 

Part (e) follows from the fact that $\rb$ is a random sub-gaussian vector (see Lemma~\ref{concentration_feature}). We refer to \cite{Vers} for bounds on the $\ell_2$ norm of random sub-gaussian vectors. 

To prove part (e), we use \eqref{refined_bth_diff_lemma} to write
\begin{align} \nonumber
& \mathbb{P}\left(| \rb_t^\sT (\tilde{\bth}(\rb) - \bth_{-k}^* )| \geq \frac{v}{\sqrt{n}}  \right) 
\leq \mathbb{P}\left(  \frac{1}{n}\left( \left| \beta_1 \rb_t^\sT  \bH_{-k}^{-1}\rb\right| + \left| \beta_2 \rb_t^\sT \bH_{-k}^{-1}\pb \right| \right) + \left| \rb_t^\sT \eb \right|    \geq \frac{v}{\sqrt{n}} \right), \\
&  \leq \mathbb{P}\left( \frac{1}{n}   \left| \beta_1 \rb_t^\sT  \bH_{-k}^{-1}\rb\right|    \geq \frac{v}{3\sqrt{n}} \right) 
+  \mathbb{P}\left( \frac{1}{n} \left| \beta_2 \rb_t^\sT \bH_{-k}^{-1}\pb \right|    \geq \frac{v}{3\sqrt{n}} \right) 
+  \mathbb{P}\left( \left| \rb_t^\sT \eb \right|    \geq \frac{v}{3\sqrt{n}} \right)  \label{bound_diff_r}
\end{align}
We will now bound each of the terms above. For the first term we have
\begin{align*} \nonumber
&\mathbb{P}\left( \frac{1}{n}   \left| \beta_1 \rb_t^\sT  \bH_{-k}^{-1}\rb\right|    \geq \frac{v}{3\sqrt{n}} \right)  \\
& \leq \mathbb{P}\left( \twonorm{\rb_t} \geq \sqrt{vn} \right) +  \mathbb{P}\left( \frac{1}{n}   \left| \beta_1 \rb_t^\sT  \bH_{-k}^{-1}\rb\right|    \geq \frac{v}{3\sqrt{n}} \,\,\, , \,\,\, \twonorm{\rb_t} < \sqrt{vn} \right)\\
& \leq  c_1 \exp(- v^{c_2}/c_1)+  \mathbb{P}\left( \frac{1}{n}   \left| \beta_1 \rb_t^\sT  \bH_{-k}^{-1}\rb\right|    \geq \frac{v}{3\sqrt{n}} \,\,\, , \,\,\, \twonorm{\rb_t} < \sqrt{vn} \right), 
\end{align*}
where the last step follows from part (e) and appropriately selecting $c_1, c_2 > 0$. Now, note that the vector $\rb$ is generated independently from $\rb_t$ and $\bH_{-k}$. As a result, to bound the second term in the RHS of the above relation, we notice that, assuming $\twonorm{\rb_t} < \sqrt{vn}$, we have $ \twonorm{ \beta_1 \rb_t^\sT  \bH_{-k}^{-1}} \leq C |\beta_1| \sqrt{vn}$. 
Hence, we can write 
\begin{align*}
& \mathbb{P}\left( \frac{1}{n}   \left| \beta_1 \rb_t^\sT  \bH_{-k}^{-1}\rb\right|    \geq \frac{v}{3\sqrt{n}} \,\,\, , \,\,\, \twonorm{\rb_t} < \sqrt{vn} \right) \\
& \leq \mathbb{P}\left( |\beta_1| \geq  v^{\frac 14} \right) + 
\mathbb{P}\left( \frac{1}{n}   \left| \beta_1 \rb_t^\sT  \bH_{-k}^{-1}\rb\right|    \geq \frac{v}{3\sqrt{n}} \,\,\, , \,\,\, \twonorm{\rb_t} < \sqrt{vn}  \,\,\, , \,\,\,  |\beta_1| < v^{\frac 14}\right) \\
&\leq c_3 \exp(- v^{c_4}/c_3)+ \mathbb{P}\left( \frac{1}{n}   \left| \beta_1 \rb_t^\sT  \bH_{-k}^{-1}\rb\right|    \geq \frac{v}{3\sqrt{n}} \,\,\, , \,\,\, \twonorm{\rb_t} < \sqrt{vn}  \,\,\, , \,\,\,  |\beta_1| < v^{\frac 14}\right)  \\
&\leq c_3 \exp(- v^{c_4}/c_3)+ c_5  \exp(- v^{c_6}/c_5),
\end{align*}
where the second inequality follows from \eqref{bounded_diff_constants} and by suitably choosing $c_3, c_4 > 0$. The third inequality follows from sub-gaussianity of $\rb$, see Lemma~\ref{concentration_feature}, and the fact that, given $|\beta_1| < v^{1/4}$ and $\twonorm{\rb_t} \leq \sqrt{v n}$, the random vector $\ub = \beta_1 \rb_t \bH_{-k}^{-1}$ satisfies $\twonorm{\ub} \leq C v^{3/4}  \sqrt{n}$ and is independently generated from $\rb$. Hence, we can bound the second term in the RHS of the above relation by using Lemma~\ref{concentration_feature} and appropriate choices of $c_5, c_6 > 0$. 

The second and third terms in \eqref{bound_diff_r} can be bounded similarly as the first term but in an easier manner. 
The second term follows by writing $|\rb_t^\sT \pb| \leq \twonorm{\rb_t} \twonorm{\pb}$, and noticing that $\twonorm{\pb}$ is upper-bounded by a constant  since the norm of $\bJ$ is bounded and the norm of $\bth_{-k}^*$ is bounded (see Lemma~\ref{simpler_min_theta_taus_lemma} and Lemma~\ref{properties_of_bth_star}). Hence, using a similar (but simpler) argument as above, we can write 
$$  \mathbb{P}\left( \frac{1}{n} \left| \beta_2 \rb_t^\sT \bH_{-k}^{-1}\pb \right|    \geq \frac{v}{3\sqrt{n}} \right)  \leq c_7 \exp(- v^{c_8}/c_7), $$
for absolute constants $c_7, c_8 > 8$.

Finally, the third term in the RHS of \eqref{bound_diff_r} can be bounded by   writing $|\rb_t^\sT \eb | \leq \twonorm{\rb_t} \twonorm{\eb}$, and noticing that $\twonorm{\eb}$ is small according to \eqref{bounded_diff_constants}. And, using similar steps as above, we reach to a similar upper bound.


Part (g) follows from Lemma~\ref{properties_of_bth_star} and Lemma~\ref{bound_product_r_bth}.

\begin{lemma} \label{concentration_feature}
[Analogous to Lemma 8 in \cite{hu2020universality}] Assume that $\ab = \sigma(\W \x)$ and $\bb = \mu_1 \W \x + \mu_2 \ub$, where $\x$ and $\ub$ are generated independently from the normal distribution.  Also,  let $\bSigma =\E[\bb \bb^\sT]$, i.e. $\bSigma = \mu_1^2\W\W^\sT + \mu_2^2 \mathbf{I}$. Then,  there exists an absolute constant $c> 0 $ such that:
\begin{equation} \label{sub-gauss-a}
\mathbb{P}( | \ab^\sT \bbeta  | \geq v  ) \leq 2 \exp(- \frac{ v^2}{c \twonorm{\bbeta}^2 ||\W||^2}),
\end{equation}
and 
\begin{equation}
\mathbb{P}( | \bb^\sT \bbeta  | \geq v  ) \leq 2 \exp(- \frac{ v^2}{c \twonorm{\bbeta}^2 ||\bSigma||}),
\end{equation}
for a fixed vector $\bbeta \in \mathbb{R}^d$ and any $v \geq 0$. Here, $||\W||$ (resp. $||\bSigma||$) denotes the operator norm of $\W$ (resp. $\bSigma$). 
\end{lemma}
\begin{proof}
We will be using the following well-known relation: For a $L$-Lipschitz continuous function $f$ and $\x \sim\mathcal{N}(\mathbf{0}, \mathbf{I})$ we have
\begin{equation} \label{lip-ineq}
\mathbb{P}(|f(\x) - \mathbb{E}[f(\x)]| \geq v) \leq 2 \exp(-\frac{v^2}{4 L^2}).
\end{equation}
Now, note that since $\sigma$ is the shifted Relu function, it is easy to see that the function $f(\x) = \bbeta^\sT \sigma(\W \x)$ is 
$\twonorm{\bbeta} || \W||$-Lipschitz continuous. Therefore, we obtain the result using the relation \eqref{lip-ineq} and the fact that $\x$ is distributed according to the normal distribution. 

The proof of the second part can similarly be done by noting that $\bb = \bSigma^{1/2} \tilde{\bb}$ where $\tilde{\bb}$ is distributed according to the standard normal distribution. 
\end{proof}

\begin{lemma} \label{covariances_are_the_same}
Let $\bH \in \reals^{N \times N}$ be such that $|| \bH || \leq C$ for an absolute constant $C > 0$.  Let $\Sigma_s = \E[ \sigma(\W \x) \sigma(\W \x)^\sT]$ and $\Sigma_f = \E[\mbf \mbf^\sT]$. We have 
\begin{equation}
\left| \frac{1}{n} {\rm {Trace}}\left\{ \bH(\Sigma_s - \Sigma_f) \right\} \right| \leq \frac{c (\log(d))^{3/2}}{\sqrt{d}}  .
\end{equation}
with probability  at least $1 - c\exp(-  (\log (d))^2/c)$ where $c > 0$ is an absolute constant. 
\end{lemma}
\begin{proof}
We first bound each element of the matrix $\Sigma_s - \Sigma_f$. Recall that the $k$-th element of the vector $\mbf$ is distributed according to 
$$\mu_1 \w_k^\sT \x + \mu_2 u_k, $$
where $u_k$ is independently generated from $\normal(0,1)$ and $\mu_1 = \tfrac{1}{2}, \mu_2 = \sqrt{\tfrac{1}{4}- \tfrac{1}{2\pi}}$. 
As a result, the $(\ell, k)$-th element of the matrix $\Sigma_f$ is
\begin{equation} \label{diff_sigmas_1}
\E[(\mu_1 \w_k^\sT \x + \mu_2 u_k)(\mu_1 \w_\ell^\sT \x + \mu_2 u_\ell) ] = \mu_1^2 \w_k^\sT \w_\ell + \mu_2^2 \ind\{ k = \ell  \}.
\end{equation}

Note that $\<\bw_\ell,\bx\>$ and $\<\bw_k,\bx\>$ are jointly Gaussian with 
\[\E[(\bw_\ell^\sT \bx)^2] = \E[ (\bw_k^\sT \bx)^2] = 1,\quad  \E[(\bw_\ell^\sT \bx) (\bw_k^\sT \bx)] = \bw_k^\sT \bw_\ell\,.\] 
Therefore, we have (see e.g.,~\cite[Table 1]{daniely2016toward})
\begin{align}
\E[\sigma( \bw_\ell^\sT \bx )\sigma( \bw_k^\sT \bx)] &=  \frac{\sqrt{ 1- ( \bw_k^\sT \bw_\ell)^2 } + (\pi - \cos^{-1}( \bw_k^\sT \bw_\ell))( \bw_k^\sT \bw_\ell)}{2\pi } - \frac{1}{2\pi}  \nonumber\\
&= \frac{1}{4} \bw_\ell^\sT\bw_k  + (\frac{1}{4} - \frac{1}{2 \pi}) \ind\{k = \ell \} + O((\bw_\ell^\sT \bw_k)^3)\nonumber\\
& =\mu_1^2 \bw_\ell^\sT\bw_k + \mu_2^2 \ind\{k = \ell \} +O\left( \left(\frac{\log(d)}{d} \right)^{\frac{3}{2}} \right) \,, \label{diff_sigmas_2}
\end{align}
where the last step follows from the fact that with probability at least $1 - c \exp( - (\log(d))^2/c)$ we have for all $k,\ell$, such that $k \neq \ell$, we have $|\w_\ell^\sT \w_k| \leq \frac{\log(d)}{d} $.
As a result, from \eqref{diff_sigmas_1} and \eqref{diff_sigmas_2} we obtain
$$\left| \left( \Sigma_s - \Sigma_f \right)_{k,\ell} \right| = O\left( \left(\frac{\log(d)}{d} \right)^{\frac{3}{2}} \right).  $$

Hence, the $\ell_2$ norm of each column of the matrix $\Sigma_s - \Sigma_f$ is of order $O( (\log(d))^{3/2}/d ) $. Now, since $||\bH|| \leq C$, then the $\ell_2$ norm of each row of $\bH$ is at most $C$. Thus, by a simple application of the Cauchy-Schwarz inequality we obtain the result of the lemma. 

\end{proof}

\subsection{Proof of Proposition~\ref{pro:opt-equiz}} 
Let us first show that $\hth^*, \hth^*_{\rm nl}$ fall in $\mathcal{C}-_\bth$ with high probability for any $\zeta > 0$.  We prove the result for $\hth^*$, and remark that the proof is exactly the same for $\hth^*_{\rm nl}$. Consider the objective
$$R(\bth) := \frac{1}{2n} \sum_{i=1}^n \left( |y_i - \bth^\sT\sigma(\bW\bx_i)|+  \eps \twonorm{\bJ \bth} \right)^2  + \frac{\zeta}{2} \bth^\sT \bOmega \bth. $$
On the one  hand we have
$$R(\mathbf{0}) = \frac{1}{n} \sum_{i=1}^n y_i^2$$
where $\mathbf{0}$ is the all-zero vector. Thus for any $v \geq 0$:
\begin{equation} \label{R_0_bound_1}
\prob\left( R_k(\mathbf{0})  \geq v  + 2\E[y_1^2]  \right) \leq c_1 \exp\left( - nv^2/c_1 \right), 
\end{equation}
for an absolute constant $c_1 > 0$.  

On the other hand, since $R(\bth)$ is $\zeta$-strongly convex, and $R(\bth) \geq 0$, we can write
$$\twonorm{\hth^*}^2 \leq \frac{1}{\zeta} R(\mathbf{0}, \rb),$$
which together with \eqref{R_0_bound_1} proves that $\twonorm{\hth^*}$ is bounded above by a constant $C$ with probability at least $1 - e \exp(-cn)$.

To bound $\left| \ones^\sT \hth^* \right|$ we note that 
$$ \zeta \frac{d}{\log(d)}( \ones^\sT \hth^* )^2 \leq R(\hth^*) \leq  R(\mathbf{0}) = \frac{1}{n} \sum_{i \neq k} y_i^2 , $$ 
 By using the bound to \eqref{R_0_bound_1}  we obtain with probability $1- c \exp(-cn)$ that $\left| \ones^\sT \hth^* \right| \leq C \sqrt{d/(\log(d))}$. Finally, the fact that $\infnorm{\hth^*}$ is with high probability of order $\sqrt{(\log(d))/d}$  follows in exactly the same manner as the proof of Lemma~\ref{th_inf_lemma} and hence we do not repeat the proof here.

We now show that $M(\hth^*) - M(\hth^*_{\rm nl}) \to 0$ in probability.  To do so, we use the argument given in \cite[Theorem 4]{abbasi2019universality}. Assume that $M(\hth^*)$ and $M(\hth^*_{\rm nl})$ converge to different values, say $M_A$ and $M_B$. Define $M = (M_A + M_B)/2$ and consider the following optimization problems
\begin{align*}
\bar{\Phi}_A &:= \min_{\bth: M(\bth) \leq M} \frac{1}{2n} \sum_{i=1}^n \left( |y_i - \bth^\sT\sigma(\bW\bx_i)|+  \eps \twonorm{\bJ \bth} \right)^2  + \frac{\zeta}{2} \bth^\sT \bOmega \bth \,, \\
\bar{\Phi}_B &:= \min_{\bth: M(\bth) \leq M} \frac{1}{2n} \sum_{i=1}^n \left( |y_i - \bth^\sT\mbf_i|+  \eps \twonorm{\bJ \bth} \right)^2  + \frac{\zeta}{2} \bth^\sT \bOmega \bth \,.
\end{align*}
Note that the values $\bar{\Phi}_A$ and $\bar{\Phi}_B$ must be different. Now, using the minimax theorem, and since the above objectives are $\zeta$-strongly convex, we can write
\begin{align*}
\bar{\Phi}_A & = \sup_{\lambda > 0} -\lambda M  + \min_{\bth} \frac{1}{2n} \sum_{i=1}^n \left( |y_i - \bth^\sT\sigma(\bW\bx_i)|+  \eps \twonorm{\bJ \bth} \right)^2  + \frac{\zeta}{2} \bth^\sT \bOmega \bth + \lambda M(\bth) \,, \\
\bar{\Phi}_B &=  \sup_{\lambda > 0} -\lambda M  + \min_{\bth} \frac{1}{2n} \sum_{i=1}^n \left( |y_i - \bth^\sT\mbf_i|+  \eps \twonorm{\bJ \bth} \right)^2  + \frac{\zeta}{2} \bth^\sT \bOmega \bth + \lambda M(\bth) \,. 
\end{align*}
Now, from the result of Theorem~\ref{main_theorem:GEP} we know that for any $\lambda > 0$ the values inside the min converge to the same value. As a result, the quantities $\bar{\Phi}_A$ and $\bar{\Phi}_B$ should converge to the same value (according to \cite[Lemma 1]{abbasi2019universality}) which is a contradiction with the claim that $M(\hth^*)$ and $M(\hth^*_{\rm nl})$ converge to different values (i.e. $M_A$ and $M_B$, respectively). \rev{A similar argument can be applied
to show that $\twonorm{\bJ\hth^*} - \twonorm{\bJ\hth^*_{\rm nl}} \to 0$, in probability.}

\section{Proofs of Step 4: Analysis of the Gaussian noisy linear model via convex Gaussian minimax framework} \label{sec:CGMT}
By the Gaussian equivalence property, we henceforth focus on optimization~\eqref{eq:nl-opt} and provide a precise characterization of $\ARoo_{\rm nl}(\hth^*_{\rm nl})$. 

Before proceeding, we will discuss another representation of the model using a few change of variables. Recall from \eqref{noisy_model} that $\mbf := \mu_0 \mathbf{1} + \mu_1 \bW\x + \mu_2 \ub$. For our activation function $\sigma(v) = {v\ind(v\ge 0)} - 1/\sqrt{2\pi}$, we have $\mu_0 = 0$, $\mu_1 = 1/2$ and $\mu_2 = \sqrt{\frac{1}{4} - \frac{1}{2\pi}}$. It is clear that 
$\mbf \sim \normal(0,\bSigma)$ with $\bSigma: = \mu_1^2\bW\bW^\sT + \mu_2^2\Iden$. Also the data generative model~\eqref{eq:linearModel} can be written as:
\begin{align}\label{eq:LM2}
y_i = \<\mbf_i, \bth_0\> + w_i, \quad \text{ with } \quad w_i\sim\normal(0,\sigma^2)\,,
\end{align}
for proper choices of $\sigma^2$ and $\bth_0$. Indeed in both models~\eqref{eq:linearModel} and~\eqref{eq:LM2}, $(y_i, \mbf_i)\in \reals^{N+1}$  is a centered Gaussian vector. By matching
their covariances we obtain 
\begin{eqnarray} 
\begin{split}\label{equiv}
&\bSigma = \mu_1^2\bW\bW^\sT + \mu_2^2\Iden\,,\\
&\bth_0 = \mu_1 \bSigma^{-1} \bW\bbeta\,,\\
&\sigma^2 = \tau^2 +\twonorm{\bbeta}^2- \mu_1^2 \bbeta^\sT \bW^\sT \bSigma^{-1} \bW \bbeta\,.
\end{split}
\end{eqnarray}
We next rewrite the objective of optimization~\eqref{eq:nl-opt} using this change of variable and also plug in for $y_i$ from~\eqref{eq:LM2} to obtain
\begin{align}\label{eq:nl-opt11}
\cL_{\rm nl}(\bth) =  \frac{1}{2n} \sum_{i=1}^n \left( |\<\mbf_i,\bth_0 - \bth\>+ w_i |+  \eps \rev{\twonorm{\bJ \bth}}\right)^2+ \frac{\zeta}{2} \bth^\sT \bOmega \bth.
\end{align}
We will use a powerful extension of a classical Gaussian process inequality due to Gordon \cite{gordon1988milman} known as \emph{Convex Gaussian Minimax Theorem (CGMT)} \cite{thrampoulidis2015regularized} to derive a precise asymptotic characterization of  $\ARnl(\hth^*_{\rm nl})$. A similar proof technique has been used in~\cite{javanmard2020precise} to understand the effect of adversarial training on linear regression models.
Indeed, for the particular case of $\mu_1 =0, \mu_2 = 1$ (so $\bSigma = \Iden$) and $\bJ = \Iden$, the loss function~\eqref{eq:nl-opt1} reduces to that studied in~\cite{javanmard2020precise}. 

The CGMT analysis will output a deterministic scalar optimization which depends on $\zeta$. We need to calculate the solution of this optimization at $\zeta\to 0$. However, as we discuss in our derivation, the objective of this optimization is strongly convex (in minimizing variables) and concave (in maximizing variables). Therefore, by continuity of its solution in the coefficients of the objective, we directly calculate the solution by setting $\zeta = 0$ in the loss $\cL_{\rm nl}(\bth)$, bringing us to the following restatement of the loss (with a slight abuse of notation):
 \begin{align}\label{eq:nl-opt1}
\cL_{\rm nl}(\bth) =  \frac{1}{2n} \sum_{i=1}^n \left( |\<\mbf_i,\bth_0 - \bth\>+ w_i |+  \eps \rev{\twonorm{\bJ \bth}}\right)^2\,.
\end{align}
Consider a change of variable of the form $\mbf_i = \bSigma^{1/2}\bg_i$ with $\bg_i\sim\normal({\mathbf 0},\Iden)$ and $\vct{z}=\bSigma^{1/2}(\vct{\theta}-\vct{\theta}_0)$. Also define 
\[\ell(v;\vct{\theta}) := \frac{1}{2}\left(\abs{v}+\eps\rev{\twonorm{\bJ \bth}}\right)^2\,.\]
Then, optimization problem \eqref{eq:nl-opt1} can be equivalently written in the form
 \begin{align}\label{eq:nl-opt2}
 \min_{\z\in\R^N,\vct{v}\in\R^n} \frac{1}{n}\sum_{i=1}^n \ell\left(v_i;\vct{\theta}_0+\bSigma^{-1/2}\vct{z}\right)\quad\text{subject to}\quad \vct{v} = \w- \Gb \z.
 \end{align}
By writing the dual of this optimization problem (with dual variable $\frac{\bu}{\sqrt{d}}$) we get
\begin{align}
 \label{lin}
 &\min_{\z\in\R^N,\vct{v}\in\R^n}  \max_{\vct{u}\in\R^n}\; \frac{1}{\sqrt{d}}\Big\{\vct{u}^\sT\Gb \vct{z} 
 - \vct{u}^T \w +\vct{u}^T \vct{v}\Big\}+\frac{1}{n}\sum_{i=1}^n \ell\left(v_i;\vct{\theta}_0+\bSigma^{-1/2}\vct{z}\right)\nn\\
 &\quad\quad\quad\quad\quad\quad\quad\quad\quad= \min_{\z\in\R^N,\vct{v}\in\R^n}  \max_{\vct{u}\in\R^n}\; \frac{1}{\sqrt{d}}\Big\{\vct{u}^\sT\Gb \vct{z} - \vct{u}^\sT \w +\vct{u}^\sT \vct{v}\Big\}+\elbar(\vct{v};\vct{z}),
 \end{align}
 where
 \begin{align*}
\elbar(\vct{v};\vct{z})&:=\frac{1}{n}\sum_{i=1}^n \ell\left(v_i;\vct{\theta}_0+\bSigma^{-1/2}\vct{z}\right)\nn\\
&= \frac{1}{2n}\twonorm{\vct{v}}^2+ \frac{\eps}{n}\onenorm{\vct{v}}\twonorm{\rev{\bJ}(\vct{\theta}_0+\bSigma^{-1/2}\vct{z})}+\frac{\eps^2}{2}\twonorm{\rev{\bJ}(\vct{\theta}_0+\bSigma^{-1/2}\vct{z})}^2.
 \end{align*}
The minimax optimization~\eqref{lin} is in a form that we can apply the CGMT framework. Formally, the CGMT framework concerns problems of the form
 \begin{align}
 \label{generalPO}
 \min_{\vct{z}\in\mathcal{S}_{\vct{z}}}\text{ }\max_{\vct{u}\in\mathcal{S}_{\vct{u}}}\quad \vct{u}^T\mtx{\Gb}\vct{z}+\psi(\vct{z},\vct{u}),
 \end{align}
 with $\Gb$ a matrix with i.i.d standard normal entries and shows that this problem is asymptotically equivalent to the following problem:
 \begin{align}
 \label{generalAO}
 \min_{\vct{z}\in\mathcal{S}_{\vct{z}}}\text{ }\max_{\vct{u}\in\mathcal{S}_{\vct{u}}}\quad\twonorm{\vct{z}}\vct{g}^T\vct{u}+\twonorm{\vct{u}}\vct{h}^T\vct{z}+\psi(\vct{z},\vct{u}),
 \end{align}
 where $\vct{g}$ and $\vct{h}$ are independent Gaussian vectors with i.i.d.~$\mathcal{N}(0,1)$ entries and $\psi(\vct{z},\vct{u})$ is convex in $\vct{z}$ and concave in $\vct{u}$. Here, the sets $\mathcal{S}_{\vct{z}}$ and $\mathcal{S}_{\vct{u}}$ are compact sets. We refer to \cite[Theorem 3]{thrampoulidis2015regularized} for precise statements regarding the equivalence of \eqref{generalPO} and \eqref{generalAO}. 
 
 Following \cite{thrampoulidis2015regularized} we shall refer to problems of the form \eqref{generalPO} as the \emph{Primal Problem (PO)} and  refer to problems of the form \eqref{generalAO} as the \emph{Auxiliary Problem (AO)}.

As described above the CGMT framework requires the minimization/maximization to be over compact sets. This technical issue can be avoided by a common trick in this literature where one introduces ``artificial'' boundedness constraint which do not effect the optimal solution. Specifically, following~\cite{thrampoulidis2015regularized} we can add constraints of the form $\cS_{\z} = \{\z|\;\; \twonorm{\vct{z}}\le K_\alpha\}$ and $\cS_{\vct{u}} = \{\vct{u}:\,\twonorm{\vct{u}}\le K_\beta \}$ for sufficiently large constants $K_\alpha$ and $K_\beta$ without changing the optimal solution of \eqref{lin} in a precise asymptotic sense. We leave out a detailed argument here and refer to \cite[Appendix B]{javanmard2020precise} for similar arguments. This allows us to replace \eqref{lin} with
 \begin{align}
 \label{linmod}
\min_{\z\in\mathcal{S}_{\vct{z}},\vct{v}\in\R^n} \max_{\vct{u}\in\mathcal{S}_{\vct{u}}}\quad \frac{1}{\sqrt{d}}\Big\{\vct{u}^\sT\Gb \vct{z}  - \vct{u}^\sT \w +\vct{u}^\sT \vct{v}\Big\}+\elbar(\vct{v};\vct{z}).
 \end{align}
 Observe that the above loss function has a bilinear term $\bu^\sT \Gb \z$, with $G_{ij}\sim \normal(0,1)$ independently, plus a function of the form
 \[
 \psi(\z,\bv,\bu) := \frac{1}{\sqrt{d}}\Big\{- \vct{u}^\sT \w +\vct{u}^\sT \vct{v}\Big\}+\elbar(\vct{v};\vct{z})\,,
 \]
which is jointly convex in $(\z,\bv)$ and concave in $\bu$. 
 
 Therefore the corresponding AO problem takes the following form 
 \begin{align}
   \label{finalAO}
 \min_{\vct{z}\in\cS_{\z}, \vct{v}} \max_{\vct{u}\in\cS_{\vct{u}}}\; \frac{1}{\sqrt{d}}\Big\{\twonorm{\vct{z}} \vct{g}^T\vct{u} + \twonorm{\vct{u}} \vct{h}^T\vct{z}- \vct{u}^\sT \w +\vct{u}^\sT \vct{v}\Big\}+\elbar(\vct{v};\vct{z})\,.
 \end{align}
 This concludes the derivation of the AO problem.
 
 \subsection{Scalarization of the AO problem}
 We next simplify the AO problem by considering this problem in the asymptotic regime. 
 We start by maximizing over $\vct{u}$. Write $\vct{u}=\beta\widetilde{\vct{u}}$ with $\widetilde{\vct{u}}\in\mathbb{S}^{n-1}$ and $0\le \beta\le K_\beta$. Using this decomposition we have
 \begin{align*}
 &\max_{\vct{u}\in\cS_{\vct{u}}}\; \twonorm{\vct{z}} \vct{g}^T\vct{u} + \twonorm{\vct{u}} \vct{h}^T\vct{z} - \vct{u}^T \w +\vct{u}^T \vct{v}\\\
 &= \max_{0\le \beta\le K_\beta}\text{ }\max_{\widetilde{\vct{u}}\in\mathbb{S}^{n-1}}\; \beta\twonorm{\vct{z}} \vct{g}^T\widetilde{\vct{u}} + \beta \vct{h}^T\vct{z}- \beta\widetilde{\vct{u}}^T \w +\beta\widetilde{\vct{u}}^T \vct{v}\\
  &= \max_{0\le \beta\le K_\beta}\text{ }\max_{\widetilde{\vct{u}}\in\mathbb{S}^{n-1}}\; \beta\widetilde{\vct{u}}^T\left(\twonorm{\vct{z}} \vct{g} - \w + \vct{v} \right)+{\beta} \vct{h}^T\vct{z}\\
  &= \max_{0\le \beta\le K_\beta}\text{ } \beta \twonorm{\twonorm{\vct{z}} \vct{g} 
   - \bw + \bv} +{\beta} \vct{h}^T\vct{z}.
 \end{align*}

 After substituting the above into AO problem~\eqref{finalAO}, it reads
  \begin{align*}
 \min_{\vct{z}\in\cS_{\z}, \vct{v}} \max_{0\le \beta\le K_\beta}\text{ } \frac{\beta}{\sqrt{d}} \twonorm{\twonorm{\vct{z}} \vct{g} 
   - \bw + \bv} +\frac{\beta}{\sqrt{d}} \vct{h}^T\vct{z}+\elbar(\vct{v};\vct{z}).
\end{align*}
  We next aim to simplify the minimization over $\vct{v}$ and $\vct{z}$, but a hurdle is that they are coupled through the term $\ell(\vct{v};\vct{z})$. To address this technical issue, we consider the conjugate of $ \elbar(\vct{v};\vct{z})$ in with respect to $\vct{z}$. That is,
 \begin{align*}
 \elbar(\vct{v};\vct{z})=\sup_{\vct{q}} \vct{q}^T\vct{z}-\widetilde{\ell}(\vct{v};\vct{q}).
 \end{align*}
The AO problem then can be written as
  \begin{align}
  \label{simpAO}
 \min_{\vct{z}\in\cS_{\vct{z}},\vct{v}}\max_{0\le\beta\le K_\beta, \vct{q}}\;\; \frac{\beta}{\sqrt{d}} \twonorm{\twonorm{\vct{z}} \vct{g} 
  - \bw + \bv} +{\frac{\beta}{\sqrt{d}}} \vct{h}^T\vct{z}+\vct{q}^T\vct{z}-\widetilde{\ell}(\vct{v};\vct{q}).
\end{align}
In the above optimization the order of minimization and maximization can be flipped using the Sion's theorem and the fact that the original PO problem is convex/concave in the min/max parameters. The argument only uses the convexity of the loss $\ell(\vct{v};\vct{q})$ and we refer to  \cite[Appendix A.2.4]{thrampoulidis2015precise} for a detailed argument.  This brings us to
  \begin{align*}
 \max_{0\le\beta\le K_\beta, \vct{q}}\;\; \min_{\vct{z}\in\cS_{\vct{z}},\vct{v}}\;\; \frac{\beta}{\sqrt{d}} \twonorm{\twonorm{\vct{z}} \vct{g} 
  - \bw + \bv} +{\frac{\beta}{\sqrt{d}}} \vct{h}^T\vct{z}+\vct{q}^T\vct{z}-\widetilde{\ell}(\vct{v};\vct{q}).
\end{align*}
We optimize over the direction and norm of $\z$ ($\twonorm{\z} = \alpha$) to get
\begin{align}\label{eq:AO-op1}
 \max_{0\le\beta\le K_\beta, \vct{q}}\;\; \min_{0\le\alpha\le K_{\alpha},\vct{v}}\;\; \frac{\beta}{\sqrt{d}} \twonorm{\alpha \vct{g} 
  - \bw + \bv} -\alpha \twonorm{\frac{\beta}{\sqrt{d}} \vct{h} +\vct{q}}-\widetilde{\ell}(\vct{v};\vct{q}).
\end{align}

Note that $\widetilde{\ell}(\bv;\qb)$ is convex in $\qb$ and so the AO objective~\eqref{eq:AO-op1} is clearly jointly concave in $\qb$ and $\beta$. Also since $\elbar$ is jointly convex in $(\vct{v},\vct{z})$, then $-\elbar(\vct{v};\vct{z})$ is jointly concave in $(\vct{v},\vct{z})$. Also $\vct{q}^T\vct{z}$ is jointly concave in $(\vct{v},\vct{z})$. Therefore,  $\vct{q}^T\vct{z}-\elbar(\vct{v};\vct{z})$ is jointly concave in $(\vct{v},\vct{z})$ and based on the partial maximization rule we can conclude that $\widetilde{\ell}(\vct{v};\vct{q})$ should be concave in $\vct{v}$. The other terms are also trivially jointly convex in $\alpha, \vct{v}$ so that overall the objective is jointly convex in $\alpha,\vct{v}$. Therefore, by virtue of Sion's min-max Theorem \cite{sion1958general}) we can change the order of the mins and maxs as we please and rewrite the AO problem as
\begin{align*}
\min_{0\le\alpha\le K_{\alpha},\vct{v}}\;\; \max_{0\le\beta\le K_\beta, \vct{q}}\;\; \;\; \frac{\beta}{\sqrt{d}} \twonorm{\alpha \vct{g} 
  - \bw + \bv} -\alpha \twonorm{\frac{\beta}{\sqrt{d}} \vct{h} +\vct{q}}-\widetilde{\ell}(\vct{v};\vct{q}).
\end{align*}

To continue we shall calculate the conjugate function $\widetilde{\ell}$. This is the subject of the next lemma.
\begin{lemma}\label{conjlemma}The conjugate of 
\begin{align*}
\elbar(\vct{v};\vct{z}):= \frac{1}{2n}\sum_{i=1}^n \left(\abs{v_i}+\eps\twonorm{\rev{\bJ} (\bth_0+\bSigma^{-1/2}\z)}\right)^2,
 \end{align*}
 with respect to the variable $\vct{z}$ is given by
\begin{align*}
\widetilde{\ell}(\vct{v};\vct{q}):=\sup_{\vct{z}} \;\;\vct{q}^T\vct{z}-\elbar(\vct{v};\vct{z})= -\<\bSigma^{1/2}\bth_0,\qb\> + 
\frac{1}{2} \Big(\frac{1}{\eps} \twonorm{\bSigma^{1/2} \rev{\bJ^{-1}}\qb} - \frac{1}{n}\onenorm{\bv} \Big)_+^2 -\frac{1}{2n}\twonorm{\bv}^2\,.
\end{align*}
\end{lemma}
We refer to Section~\ref{sec:aux-CGMT} for the proof of this lemma.
Plugging in for $\widetilde{\ell}$ from Lemma~\ref{conjlemma} in the AO problem we arrive at 
\begin{align}\label{eq:AO-op2}
\min_{0\le \alpha< K_\alpha,\vct{v}}\;\;  \max_{0\le\beta\le K_\beta, \vct{q}}\;\;&\frac{\beta}{\sqrt{d}} \twonorm{\alpha \vct{g} 
  - \bw + \bv} -\alpha \twonorm{\frac{\beta}{\sqrt{d}} \vct{h} +\vct{q}} \nn\\
&+\<\bSigma^{1/2}\bth_0,\qb\> - 
\frac{1}{2} \Big(\frac{1}{\eps} \twonorm{ \rev{\bJ^{-1}} \bSigma^{1/2}\qb} - \frac{\onenorm{\bv}}{n} \Big)_+^2 +\frac{1}{2n}\twonorm{\bv}^2\,.
\end{align}

\bigskip

{\bf Optimization over $\qb$:}
To simplify the AO problem further, we next focus on maximization over $\qb$. Consider the change of variable 
$\tqb : =\rev{\bJ^{-1}}\bSigma^{1/2}\qb$ and keep only the terms in the AO objective which involve $\qb$. 
\begin{align*}
&\max_{\qb}\;\; -\alpha \twonorm{\frac{\beta}{\sqrt{d}} \vct{h} +\vct{q}}
+\<\bSigma^{1/2}\bth_0,\qb\> - 
\frac{1}{2} \Big(\frac{1}{\eps} \twonorm{ \rev{\bJ^{-1}} \bSigma^{1/2}\qb} - \frac{\onenorm{\bv}}{n} \Big)_+^2\\
&=\max_{\tqb}\;\; -\alpha \twonorm{\frac{\beta}{\sqrt{d}} \vct{h} +\bSigma^{-1/2}\rev{\bJ} \tqb}
+\<\bth_0,\rev{\bJ}\tqb\> - 
\frac{1}{2} \Big(\frac{1}{\eps} \twonorm{\tqb} - \frac{\onenorm{\bv}}{n} \Big)_+^2\\
&=\max_{\tqb, 0\le \tau_q}\;\;  -\frac{\alpha}{2\tau_q} \twonorm{\frac{\beta}{\sqrt{d}} \vct{h} +\bSigma^{-1/2}\rev{\bJ} \tqb }^2 - \frac{\alpha \tau_q}{2}
+\<\bth_0,\rev{\bJ}\tqb\>  - 
\frac{1}{2} \Big(\frac{1}{\eps} \twonorm{\tqb} - \frac{\onenorm{\bv}}{n} \Big)_+^2\\
&=\max_{\tqb, 0\le \tau_q}\;\; 
-\frac{\alpha}{2\tau_q} \twonorm{\frac{\beta}{\sqrt{d}} \vct{h} +\bSigma^{-1/2}\rev{\bJ} \tqb - \frac{\tau_q}{\alpha} \bSigma^{1/2}\bth_0 }^2 + \frac{\tau_q}{2\alpha}\twonorm{\bSigma^{1/2}\bth_0}^2
-\<\frac{\beta}{\sqrt{d}} \bh,\bSigma^{1/2}\bth_0\> \\
&\quad\quad\quad\quad - \frac{\alpha \tau_q}{2}-\frac{1}{2} \Big(\frac{1}{\eps} \twonorm{\tqb} - \frac{\onenorm{\bv}}{n} \Big)_+^2
\end{align*}
We next maximize over $\tqb$ by introducing  a new dummy variable $\gamma$ for $\twonorm{\tqb}$. This brings us to the following problem
\begin{align}\label{eq:trust1}
\min_{\tqb, 0\le \gamma} \;\; &\frac{\alpha}{2\tau_q} \twonorm{\frac{\beta}{\sqrt{d}} \vct{h} +\bSigma^{-1/2}\rev{\bJ} \tqb - \frac{\tau_q}{\alpha} \bSigma^{1/2}\bth_0 }^2 +\frac{1}{2} \Big(\frac{1}{\eps} \gamma - \frac{\onenorm{\bv}}{n} \Big)_+^2 \\
\text{subject to } & \twonorm{\tqb} = \gamma\,\nn.
\end{align}
%

We continue by the following lemma and refer to Section~\ref{sec:aux-CGMT} for its proof.

\begin{lemma}\label{lem:trust}
Let $H\in \reals^{d\times d}$ be invertible and $\br\in \reals^d, c_0,c_1\in \reals$. Consider the following optimization problem:
\begin{align}\label{eq:opt-quad}
\min_{\tqb,0\le \gamma}\;\; & \frac{c_0}{2} \twonorm{\bH\tqb - \br}^2  + \frac{1}{2}\Big(\frac{1}{\eps}\gamma-c_1\Big)_+^2\\
\text{s.t.}\;\; & \twonorm{\tqb}= \gamma\nn
\end{align}
Define
\begin{align}\label{eq:Q}
Q(\bH,\br,\gamma) = \sup_{\lambda\ge0}\;\;\; \frac{\lambda}{2} \left(
\br^\sT (\bH\bH^\sT + \lambda \Iden)^{-1} \br - \gamma^2
\right).
\end{align}
Then, the optimal objective value of \eqref{eq:opt-quad} is given by 
\[
\min_{\gamma\ge 0}\;\; c_0Q(\bH,\br,\gamma) + \frac{1}{2}\Big(\frac{1}{\eps}\gamma-c_1\Big)_+^2\,.
\]  
\end{lemma}

Using Lemma~\ref{lem:trust}, the optimal value of \eqref{eq:trust1} is given by 
\[\min_{\gamma\ge0}\;\; \frac{\alpha}{\tau_q}Q(\bSigma^{-1/2}\rev{\bJ}, \frac{\tau_q}{\alpha}\bSigma^{1/2}\bth_0 - \frac{\beta}{\sqrt{d}}\bh,\gamma) + \frac12\left(\frac{\gamma}{\eps} - \frac{\onenorm{\bv}}{n}\right)_+^2\,.
\]
Therefore, by substituting in \eqref{eq:AO-op1} the AO optimization can be simplified as
\begin{align}\label{eq:AO-op3}
\min_{0\le \alpha< K_\alpha,\vct{v}}\;\;  \max_{0\le\beta\le K_\beta, 0\le\gamma,\tau_q}\;\;&\frac{\beta}{\sqrt{d}} \twonorm{\alpha \vct{g} 
  - \bw + \bv} - \frac{\alpha}{\tau_q}Q(\bSigma^{-1/2}\rev{\bJ}, \frac{\tau_q}{\alpha}\bSigma^{1/2}\bth_0 - \frac{\beta}{\sqrt{d}}\bh,\gamma) \nn\\
& +  \frac{\tau_q}{2\alpha}\twonorm{\bSigma^{1/2}\bth_0}^2
-\<\frac{\beta}{\sqrt{d}} \bh,\bSigma^{1/2}\bth_0\> 
 - \frac{\alpha \tau_q}{2}
-\frac{1}{2} \Big(\frac{\gamma}{\eps}- \frac{\onenorm{\bv}}{n} \Big)_+^2 +\frac{1}{2n}\twonorm{\bv}^2\,.
\end{align}
Before proceeding further with our simplification of the AO problem, let us state the following lemma which is used to discuss the convexity-concavity of the objective and justification of changing the order of maximization and minimization. We refer to Section~\ref{sec:aux-CGMT} for its proof.
\begin{lemma}\label{lem:conv-conc}
The function
\[
 f(\gamma,\beta,\tau_q):=\frac{1}{\tau_q}Q(\bSigma^{-1/2}\rev{\bJ}, \frac{\tau_q}{\alpha}\bSigma^{1/2}\bth_0 - \frac{\beta}{\sqrt{d}}\bh,\gamma) , 
\]
is jointly convex in the variables $(\gamma,\frac{\beta}{\sqrt{d}},\tau_q)$.
\end{lemma}


As a result this lemma, the objective~\eqref{eq:AO-op3} is jointly concave in $(\gamma,\beta,\tau_q)$. Also recall that since $\widetilde{\ell}$ was concave the objective~\eqref{eq:AO-op1} was jointly convex in $(\alpha,\bv)$. Since maximization (with respect to direction of $\tqb$) preserves convexity (pointwise maximum of convex functions is convex), therefore the objective~\eqref{eq:AO-op3} is jointly convex in $(\alpha,\bv)$.

Therefore, by another use of Sion's min-max theorem, we can change the order of min and max in~\eqref{eq:AO-op3} and write it equivalently as 
 
\begin{align}\label{eq:AO-op4}
 \max_{0\le\beta\le K_\beta, 0\le\gamma,\tau_q}\;\; \min_{0\le \alpha< K_\alpha,\vct{v}}\;\;  &\frac{\beta}{\sqrt{d}} \twonorm{\alpha \vct{g} 
  - \bw + \bv} - \frac{\alpha}{\tau_q}Q(\bSigma^{-1/2}\rev{\bJ}, \frac{\tau_q}{\alpha}\bSigma^{1/2}\bth_0 - \frac{\beta}{\sqrt{d}}\bh,\gamma) \nn\\
& +  \frac{\tau_q}{2\alpha}\twonorm{\bSigma^{1/2}\bth_0}^2
-\frac{1}{\sqrt{d}}\<\beta \bh,\bSigma^{1/2}\bth_0\> 
 - \frac{\alpha \tau_q}{2}
-\frac{1}{2} \Big(\frac{\gamma}{\eps}- \frac{\onenorm{\bv}}{n} \Big)_+^2 +\frac{1}{2n}\twonorm{\bv}^2\,.
\end{align} 

We next focus on minimization over $\bv$. Keeping only the terms in~\eqref{eq:AO-op3} that depend on $\bv$ we have
\begin{align}
\label{eq10}
&\min_{\vct{v}}\;\;  {\frac{\beta}{\sqrt{d}}} \twonorm{\alpha\vct{g}-\w+\vct{v}}+\frac{1}{2n}\twonorm{\vct{v}}^2 -\frac{1}{2}\left(\frac{\gamma}{\eps}-\frac{\onenorm{\vct{v}}}{n}\right)_+^2\nn\\
&= \min_{\tau_g\ge0,\vct{v}}\;\;  \frac{\beta}{2n\tau_g } \twonorm{\alpha\vct{g}-\w+\vct{v}}^2 +\frac{\beta\tau_g n}{2d} +\frac{1}{2n}\twonorm{\vct{v}}^2 -\frac{1}{2n^2}\left(\frac{n\gamma}{\eps}-\onenorm{\vct{v}}\right)_+^2.
\end{align}
Recall the definition of the Moreau envelope function of a function $f$ at a point $\vct{x}$ with parameter $\mu$,
\[
e_f(\vct{x};\rho)\equiv \min_{\bv} \frac{1}{2\rho} \twonorm{\vct{x}-\bv}^2+ f(\bv)\,.
\]
and define
\begin{align}\label{eq:f}
f(\bv;\gamma):= \frac{1}{2} \twonorm{\bv}^2 -\frac{1}{2n} (\frac{n}{\eps}\gamma-\onenorm{\bv})_+^2\,.
\end{align}
Note that $f(\bv;\gamma)$ is convex in $\bv$ (since $-\widetilde{\ell}(\bv;\vct{q})$ was convex in $\bv$). Thus, \eqref{eq10} can be rewritten in the more compact form
\begin{align}\label{eq:dum0}
\min_{\tau_g\ge0}\;\;  &\frac{1}{n} e_f\left(\w-\alpha\vct{g};\frac{\tau_g}{\beta}\right) + \frac{\beta\tau_g}{2}\frac{n}{d}.
  \end{align}
We next invoke the result of~\cite[Lemma 6.3]{javanmard2020precise} which gives a characterization of the Moreau envelope function $e_f(\bx;\mu)$. 

\begin{lemma}\label{meenv} (\cite[Lemma 6.3]{javanmard2020precise})
Consider the function $f$ given by~\eqref{eq:f}. Then,
\begin{align*}
e_f(\vct{x};\rho) &= \frac{1}{2(\rho+1)}\twonorm{\bx}^2 +
\min_{\nu\ge 0} G_n(\bx;\rho,\gamma,\nu),
\end{align*} 
where
\begin{align}
&G_n(\bx;\rho,\gamma,\nu) = \frac{1}{2\rho(\rho+1)}\twonorm{\vct{x}-\ST(\bx;\nu)}^2-\frac{1}{2n}\left(\frac{n}{\eps}\gamma-\frac{1}{1+\rho}\onenorm{\ST(\bx;\nu)}\right)_+^2,\label{eq:Gchar}
\end{align}
\rev{and $\ST(\bx;\nu)$ is the soft-thresholding function defined as
\begin{align*}
[\ST(\bx;\nu)]_i = \begin{cases}
x_i-\lambda, &\text{if }\quad x_i\ge \lambda,\\
0 &\text{if }\quad |x_i|\le \lambda,\\ 
x_i+\lambda &\text{if }\quad x_i\le -\lambda,
\end{cases}
\end{align*}
for each coordinate $i$.} Furthermore, $e_f(\bx;\tau)$ is strictly convex in $\bx$.
\end{lemma}

Using this characterization in~\eqref{eq10} we get
 \begin{align}
\label{eq11}
&\min_{\vct{v}}\;\;  {\frac{\beta}{\sqrt{d}}} \twonorm{\alpha\vct{g}-\w+\vct{v}}+\frac{1}{2n}\twonorm{\vct{v}}^2 -\frac{1}{2}\left(\frac{\gamma}{\eps}-\frac{\onenorm{\vct{v}}}{n}\right)_+^2\nn\\
&= \min_{\tau_g\ge0}\;\;  \frac{1}{n} e_f\left(\w-\alpha\vct{g};\frac{\tau_g}{\beta}\right) + \frac{\beta\tau_g}{2}\frac{n}{d}\nn\\
&=\min_{\tau_g\ge0}\;\;  \frac{\beta\tau_g}{2}\frac{n}{d} + \frac{1}{n} \frac{\beta}{2(\tau_g+\beta)}\twonorm{\w-\alpha\bg}^2 +
\frac{1}{n}\min_{\nu\ge 0} G_n(\w-\alpha\bg;\frac{\tau_g}{\beta},\gamma,\nu)  \,.
\end{align}
Next by plugging \eqref{eq11} in \eqref{eq:AO-op4} we arrive at the following AO formulation:
\begin{align}\label{eq:AO-op5}
 \max_{0\le\beta\le K_\beta, 0\le\gamma,\tau_q}\;\; \min_{0\le \alpha< K_\alpha, 0\le \tau_g,\nu}\;\;  & - \frac{\alpha}{\tau_q}Q(\bSigma^{-1/2}\rev{\bJ}, \frac{\tau_q}{\alpha}\bSigma^{1/2}\bth_0 - \frac{\beta}{\sqrt{d}}\bh,\gamma)  +  \frac{\tau_q}{2\alpha}\twonorm{\bSigma^{1/2}\bth_0}^2
 - \frac{\alpha \tau_q}{2}\nn\\
& -\frac{1}{\sqrt{d}}\<\beta \bh,\bSigma^{1/2}\bth_0\> +\frac{\beta\tau_g}{2}\frac{n}{d} + \frac{\beta}{2(\tau_g+\beta)} \frac{1}{n}\twonorm{\w-\alpha\bg}^2 \\
& +
\frac{1}{n} G_n(\w-\alpha\bg;\frac{\tau_g}{\beta},\gamma,\nu) 
\,.
\end{align} 
Recall that the problem \eqref{eq10} was jointly convex in $(\vct{v},\alpha,\tau_g)$ and \eqref{eq:AO-op4} jointly concave in $(\beta,\gamma,\tau_q)$. Since partial minimization preserves convexity we therefore conclude that the objective~\eqref{eq:AO-op5} is jointly convex in $(\alpha,\tau_g)$ and jointly concave in $(\beta,\gamma,\tau_q)$ (after the minimization over $\nu\ge 0$ has been carried out). 
\subsection{Convergence analysis of the AO problem}
\subsubsection{Pointwise convergence}\label{sec:pointwise}
We next derive the pointwise limit of the AO objective in the asymptotic
regime that $N/d\to \psi_1$ and $n/d\to \psi_2$, as $n\to\infty$.

Recalling the definition of $\bth_0$ from~\eqref{equiv} we have
\[
\twonorm{\bSigma^{1/2}\bth_0}^2 = \mu_1^2\twonorm{\bSigma^{-1/2}\bW\bbeta} ^2= 
\frac{\mu_1^2\twonorm{\bbeta}^2}{d}\tr(\bW^\sT\bSigma \bW)\,,
\]
where we used the fact that the distribution of $\W$ is rotationally invariant. By our assumption $\twonorm{\bbeta}\to 1$. Let $0\le s_1, \dots, s_N$ denote the eigenvalues of $\bW\bW^\sT$. By invoking the definition of $\bSigma$ from~\eqref{equiv} we have
\begin{align}
\twonorm{\bSigma^{1/2}\bth_0}^2  &= \frac{\psi_1\rev{\mu_1^2}\twonorm{\bbeta}^2}{N} \rev{\tr}(\bW^\sT\bSigma^{-1} \bW)\nn\\
& = \frac{\psi_1\mu_1^2\twonorm{\bbeta}^2}{N} \sum_{i=1}^d \frac{s_i}{\mu_1^2 s_i+\mu_2^2}\nn\\
 &=\frac{\psi_1 \mu_1^2\twonorm{\bbeta}^2}{N} \sum_{i=1}^d \frac{1}{\mu_1^2} \left(1 - \frac{\mu_2^2/\mu_1^2}{s_i+\mu_2^2/\mu_1^2}\right) \to {\psi_1} \Big(1 +\frac{\mu_2^2}{\mu_1^2} S\Big(-\frac{\mu_2^2}{\mu_1^2};\psi_1\Big)\Big)\,, \label{eq:Sigma-th-norm}
\end{align}
in probability where $S(z) = \int \frac{\rho(s)}{z- s}\de s$ is the Stieltjes transform of the spectral density $\rho$ of the matrix $\W\W^\sT$. The formula for $S(z)$ is given in Proposition~\ref{propo:stiel} and since it is a function of $\psi_1$, we make this dependence clear in the notation and write $S(z;\psi_1)$ henceforth. 

Since for our activation $\mu_1 = \frac{1}{2}$ and $\mu_2 = \sqrt{\frac{1}{4} - \frac{1}{2\pi}}$, this simplifies to 
\begin{align}\label{eq:Sigma-th0}
\twonorm{\bSigma^{1/2}\bth_0}^2 \to \psi_1\left(1+\Big(1-\frac{2}{\pi}\Big) S\Big(\frac{2}{\pi}-1;\psi_1\Big)\right)\,, 
\end{align}
in probability. This together with \eqref{equiv} implies that
\begin{align}\label{sigma-formula}
\sigma^2= \tau^2 +\twonorm{\bbeta}^2- \twonorm{\bSigma^{-1/2}\bth_0}^2\to
\tau^2+1 -  \psi_1\left(1+\Big(1-\frac{2}{\pi}\Big) S\Big(\frac{2}{\pi}-1;\psi_1\Big)\right)\,.
\end{align}

We next note that since $\twonorm{\bSigma^{1/2}\bth_0} = O(1)$, for $\bh\sim \normal(0,\Iden)$ we have 
\begin{align}\label{eq:cross}
\frac{1}{\sqrt{n}}\<\bh,\bSigma^{1/2}\bth_0 \> \to 0, 
\end{align}
in probability, as $n\to\infty$. In addition, since $\bg\sim\normal(0,\Iden_n)$ and $\bw\sim \normal(0,\sigma^2\Iden_n)$, we have
\begin{align}\label{eq:g-w}
\frac{1}{n} \twonorm{\alpha\bg - \bw}^2 \to \alpha^2+ \sigma^2\,,
\end{align}
in probability. 

We next proceed by calculating the limit of the $Q$ function. 
By definition,
\begin{align}
&Q(\bSigma^{-1/2}\rev{\bJ}, \frac{\tau_q}{\alpha}\bSigma^{1/2}\bth_0 - \frac{\beta}{\sqrt{d}}\bh,\gamma) \nn\\
&= \sup_{\lambda\ge0}\;\;\; \frac{\lambda}{2} \left[
(\frac{\tau_q}{\alpha}\bSigma^{1/2}\bth_0 - \frac{\beta}{\sqrt{d}}\bh)^\sT (\bSigma^{-1/2}\rev{\bJ^2}\bSigma^{-1/2} + \lambda \Iden)^{-1} (\frac{\tau_q}{\alpha}\bSigma^{1/2}\bth_0 - \frac{\beta}{\sqrt{d}}\bh) - \gamma^2
\right]\nn\\
&= \sup_{\lambda\ge0}\;\;\; \frac{\lambda}{2} \left[
(\frac{\tau_q}{\alpha}\bSigma\bth_0 - \frac{\beta}{\sqrt{d}}\bSigma^{1/2}\bh)^\sT (\rev{\bJ^2} + \lambda \bSigma)^{-1} (\frac{\tau_q}{\alpha}\bSigma\bth_0 - \frac{\beta}{\sqrt{d}}\bSigma^{1/2}\bh) - \gamma^2
\right].\label{eq:conv1}
\end{align}
We compute the limit of the right hand side for any fixed value of $\lambda\ge 0$. First note that since $\opnorm{(\bSigma^{-1/2}\rev{\bJ^{2}}\bSigma^{-1/2} + \lambda \Iden)^{-1}}=O_p(1)$ and by invoking \eqref{eq:cross}, the cross terms vanish in the limit and we have
\begin{align}
&\lim_{n\to\infty} (\frac{\tau_q}{\alpha}\bSigma^{1/2}\bth_0 - \frac{\beta}{\sqrt{d}}\bh)^\sT (\bSigma^{-1/2}\rev{\bJ^{2}}\bSigma^{-1/2} + \lambda \Iden)^{-1} (\frac{\tau_q}{\alpha}\bSigma^{1/2}\bth_0 - \frac{\beta}{\sqrt{d}}\bh) \nn\\
&=\lim_{n\to\infty} \frac{\tau_q^2}{\alpha^2}\bth_0\bSigma^{1/2}  (\bSigma^{-1/2}\rev{\bJ^{2}}\bSigma^{-1/2} + \lambda \Iden)^{-1} \bSigma^{1/2}\bth_0 + \lim_{n\to\infty}  \frac{\beta^2}{d}\bh^\sT (\bSigma^{-1/2}\rev{\bJ^{2}}\bSigma^{-1/2} + \lambda \Iden)^{-1} \bh.\label{eq:conv2}
\end{align}
We treat each term separately. Plugging for $\bth_0$ from~\eqref{equiv} we have
\begin{align}
\lim_{n\to\infty} \frac{\tau_q^2}{\alpha^2}\bth_0\bSigma^{1/2}  (\bSigma^{-1/2}\rev{\bJ^{2}}\bSigma^{-1/2} + \lambda \Iden)^{-1} \bSigma^{1/2}\bth_0  &= \lim_{n\to\infty}  (\frac{\mu_1\tau_q}{\alpha})^2\bbeta^\sT \W^\sT (\rev{\bJ^{2}}+\lambda\bSigma)^{-1}\W\bbeta\nn\\
&= \lim_{n\to\infty} (\frac{\mu_1\tau_q}{\alpha})^2 \frac{1}{d}\tr(\W^\sT (\rev{\bJ^{2}}+\lambda\bSigma)^{-1}\W)\,,\label{eq:conv3}
\end{align}
where in the last step we used the fact that the distribution of $\bW$ is rotationally invariant and $\twonorm{\bbeta}\to 1$.

Similarly since $\bh\sim\normal(0,\Iden_N)$ we have
\begin{align}
\lim_{n\to\infty}  \frac{\beta^2}{d}\bh^\sT (\bSigma^{-1/2}\rev{\bJ^{2}}\bSigma^{-1/2} + \lambda \Iden)^{-1} \bh &= \lim_{n\to\infty} \frac{\beta^2}{d} \<\bSigma^{1/2}(\rev{\bJ^{2}}+\lambda\bSigma)^{-1} \bSigma^{1/2}, \bh\bh^\sT\>\nn\\
&= \lim_{n\to\infty} \frac{\beta^2}{d} \tr(\bSigma^{1/2}(\rev{\bJ^{2}}+\lambda\bSigma)^{-1} \bSigma^{1/2}).\label{eq:conv4}
\end{align}
Combining \eqref{eq:conv3} and \eqref{eq:conv4} into \eqref{eq:conv2} we get
\begin{align}
&\lim_{n\to\infty} (\frac{\tau_q}{\alpha}\bSigma^{1/2}\bth_0 - \frac{\beta}{\sqrt{d}}\bh)^\sT (\bSigma^{-1/2}\rev{\bJ^{2}}\bSigma^{-1/2} + \lambda \Iden)^{-1} (\frac{\tau_q}{\alpha}\bSigma^{1/2}\bth_0 - \frac{\beta}{\sqrt{d}}\bh) \nn\\
&=\lim_{n\to\infty} \tr\left\{(\rev{\bJ^{2}}+\lambda\bSigma)^{-1} \Big(\Big(\frac{\mu_1\tau_q}{\alpha}\Big)^2 \frac{1}{d} \W\W^\sT + \frac{\beta^2}{d}  \bSigma\Big) \right\}\,,
\label{eq:conv5}
\end{align}
where we used that $\tr(\bA \bB) = \tr(\bB\bA)$ for any two matrices $\bA$ and $\bB$.

To calculate the limit on the right-hand side of \eqref{eq:conv5}, we use Proposition~\ref{pro:spectral} on the spectrum of inner product kernel random matrices. 
 
 Since $\opnorm{\W} = O_p(1)$ we also have $\opnorm{(\frac{\mu_1\tau_q}{\alpha})^2 \frac{1}{d} \W\W^\sT + \frac{\beta^2}{d}  \bSigma} = O_p(1)$ and as an immediate corollary of Proposition~\ref{pro:spectral}, in~\eqref{eq:conv5} we can replace 
$\rev{\bJ^{2}}$ with $\bK$ since they have the same spectrum. This brings us to
\begin{align}
&\lim_{n\to\infty} (\frac{\tau_q}{\alpha}\bSigma^{1/2}\bth_0 - \frac{\beta}{\sqrt{d}}\bh)^\sT (\bSigma^{-1/2}\rev{\bJ^{2}}\bSigma^{-1/2} + \lambda \Iden)^{-1} (\frac{\tau_q}{\alpha}\bSigma^{1/2}\bth_0 - \frac{\beta}{\sqrt{d}}\bh) \nn\\
&=\lim_{n\to\infty} \tr\left\{(\bK+\lambda\bSigma)^{-1} \Big(\Big(\frac{\mu_1\tau_q}{\alpha}\Big)^2 \cdot\frac{1}{d} \W\W^\sT + \frac{\beta^2}{d}  \bSigma\Big) \right\}\nn\\
&= \lim_{n\to\infty} \frac{1}{d}\tr\left\{\Big(\Big(\frac{1}{4}+\lambda\mu_1^2\Big)\W\W^\sT + \Big(\frac{1}{4}+\lambda\mu_2^2\Big)\Iden\Big)^{-1} 
\Big(\Big(\frac{\mu_1^2\tau_q^2}{\alpha^2} + {\beta^2\mu_1^2}\Big) \W\W^\sT + {\beta^2\mu_2^2}  \Iden\Big)
\right\}.
\label{eq:conv6}
\end{align}
Note that the latter only depends on the spectral density of $\W\W^\sT$ and can be written in terms of its Stieltjes transform.

By \rev{law of large numbers} and with simple algebraic manipulations it is easy to see that for any constants $b_0,b_1,c_0, c_1$ we have
\begin{align}
\frac{1}{N}\sum_{i=1}^N \frac{b_0 s_i+b_1}{c_0s_i+c_1} \to \frac{b_0}{c_0} + \frac{b_0\frac{c_1}{c_0} - b_1}{c_0} S(-c_1/c_0;\psi_1)\,,
\end{align}
with $S(t;\psi_1)$ representing the Stieltjes transform of the spectral density of $\W\W^\sT$.

Using this with~\eqref{eq:conv6} we obtain
\begin{align}
&\lim_{n\to\infty} (\frac{\tau_q}{\alpha}\bSigma^{1/2}\bth_0 - \frac{\beta}{\sqrt{d}}\bh)^\sT (\bSigma^{-1/2}\rev{\bJ^{2}}\bSigma^{-1/2} + \lambda \Iden)^{-1} (\frac{\tau_q}{\alpha}\bSigma^{1/2}\bth_0 - \frac{\beta}{\sqrt{d}}\bh) \nn\\
& =  \frac{4\psi_1}{1+4\lambda\mu_1^2} \Big(\frac{\mu_1^2\tau_q^2}{\alpha^2} + {\beta^2\mu_1^2}\Big) + \frac{4\psi_1}{1+4\lambda\mu_1^2} \left\{\Big(\frac{\mu_1^2\tau_q^2}{\alpha^2} + {\beta^2\mu_1^2}\Big) \Big(\frac{1+4\lambda \mu_2^2}{1+4\lambda \mu_1^2} \Big) - {\beta^2\mu_2^2}  \right\} S\Big(-\frac{1+4\lambda \mu_2^2}{1+4\lambda \mu_1^2};\psi_1\Big).
\label{eq:conv7}
\end{align}

By combining~\eqref{eq:conv7} and \eqref{eq:conv1} we get
\begin{align}
&\lim_{n\to\infty}Q(\bSigma^{-1/2}\rev{\bJ}, \frac{\tau_q}{\alpha}\bSigma^{1/2}\bth_0 - \frac{\beta}{\sqrt{n}}\bh,\gamma) \nn\\
&= \sup_{\lambda\ge0}\;\;\; \frac{\lambda}{2} \left[
 \frac{4\psi_1}{1+4\lambda\mu_1^2} \Big(\frac{\mu_1^2\tau_q^2}{\alpha^2} + {\beta^2\mu_1^2}\Big) + \frac{4\psi_1}{1+4\lambda\mu_1^2} \left\{\Big(\frac{\mu_1^2\tau_q^2}{\alpha^2} + {\beta^2\mu_1^2}\Big) \Big(\frac{1+4\lambda \mu_2^2}{1+4\lambda \mu_1^2} \Big) - {\beta^2\mu_2^2}  \right\} S\Big(-\frac{1+4\lambda \mu_2^2}{1+4\lambda \mu_1^2};\psi_1\Big)
 - \gamma^2
\right].\label{eq:conv8}
\end{align}

Plugging for $\mu_1 = \frac{1}{2}$ and $\mu_2 = \sqrt{\frac{1}{4}-\frac{1}{2\pi}}$ we have
\begin{align}\label{eq:Qlim}
&\lim_{n\to\infty}Q(\bSigma^{-1/2}\rev{\bJ}, \frac{\tau_q}{\alpha}\bSigma^{1/2}\bth_0 - \frac{\beta}{\sqrt{n}}\bh,\gamma) = 
\sF\Big(\frac{\tau_q}{\alpha} , \beta
, \psi_1, \gamma \Big),
\end{align}
with the definition
\begin{align}
\sF(a,b,\psi_1,\gamma): = \sup_{\lambda\ge0}\;\;
\frac{\lambda\psi_1}{2(1+\lambda)}\left\{a^2+b^2 +\Big(a^2\Big(1 - \frac{2}{\pi}\frac{\lambda}{1+\lambda}\Big)+\frac{2b^2}{\pi(1+\lambda)}\Big) S\Big(\frac{2}{\pi}\frac{\lambda}{1+\lambda} - 1;\psi_1\Big)\right\}-\frac{\lambda}{2}\gamma^2 \,.
\end{align}
By the change of variable $\tilde{\lambda} = \frac{\lambda}{1+\lambda}$, the function $\sF$ can be written as
\begin{align}
\sF(a,b,\psi_1,\gamma): = \sup_{0\le \tilde{\lambda}< 1}\;\;
\frac{\tilde{\lambda}\psi_1}{2}\left\{a^2+b^2 +\Big(a^2\Big(1 - \frac{2}{\pi}\tilde{\lambda} \Big)+\frac{2(1- \tilde{\lambda})b^2}{\pi}\Big) S\Big(\frac{2}{\pi}\tilde{\lambda} - 1;\psi_1\Big)\right\}
-\frac{\tilde{\lambda}}{2(1- \tilde{\lambda})}\gamma^2 \,.
\end{align}
We next proceed to characterize the limit of $\frac{1}{n} G_n(\w-\alpha\bg;\frac{\tau_g}{\beta},\gamma,\nu)$. To this end, we recall the result of~\cite[Lemma 6.4]{javanmard2020precise}.
\begin{lemma}\label{lem:Glim}
Let $\vct{u}\in\R^n$ be a Gaussian random vector distributed as \rev{$\normal(\vct{0},\omega^2\mtx{I}_n)$}. Then,
\begin{align}
\lim_{n\to\infty}  \frac{1}{2n\rho(\rho+1)}\twonorm{\vct{u}-\ST(\bu;\nu)}^2 
&= \frac{\omega^2}{2\rho(\rho+1)}\left(\left(1-\sqrt{\frac{2}{\pi}}\frac{\nu}{\omega} e^{-\frac{\nu^2}{2\omega^2}}\right)\right)\,,\\
\lim_{n\to\infty} \frac{1}{2n^2}\left(\frac{n}{\eps}\gamma-\frac{1}{1+\rho}\onenorm{\ST(\bu;\nu)}\right)_+^2
&=\frac{\omega^2}{2(\rho+1)^2}\left(\frac{\gamma(\rho+1)}{\eps\omega}+\frac{\nu}{\omega}\cdot\erfc\left(\frac{1}{\sqrt{2}}\frac{\nu}{\omega}\right)-\sqrt{\frac{2}{\pi}} e^{-\frac{\nu^2}{2\omega^2}}\right)_+^2.\label{eq:v0}
\end{align}
Therefore, by~\eqref{eq:Gchar} we have
\begin{align*}
\underset{n \rightarrow \infty}{\lim}\text{ }\frac{1}{n}G_n(\vct{u};\rho,\gamma,\nu)= \; &\frac{\omega^2}{2\rho(\rho+1)}\left(\left(1-\sqrt{\frac{2}{\pi}}\frac{\nu}{\omega} e^{-\frac{\nu^2}{2\omega^2}}\right)+\left(\frac{\nu^2}{\omega^2}-1\right)\erfc\left(\frac{1}{\sqrt{2}}\frac{\nu}{\omega}\right)\right)\\
&-\frac{\omega^2}{2(\rho+1)^2}\left(\frac{\gamma(\rho+1)}{\eps\omega}+\frac{\nu}{\omega}\cdot\erfc\left(\frac{1}{\sqrt{2}}\frac{\nu}{\omega}\right)-\sqrt{\frac{2}{\pi}} e^{-\frac{\nu^2}{2\omega^2}}\right)_+^2.
\end{align*}
Furthermore, 
\begin{align*}
&\underset{\nu\ge 0}{\min}\text{ }\underset{n \rightarrow \infty}{\lim}\text{ }\frac{1}{n}G_n(\vct{u};\rho,\gamma,\nu)\\
&=\sG(\omega;\rho,\gamma) :=
\begin{cases}
0 \quad &\text{ if }\gamma(\rho+1)\le \sqrt{\frac{2}{\pi}} \eps\omega\\
\frac{\omega^2}{2\rho(\rho+1)}\left(\erf\left(\frac{\nu^*\left(\frac{\gamma(\rho+1)}{\eps\omega},\rho\right)}{\sqrt{2}}\right)-\frac{\gamma(\rho+1)}{\eps\omega}\nu^*\left(\frac{\gamma(\rho+1)}{\eps\omega},\rho\right)\right)&\text{ if }\gamma(\rho+1)>\sqrt{\frac{2}{\pi}} \eps \omega
\end{cases}
\end{align*}
where $\nu^*(a,\rho)$ is the unique solution to 
\begin{align*}
a-\frac{1}{\rho}\nu-\nu\cdot\erf\left(\frac{\nu}{\sqrt{2}}\right)-\sqrt{\frac{2}{\pi}} e^{-\frac{\nu^2}{2}}= 0\,.
\end{align*}
\end{lemma}
Combining~\eqref{eq:Sigma-th0}, \eqref{eq:cross}, \eqref{eq:g-w}, \eqref{eq:Qlim}, and Lemma~\ref{lem:Glim}, we obtain the following scalarized AO problem:
\begin{align}\label{eq:AO-final}
 \max_{0\le\beta\le K_\beta, 0\le\gamma,\tau_q}\;\; \min_{0\le \alpha< K_\alpha, 0\le \tau_g}\;\;  & - \frac{\alpha}{\tau_q}
 \sF\Big(\frac{\tau_q}{\alpha} , {\beta}
, \psi_1, \gamma \Big)  +  \frac{\tau_q}{2\alpha} (\tau^2+1-\sigma^2)
 - \frac{\alpha \tau_q}{2}\nn\\
& +\frac{\beta\tau_g}{2}\psi_2 + \frac{\beta}{2(\tau_g+\beta)} (\sigma^2+\alpha^2) +
\sG\Big(\sqrt{\sigma^2+\alpha^2};\frac{\tau_g}{\beta},\gamma\Big) 
\,,
\end{align} 
where $\sigma^2 = \tau^2 +1 - \psi_1\left(1+\Big(1-\frac{2}{\pi}\Big) S\Big(\frac{2}{\pi}-1;\psi_1\Big)\right)$.

We conclude this part by a lemma on the convexity-concavity of the above scalarized AO problem and the uniqueness of the solution to the AO problem.
\begin{lemma}({\bf Strict convexity and uniqueness of the solution})\label{lem:SAO-CC}
The objective function~\eqref{eq:AO-final} is strictly jointly convex in $(\alpha,\tau_g)$ and jointly concave in $(\beta,\gamma,\tau_q)$. Also the solution $(\alpha_*, \frac{\tau_{g*}}{\beta_*})$  to this problem is unique. 
\end{lemma}
We defer the proof of Lemma~\ref{lem:SAO-CC} to Section~\ref{sec:aux-CGMT}. This concludes the proof of Theorem~\ref{thm:main}(a).

%
%
%
%
%
\subsubsection{Uniform convergence}
In Section~\ref{sec:pointwise} we showed that the objective function in \eqref{eq:AO-op5} converges point-wise to the
objective function in \eqref{eq:AO-final}. \rev{However, for our goal we need to show that the 
 minimax solutions of the converging sequence of the objectives in \eqref{eq:AO-op5} converges 
to the minimax solution of the AO objective in \eqref{eq:AO-final}, denoted by $\cR(\alpha,\tau_g, \beta,\gamma,\tau_q) $. }
Convexity/concavity of $\cR$ plays a crucial role here since it is being
used to conclude local uniform convergence from the point-wise convergence.

This can be shown by following
similar arguments as in \cite[Lemma A.5]{thrampoulidis2015precise} that is essentially based on a result known as “convexity
lemma” in the literature (see e.g. \cite[Lemma 7.75]{StatDecision}) by which point-wise convergence of convex
functions, of a finite number of variables, implies uniform convergence in compact subsets. Since the argument here is general, we leave out a detailed discussion and refer to \cite[Lemma A.5]{thrampoulidis2015precise}.

 \subsection{Proof of Theorem~\ref{thm:main}(b)}


In Proposition~\ref{pro:risk-equi} we gave a characterization of $\ARoo_{\rm nl}$. We first provide an alternative characterization in terms of the equivalent model of~\eqref{equiv}. 

Recall the key quantity $a$ from Proposition~\ref{pro:risk-equi}, given by
\begin{align}
a^2 = \tau^2 + \twonorm{\frac{1}{2}\bW^\sT\bth - \bbeta}^2 + \Big(\frac{1}{4}-\frac{1}{2\pi}\Big) \twonorm{\bth}^2\,.
\end{align}
We claim that $a^2 = \sigma^2+ \twonorm{\bSigma^{1/2}(\bth-\bth_0)}^2$.
To see this, we expand this expression as follows:
\begin{align*}
\sigma^2+ \twonorm{\bSigma^{1/2}(\bth-\bth_0)}^2
& = \sigma^2 + \<\bth,\bSigma\bth\>+ \<\bth_0,\bSigma\bth_0\>
-2\<\bth_0,\bSigma\bth\>\nn\\
&= \sigma^2+\mu_1^2 \twonorm{\bW^\sT\bth}^2+ \mu_2^2 \twonorm{\bth}^2
+\mu_1^2\bbeta^\sT \bW^\sT \bSigma^{-1} \bW\bbeta - 2\mu_1 \<\bW,\bbeta,\bth\>\\
&=\sigma^2+ \twonorm{\mu_1\bW^\sT\bth-\bbeta}^2 - \twonorm{\bbeta}^2
+\mu_2^2\twonorm{\bth}^2 +\mu_1^2\bbeta^\sT \bW^\sT \bSigma^{-1} \bW\bbeta \\
&= \tau^2+ \twonorm{\mu_1\bW^\sT\bth-\bbeta}^2+\mu_2^2\twonorm{\bth}^2\,,
\end{align*}
where we used the definition of $\bSigma$, $\bth_0$ and $\sigma^2$ as per \eqref{equiv}. The claim follows by recalling that for the shifted Relu activation, $\mu_1 =\frac{1}{2}$ and $\mu_2 = \sqrt{\frac{1}{4}-\frac{1}{2\pi}}$.

By the above characterization of quantity $a$ we obtain
\begin{align}\label{eq:SRlim2}
a^2 &= \sigma^2+ \twonorm{\bSigma^{1/2}(\bth-\bth_0)}^2\,.
\end{align} 
We next note that by definition of the variables in the AO problem, we have $\bz = \bSigma^{1/2}(\bth - \bth_0)$ and $\alpha = \twonorm{\bz}$. Therefore,
\begin{align*}
\lim_{n\to\infty} \twonorm{\bSigma^{1/2}(\bth - \bth_0)} = \alpha_*\,.
\end{align*}
Invoking the limit of $\sigma^2$ given by~\eqref{sigma-formula}, we get 
\begin{align}
\lim_{n\to\infty} a^2 = \tau^2 + 1 - \psi_1\left(1+\Big(1-\frac{2}{\pi}\Big) S\Big(\frac{2}{\pi}-1\Big)\right) + \alpha_*^2\,.
\end{align} 
We next characterize $\lim_{n\to\infty} \rev{\twonorm{\bJ \bth}}$. We will use the same AO problem to calculate this quantity. 

Recall that $\widehat{\z} = \bSigma^{1/2} (\hth^*_{\rm nl} - \bth_0)$ satisfies the following relation with $\qb_*$ the optimizer in~\eqref{eq:AO-op2}:
\[
\qb_* = \arg\max_{\qb} \qb^\sT \widehat{\z} - \widetilde{\ell}(\bv;\qb)\,,
\] 
where $\widetilde{\ell}(\bv;\qb)$ is the convex conjugate of $\bar{\ell}(\bv;\z)$. Since conjugate of a conjugate function is the function itself we then have
\begin{align}
\widehat{\z} &= \arg\max_{\z} \qb_*^\sT \z - \bar{\ell}(\bv;\z)\nn\\
&= \arg\max_{\z} \qb_*^\sT\z  -  \frac{1}{2n}\twonorm{\vct{v}}^2- \frac{\eps}{n}\onenorm{\vct{v}}\twonorm{\rev{\bJ}(\vct{\theta}_0+\bSigma^{-1/2}\vct{z})}-\frac{\eps^2}{2}\twonorm{\rev{\bJ}(\vct{\theta}_0+\bSigma^{-1/2}\vct{z})}^2\,.\label{eq:widehatz}
\end{align}
We consider two cases:

\noindent {\bf Case 1:} $\twonorm{\rev{\bJ}(\bth_0+\bSigma^{-1/2}\widehat{\z})} \neq 0$.
Setting derivative with respect to $\widehat{\z}$ to zero we obtain
\[
\qb_* - \frac{\eps}{n}\onenorm{\bv} \frac{\bSigma^{-1/2}\rev{\bJ^{2}}(\bth_0+\bSigma^{-1/2}\widehat{\z}) }{\twonorm{\rev{\bJ}(\bth_0+\bSigma^{-1/2}\widehat{\z})}}
-\eps^2\bSigma^{-1/2} \rev{\bJ^{2}}(\bth_0+\bSigma^{-1/2}\widehat{\z}) = 0\,.
\]
By rearranging the terms we write it as
\[
\rev{\bJ}(\bth_0+\bSigma^{-1/2}\widehat{\z}) = \left(\frac{\eps}{n}\frac{\onenorm{\bv}}{\twonorm{\rev{\bJ}(\bth_0+\bSigma^{-1/2}\widehat{\z})}} + \eps^2  \right)^{-1} \rev{\bJ^{-1}} \bSigma^{1/2}\qb_*\,.
\]
By taking the $\ell_2$ norm of both sides and then solving for $\twonorm{\rev{\bJ}(\bth_0+\bSigma^{-1/2}\widehat{\z})}$, we get
\begin{align}\label{eq:opt-case1}
\twonorm{\rev{\bJ}(\bth_0+\bSigma^{-1/2}\widehat{\z})} = \frac{1}{\eps^2} \twonorm{\rev{\bJ^{-1}} \bSigma^{1/2} \qb_*} - \frac{1}{n\eps} \onenorm{\bv}\,.
\end{align}

\noindent {\bf Case 2:} $\twonorm{\rev{\bJ}(\bth_0+\bSigma^{-1/2}\widehat{\z})} = 0$. In this case, $\widehat{z} = -\bSigma^{1/2}\bth_0$ and by comparing the objective function of \eqref{eq:widehatz} at the optimal solution under case 1 and case 2, it is easy to verify that case 2 happens only when the right-hand side in~\eqref{eq:opt-case1} becomes negative. Therefore, the two cases can be combined together in the following form:
\begin{align}\label{eq:thJth}
\<\hth^*_{\rm nl},\rev{\bJ^{2}} \hth^*_{\rm nl}\>= \twonorm{\rev{\bJ}(\bth_0+\bSigma^{-1/2}\widehat{\z})}^2 = \left(\frac{1}{\eps^2} \twonorm{\rev{\bJ^{-1}} \bSigma^{1/2} \qb_*} - \frac{1}{n\eps} \onenorm{\bv}\right)_+^2\,.
\end{align}
So in order to get the asymptotic value of the left hand side we can work with the right-hand side with $\bv$ and $\gamma = \twonorm{\rev{\bJ^{-1}} \bSigma^{1/2} \qb_*}$ the optimal solutions of the AO problem.
 
 In Lemma~\ref{meenv} (which is a restatement of \cite[Lemma 6.3]{javanmard2020precise}), the Moreau envelop function $e_f(\bx,\rho)$ was characterized. Following the proof of \cite[Lemma 6.3]{javanmard2020precise}, we can verify that the optimal $\bv$ is given by
 \[
 \bv = \frac{1}{1+\frac{\tau_g}{\beta}} \ST(\bw-\alpha \bg; \nu)\,.
 \] 
Therefore, by invoking the relation~\eqref{eq:v0} we have
\[
\lim_{n\to\infty} \frac{1}{n^2}\left(\frac{n}{\eps}\gamma- \onenorm{\bv}\right)_+^2 =
\frac{\omega^2}{(\rho+1)^2}\left(\frac{\gamma(\rho+1)}{\eps\omega}+\nu_*\cdot\erfc\left(\frac{1}{\sqrt{2}}\nu_*\right)-\sqrt{\frac{2}{\pi}} e^{-\nu_*^2}\right)_+^2\,,
\]
with $\omega = \sqrt{\alpha^2+\sigma^2}$, $\rho = \frac{\tau_g}{\beta}$,  $\nu_* = \nu_*(\frac{\gamma(\rho+1)}{\eps\omega},\rho)$ and $\nu_*(a,\mu)$ the unique solution to the following equation:
\begin{align*}
a-\frac{1}{\rho}\nu-\nu\cdot\erf\left(\frac{\nu}{\sqrt{2}}\right)-\sqrt{\frac{2}{\pi}} e^{-\frac{\nu^2}{2}}= 0\,.
\end{align*}
Plugging the above relation into~\eqref{eq:thJth} we obtain
\begin{align}\label{eq:thJth2}
\lim_{n\to\infty} \<\hth^*_{\rm nl},\rev{\bJ^{2}} \hth^*_{\rm nl}\> 
&= \lim_{n\to\infty} \left(\frac{1}{\eps^2} \gamma - \frac{1}{n\eps} \onenorm{\bv}\right)_+^2\nn\\
&= \frac{1}{\eps^2n^2} \lim_{n\to\infty} \left(\frac{n}{\eps} \gamma -  \onenorm{\bv}\right)_+^2\nn\\
&= \frac{\omega^2}{\eps^2(\rho+1)^2}\left(\frac{\gamma(\rho+1)}{\eps\omega}+\nu_*\cdot\erfc\left(\frac{1}{\sqrt{2}}\nu_*\right)-\sqrt{\frac{2}{\pi}} e^{-\nu_*^2}\right)_+^2\nn\\
&=\frac{\omega^2}{\eps^2(\rho+1)^2}\left(\frac{\gamma(\rho+1)}{\eps\omega}+\frac{\rho+1}{\rho} \nu_* - \frac{\gamma(\rho+1)}{\eps\omega} \right)_+^2\nn\\
&= \frac{\omega^2}{\eps^2(\rho+1)^2}\left(\frac{1+\rho}{\rho} \nu_* \right)_+^2\nn\\
&= \frac{\omega^2\nu_*^2}{\eps^2 \rho^2}\nn\\
&= \frac{(\alpha_*^2+\sigma^2)\nu_*^2}{\eps^2\rho^2}\nn\\
&= \frac{\beta^2\nu_*^2(\alpha_*^2+\sigma^2)}{\eps^2\tau_g^2} \,.
\end{align}
Now by recalling the characterization~\eqref{ARnl-char} along with~\eqref{eq:thJth2} and \eqref{eq:SRlim2} we get the desired result of~\eqref{eq:AR-final-form}.


\rev{
\subsection{Proof of Proposition~\ref{propo:non-ad}}\label{proof:propo:non-ad}
Recall the objective $\mathcal{R}(\alpha,\tau_g,\beta,\gamma,\tau_q)$. We start by considering the change of variable $\tilde{\gamma}=\gamma/\eps$. Note that the term $-\lambda/(1-\lambda)\gamma^2 = -\lambda/(1-\lambda)\eps^2\tilde{\gamma}^2$ will be dropped as it is zero. We next argue that $\nu^* = 0$. The reason is that if the indicator in the objective is inactive then the corresponding term is void, which is equivalent to $\nu^* = 0$. If the indicator is active, since the problem is maximization over $\gamma$, we deduce that $\frac{\gamma(\tau_g+\beta)}{\eps\beta\sqrt{\alpha^2+\sigma^2}} = \sqrt{\frac{2}{\pi}}$, which by the equation defining $\nu^*$ implies that $\nu^* = 0$. Next, by straightforward calculation, and using definition of 
Stieltjes transform $S$, we have that
 the expression inside $\sup_{0\le \lambda<1}$ is increasing in $\lambda$ and so we have that the optimal $\lambda \to 1$.
Using these values, the objective reduces to 
\begin{align*}
\mathcal{R}(\alpha,\tau_g,\beta,\tau_q)= & \frac{\tau_q}{2\alpha}(\tau^2+1-\sigma^2) - \frac{\alpha\tau_q}{2}+
\frac{\beta \tau_g}{2}\psi_2 + \frac{\beta}{2(\tau_g+\beta)}(\sigma^2+\alpha^2)\\
&-\frac{\psi_1}{2}\left\{\frac{\tau_q}{\alpha} + \frac{\alpha}{\tau_q}\beta^2 + \frac{\tau_q}{\alpha} \Big(1-\frac{2}{\pi}\Big) S\Big(\frac{2}{\pi}-1;\psi_1 \Big) \right\}\,.
\end{align*}
Using the definition 
$$\sigma^2 = \tau^2+1-\psi_1\Big(1+\Big(1-\frac{2}{\pi}\Big) S\Big(\frac{2}{\pi}-1;\psi_1\Big) \Big),$$
we can further simplify the objective as
\begin{align*}
\mathcal{R}(\alpha,\tau_g,\beta,\tau_q)= & -\frac{\psi_1}{2} \frac{\beta^2\alpha}{\tau_q} - \frac{\alpha\tau_q}{2}+
\frac{\beta \tau_g}{2}\psi_2 + \frac{\beta}{2(\tau_g+\beta)}(\sigma^2+\alpha^2)
\,.
\end{align*}
Optimization over $\tau_q$ can be done easily resulting in $\tau_q = \beta \sqrt{\psi_1}$, which gives
\begin{align*}
\mathcal{R}(\alpha,\tau_g,\beta)= & -\alpha \beta \sqrt{\psi_1} +
\frac{\beta \tau_g}{2}\psi_2 + \frac{\beta}{2(\tau_g+\beta)}(\sigma^2+\alpha^2)
\,.
\end{align*}
Writing the stationary condition for $\alpha, \tau_g,\beta$ we arrive at the following system of equations:
\begin{align}
\begin{cases}
&\dfrac{\alpha}{\tau_g+\beta} = \sqrt{\psi_1},\\
&\psi_2 = \dfrac{\sigma^2+\alpha^2}{(\tau_g+\beta)^2},\\
&-\alpha\sqrt{\psi_1} + \dfrac{\tau_g\psi_2}{2} + \dfrac{\tau_g}{2}\dfrac{\sigma^2+\alpha^2}{(\tau_g+\beta)^2} = 0\,.
\end{cases}
\end{align}
Solving the above system of equations we obtain $\alpha^2 = \sigma^2\psi_1/(\psi_2-\psi_1)$. Recalling that $\nu^*$, using Theorem 4.2 (b) we get the standard risk of the estimator to be
\[
\SR(\hth) = \alpha_*^2+ \sigma^2 = \sigma^2 \Big(\frac{\psi_2}{\psi_2-\psi_1} \Big).
\]
}
\subsection{Proofs of the Auxiliary Lemmas}\label{sec:aux-CGMT}
\subsubsection{Proof of Lemma~\ref{conjlemma}}
We start by considering the following related but different function
\[
\ell_0(\bv;\z) = \frac{1}{2n} \sum_{i=1}^n \Big(|v_i| + \eps \twonorm{\bz} \Big)^2\,.
\]
As shown in the proof of Lemma 6.1 in \cite{javanmard2020precise}, the conjugate of this function is given by
\[
\ell_0^*(\bv;\qb) = \frac{1}{2} \Big(\frac{ \twonorm{\qb}}{\eps} - \frac{\onenorm{\bv}}{n} \Big)_+^2 -\frac{1}{2n}\twonorm{\bv}^2\,.
\]
Note that $\elbar(\bv;\bz) = \ell_0(\bv;\rev{\bJ} (\bth_0+\bSigma^{-1/2}\z))$. We next use the result that if $f(\bx) = g(\bA\bx+\bx_0)$ then the conjugate of $f$ can be written in terms of the conjugate of $g$ as follows:
\[
f^*(\by) = -\<\bA^{-1}\bx_0,\by\> + g^*(\bA^{-\sT}\by)\,.
\]
Using this result with $\bx_0 = \rev{\bJ \bth_0}$ and $\bA = \rev{\bJ}\bSigma^{-1/2}$ we obtain
\[
\widetilde{\ell}(\bv;\qb) = -\<\bSigma^{1/2}\bth_0,\qb\> + 
\frac{1}{2} \Big(\frac{1}{\eps} \twonorm{ \rev{\bJ^{-1}} \bSigma^{1/2}\qb} - \frac{\onenorm{\bv}}{n} \Big)_+^2 -\frac{1}{2n}\twonorm{\bv}^2
\]
\subsubsection{Proof of Lemma~\ref{lem:trust}}
Consider a slightly different optimization than~\eqref{eq:opt-quad} where the equality constraint is replaced by the inequality constraint $\twonorm{\tqb}\le \gamma$:
\begin{align}\label{eq:opt-quad2}
\min_{\tqb, 0\le \gamma}\;\; & \frac{c_0}{2} \twonorm{\bH\tqb - \br}^2 +\frac{1}{2}\Big(\frac{1}{\eps}\gamma-c_1\Big)_+^2 \\
\text{s.t.}\;\; & \twonorm{\tqb} \le \gamma\,.\nn
\end{align}
Due to this change, $\eqref{eq:opt-quad2}$ is now a convex optimization.
 Denote by ${\sf OPT}_1$ the optimal objective value of the original problem~\eqref{eq:opt-quad} and by ${\sf OPT}_2$ the optimal objective value of the modified problem~\eqref{eq:opt-quad2}. We argue that ${\sf OPT}_1 = {\sf OPT}_2$. Clearly ${\sf OPT}_1 \ge {\sf OPT}_2$ because \eqref{eq:opt-quad2} has a larger feasible set. Now suppose that this inequality is strict (${\sf OPT}_1 > {\sf OPT}_2$) and  let $(\tqb_*,\gamma_*)$ be a solution to \eqref{eq:opt-quad2}. Then we should have
$\twonorm{\tqb_*}< \gamma_*$. Consider the point $(\tqb_*, \twonorm{\tqb_*})$ which is a feasible point for both optimization problems and so the objective value at this point is at least ${\sf OPT_1}$ and therefore strictly larger than ${\sf OPT}_2$. But this is a contradiction  because $\Big(\frac{1}{\eps}\gamma-c_1\Big)_+^2$ is non-decreasing in $\gamma \ge 0$.
  
 To characterize ${\sf OPT}_2$, we first focus on the minimization over $\tqb$.  The corresponding Lagrangian with Lagrange multiplier $\frac{\lambda c_0}{2}$ reads
\[
\sup_{\lambda\ge 0}\min_{\tqb} \;\frac{c_0}{2}\twonorm{\bH\tqb - \br}^2-\frac{\lambda c_0}{2} (\gamma^2 - \twonorm{\tqb}^2)\,.
\] 
Solving the inner minimization, we have $\tqb_* = (\bH^\sT\bH+\lambda \Iden)^{-1}\bH^\sT \br$ and the dual problem becomes
\begin{align}\label{eq:dual}
&\sup_{\lambda\ge 0} \;\frac{c_0}{2}\twonorm{\bH\tqb_* - \br}^2-\frac{\lambda c_0}{2} (\gamma^2 - \twonorm{\tqb_*}^2)\nn\\
&=\sup_{\lambda\ge 0} \;\frac{c_0}{2}\tqb_*^\sT[\bH^\sT (\bH\tqb_* - \br) + \lambda \tqb_*] - \frac{c_0}{2}\br^\sT(\bH\tqb_* - \br)  - \frac{\lambda c_0}{2} \gamma^2 \nn\\
&=\sup_{\lambda\ge 0} \; -\frac{c_0}{2}\br^\sT(\bH\tqb_* - \br)  - \frac{\lambda c_0}{2} \gamma^2 \nn\\
&=\sup_{\lambda\ge 0} \; -\frac{c_0}{2}\br^\sT(\bH (\bH^\sT\bH+\lambda \Iden)^{-1}\bH^\sT - \Iden)\br  - \frac{\lambda c_0}{2} \gamma^2 \nn\\
&=\sup_{\lambda\ge 0} \; \frac{\lambda c_0}{2}\br^\sT (c_0\bH^\sT\bH+\lambda \Iden)^{-1}\br  - \frac{\lambda c_0}{2} \gamma^2\nn\\
& = c_0Q(\bH,\br,\gamma)\,.
\end{align}
By the Slater's condition the duality gap is zero and hence by next minimizing over $\gamma\ge0$, we obtain that the optimal value of \eqref{eq:opt-quad2} is given by
\[
\min_{\gamma\ge 0} \;\; c_0Q(\bH,\br,\gamma) + \frac{1}{2}\Big(\frac{1}{\eps}\gamma-c_1\Big)_+^2\,.
\]
\subsubsection{Proof of Lemma~\ref{lem:conv-conc}}
We first show that the function
\[
g(\gamma,\beta) = Q(\bSigma^{-1/2}\rev{\bJ}, \frac{1}{\alpha}\bSigma^{1/2}\bth_0 - \frac{\beta}{\sqrt{d}}\bh,\gamma) 
\]
is jointly convex in $(\gamma,\beta)$.  By the change of variable $\tilde{\lambda} = \lambda \gamma$, $\tilde{\bth} = \bSigma^{1/2}\bth_0/\alpha$, $\bH= \bSigma^{-1/2}\rev{\bJ}$, this function can be written as
\begin{align*}
g(\gamma,\beta) &=\sup_{\tilde{\lambda}\ge 0}\;\; \frac{\tilde{\lambda}}{2\gamma} \left(
(\tilde{\bth} - \frac{\beta}{\sqrt{d}}\bh)^\sT (\bH\bH^\sT + \frac{\tilde{\lambda}}{\gamma} \Iden)^{-1} (\tilde{\bth} - \frac{\beta}{\sqrt{d}}\bh) - \gamma^2
\right) \nn\\
&= \sup_{\tilde{\lambda}\ge 0}\;\; \frac{\tilde{\lambda}}{2} \left(
(\tilde{\bth} - \frac{\beta}{\sqrt{d}}\bh)^\sT (\gamma \bH\bH^\sT + \tilde{\lambda} \Iden)^{-1} (\tilde{\bth} - \frac{\beta}{\sqrt{d}}\bh) - \gamma
\right).
\end{align*}
We show that for any fixed $\tilde{\lambda} \ge 0$ the inner function above is jointly convex in $(\gamma,\beta)$ and since the pointwise maximum of convex functions is also convex, we conclude that $g(\gamma,\beta)$ is jointly convex in $(\gamma,\beta)$. 

The Hessian of the inner function reads 
\begin{align*}
&\frac{1}{2}\nabla^2_{\frac{\beta}{\sqrt{d}},\gamma}\left[
(\tilde{\bth} - \frac{\beta}{\sqrt{d}}\bh)^\sT (\gamma \bH\bH^\sT + \tilde{\lambda} \Iden)^{-1} (\tilde{\bth} - \frac{\beta}{\sqrt{d}}\bh) - \gamma
\right]=\begin{bmatrix}
A & \quad C\\
C& \quad B
\end{bmatrix}\,,
\end{align*}
where
\begin{align*}
A&:= \bh^\sT (\gamma \bH\bH^\sT + \tilde{\lambda} \Iden)^{-1} \bh\\
B&:=  \twonorm{(\gamma \bH\bH^\sT +\tilde{\lambda} \Iden)^{-1/2} \bH\bH^\sT (\gamma \bH\bH^\sT +\tilde{\lambda} \Iden)^{-1}(\tilde{\bth}- \frac{\beta}{\sqrt{d}}\bh)}^2\\
C&:= \bh^\sT(\gamma \bH\bH^\sT +\tilde{\lambda} \Iden)^{-1} \bH\bH^\sT(\gamma \bH\bH^\sT +\tilde{\lambda} \Iden)^{-1} (\tilde{\bth}-\frac{\beta}{\sqrt{d}}\bh)\,.
\end{align*}
Here we repeatedly used the identity $\frac{\partial \bK^{-1}}{\partial \gamma} =  - \bK^{-1} \frac{\partial \bK}{\partial \gamma} \bK^{-1}$, for a matrix $\bK$. 

To lighten the notation, we set $\bM:= (\gamma \bH\bH^\sT +\tilde{\lambda} \Iden)^{-1}$ and $\bv: = \tilde{\bth}- \frac{\beta}{\sqrt{d}}\bh$. Using these shorthands the determinant of the Hessian is equal to
\[
\twonorm{\bM^{1/2}\bh}^2 \twonorm{\bM^{1/2}\bH\bH^\sT \bM\bv}^2
-(\bh^\sT \bM\bH\bH^\sT \bM\bv)^2 \ge0,
\] 
using the Cauchy–Schwarz inequality. This completes the proof of $g(\gamma,\beta)$ being jointly convex in $(\gamma,\beta)$.

Next note that its perspective function is given by
\begin{align*}
\tau_q g(\gamma/\tau_q, \beta/\tau_q) &= \tau_qQ(\bSigma^{-1/2}\rev{\bJ}, \frac{1}{\alpha}\bSigma^{1/2}\bth_0 - \frac{\beta}{\tau_q\sqrt{d}}\bh,\frac{\gamma}{\tau_q})\\
&=   \frac{1}{\tau_q}Q(\bSigma^{-1/2}\rev{\bJ}, \frac{\tau_q}{\alpha}\bSigma^{1/2}\bth_0 - {\frac{\beta}{\sqrt{d}}}\bh,{\gamma}),
\end{align*}
and therefore is jointly convex in $(\gamma,\beta,\tau_q)$. 
\subsubsection{Proof of Lemma~\ref{lem:SAO-CC}}
As we discussed after \eqref{eq:AO-op5}, the objective function in \eqref{eq:AO-op5} is jointly convex in $(\alpha,\tau_g)$ and jointly concave in $(\beta,\gamma,\tau_q)$. Since convexity/concavity is preserved by point-wise limits, the objective~\eqref{eq:AO-final} is jointly convex in $(\alpha,\tau_g)$ and jointly concave in $(\beta,\gamma,\tau_q)$. To prove strict convexity in $(\alpha,\tau_g)$, note that in our derivation we wrote \eqref{eq10} (the part of the objective~\ref{eq:AO-op4} that involves $\bv$) in terms of the Moreau envelope $\frac{1}{n} e_f\left(\w-\alpha\vct{g};\frac{\tau_g}{\beta}\right)$, cf.~\eqref{eq:dum0}. As $d\to \infty$, its limit goes to the \emph{expected Moreau envelope}. By using the result of \cite[Lemma
4.4]{thrampoulidis2015precise} the expected Moreau envelope of a function is strictly convex in $\reals_{>0}\times \reals_{>0}$ without requiring any strong
or strict convexity \rev{assumption on the function itself}. Therefore, the objective \eqref{eq:AO-op5} (and so objective of~\eqref{eq:AO-final} after taking point-wise limit) is jointly strictly convex in $(\alpha,\tau_g)$.

To prove the uniqueness, note that $ \max_{0\le\beta, \gamma,\tau_q} \cR(\alpha,\tau_g, \beta,\gamma,\tau_q)$ is strictly convex in $(\alpha,\tau_g)$. This follows from
the fact that if $f(\bx, \by)$ is strictly convex in $\bx$, then $\max_{\by} f(\bx, \by)$ is also strictly convex in $\bx$. We
next use \cite[Lemma C.5]{thrampoulidis2015precise} to conclude that 
$\min_{\tau_g>0}\max_{0\le\beta, \gamma,\tau_q} \cR(\alpha,\tau_g, \beta,\gamma,\tau_q)$ is strictly convex in
$\alpha\ge 0$.Therefore, its minimizer over $\alpha\ge 0$ is unique. By a similar argument, we show that $\frac{\tau_{g*}}{\beta_*}$ is unique. Consider the change of variable $\tau_g \to \tilde{\tau_g} = \frac{\tau_g}{\beta}$. Then part of the objective~\eqref{eq:AO-op5} that depends on $\tilde{\tau_g}$ can be written as
\[
\frac{\beta^2\tilde{\tau_g}}{2}\frac{n}{d} + \frac{1}{2(\tilde{\tau_g}+1)} \frac{1}{n}\twonorm{\w-\alpha\bg}^2 +
\frac{1}{n} G_n(\w-\alpha\bg;\tilde{\tau_g},\gamma,\nu) =  
\frac{\beta^2\tilde{\tau_g}}{2}\frac{n}{d} + \frac{1}{n} e_f\left(\w-\alpha\vct{g};\tilde{\tau_g}\right)\,,
\]
 using~\eqref{eq10}. As explained above, this converges to the expected Moreau envelope, which is strictly convex in $\tilde{\tau_g}$. Following by the same reasoning for $\alpha$, one can show that $\min_{\alpha>0}\max_{0\le\beta, \gamma,\tau_q} \cR(\alpha,\tilde{\tau_g}, \beta,\gamma,\tau_q)$ is strictly convex in
$\tilde{\tau_g}> 0$.Therefore, its minimizer over $\tilde{\tau_g} > 0$ is unique.

\section{Some useful lemmas}\label{sec:useful-lemma}
Here we state some of the technical lemmas that are used in deriving our analytical results.

The first lemma is about the Stieltjes transform of  the Marchenko-Pastur distribution.
\begin{definition}
The Stieltjes transform $S_{\rho}(z)$ of a measure of density $\rho$ on a real interval $I$ is the function of the complex variable
$z$ defined outside $I$ by the formula
\[
S_{\rho}(z) = \int_I \frac{\rho(t)\de t}{z- t}\, \quad z\in \mathbb{C}\backslash I\,.
\]
\end{definition}
\begin{lemma}\label{propo:stiel}
Suppose that $\W\in \reals^{N\times d}$ has rows drawn independently from unit sphere. As $N,d\to \infty$ and  $N/d\to \psi_1$, the spectral density of $\W\W^\sT$ converges (in weak topology in distribution) to the Marchenko-Pastur distribution with Stieltjes transform given by
\[
S(z;\psi_1) = \frac{1-\psi_1-z - \sqrt{(1-\psi_1-z)^2-4\psi_1 z}}{-2\psi_1 z},
\]
for $z<0$.
\end{lemma}
\begin{proof}
We refer to \cite[page 52]{bai2010spectral} for the proof of this proposition.
\end{proof}

The next lemma is about the spectrum of matrix $\bJ$ given by
\begin{align}\label{def:J}
 \bJ = \left(\bW\bW^\sT) \odot \Big(\frac{\pi-\cos^{-1}(\bW\bW^\sT)}{2\pi}\Big)\right)^{1/2}\,.
 \end{align}
\begin{lemma}\label{pro:spectral}
 Suppose that $\bW\in \reals^{N\times d}$ has rows chosen randomly and independently of data form the unit sphere, $\Unif(\mathbb{S}^{d-1})$. Let $\bJ$ be given by~\eqref{def:J} and suppose that $N/d \to\psi_1\in (0,\infty)$, as $n\to \infty$. Then, the matrix $\bJ^2$ can (in probability) be approximated consistently in operator norm by the matrix $\bK$ given by
 \[
 \bK = \frac{1}{4}(\bW\bW^\sT + \Iden).
 \]
 In other words, $\opnorm{\bJ^2-\bK}\to 0$, in probability, when $n\to\infty$.
 \end{lemma}
\begin{proof}
The claim follows from the result of~\cite[Theorem 2.1]{el2010spectrum} about the spectrum of inner product kernel random matrices, specialized to matrix $\bJ^2$. Specifically, let $f(z) = z(\pi-\cos^{-1}(\z))/(2\pi)$. Then $\bJ^2_{ij} = f(\bw_i^\sT\bw_j)$. By employing~\cite[Theorem 2.1]{el2010spectrum}, the kernel matrix $\bJ^2$ can (in probability) be
approximated consistently in operator norm by the matrix $\bK$, given by
\[
\bK = f(0) \ones\ones^\sT + f'(0) \bW\bW^\sT + (f(1)-f(0)-f'(0))\Iden\,.
\]
For our specific $f$ we have $f(0) = 0$, $f(1) =1/2$, $f'(0) = 1/4$. 
\end{proof}

\vspace{1.5cm}

\end{document}